\documentclass{article}

\usepackage{microtype}
\usepackage{graphicx}
\usepackage{booktabs} 

\usepackage{thm-restate}

\usepackage{lipsum}

\usepackage{hyperref}
\usepackage{xcolor}

\usepackage[preprint]{neurips_2021}
\usepackage{amsthm}
\usepackage{amsmath}
\usepackage{bm}
\usepackage{amsfonts}
\usepackage{physics}
\usepackage{subfig}
\usepackage{multirow}

\usepackage[normalem]{ulem}
\usepackage{enumitem}

\title{Learning High Dimensional Wasserstein Geodesics}

\newcommand{\equal}{\textsuperscript{*}}

\newtheorem{definition}{Definition}
\newtheorem{remark}{Remark}
\newtheorem{theorem}{Theorem}
\newtheorem{lemma}{Lemma}

\newtheorem{manualtheoreminner}{Theorem}
\newenvironment{manualtheorem}[1]{%
  \renewcommand\themanualtheoreminner{#1}%
  \manualtheoreminner
}{\endmanualtheoreminner}
\usepackage{algorithm}
\usepackage{algpseudocode}%

\newtheorem{innercustomthm}{Theorem}
\newenvironment{customthm}[1]
  {\renewcommand\theinnercustomthm{#1}\innercustomthm}
  {\endinnercustomthm}

\begin{document}

\author{
  Shu Liu\thanks{Equal contribution} \\
  Department of Mathematics\\
  Georgia Institute of Technology\\
  Atlanta, GA 30332, USA \\
  \texttt{sliu459@gatech.edu} \\
   \And
   Shaojun Ma\equal\\
   Department of Mathematics \\
   Georgia Institute of Technology\\
    Atlanta, GA 30332, USA\\
   \texttt{shaojunma@gatech.edu} \\
   \And
   Yongxin Chen \\
   Department of Aerospace Engineering \\
   Georgia Institute of Technology\\
   Atlanta, GA 30332, USA \\
   \texttt{yongchen@gatech.edu} \\
   \And
   Hongyuan Zha \thanks{The research of Hongyuan Zha is supported in part by a grant from Shenzhen Research Institute of Big Data}\\
   School of Data Science \\
   Shenzhen Research Institute of Big Data\\
   The Chinese University of Hong Kong\\
   Shenzhen, Guangdong Province 518172, China \\
   \texttt{zhahy@cuhk.edu.cn} \\
   \And
   Haomin Zhou \\
   Department of Mathematics \\
   Georgia Institute of Technology\\
   Atlanta, GA 30332, USA \\
   \texttt{hmzhou@math.gatech.edu} \\
}

\maketitle






\begin{abstract}
We propose a new formulation and learning strategy for computing the Wasserstein geodesic between two probability distributions in high dimensions. By applying the method of Lagrange multipliers to the dynamic formulation of the optimal transport (OT) problem, we derive a minimax problem whose saddle point is the Wasserstein geodesic. We then parametrize the functions by deep neural networks and design a sample based bidirectional learning algorithm for training. The trained networks enable sampling from the Wasserstein geodesic. As by-products, the algorithm also computes the Wasserstein distance and OT map between the marginal distributions. We demonstrate the performance of our algorithms through a series of experiments with both synthetic and real-world data. 

\end{abstract}

\section{Introduction}
As a key concept of optimal transport (OT) \citep{villani2003topics}, Wasserstein distance has been widely used to evaluate the distance between two distributions. Suppose we are given any two probability distributions $\rho_a$ and $\rho_b$ on $\mathbb{R}^d$, the Wasserstein distance between $\rho_a$ and $\rho_b$ is defined as
\begin{equation}
    \inf_{\pi}\left\{\int_{\mathbb{R}^d\times \mathbb{R}^d}c(x,y)d\pi (x,y)\big|\pi\in \Pi(\rho_a,\rho_b)\right\}
    \label{Static OT}
\end{equation}
in the Kantorovich form. Here $c(x,y)$ is the cost function that quantifies the effort of moving one unit of mass from location $x$ to location $y$, and $\Pi(\rho_a,\rho_b)$ is the set of joint distribution of $\rho_a$ and $\rho_b$. The solution of OT refers to an optimal $\pi^*$ to attain the Wasserstein distance between two distributions as well as an optimal map $g$ such that $g(x)$ and $y$ have the same distribution when $x \sim \rho_a$. 

\begin{figure}[ht]
\centering
\subfloat[][t=0]{\includegraphics[width=.2\linewidth]{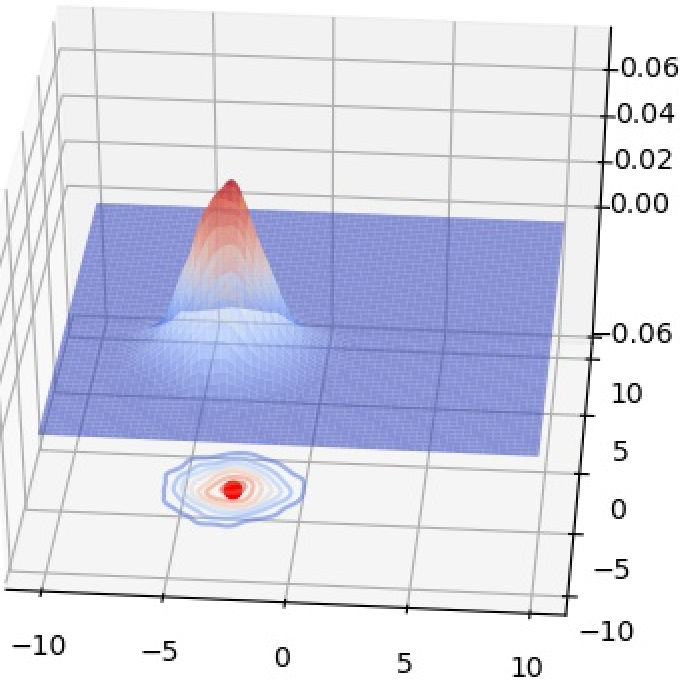}}
\subfloat[][t=0.25]{\includegraphics[width=.2\linewidth]{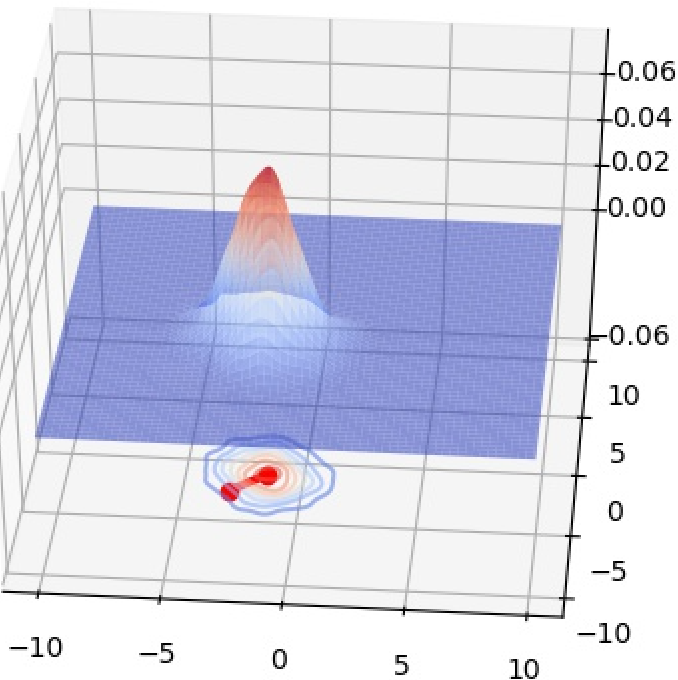}}
\subfloat[][t=0.5]{\includegraphics[width=.2\linewidth]{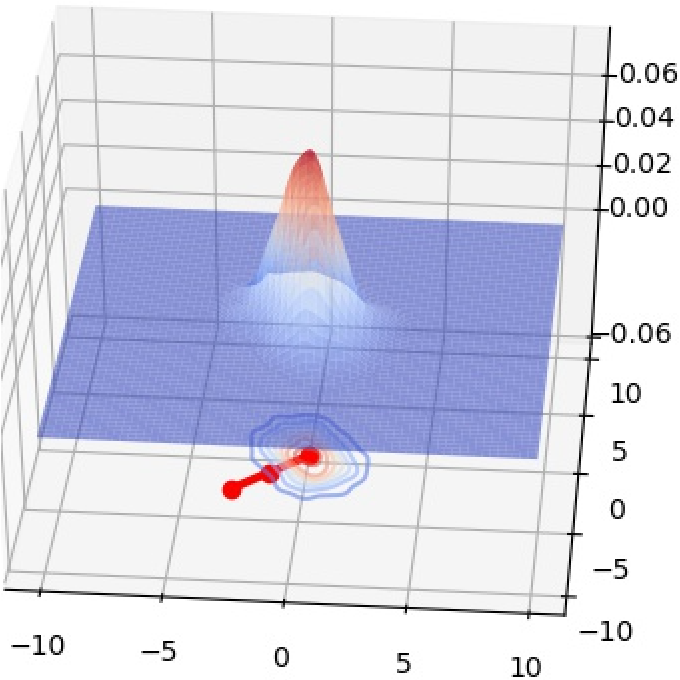}}
\subfloat[][t=0.75]{\includegraphics[width=.2\linewidth]{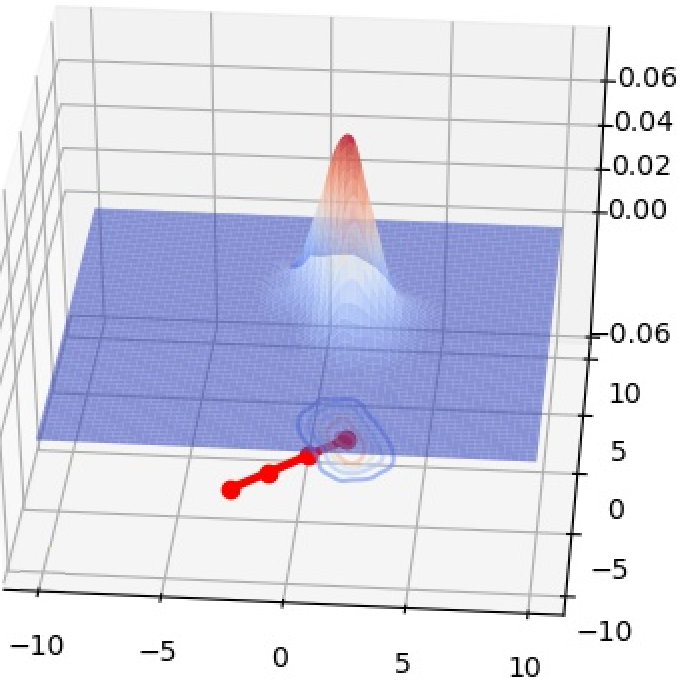}}
\subfloat[][t=1]{\includegraphics[width=.2\linewidth]{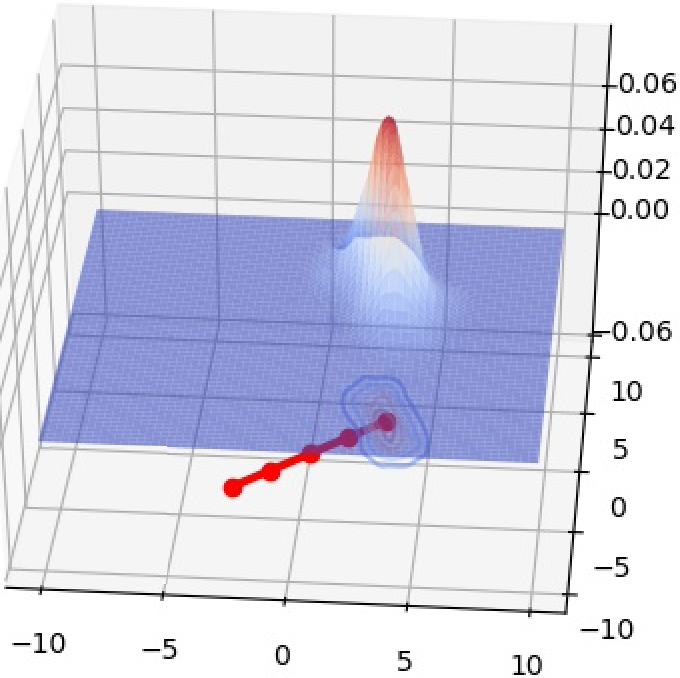}}
\caption{Wasserstein geodesic between two Gaussians}
\label{fig:geodesic}
\end{figure}

When applying OT-related computations to applications, the Kantorovich duality form has been widely studied and applied. For example, many regularized OT problems have been investigated,
including entropic regularized OT \citep{cuturi2013sinkhorn, largeot}, Laplacian regularization \citep{lapregular}, Group-Lasso regularized OT \citep{grouplasso}, Tsallis regularized OT \citep{tsallis} and OT with $L^2$ regularization \citep{l2regular}.

Based on the Kantorovich form, specially if choosing cost function as $c(x,y) = \frac{1}{2}|x-y|^2$ which results in the well-known $2$-Wasserstein distance, 
the OT problem can be rewritten in a fluid dynamics perspective \citep{benamou2}: 
$W_2$ defined below is the square of $2$-Wasserstein distance,
\begin{align}
    W_2(\mu,\nu) = & \textrm{inf} \left\{ \int_{\mathbb{R}^d}\int_0^1\rho(x,t)|v(x,t)|^2dxdt\right\},\nonumber\\
    \text{subject to:}~~&\partial_t\rho + \nabla\cdot(\rho v)=0,~~~\rho(\cdot,0) = \rho_a, ~~ \rho(\cdot,1) = \rho_b.
    \label{eqn:fluiddynamics}
\end{align}
Solving \eqref{eqn:fluiddynamics} leads to the following partial differential equation (PDE) system:
\begin{align}
    &\partial_t\rho + \nabla\cdot(\rho \nabla \Phi)=0, ~~~\frac{\partial \Phi}{\partial t} + \frac{1}{2}|\nabla \Phi|^2 = 0,  \nonumber\\
    &\text{subject to:}~~\rho(\cdot,0) = \rho_a, ~~~ \rho(\cdot, 1) = \rho_b.
    \label{eqn:otpde}
\end{align}
The solution of Problem \eqref{eqn:fluiddynamics} or \eqref{eqn:otpde} gives the definition of Wasserstein {\it geodesic}: if there exists an optimal pushforward map $\Xi^*$ between $\rho_a$ and $\rho_b$, then the constant speed geodesic is defined as a continuous curve: $\rho_t = (\textrm{Id} + t \Xi^*)_\sharp \rho_a$\footnote{Please check Definition \ref{pushfwd of distribution}.
}, which depicts the trajectory of $x_0 \in \rho_a$ moving towards $x_1 \in \rho_b$. Here $\textrm{Id}$ stands for the identity map.

Knowing Wasserstein geodesic between $\rho_a$ and $\rho_b$ (refer to Figure \ref{fig:geodesic} as an example, trajectory of the distribution mean is indicated by a red line) provides ample information for their Wasserstein distance and optimal pushforward map. More importantly, since the Wasserstein geodesic is automatically kinetic-energy-minimizing, it offers a natural sampling mechanism without using additional artificial regularization to generate samples not only for the target distribution $\rho_b$, but also for all distributions along the Wasserstein geodesic. This is different from several recent OT based models for computing the optimal pushforward map, such as Jacobian and Kinetic regularized OT \citep{trainode} and $L^2$ regularized OT \citep{pfg}.  


Wasserstein geodesic also finds applications in robotics and optimal control communities. 
\citet{otrobotics1} apply Brenier-Benamou OT to swarm control and updates the velocity of each agent. \citet{otrobotics2} study the locations of robots by minimizing the Wasserstein distance between original and target distributions. \citet{otcontrol1, otcontrol2} investigate diverse OT applications in control theories. Currently, the research that combines OT with robotics or control is still limited to low dimensions. We believe, having a method to compute the Wasserstein geodesic, especially in high dimensional setups, will be very beneficial for developing novel algorithms and applications in robotics and control, such as path planning for multi-agent systems.   

Last but not least, finding an efficient method to compute the Wasserstein geodesic is important and challenging in applied mathematics as well. It is well-known that  directly solving Problem \eqref{eqn:fluiddynamics} or \eqref{eqn:otpde} by the traditional numerical PDE methods, such as finite difference or finite element method which requires spatial discretization, must face the {\it curse of dimensionality}, meaning that the computational cost grows exponentially as the dimension increases.

To this end, we first formulate the OT problem as a saddle point problem without introducing any regularizers. We also further reduce the search space for the saddle point problem by leveraging KKT conditions. We then parametrize the optimal maps as well as the Lagrange multipliers via deep neural networks, which can be trained alternately. The resulting method is a sample based algorithm that is capable of handling high dimensional Wasserstein geodesic. It is worth mentioning that there is no Lipschitz or convexity constraint on the neural networks in our computation. Those constraints become thorny issues when they can only be approximately enforced in methods for other Wasserstein related problems. Furthermore, our formulation is based on a Lagrangian framework which can be readily generalized to various alternative cost functions, including the $L^p$-Wasserstein distance. To the best of our knowledge, there is no method to compute high dimensional Wasserstein distance, optimal map as well as Wasserstein geodesic all at once for a general cost function. We validate our method through a series experiments on both synthetic and realistic data sets. To summarize, our contributions are:
\begin{itemize}[leftmargin=*]
    \item We develop a novel saddle point formulation so that high dimensional Wasserstein geodesic, optimal map as well as Wasserstein distance between two given distributions can be computed in one single framework.
    \item Our scheme is formulated to handle general convex cost functions, including the general $L^p$-Wasserstein distance. More importantly, it provides a method without requiring convexity or Lipschitz constraint.
    \item We show the effectiveness of our method through extensive numerical experiments with both synthetic and realistic data sets. 
\end{itemize}
\textbf{Related work:}
Traditional approaches \citep{benamou2,benamou1,parallelot,unnormot} for OT problems are designed to handle low dimensional settings but become infeasible in high dimensional cases since those methods are based on spatial discretizations. In machine learning community, researchers utilize OT to drive one distribution to another, and the Sinkhorn algorithm has been widely adopted \citep{lineartimeotsinkhorn, genwithsinkhorn, otmatch, ipot} due to its convenient implementation even in high dimensional settings. 
However, the algorithm does not scale well to a large number of samples or can't readily handle continuous probability measures \citep{soptilargescaleot}. 

To overcome the challenges of Sinkhorn algorithm, researchers bring neural networks to study OT problems. \citet{largeot} study the regularized OT by applying neural networks such that large scale and high dimensional OT can be computed. \citet{wgan} propose Wasserstein-1 GAN, which is a generative model to minimize Wasserstein distance. Though we have witnessed a great success of Wasserstein-1 GAN and its related work in high dimensional and large scale applications \citep{gulrajani2017improved, wasserautoencoder, spectral, wassproximalofgan, wofwforgen, scalsot}, the non-trivial Lipschitz-1 constraint of the discriminator is hard to be theoretically satisfied.

In order to avoid Lipschitz-1 constraint in Wasserstein-1 GAN, by using input convex neural networks (ICNN) \citep{icnn} to approximate the potential function, \citet{wasserstein2gan} propose Wasserstein-2 GAN to rewrite OT as a minimization problem, \citet{icnnwasserbary} and \citet{icnnot} extend the semi-dual formulation of OT to new minimax problems. 

\citet{ruthotto2020machine} propose a machine learning framework for solving general Mean-Field Games in high dimensional spaces. Such method can be applied to solving Optimal Transport problems by adding penalty term to enforce the terminal constraint. Many other OT based approaches have been proposed to drive one distribution to another, including continuous normalizing flows (CNF) \citep{cnf1, cnf2} and the approaches developed by \citet{tabak1, tabak2}. 

Various OT models bring numerous applications in domain adaptation \citep{largeot}, generative modeling \citep{wgan}, partial differential equations \citep{learnsde1}, stochastic control \citep{wassercontrol}, robotics \citep{otrobotics1,otrobotics2}, as well as color transfer \citep{wasserstein2gan}, which is also one of the experiments in this paper.

Although Wasserstein distance and optimal map have been widely studied within diverse frameworks, most of them focus on the cases that the cost function is either $L^1$ or $L^2$ based, and Wasserstein geodesic is seldom computed, especially in high dimensional settings. In this paper we compute all in our novel designed framework. We also note that a similar strategy formulated by \citet{apac} derives a saddle point optimization scheme for solving the mean field game equations. We should point out that our problem setting and sampling method are distinct from them.

\section{Background of the Wasserstein Geodesic}
We consider two probability distributions $\rho_a, \rho_b$ defined on $\mathbb{R}^d$. For ease of discussion, we assume the density functions of these two distributions exist, then we focus on computing an interpolation curve $\{\rho_t\}_{t=0}^1$ between $\rho_a$ and $\rho_b$. To be more precise, we aim to find the length-minimizing curve joining $\rho_a$ and $\rho_b$, which is also known as the {\it Wasserstein geodesic}. 
To this end, we start from the classical Optimal Transport (OT) problem \eqref{Static OT}, and assume the cost function $c(x,y)$ is the optimal value of the following control problem,
\begin{align}
  & c(x, y) =L(y-x)  = \min_{\{{v}(\cdot, t)\}}  \left\{ \int_0^1 L({v}_t)~dt \right\},  \nonumber\\
  & \textrm{subject to:}~~\dot{x}_t={v}(x_t, t),~ x_0 = x,~ x_1 = y. \nonumber
\end{align}
Here we assume the Lagrangian $L(\cdot)\in C^1(\mathbb{R}^d)$ satisfies $L(-u)=L(u)$ for arbitrary $u\in \mathbb{R}^d$ and is strictly convex and superlinear.\footnote{$L(\cdot)$ is super linear if $\lim_{u\rightarrow \infty}~\frac{L(u)}{|u|}=\infty$.} Then $\nabla L:\mathbb{R}^d\rightarrow \mathbb{R}^d$ is an invertible map, we denote $\nabla L^{-1}$ the inverse of $\nabla L$ in the sequel. We define the Hamiltonian $H(\cdot)$ associated with Lagrangian $L(\cdot)$ as
\begin{align}
  H(p) & = \max_{v\in\mathbb{R}^d}\{v\cdot p - L(v)\} = \nabla L^{-1}(p)\cdot p - L(\nabla L^{-1}(p))\label{Legendre},  
\end{align}
This is useful in our future discussion.
\begin{remark}
 One important instance of $L$ is $L(v)=\frac{|v|^p}{p}$, $p>1$, which leads to $p$-Wasserstein distance. When $p=2$, it recovers the classical $2$-Wasserstein distance.
\end{remark}
Since \eqref{Static OT} doesn't involve time, we denote it as \textbf{Static OT} problem and we denote its optimal value as $W_{\textrm{Static}}(\rho_a, \rho_b)$. Before introducing the dual problem of \eqref{Static OT}, we introduce the following definition.
\begin{definition}\label{pushfwd of distribution}
 Given measurable map $T:\mathbb{R}^d\rightarrow \mathbb{R}^d$ and $\lambda$ as a probability distribution on $\mathbb{R}^d$.  We denote the pushforward of $\lambda$ by $T$ as $T_{\sharp}\lambda$, which is defined as 
 \begin{equation*}
   T_{\sharp}\lambda(E) = \lambda(T^{-1}(E)) \quad  \textrm{for all measurable set} ~ E \subset \mathbb{R}^d.
 \end{equation*}
\end{definition}
Problem \eqref{Static OT} has the Kantorovich dual form \citep{villani2008optimal}
\begin{equation}
 \max_{\phi(y)-\psi(x)\leq c(x,y)}\left\{ \int \phi(y)\rho_b(y)~dy - \int \psi(x)\rho_a(x)~dx \right\} . \label{Kant duality}
\end{equation}
Here and in the sequel, we will denote $\int$ as $\int_{\mathbb{R}^d}$ for conciseness. One can show that the optimal value of \eqref{Kant duality} equals $W_{\textrm{Static}}(\rho_a, \rho_b)$. Let us denote the optimizer to \eqref{Kant duality} as $ (\phi^*,\psi^*)$.
Then $\nabla\psi^*$ (as well as $\nabla\phi^*$) provides optimal transport maps from $\rho_a$ to $\rho_b$ (as well as from $\rho_b$ to $\rho_a$) in the sense that
\begin{equation}
(\textrm{Id}+\nabla L^{-1}(\nabla\psi^*))_{\sharp}\rho_a=\rho_b, ~  
(\textrm{Id}-\nabla L^{-1}(\nabla\phi^*))_{\sharp}\rho_b = \rho_a.  \label{phi, psi as monge maps}
\end{equation}

We then consider the dynamic version of \eqref{Static OT}
\begin{align}
    & W_{\textrm{Dym}}(\rho_a,\rho_b) = \min_{\rho,{v}}\left\{ \int_0^1\int L({v}(x,t))~\rho(x,t)~dxdt  \right\}, \label{dynamical OT}\\
    & \textrm{subject to:}~~\frac{\partial\rho}{\partial t} + \nabla\cdot(\rho{v}) = 0, \rho(\cdot, 0)=\rho_a, \rho(\cdot,1) = \rho_b.  \label{constraints}
\end{align}
This problem also has an equivalent particle control version 
\begin{align}
  \min_{{v}}\left\{ \int_0^1\mathbb{E} [L({v}(\boldsymbol{X}_t,t))]~dt  \right\}, \label{particle dynamical OT} \quad
  \textrm{subject to}~ \frac{d}{dt}\boldsymbol{X}_t = {v}(\boldsymbol{X}_t, t),~ \boldsymbol{X}_0 \sim \rho_a, ~ \boldsymbol{X}_1 \sim \rho_b. 
\end{align}
Such particle control version is more helpful for designing sample based formulation than its PDE counterpart, \eqref{dynamical OT} as we will stated in Section \ref{geodesic pushforward}.
Since we introduce the dynamics of $\{\rho_t\}$ and $\{\boldsymbol{X}_t\}$ into the new definition and reformulate the original problem \eqref{Static OT} as an optimal control problem \eqref{dynamical OT} and \eqref{particle dynamical OT}, we thus denote \eqref{dynamical OT}, \eqref{particle dynamical OT} as \textbf{Dynamical OT} problem. The optimal solution of Dynamical OT is given by the following coupled PDE system, c.f. Chapter 13 of \cite{villani2008optimal}
\begin{align}
\label{geodesic eq}
\begin{split}
  &\frac{\partial \rho(x,t)}{\partial t} + \nabla\cdot(\rho(x,t) \nabla L^{-1}(\nabla \Phi(x,t)))=0,~~\frac{\partial \Phi(x,t)}{\partial t} + H(\nabla \Phi(x,t)) = 0 ,  \\
  &\textrm{subject to:}~~\rho(\cdot,0) = \rho_a, ~\rho(\cdot, 1) = \rho_b.
\end{split}
\end{align}
We denote the solution to \eqref{geodesic eq} as $(\rho^*, \Phi^*)$. Then the optimal vector field is
\begin{equation}
  {v}^*(x,t) = \nabla L^{-1}(\nabla\Phi^*(x, t)). \label{vector field = inv grad L grad Phi}
\end{equation}
From a geometric perspective, Problem \eqref{dynamical OT} can be treated as a computing scheme  for the geodesic on the probability manifold equipped with Wasserstein distance $W_{\textrm{Dym}}(\cdot,\cdot)$. Following this point of view, we also treat \eqref{dynamical OT} as the problem for evaluating the \textbf{Wasserstein geodesic} joining $\rho_a$ and $\rho_b$. Meanwhile, the PDE system \eqref{geodesic eq} serves as the geodesic equation for the Wasserstein geodesic. 

Static OT problems and Dynamical OT problems are closely related. We summarize some of their connections, which are useful for our future derivations in section \ref{derivation }. 

Recall $(\phi^*,\psi^*)$ as optimal solution to \eqref{Kant duality}, and $\Phi^*$ as solution of \eqref{geodesic eq}, then
\begin{equation}
 \Phi^*(x, 0) = \psi^*(x) + C_0,\quad \Phi^*(x,1) = \phi^*(x) + C_1.   \label{static dual = dynamic dual}
\end{equation}
Here $C_0, C_1$ are constants. Furthermore, from \eqref{phi, psi as monge maps} we have
\begin{align}
\label{Phi0, Phi1 as monge maps}
 \begin{split}
  (\textrm{Id}+\nabla L^{-1}(\nabla\Phi^*(\cdot, 0)))_{\sharp}\rho_a=\rho_b,~~~  
  (\textrm{Id}-\nabla L^{-1}(\nabla\Phi^*(\cdot, 1)))_{\sharp}\rho_b = \rho_a. 
 \end{split}
\end{align}
In addition, Static and Dynamical OT produce the same distance, i.e. $W_{\textrm{Static}}(\rho_a, \rho_b) = W_{\textrm{Dym}}(\rho_a, \rho_b)$. 

\section{Proposed Methods }\label{derivation } 
\subsection{Primal-Dual based saddle point scheme}\label{prototype saddle problem}
In this section, we present an approach of solving \textbf{Dynamical OT} problem \eqref{dynamical OT} by applying Lagrange Multiplier method. We introduce Lagrange Multiplier $\Phi(x,t)$ for the PDE constraint $\frac{\partial\rho(x,t)}{\partial t} + \nabla\cdot(\rho(x,t){v}(x,t)) = 0 $ and $\Psi(x)$ for one of the boundary constraints $\rho(\cdot, 1) = \rho_b$. (The constraint $\rho( \cdot,0)=\rho_a$ can be naturally treated as the initial condition of $\rho_t$.) Then we consider the functional
\begin{align}
  \mathfrak{L}(\rho,{v},\Phi,\Psi) &=  \int_0^1\int L({v})~\rho + \left(\frac{\partial\rho}{\partial t} + \nabla\cdot(\rho{v})\right)\Phi(x,t)~dxdt + \int \Psi(x)(\rho(x,1)-\rho_b(x))~dx  \nonumber\\
 &=  \int_0^1\int \left( L({v}) -\frac{\partial\Phi}{\partial t} - \nabla\Phi\cdot{v} \right) ~\rho(x,t)~dx dt  \label{reformulation of L of original problem} + \int \Phi(x,1)\rho(x,1)~dx \\
 &- \int \Phi(x,0)\rho_a(x)~dx +\int \Psi(x) (\rho(x,1)-\rho_b(x))~dx. \nonumber
\end{align}
For the second equality, we apply integration by parts on $[0,1]$ and use the initial condition $\rho(\cdot, 0)=\rho_a$.
Solving the constrained optimization problem \eqref{dynamical OT} is equivalent to investigating the following saddle point optimization problem
\begin{equation}
    \min_{\rho,{v}}  \max_{\Phi,\Psi}   ~\mathfrak{L}(\rho,{v},\Phi,\Psi).  \label{original saddle scheme}
\end{equation}
Problem \eqref{original saddle scheme} contains $(\rho,{v},\Phi,\Psi)$ as variables. We eliminate some of the variables by leveraging
the Karush–Kuhn–Tucker (KKT) conditions \citep{kuhn1951}
\begin{equation}
   \frac{\partial \mathfrak{L}}{\partial \Phi(x,t)}=0, \frac{\partial \mathfrak{L}}{\partial \Psi(x)} = 0, 
  \frac{\partial \mathfrak{L}}{\partial  \rho(x,t)} = 0,  \frac{\partial \mathfrak{L}}{\partial {v}(x,t)} = 0. \label{kkt}
\end{equation}
The first two conditions lead to the constraints \eqref{constraints}. The third condition in \eqref{kkt} yields
\begin{align}
   -\frac{\partial\Phi}{\partial t} - (\nabla\Phi(x,t)\cdot{v}(x,t) - L({v}(x,t))) = 0, \label{condition A original HJ}\\
   \Phi(x,1)+\Psi(x)=0.  \label{condition B}
\end{align}
The fourth condition in \eqref{kkt} yields $\nabla L({v}(x, t)) -\nabla\Phi(x,t) = 0$, 
which can be rewritten as
\begin{equation}
  {v}(x,t) = \nabla L^{-1} (\nabla\Phi(x,t)).  \label{condition C  rel of v and nabla Phi}
\end{equation}
The KKT conditions \eqref{condition B} and \eqref{condition C  rel of v and nabla Phi} reveal explicit relations among ${v}$, $\Phi$ and $\Psi$, which can be incorporated in \eqref{original saddle scheme} by plugging $\Psi(x) = -\Phi(x,1)$ and ${v}(x,t) = \nabla L^{-1}(\nabla\Phi(x,t))$ back into \eqref{reformulation of L of original problem}. 
Recall definition of $H$ in \eqref{Legendre}, the first term of \eqref{reformulation of L of original problem} then becomes $-\left(\frac{\partial\Phi(x,t)}{\partial t}+H(\nabla \Phi(x,t))\right)$. Our new optimization problem can be formulated as
\begin{align}
  \min_{\rho}\max_{\Phi}
\int_0^1\int -\left(\frac{\partial\Phi}{\partial t}+H(\nabla \Phi)\right) \rho(x,t)~dxdt  + \int\Phi(x,1)\rho_b(x) - \Phi(x,0)\rho_a(x)~dx.
\label{functional L}
\end{align}


\subsection{Simplification via geodesic pushforward map}\label{geodesic pushforward}
In the saddle point problem \eqref{functional L}, both variables $\rho(\cdot,t)$ and $\Phi(\cdot,t)$ are time-varying functions. It is a weak formulation of the OT problem that may accommodate general cost $L$. On the other hand, it requires to optimize over rather large space of time-varying functions, which may increase the computational cost as well as the chance of encountering local optima. 
To mitigate the challenge, we reduce the search space by leveraging the the following geodesic property of optimal transporting trajectory if $L$ is strictly convex. 
\begin{theorem}\label{transport along geod}
    Suppose $\{\boldsymbol{X}^*_t\}_{t=0}^1$ is the trajectory obeying the optimal vector field of \eqref{particle dynamical OT} with strictly convex Lagrangian $L$, then $\frac{d^2\boldsymbol{X}^*_t}{dt^2}=0$.
\end{theorem}

This result was discussed by \citet{hjd2} for 2-Wasserstein case ($L(\cdot)=\frac{|\cdot|^2}{2}$). The more general result was presented in Theorem 5.5 of \cite{villani2003topics}. This theorem illustrates that, from the particle point of view, the optimal transport process can be treated as a pushforward operation along the geodesics (straight lines) in $\mathbb{R}^d$. To be more precise, we denote $\{\rho^*(\cdot, t)\}_{t=0}^1$ as the optimal solution to \eqref{dynamical OT}. Denote $\{{v}^*(\cdot, t)\}_{t=0}^1$ as the optimal vector field in \eqref{particle dynamical OT}. Then $\{\boldsymbol{X}_t^*\}_{t=0}^1$ solves $\frac{d}{dt}\boldsymbol{X}_t^* = {v}^*(\boldsymbol{X}_t,t)$. Since $\frac{d^2}{dt^2}\boldsymbol{X}_t^*=0$, this implies $\boldsymbol{X}^*_t = \boldsymbol{X}^*_0+t{v}^*(\boldsymbol{X}^*_0, 0)$, $t\in [0,1]$. Finally, due to the equivalence between \eqref{dynamical OT} and \eqref{particle dynamical OT}, we are able to verify that $\boldsymbol{X}^*_t\sim\rho^*(\cdot, t)$, which yields $\boldsymbol{X}^*_0+t{v}^*(\boldsymbol{X}^*_0, 0)\sim\rho^*(\cdot, t)$. Since $\boldsymbol{X}_0^*$ follows the distribution with density $\rho_a$, this leads to the following relation between optimal density $\rho^*(\cdot,t)$ and optimal vector field ${v}^*(\cdot, 0)$ at $t=0$
\begin{equation}
  \rho^*(\cdot, t) = (\textrm{Id}+t{v}^*(\cdot, 0))_{\sharp}\rho_a.  \label{optimal rho given by geodesic pushfwd}
\end{equation}
Equation \eqref{optimal rho given by geodesic pushfwd} justifies that the optimal $\rho^*(\cdot, t)$ can be obtained by pushforwarding the initial distribution $\rho_a$ along certain straight lines with initial direction ${v}^*(\cdot, 0)$. This observation motivates us to restrict the search space of $\{\rho_t\}_{t=0}^1$ on the following space $\mathcal{P}_{\textrm{restrict}}$
\begin{equation*}
  \{~\{\rho(\cdot, t)\}_{t=0}^1~|~ ~ \rho(\cdot, t) = (\textrm{Id}+t{F})_{\sharp}\rho_a~~\textrm{for}~  t \in[0,1]~~ \}.
\end{equation*}
Here ${F}$ is an arbitrary vector field defined on $\mathbb{R}^d$. Combining the previous discussions together, we propose the scheme $ \min_{\rho\in \mathcal{P}_{\textrm{restrict}}}\max_{\Phi} \hat{\mathfrak{L}}(\rho,\Phi)$. In particular, since $\rho$ is uniquely determined by ${F}$, we reformulate our scheme as
\begin{align}
  &  \min_{{F}}\max_{\Phi}  \mathcal{L}({F},\Phi),~~~ \mathcal{L}({F}, \Phi) = \hat{\mathfrak{L}}((\textrm{Id}+t{F})_{\sharp}\rho_a,\Phi).\label{chosen scheme}
\end{align}
We have the following theoretical property for scheme \eqref{chosen scheme}.
\begin{restatable}{theorem}{theoremA}
\label{main}
  Denote the solution \eqref{geodesic eq} as $\rho^*(x,t)$ and $\Phi^*(x,t)$, and set $\Phi_0^*(\cdot) = \Phi^*(\cdot,0)$. Assume $\Phi^*(\cdot, t)\in C^2(\mathbb{R}^d)$, then 
  $(\nabla L^{-1}(\nabla \Phi_0^*), \Phi^*)$ is a critical point to the functional $\mathcal{L}$, i.e.
\[ \frac{\partial \mathcal{L}}{\partial {F}}(\nabla L^{-1}(\nabla \Phi_0^*),\Phi^*) = 0, \quad \frac{\partial \mathcal{L}}{\partial \psi}(\nabla L^{-1}(\nabla \Phi_0^*),\Phi^*) = 0.\]
Furthermore,  $\mathcal{L}(\nabla L^{-1}(\nabla\Phi_0^*), \Phi^*)=W_{\textrm{Dym}}(\rho_a,\rho_b)$.
\end{restatable}
Theorem \ref{main} shows that the optimal solution of Dynamical OT is also a critical point of the functional used in our saddle point optimization scheme \eqref{chosen scheme}. The optimal value of \eqref{chosen scheme} is exactly the optimal transport distance. The theorem is proved in the supplemental materials (Appendix B).

\subsection{Bidirectional dynamical formulation}
To improve the stability and avoid local traps in the training processing, we propose a {\it bidirectional} scheme by exploiting the symmetric status of $
\rho_a$ and $\rho_b$ in \eqref{Static OT}.  
Let's consider two OT problems
\begin{equation*}
    \min_{{F}}\max_{\Phi_F} ~ \mathcal{L}^{ab}({F}, \Phi_F), \quad 
  \min_{ {G} }\max_{\Phi_G} ~ \mathcal{L}^{ba}({G}, \Phi_G),
\end{equation*}
where $\mathcal{L}^{ab}$ is defined in \eqref{chosen scheme}, and $\mathcal{L}^{ba}$ is defined by switching $\rho_a$ and $\rho_b$ in \eqref{chosen scheme}. Clearly, the first one is for $W(\rho_a,\rho_b)$ while the second one is for $W(\rho_b, \rho_a)$.
When reaching optima, the vector fields $F$ and $G$ are transport vectors in the opposite directions. At a specific point $x\in\mathbb{R}^d$, moving along straight line in the direction ${F}$ ends up at $x+{F}(x)$. The direction of ${G}$ at $x+{F}(x)$ must point to the opposite direction of ${F}(x)$, which leads to ${G}(x+{F}(x))=-{F}(x)$. Similarly, we also have ${F}(x+{G}(x))=-{G}(x)$. Thus we introduce two constraints for ${F}$ and ${G}$
\begin{align*}
\begin{split}
 \mathcal{R}^{ab}({F}, {G}) = & \int |{G}(x+{F}(x))+{F}(x)|^2~\rho_a(x)~dx, \\
 \mathcal{R}^{ba}({F}, {G}) = & \int |{F}(x+{G}(x))+{G}(x)|^2\rho_b(x)~dx.
\end{split}
\end{align*}
Our final saddle-point problem becomes
\begin{align}
  \label{bidirection}
  \begin{split}
  \min_{ {F}, {G} }\max_{\Phi_F,\Phi_G} ~& \mathcal{L}^{ab}({F}, \Phi_F) + \mathcal{L}^{ba}({G}, \Phi_G) + \lambda(\mathcal{R}^{ab}({F}, {G}) + \mathcal{R}^{ba}({F}, {G})),
  \end{split}
\end{align}
where $\lambda$ is a tunable coefficient of our constraint terms. 

\subsection{Overview of the algorithm}
To solve \eqref{bidirection}, we propose an algorithm that is summarized in the following steps. Please also check its detailed discussions in Appendix A.
\begin{itemize}[leftmargin=*]
  \item \textbf{Preconditioning}  We can apply preconditioning techniques to 2-Wasserstein cases in order to make our computation more efficient.
  \item \textbf{Main Algorithm} We set ${F}_{\theta_1}, {G}_{\theta_2}$ and $\Phi^F_{\omega_1},\Phi^G_{\omega_2}$ as fully connected neural networks and optimize over their parameters $\omega_1,\omega_2$ and $\theta_1,\theta_2$ alternatively via stochastic gradient ascend and descend.
    \item \textbf{Stopping Criteria} When computed ${F}$ (or ${G}$) is close to the optimal solution, the Wasserstein distance $W(\rho_a,\rho_b)$ (or $W(\rho_b, \rho_a)$) can be approximated by 
    \begin{equation*}
      \widehat{W}^{ab} = \int L({F}(x))~\rho_a(x)~dx,  ~ \widehat{W}^{ba} = \int L({G}(x))~\rho_b(x)~dx.
    \end{equation*}
    For a chosen threshold $\epsilon>0$, we treat $|\widehat{W}^{ab} - \widehat{W}^{ba}|<\epsilon$ as the stopping criteria of our algorithm.
\end{itemize}

\section{Experiments}
\textbf{Experiment Setup:}
We test our algorithm through a series of synthetic data sets and compare our numerical results with the ground truths (only for the Gaussian cases) or with the computational methods introduced in the Python library (Python Optimal Transport (POT)) \citep{JMLR:v22:20-451}. We also test our algorithm for realistic data sets including color transfer \citep{10.1109/38.946629} and transportation between MNIST digits \citep{726791}. 

For low dimensional cases, namely, 2 and 10 dimensional cases, we set $\Phi_F, \Phi_G$ and $F, G$ as fully connected neural networks, where $\Phi_F, \Phi_G$ have 6 hidden layers and $F$ and $G$ have 5 hidden layers. Each layer has 48 nodes, the activation function is chosen as Tanh. For high dimensional cases, where we deal with MNIST handwritten digits data set, we adopt similar structures of neural networks, the only difference is that in each layer we extend the number of nodes from 48 to 512. In terms of training process, for all synthetic and realistic cases we use the Adam optimizer \citep{kingma2014adam} with learning rate $10^{-4}$. Notice that we are computing Wasserstein geodesic, namely, starting with an initial distribution $\rho_{t_0}$, in most cases we generate evolving distributions for next ten time steps, from $t_1=0.1$ to $t_{10}=1$. The cost functions are chosen as $L(v)=|v|^{\frac{3}{2}}$ in Synthetic-2 and $L(v)=|v|^2$ in all other tests. We only show the final state of the generated distribution due to space limitation, more experiments and details are included in the last part of Appendix.

\textbf{Synthetic-1}: This is a 2-dimensional case, we set $\rho_a$ as a standard Gaussian distribution $N(\mu_0,\Sigma_0)$ while $\rho_b$ as six surrounding Gaussian distributions with the same $\Sigma_0$. In Figure \ref{fig:syn-1}, we show the generated distribution that follows $\rho_b$ as well as the one that follows $\rho_a$. We also show the start-end tracks of points and vector field.

\vspace{-0.2em}
\textbf{Synthetic-2}: In this 5-dimensional case we treat $\rho_a$ and $\rho_b$ both as two Gaussian distributions. We show the results of two dimensional projection in Figure \ref{fig:syn-2}.
\vspace{-0.2em}

\textbf{Synthetic-3}: As a 10-dimensional case, here we set $\rho_a$ as a standard Gaussian distribution and $\rho_b$ as a special distribution where samples are unevenly distributed around four corners. We show the results of two dimensional projection in Figure \ref{fig:syn-3}.

\vspace{-0.2em}

\textbf{Training and Results}: For synthetic data sets, in the training process we set the batch size $N_t=2000$ and sample size for prediction $N_p=1000$. From Figures \ref{fig:syn-1}, \ref{fig:syn-2} and \ref{fig:syn-3} we see that in all cases, within various dimensional settings, the generated samples closely follow the ground-truth distributions.

\vspace{-0.2em}

\textbf{Comparisons}: We now compare the numerical results listed in the Synthetic cases with either the ground truth (Synthetic-2) or with the results computed by POT (Synthetic-1, Synthetic-3) in Figure \ref{fig:com}. Here all the examples are computed under quadratic cost $L(v)=|v|^2/2$.

\begin{figure}[t!]
\centering
\subfloat[][]{\includegraphics[width=.125\linewidth]{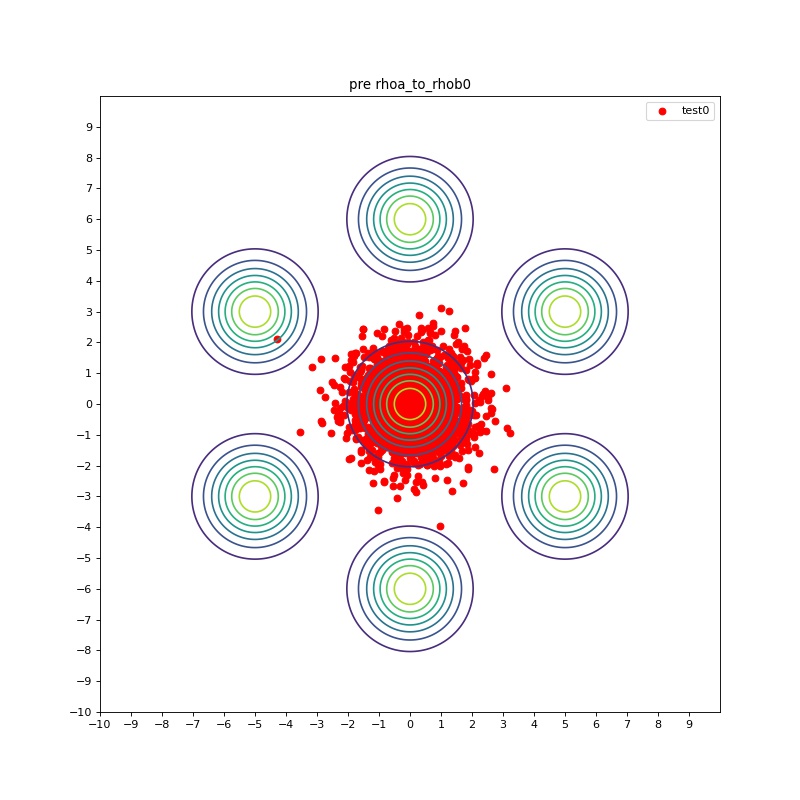}}
\subfloat[][]{\includegraphics[width=.125\linewidth]{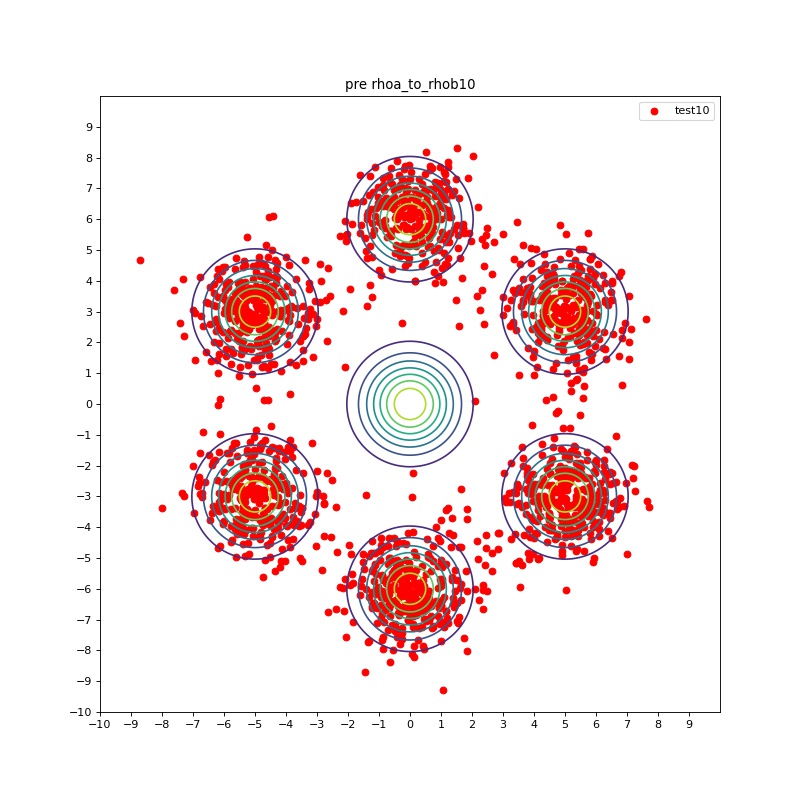}}
\subfloat[][]{\includegraphics[width=.125\linewidth]{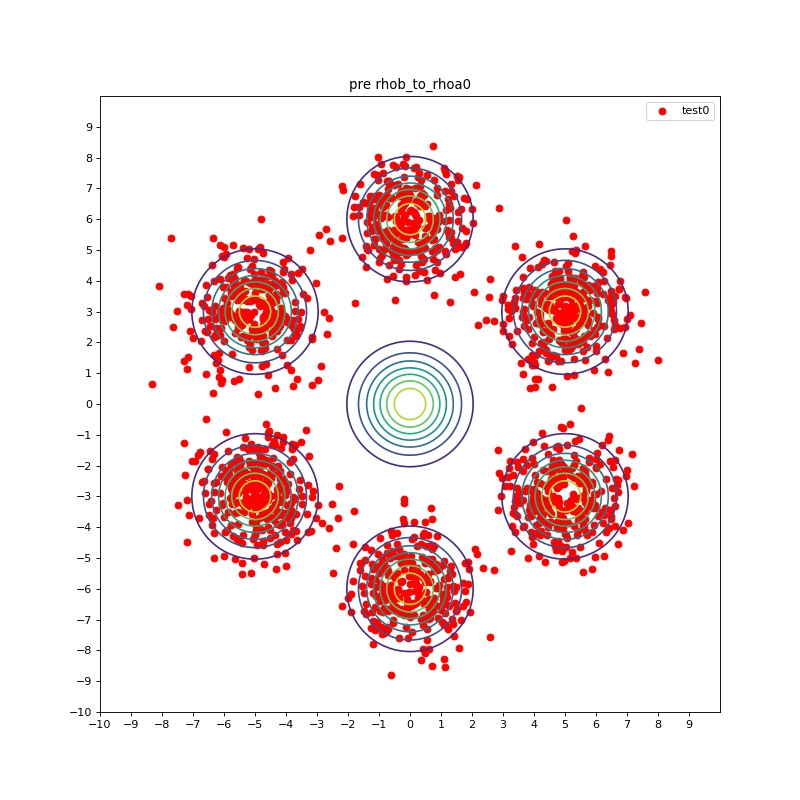}}
\subfloat[][]{\includegraphics[width=.125\linewidth]{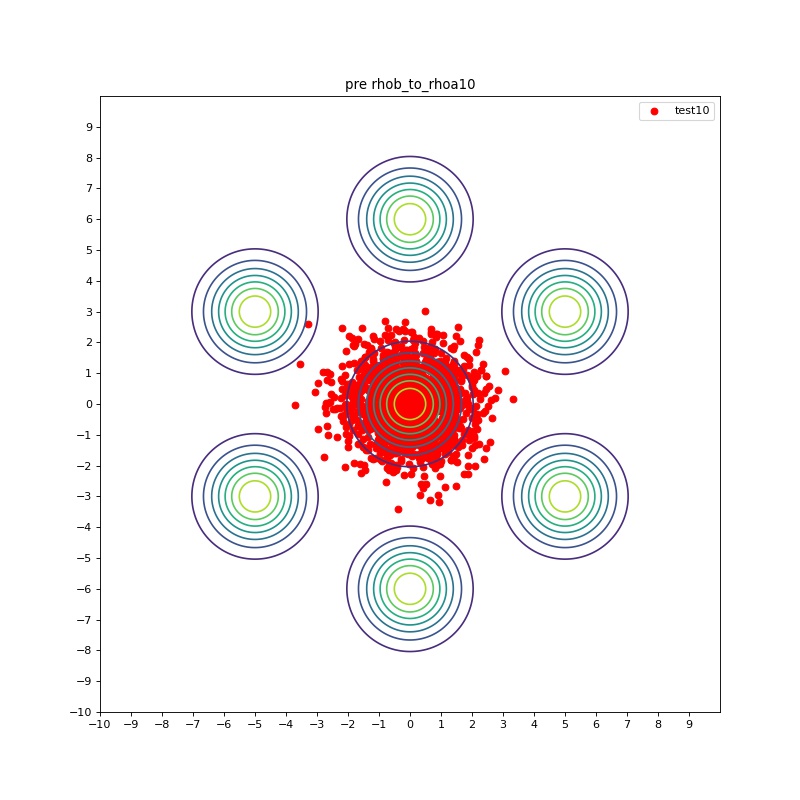}}
\subfloat[][]{\includegraphics[width=.125\linewidth]{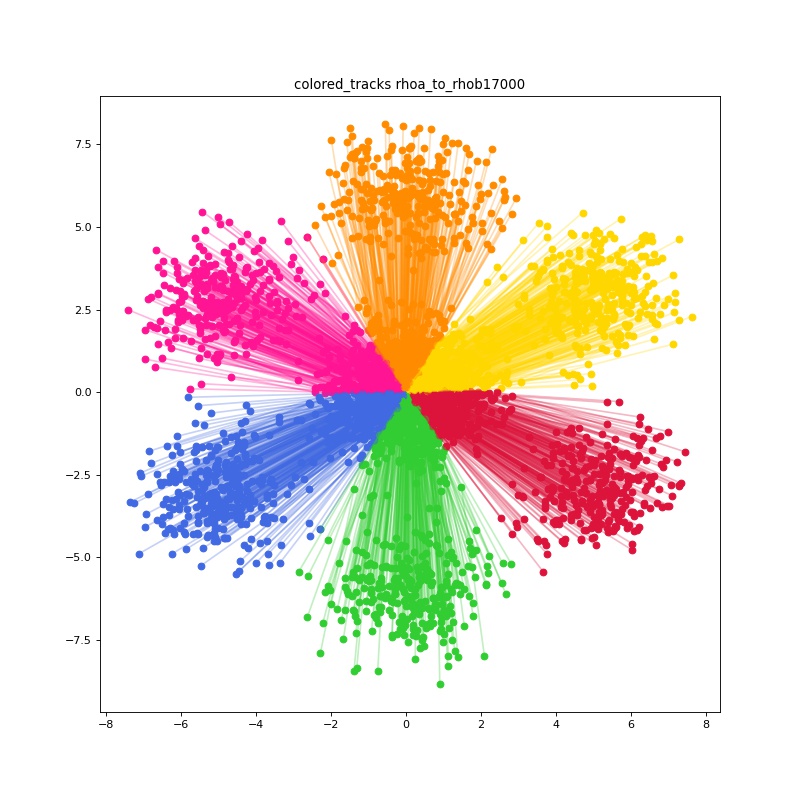}}
\subfloat[][]{\includegraphics[width=.125\linewidth]{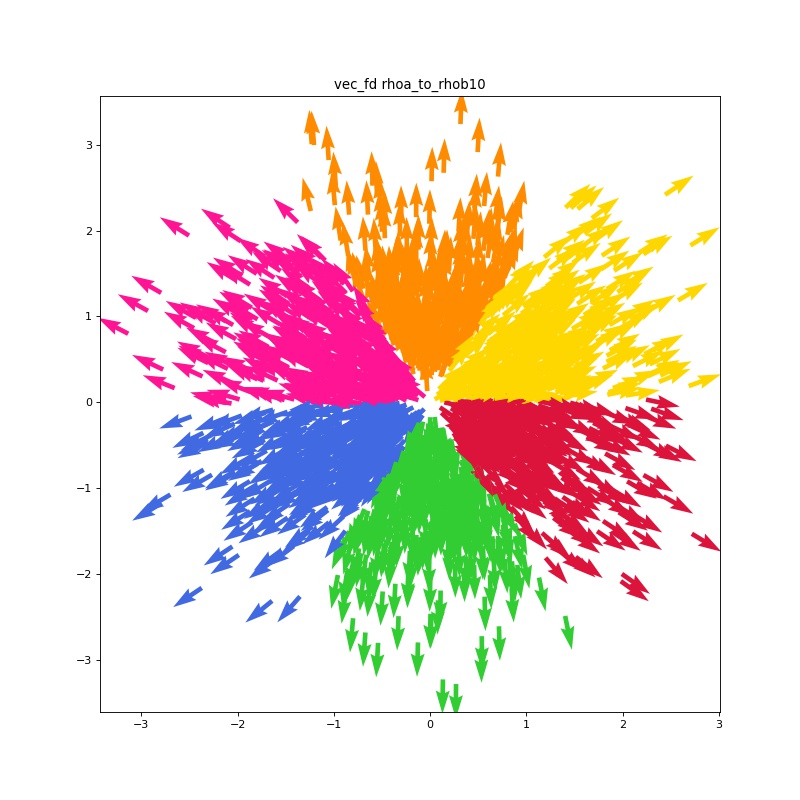}}
\subfloat[][]{\includegraphics[width=.125\linewidth]{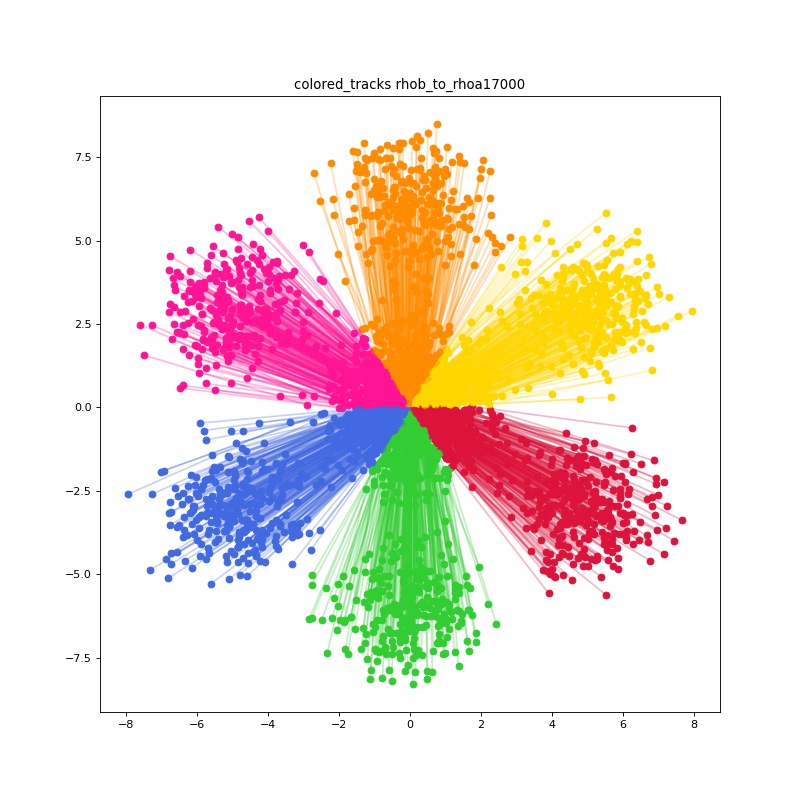}}
\subfloat[][]{\includegraphics[width=.125\linewidth]{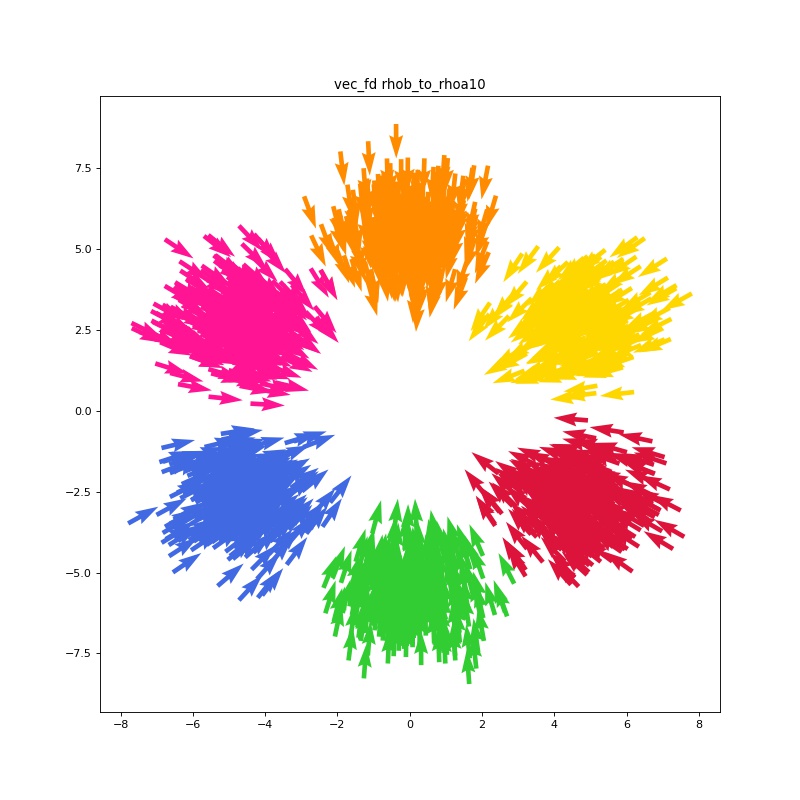}}
\caption{Syn-1. (a)(b) true $\rho_a$ and generated $\rho_b$, (c)(d) true $\rho_b$ and generated $\rho_a$, (e)(g) tracks of sample points from $\rho_a$($\rho_b$) to $\rho_b$($\rho_a$), (f)(h) vector fields from $\rho_a$($\rho_b$) to $\rho_b$($\rho_a$).}
\label{fig:syn-1}
\vspace{-2em}
\end{figure}

\vspace{-0.2em}

\begin{figure}[t!]
\subfloat[][]{\includegraphics[width=.125\linewidth]{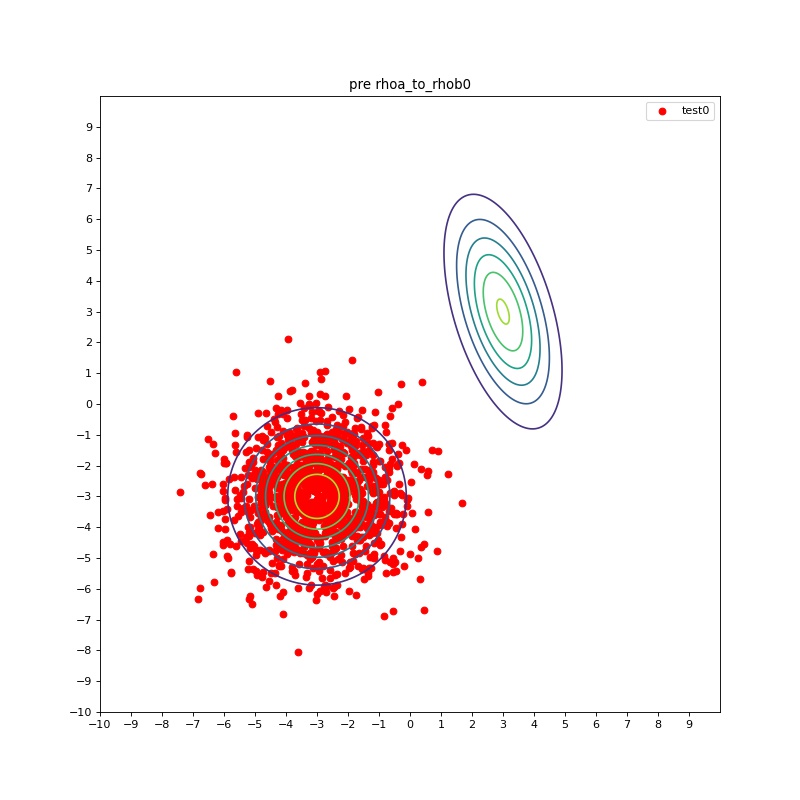}}
\subfloat[][]{\includegraphics[width=.125\linewidth]{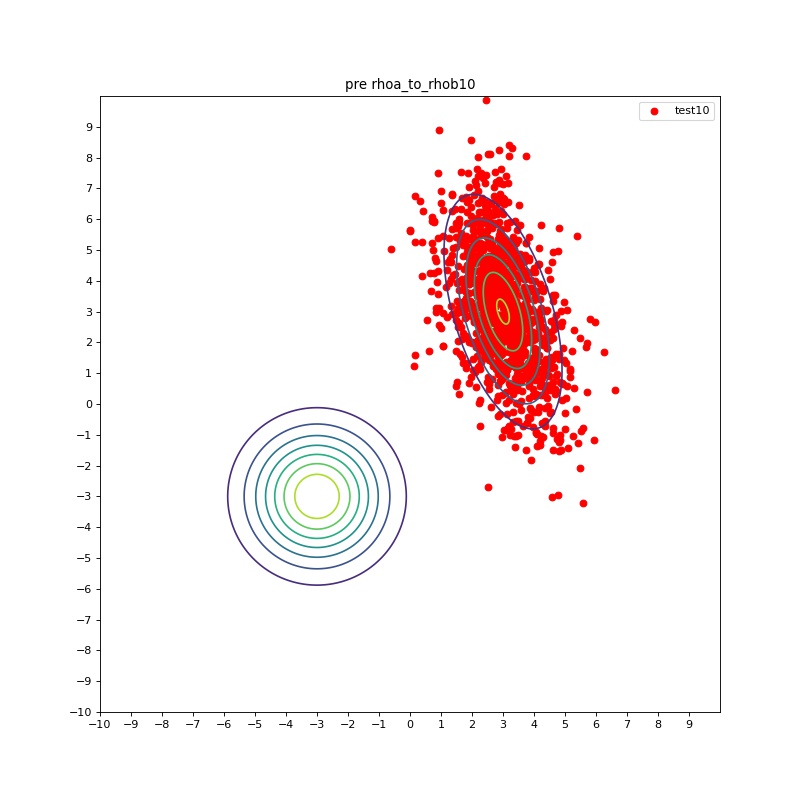}}
\subfloat[][]{\includegraphics[width=.125\linewidth]{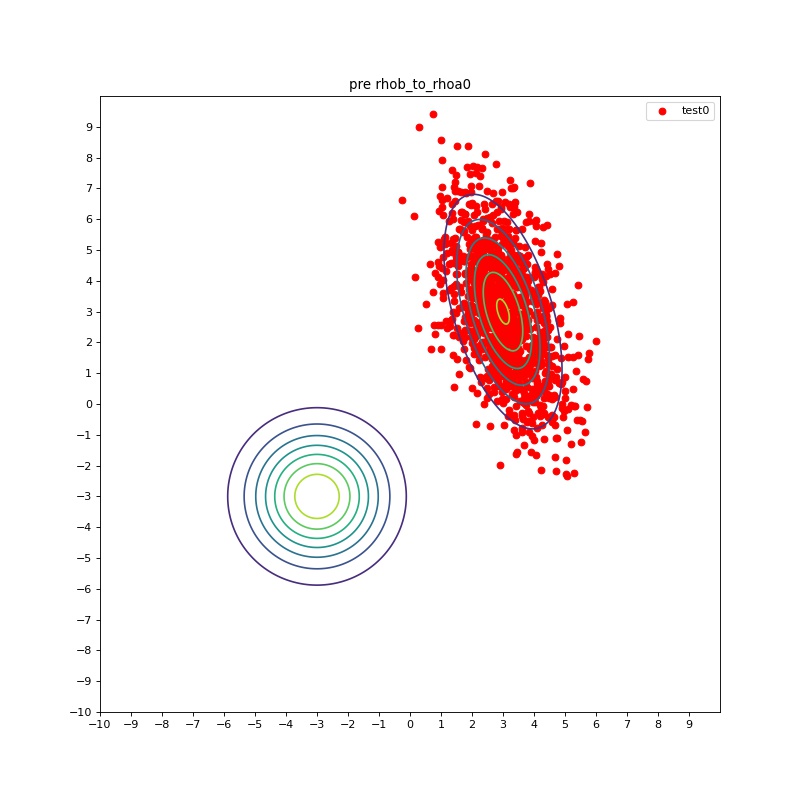}}
\subfloat[][]{\includegraphics[width=.125\linewidth]{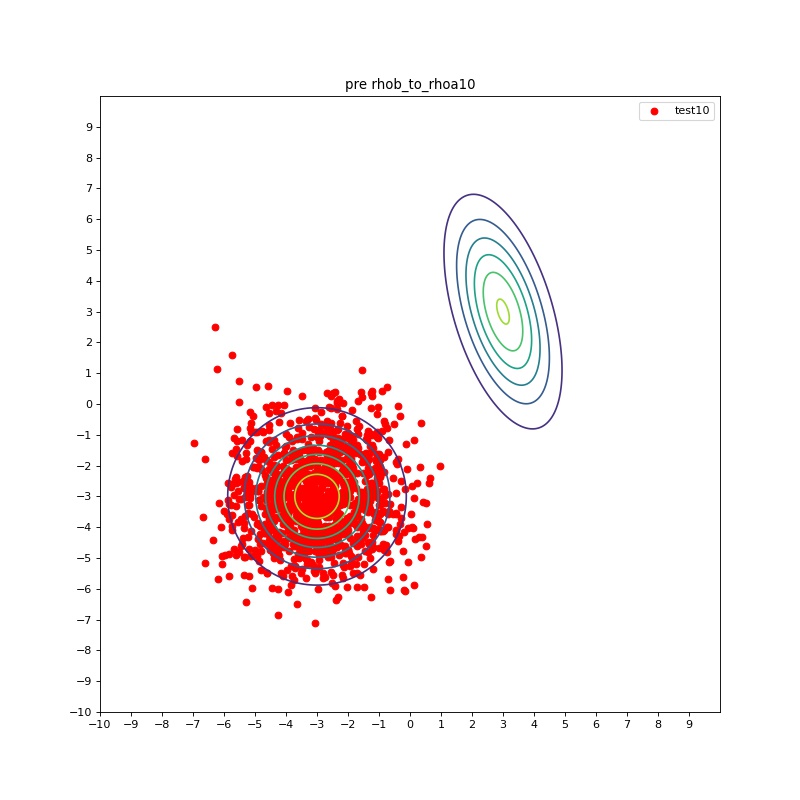}}
\subfloat[][]{\includegraphics[width=.125\linewidth]{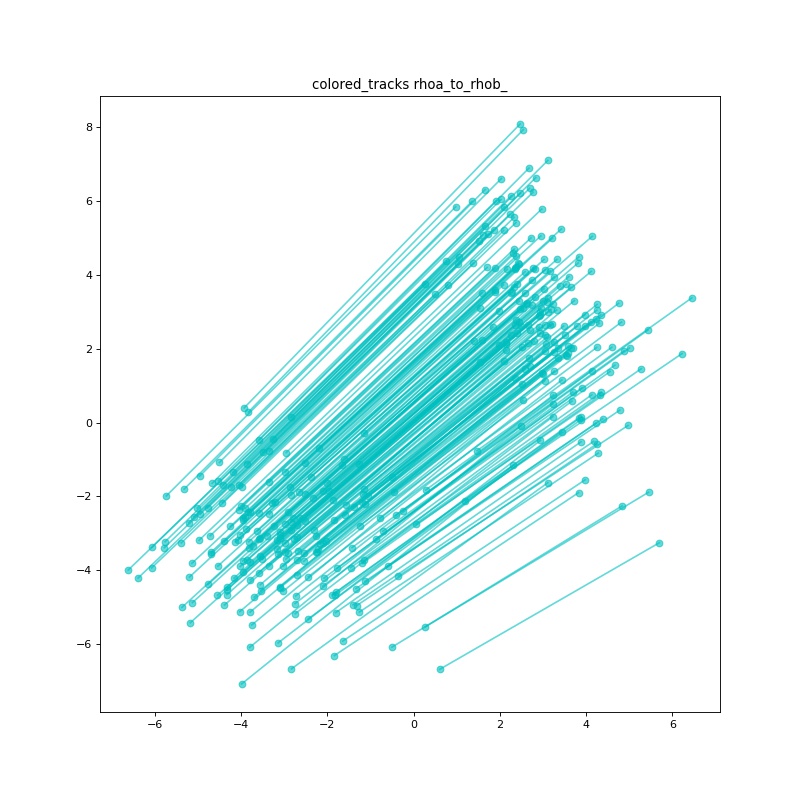}}
\subfloat[][]{\includegraphics[width=.125\linewidth]{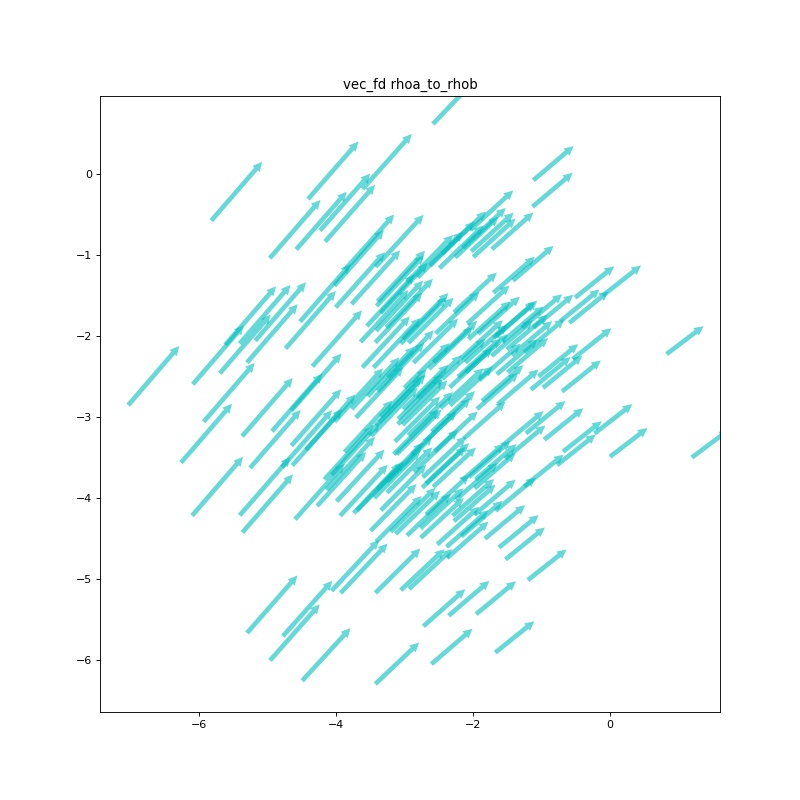}}
\subfloat[][]{\includegraphics[width=.125\linewidth]{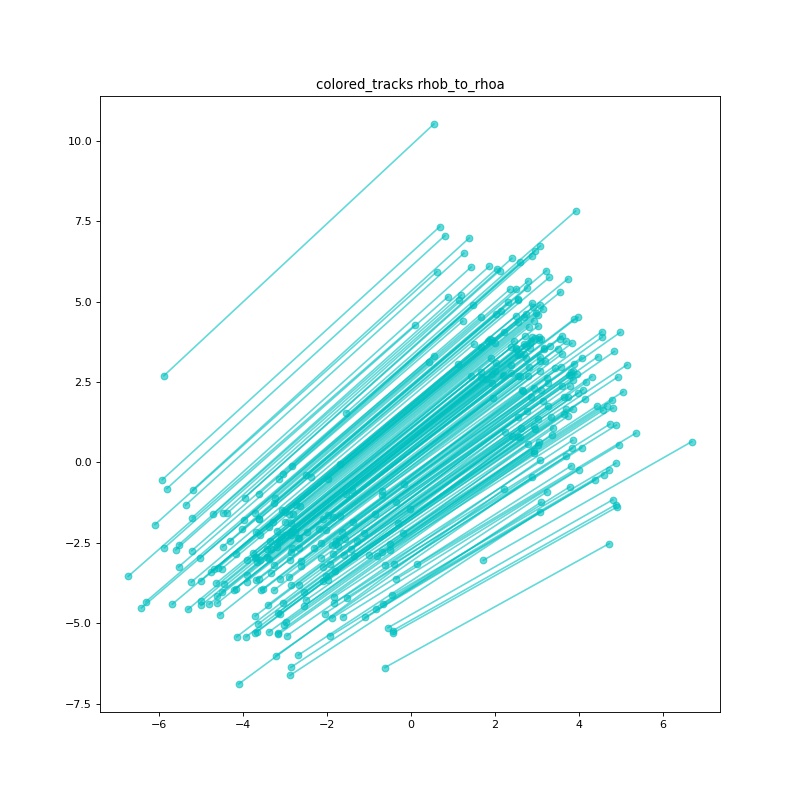}}
\subfloat[][]{\includegraphics[width=.125\linewidth]{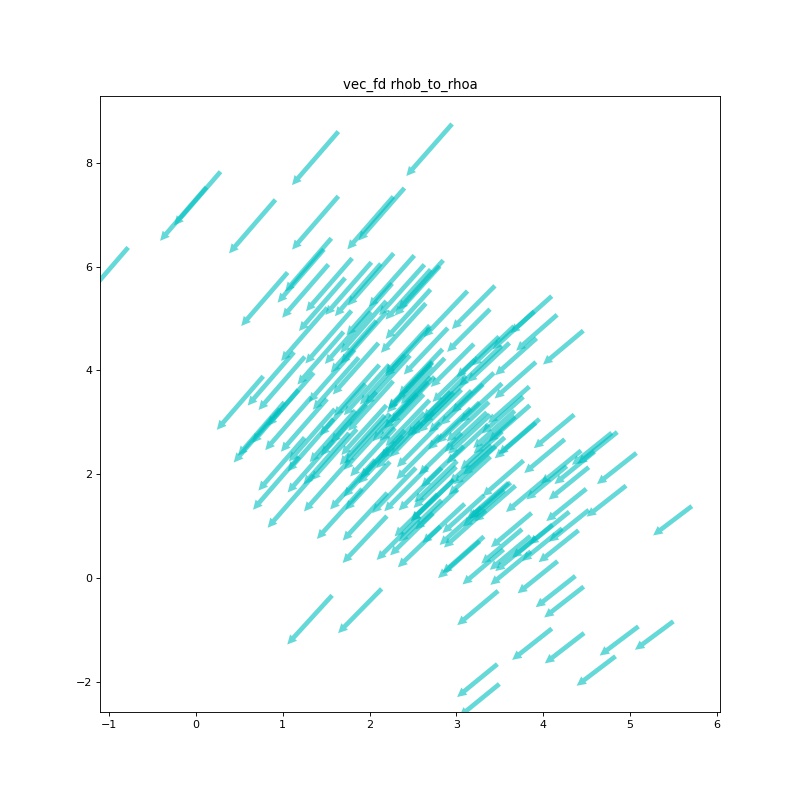}}
\caption{Syn-2. (a)(b) true $\rho_a$ and generated $\rho_b$, (c)(d) true $\rho_b$ and generated $\rho_a$, (e)(g) tracks of sample points from $\rho_a$($\rho_b$) to $\rho_b$($\rho_a$), (f)(h) vector fields from $\rho_a$($\rho_b$) to $\rho_b$($\rho_a$).}
\label{fig:syn-2}
\vspace{-2em}
\end{figure}

\begin{figure}[t!]
\subfloat[][]{\includegraphics[width=.125\linewidth]{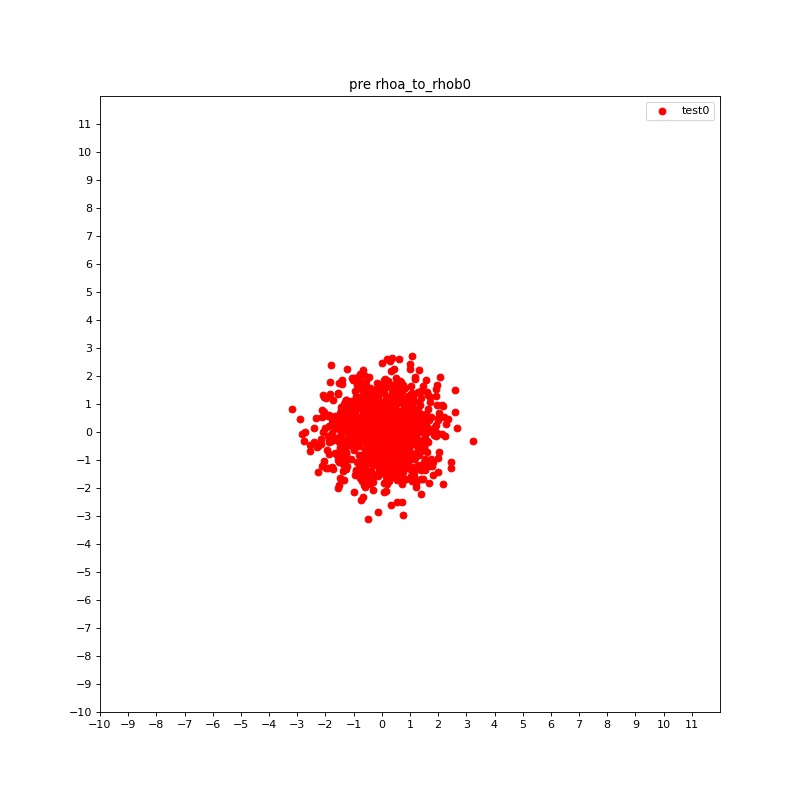}}
\subfloat[][]{\includegraphics[width=.125\linewidth]{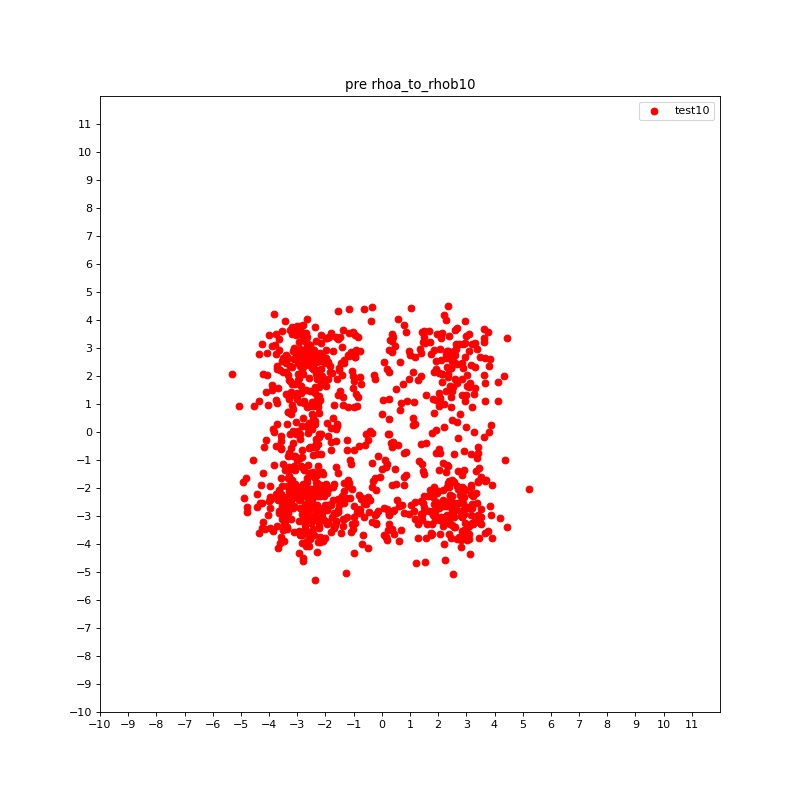}}
\subfloat[][]{\includegraphics[width=.125\linewidth]{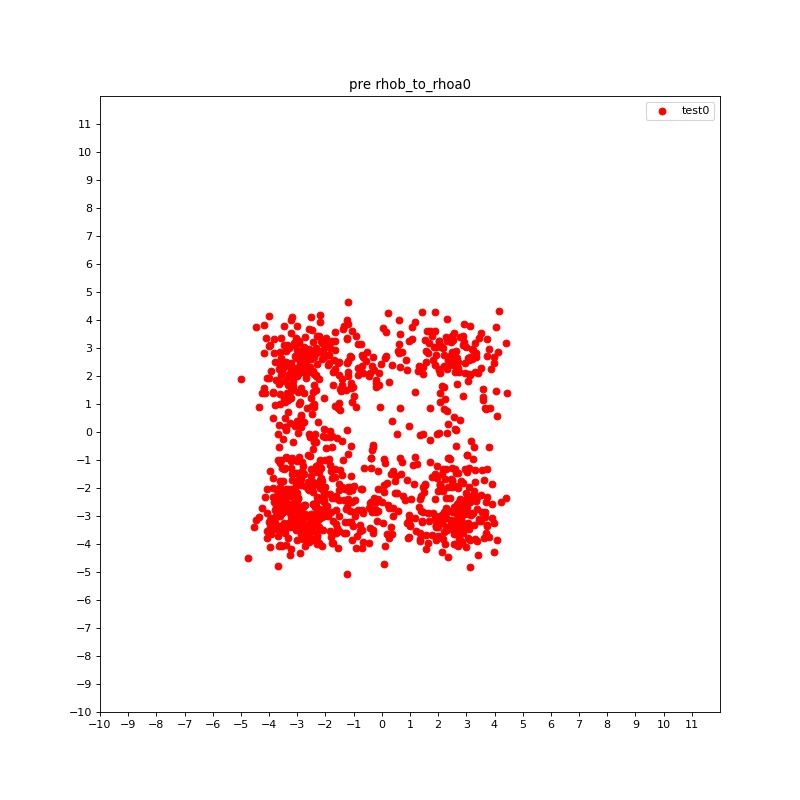}}
\subfloat[][]{\includegraphics[width=.125\linewidth]{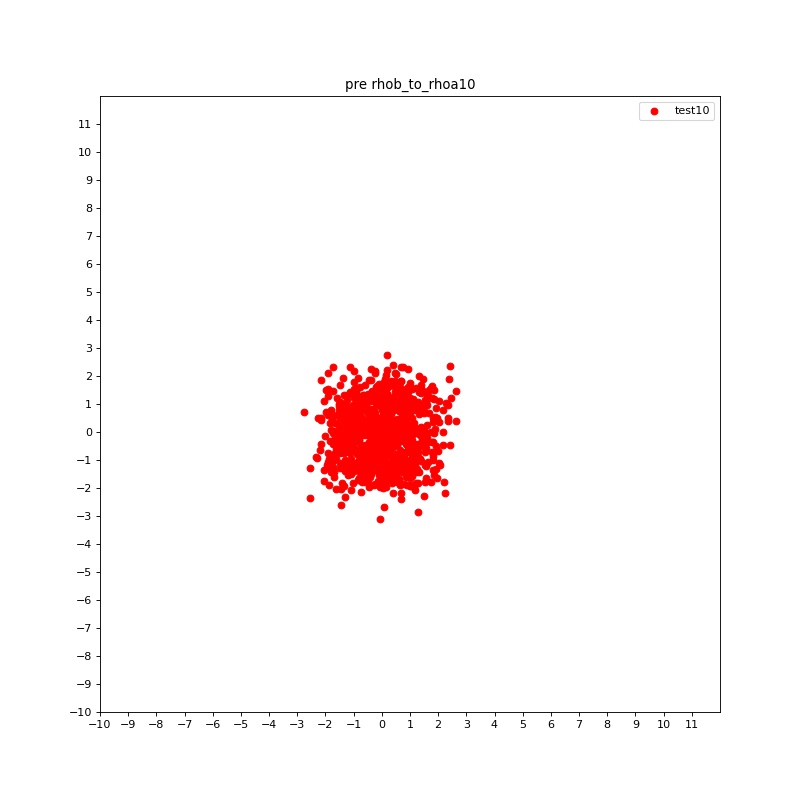}}
\subfloat[][]{\includegraphics[width=.125\linewidth]{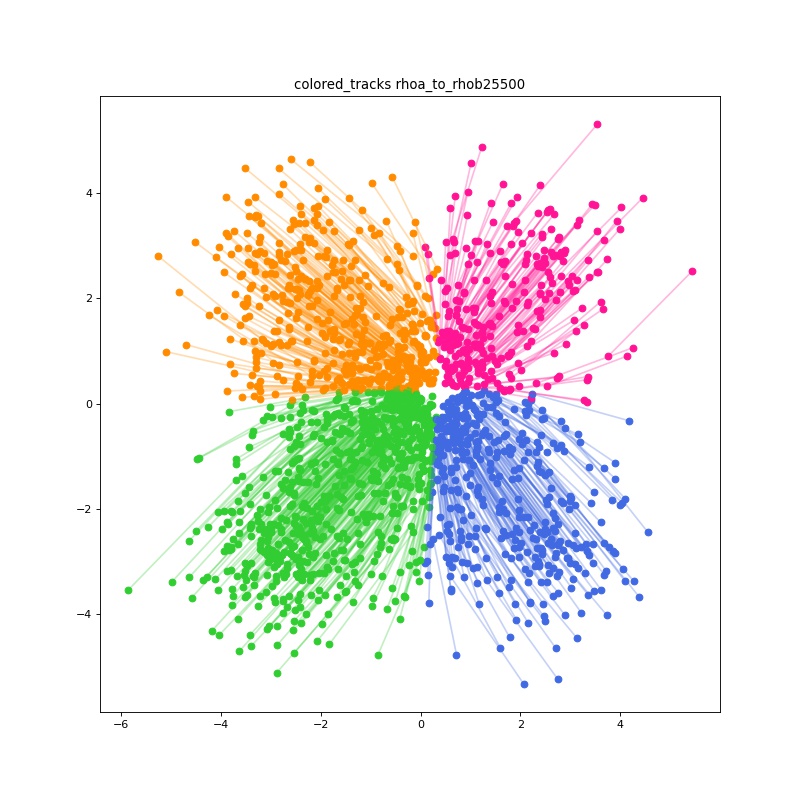}}
\subfloat[][]{\includegraphics[width=.125\linewidth]{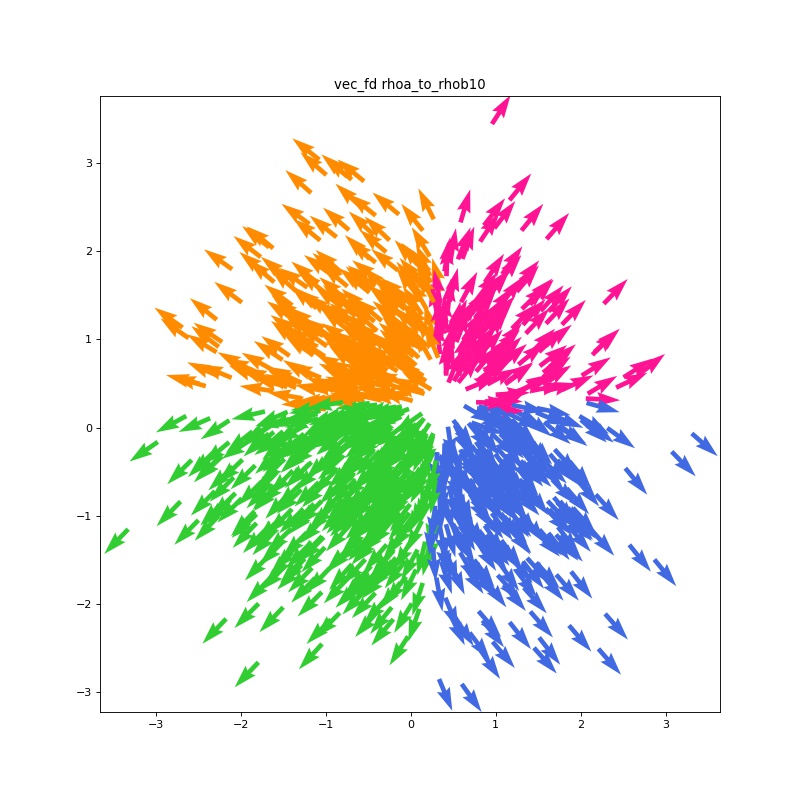}}
\subfloat[][]{\includegraphics[width=.125\linewidth]{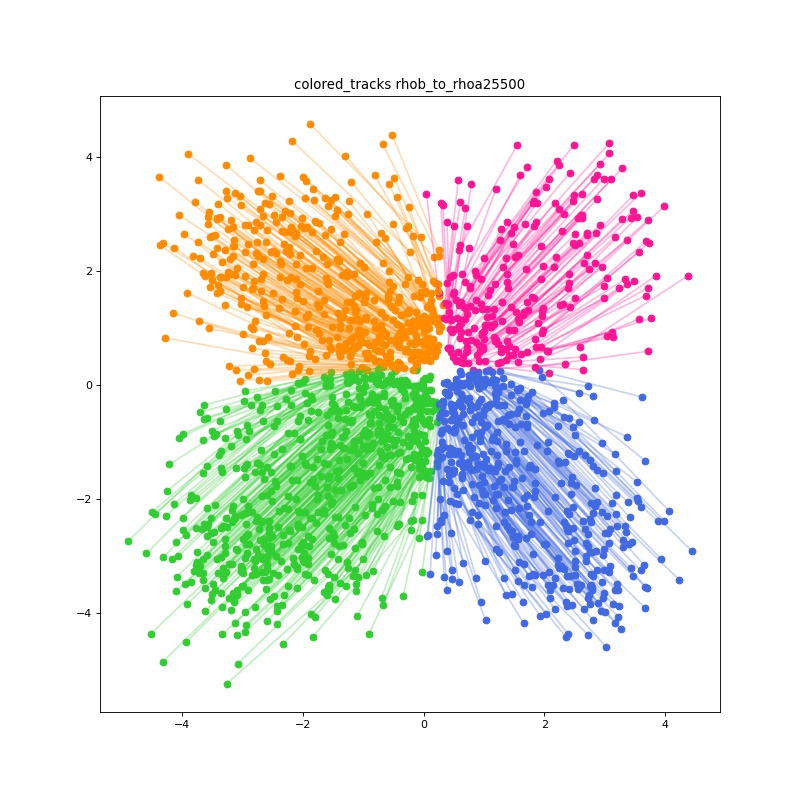}}
\subfloat[][]{\includegraphics[width=.125\linewidth]{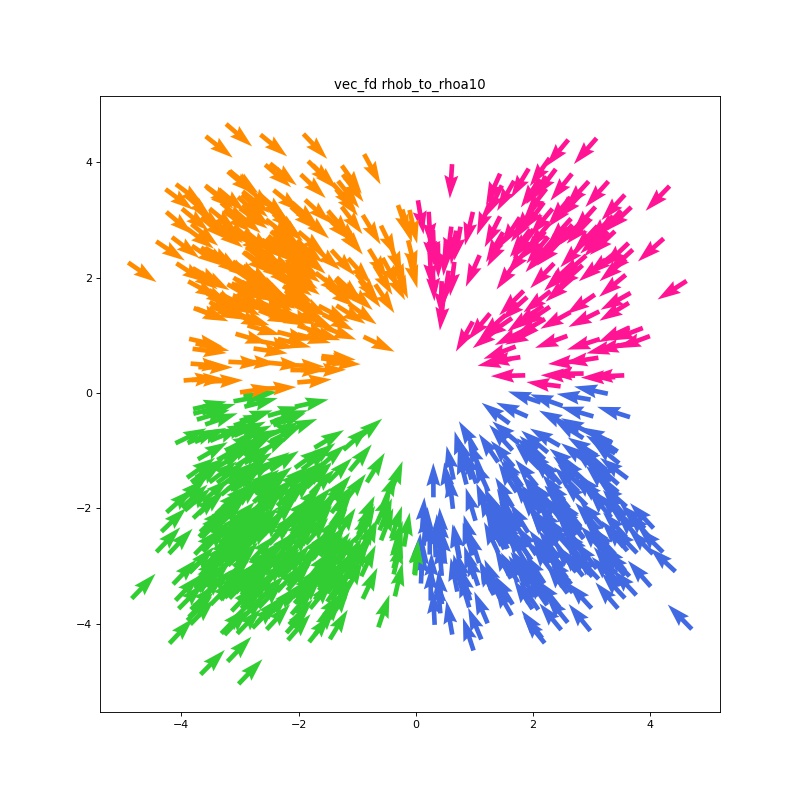}}
\caption{Syn-3. (a)(b) true $\rho_a$ and generated $\rho_b$, (c)(d) true $\rho_b$ and generated $\rho_a$, (e)(g) tracks of sample points from $\rho_a$($\rho_b$) to $\rho_b$($\rho_a$), (f)(h) vector fields from $\rho_a$($\rho_b$) to $\rho_b$($\rho_a$).}
\label{fig:syn-3}
\vspace{-2em}
\end{figure}

\begin{figure}[t!]
\centering
\subfloat[][]{\includegraphics[width=.26\linewidth]{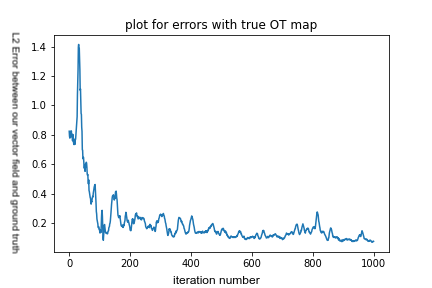}}
\subfloat[][]{\includegraphics[width=.18\linewidth]{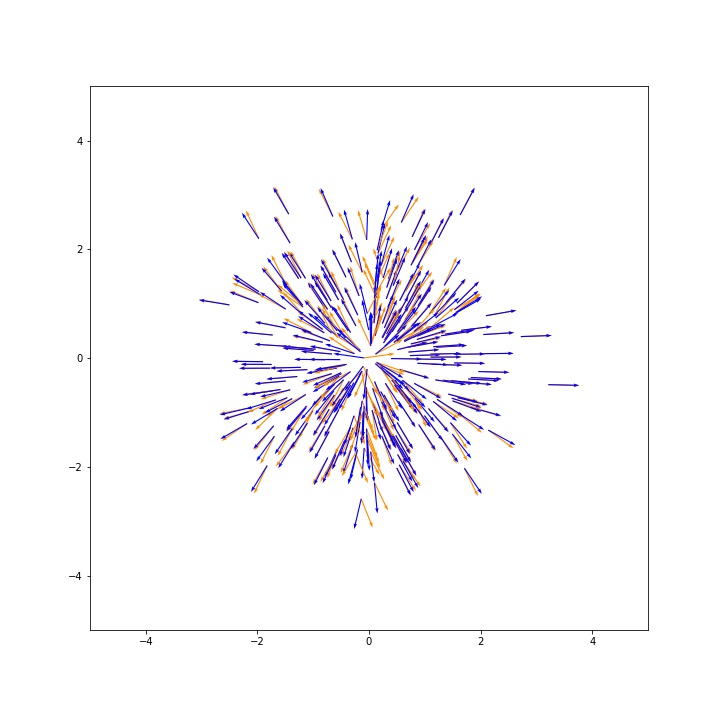}}
\subfloat[][]{\includegraphics[width=.18\linewidth]{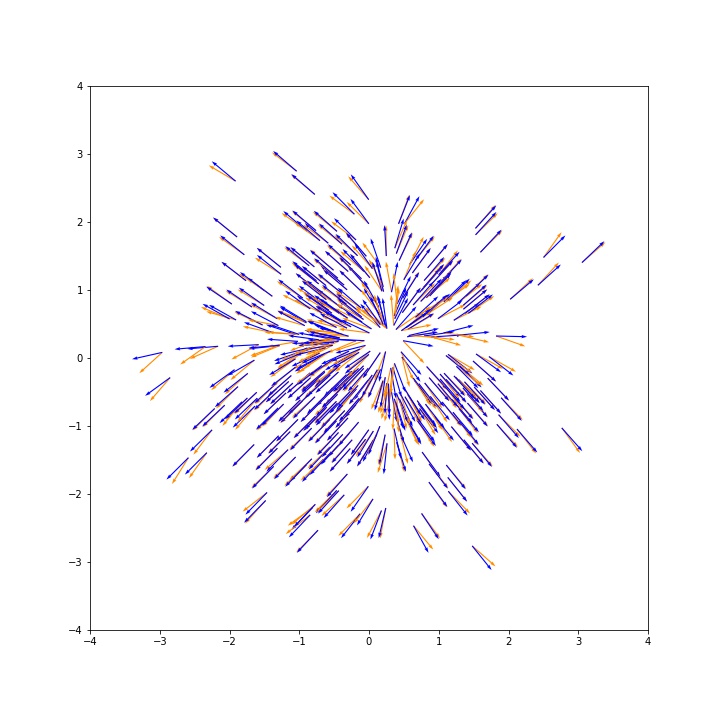}}
\caption{Left: Syn2: $L^2(\rho_a)$ error between our computed $F$ and the real OT map vs iteration number; Middle: Syn1: Plot of our computed $F$ (blue) and the OT vector field computed by POT (orange); Right: Syn3: Plot of our computed $F$ and the OT vector field computed by POT (orange).}
\label{fig:com}
\vspace{-1em}
\end{figure}

\begin{figure}[t!]
\centering
\subfloat[][]{\includegraphics[width=.125\linewidth]{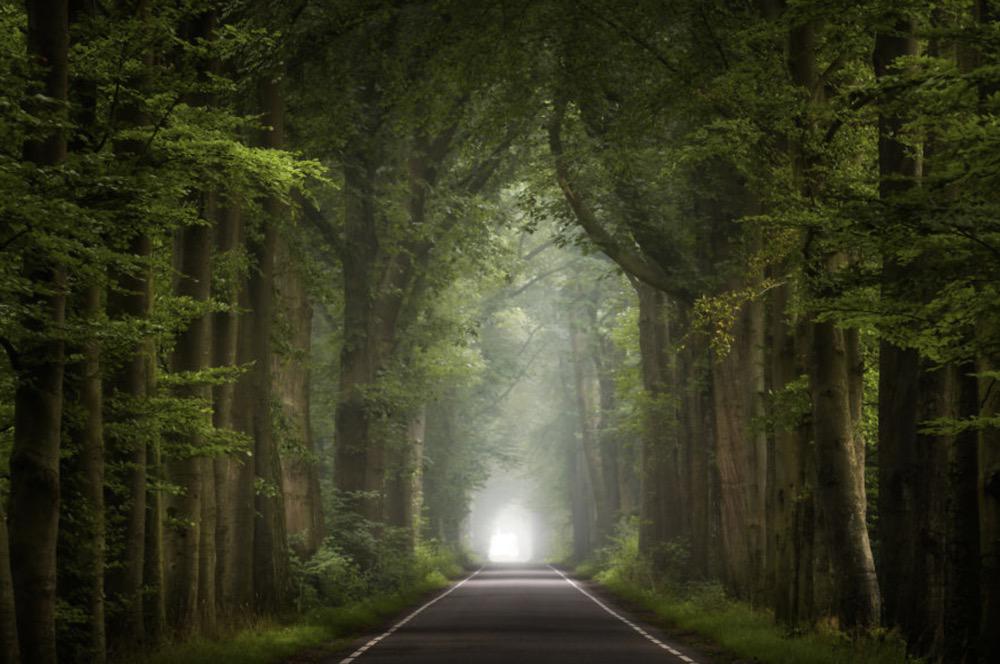}}
\subfloat[][]{\includegraphics[width=.125\linewidth]{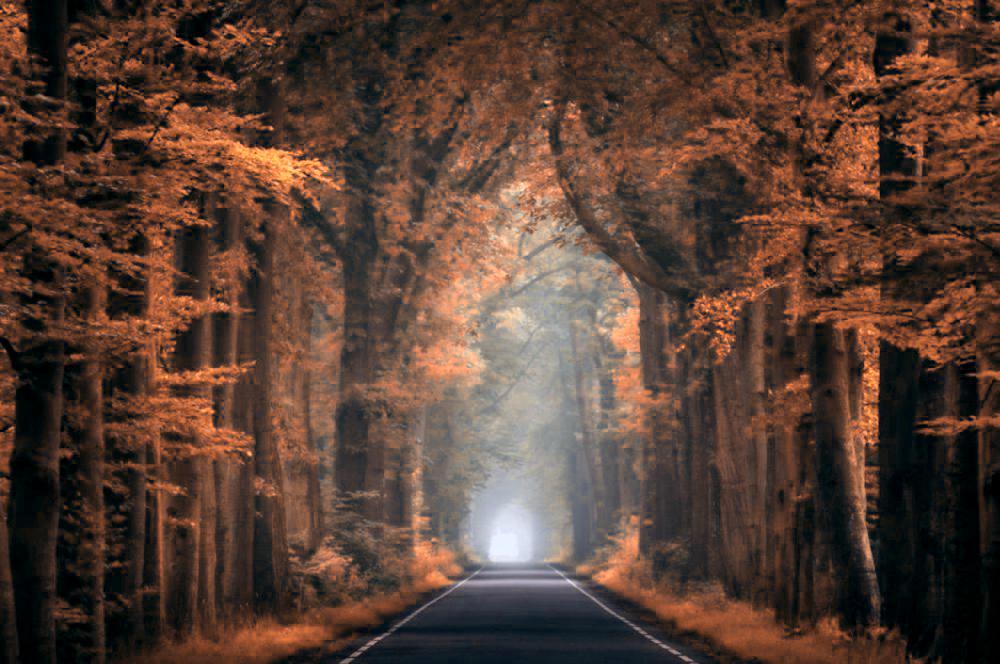}}
\subfloat[][]{\includegraphics[width=.125\linewidth]{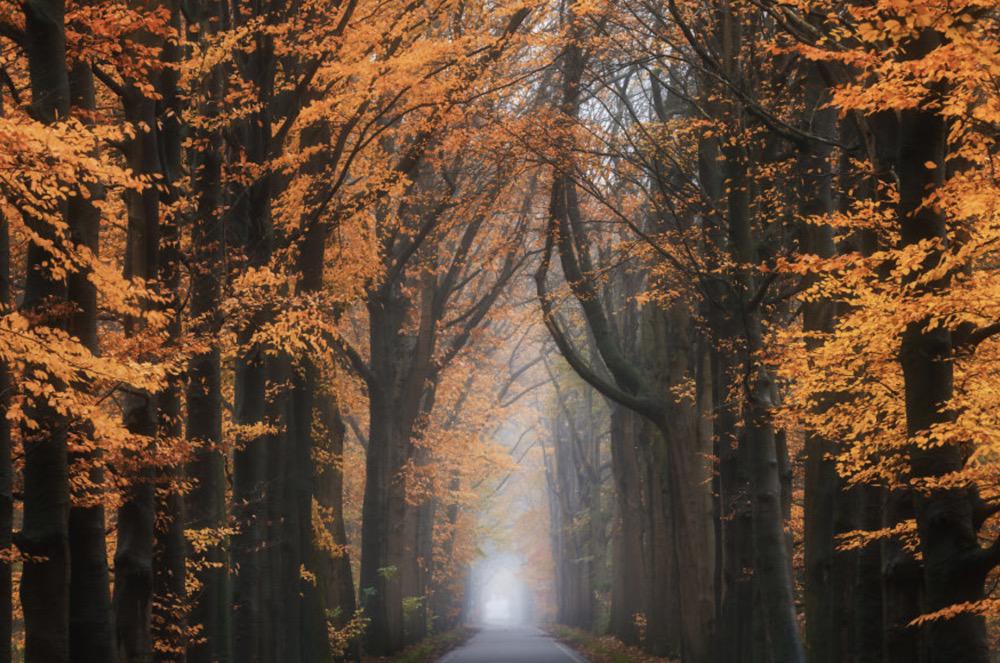}}
\subfloat[][]{\includegraphics[width=.125\linewidth]{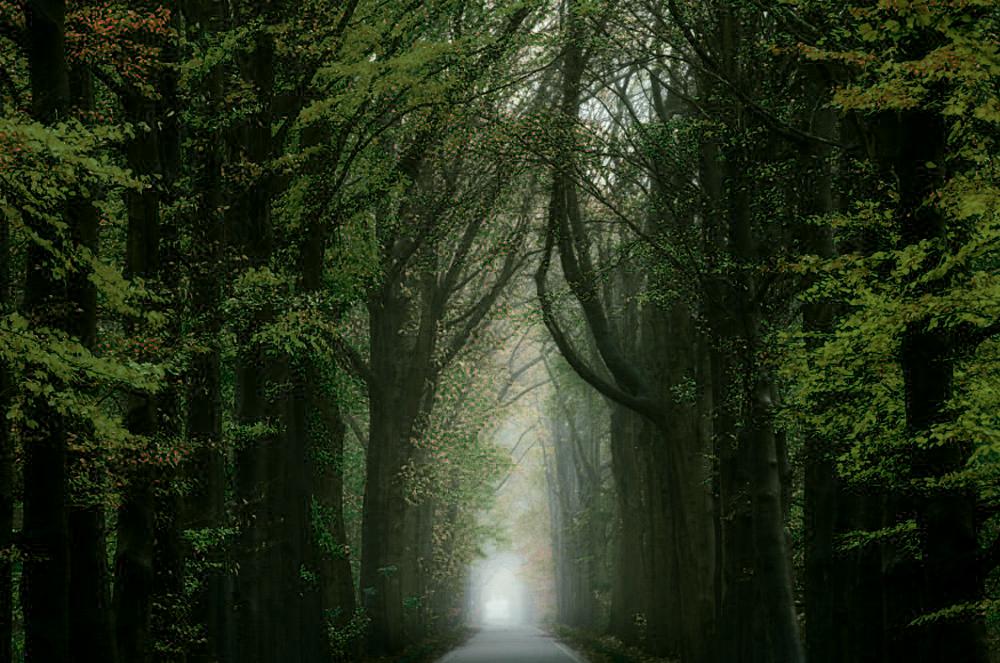}}
\subfloat[][]{\includegraphics[width=.125\linewidth]{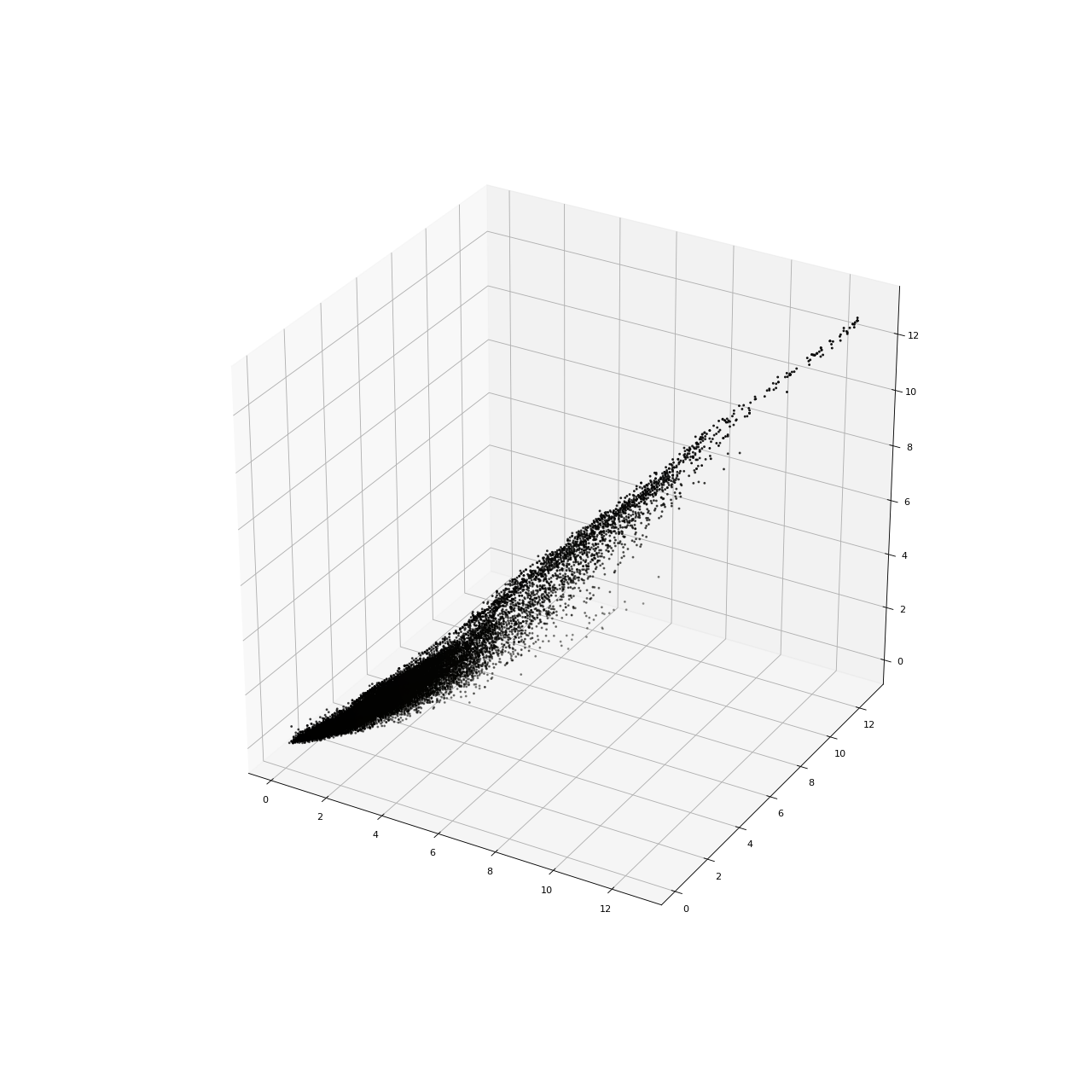}}
\subfloat[][]{\includegraphics[width=.125\linewidth]{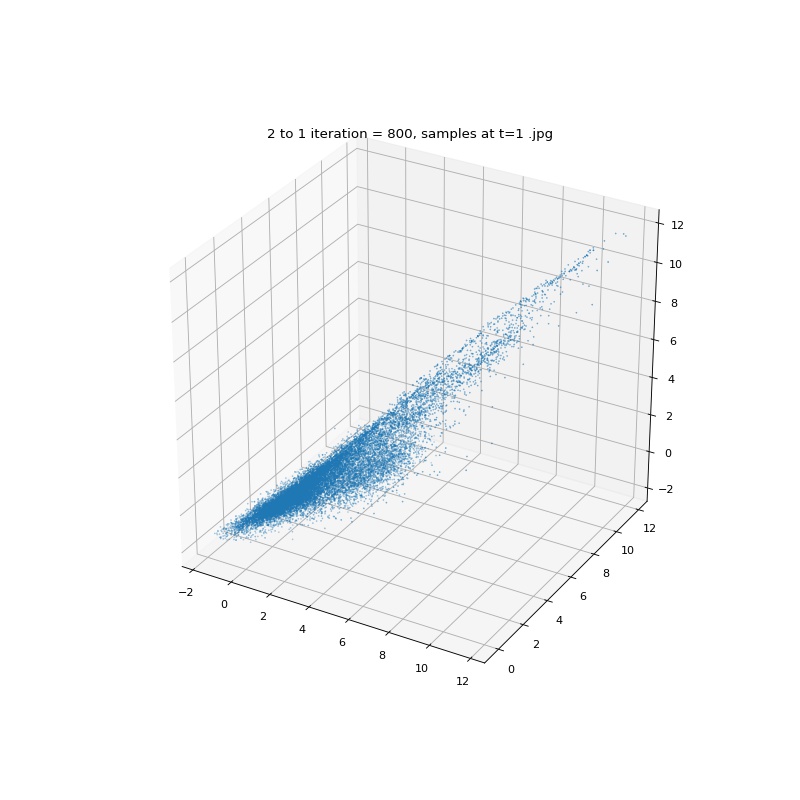}}
\subfloat[][]{\includegraphics[width=.125\linewidth]{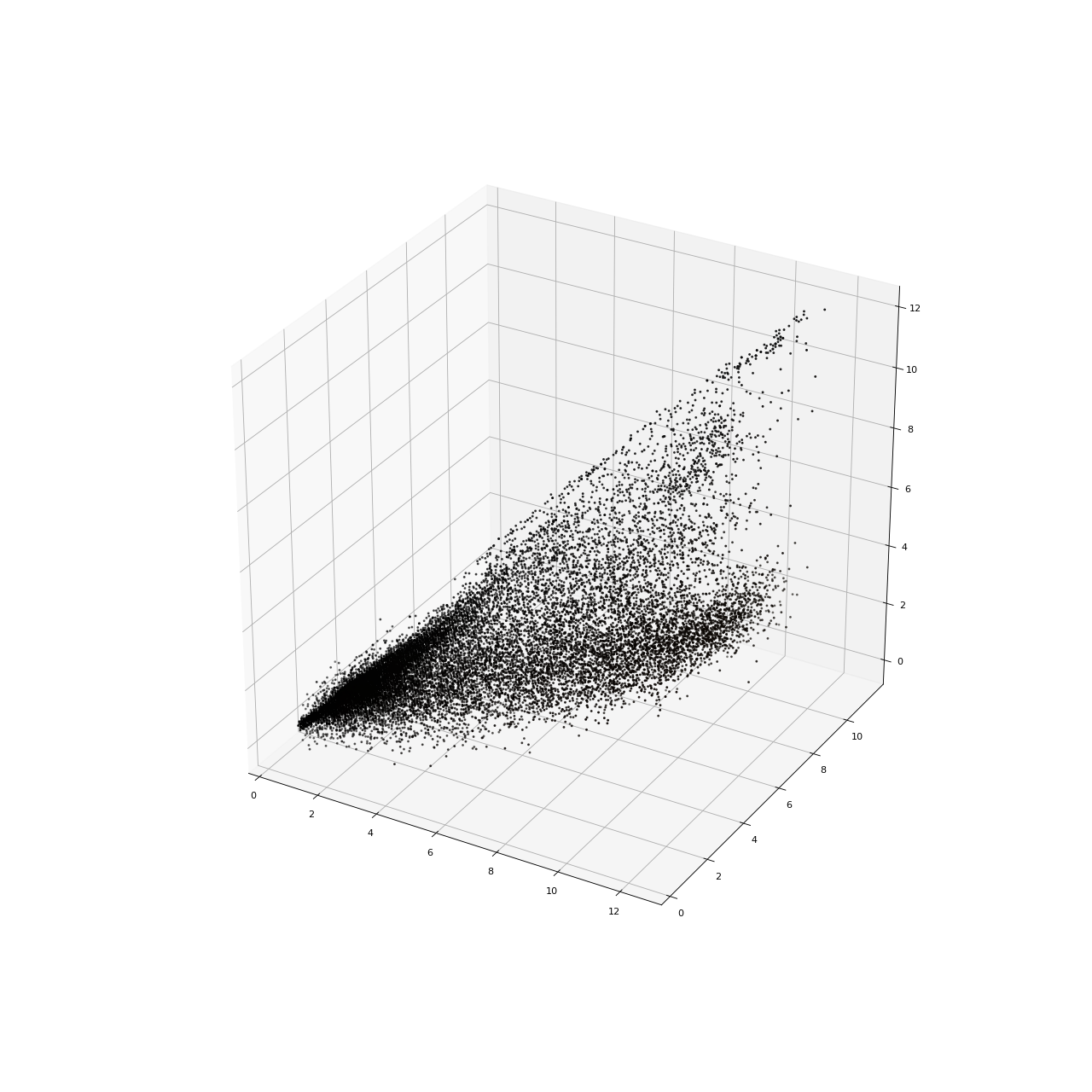}}
\subfloat[][]{\includegraphics[width=.125\linewidth]{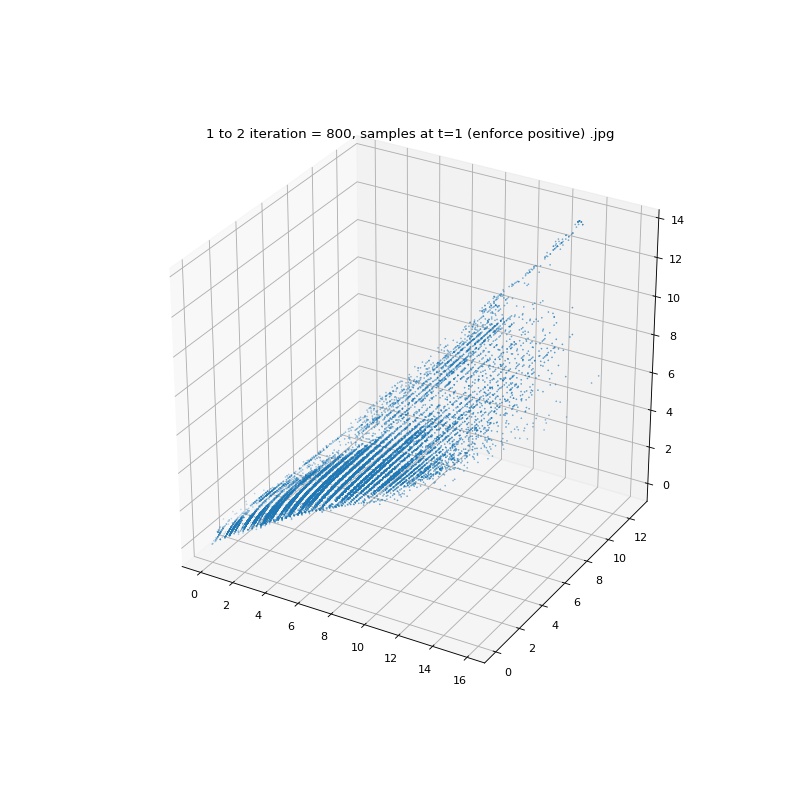}}
\caption{Real-1, Forest views. (a)(c) true summer(autumn) view, (b)(d) generated autumn(summer) view when true summer(autumn) view is given, (e)(g) true palette distribution of summer(autumn) view, (f)(h) generated palette distribution of summer(autumn) view.}
\label{fig:real-1}
\vspace{-2em}
\end{figure}

\begin{figure}[t!]
\centering
\subfloat[][]{\includegraphics[width=.125\linewidth]{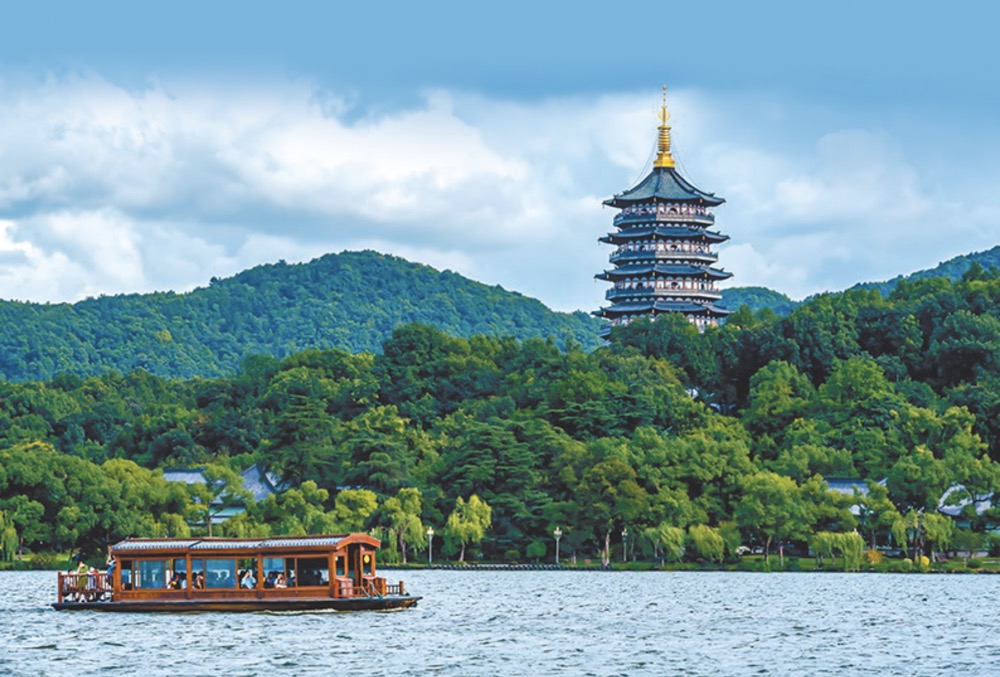}}
\subfloat[][]{\includegraphics[width=.125\linewidth]{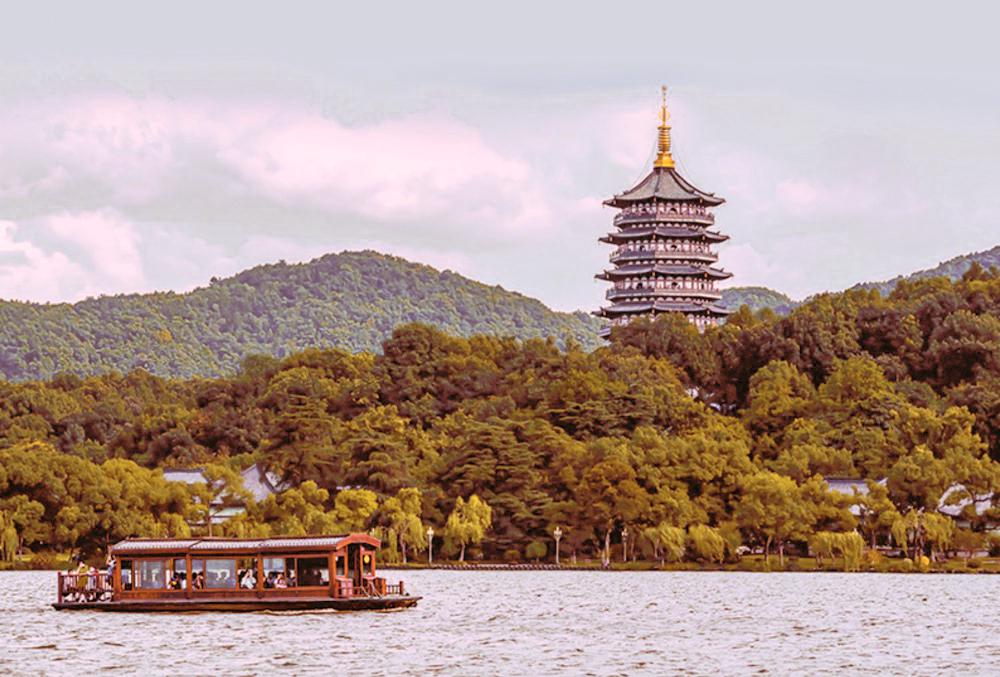}}
\subfloat[][]{\includegraphics[width=.125\linewidth]{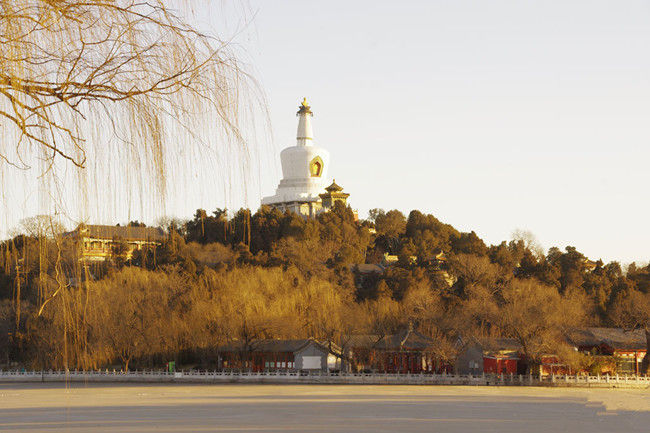}}
\subfloat[][]{\includegraphics[width=.125\linewidth]{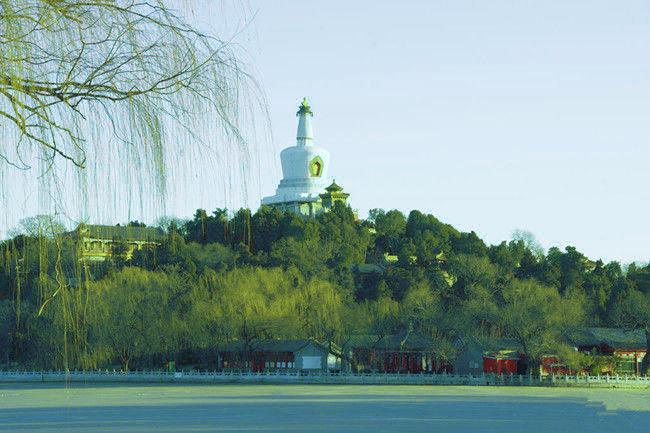}}
\subfloat[][]{\includegraphics[width=.125\linewidth]{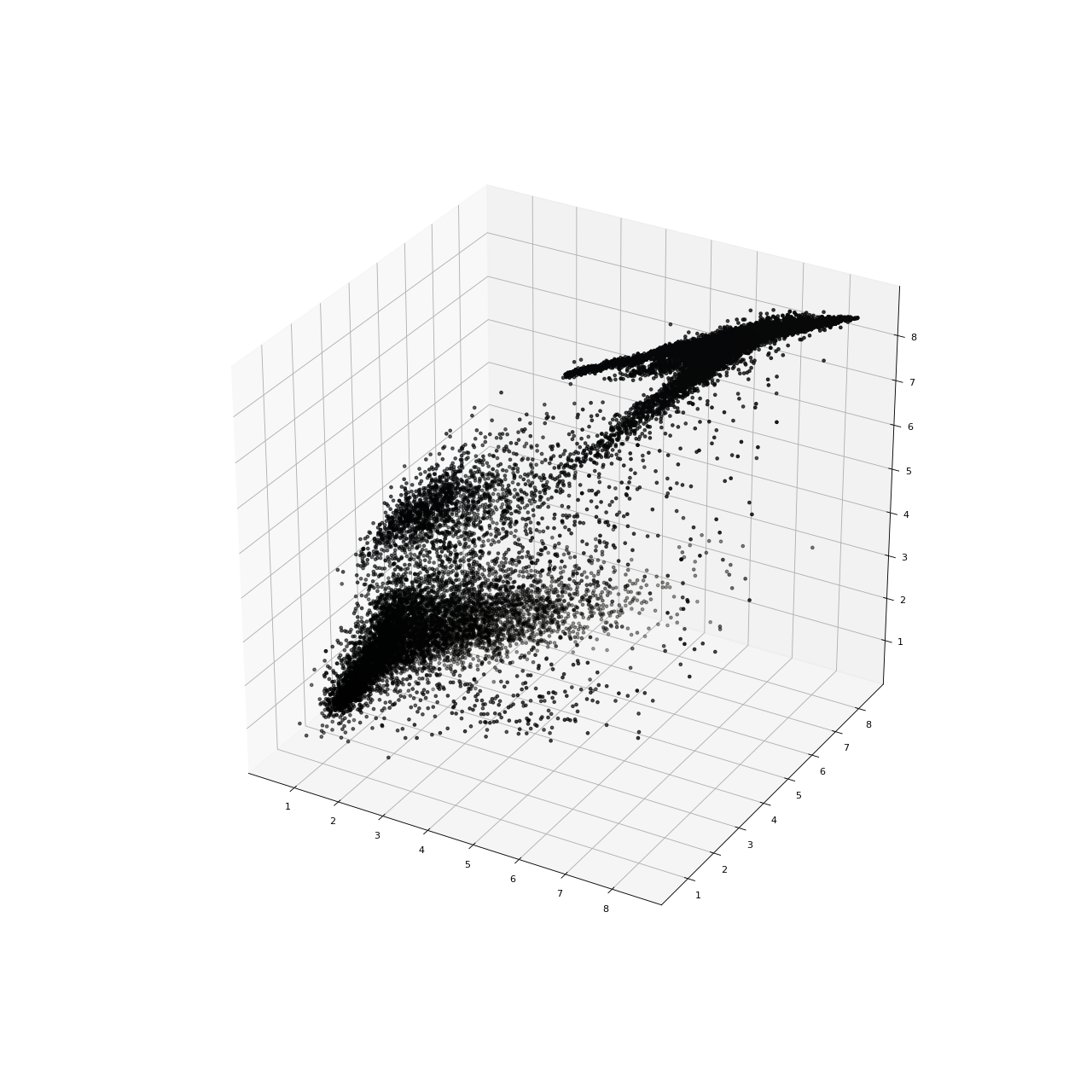}}
\subfloat[][]{\includegraphics[width=.125\linewidth]{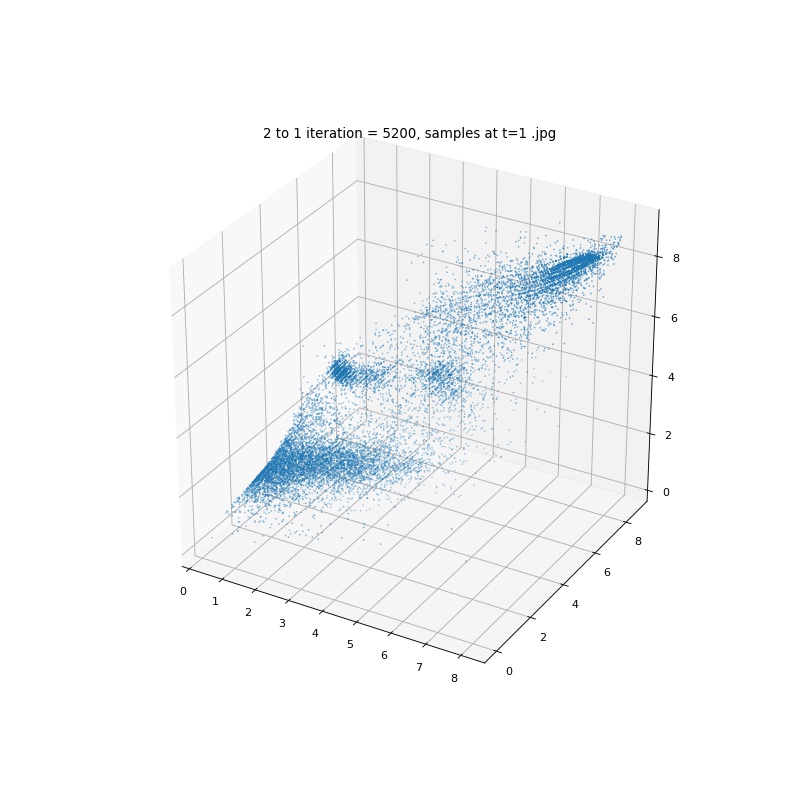}}
\subfloat[][]{\includegraphics[width=.125\linewidth]{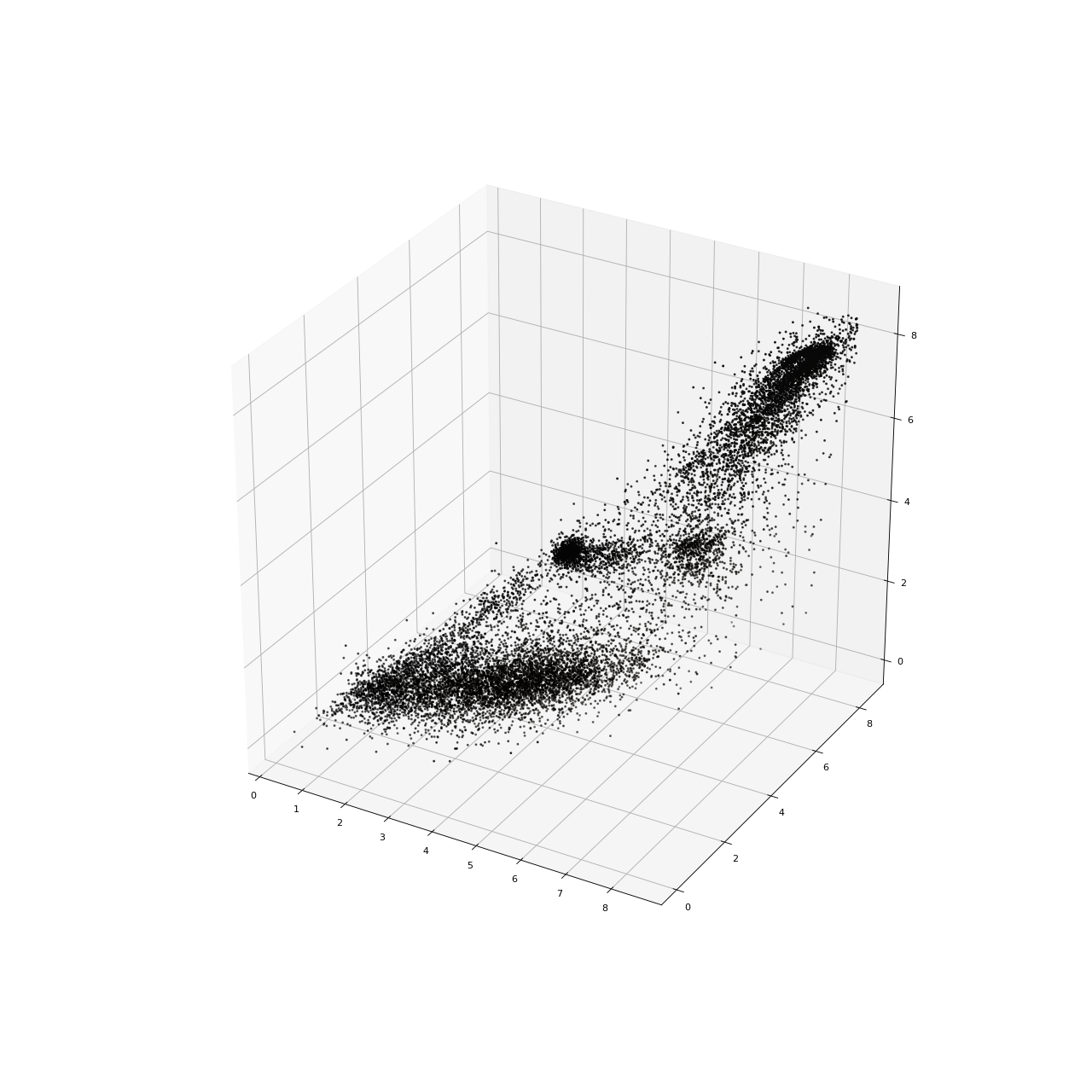}}
\subfloat[][]{\includegraphics[width=.125\linewidth]{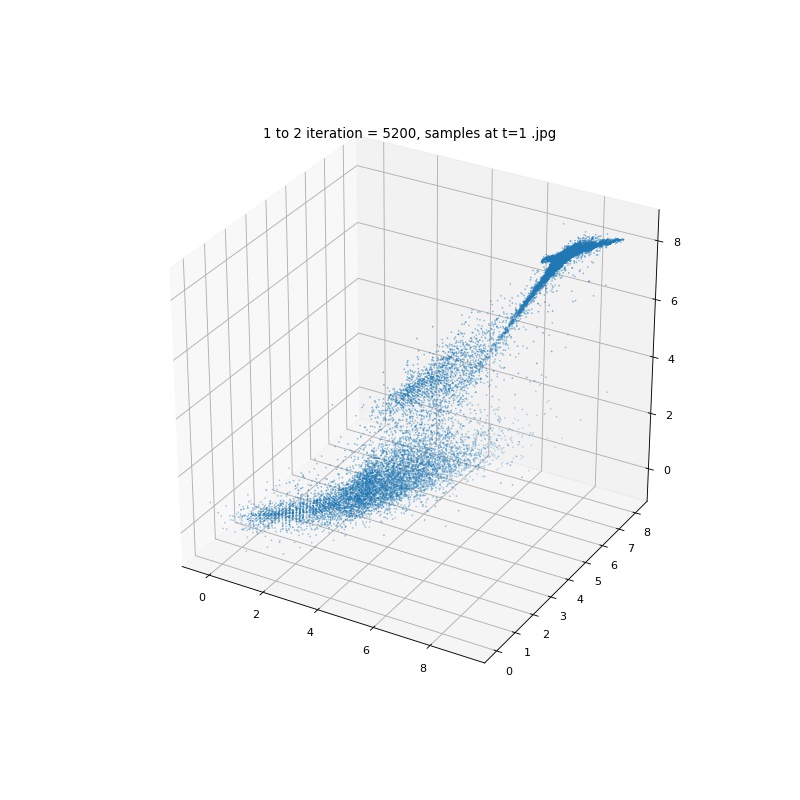}}
\caption{Real-2, Color transfer. (a)(b) true summer(generated autumn) view of the West Lake, (c)(d) true autumn(generated summer) view of the White Tower, (e)(f) palette distribution of the true summer West Lake(generated summer White Tower), (g)(h) palette distribution of the true autumn White Tower(generated autumn West Lake).}
\label{fig:real-2}
\vspace{-1em}
\end{figure}

\begin{figure}[t!]
\centering
\subfloat[][]{\includegraphics[width=.125\linewidth]{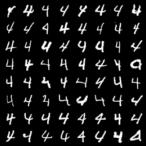}}
\subfloat[][]{\includegraphics[width=.125\linewidth]{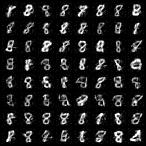}}
\subfloat[][]{\includegraphics[width=.125\linewidth]{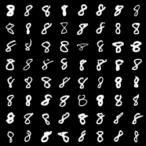}}
\subfloat[][]{\includegraphics[width=.125\linewidth]{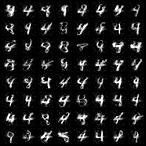}}
\subfloat[][]{\includegraphics[width=.125\linewidth]{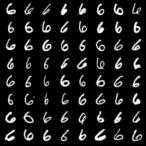}}
\subfloat[][]{\includegraphics[width=.125\linewidth]{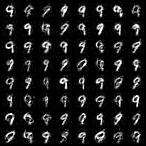}}
\subfloat[][]{\includegraphics[width=.125\linewidth]{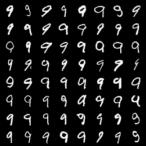}}
\subfloat[][]{\includegraphics[width=.125\linewidth]{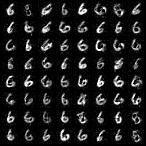}}
\caption{Real-3, Digits transformation. (a)(b) true(generated) digit 4(8), (c)(d) true(generated) digit 8(4), (e)(f) true(generated) digit 6(9), (g)(h)true(generated) digit 9(6).}
\label{fig:real-3}
\end{figure}

\textbf{Realistic-1}: Two given pictures describe the summer view$(\rho_a)$ and autumn view$(\rho_b)$ of a forest. We generate the autumn view starting with the summer view and generate the summer view starting with the autumn view. We also show the ground-truth and generated palette distributions in Figure \ref{fig:real-1}.

\vspace{-0.2em}
\textbf{Realistic-2}: In this case the view of West Lake in summer and the view of White Tower in autumn are given, then we aim to do a color transfer and simulate the summer view of White Tower and the autumn view of West Lake, the results are shown in Figure \ref{fig:real-2}, the ground-truth and generated palette distributions are also included.

\vspace{-0.2em}
\textbf{Realistic-3}: We choose MNIST as our data set ($28 \times 28$ dimensional) and study the Wasserstein mappings as well as geodesic between digit 0$(\rho_a)$ and digit 1$(\rho_b)$, digit 4($\rho_a$) and digit 8($\rho_b$), digit 6($\rho_a$) and digit 9($\rho_b$). Only partial results are presented in Figure \ref{fig:real-3} due to space limit.

\vspace{-0.2em}
\textbf{Training and Results}: For realistic-1 and realistic-2 we set the batch size $N_t=1000$, for realistic-3 in each iteration we take $N_t=500$ pictures for training, especially for realistic-3 we also add small noise to samples during the training process. From Figure \ref{fig:real-1} and \ref{fig:real-2}, the generated autumn(summer) views are very close to the true autumn(summer) views, moreover, the similarity between true and generated palette distributions also demonstrate that our algorithm works well in these cases. In realistic-3 we study the Wasserstein geodesic and mappings in original space without any dimension reduction techniques. Our generated handwritten digits follow similar patterns with the ground truth.


\vspace{-0.6em}
\section{Conclusion}
OT problem has been drawing more attention in machine learning recently. Though many algorithms have been proposed during the past several years for efficient computations, most of them do not consider the Wasserstein geodesics, neither be suitable for estimating optimal transport map with general cost in high dimensions. In this paper we present a novel method to compute Wasserstein geodesic between two given distributions with general convex cost. In particular, we consider the primal-dual scheme of the dynamical OT problem and simplify the scheme via the KKT conditions as well as the geodesic transporting properties. By further introducing preconditioning techniques and bidirectional dynamics into our optimization, we obtain a stable and effective algorithm that is capable of high dimensional computation, which is one of the very first scalable algorithms for directly computing geodesics with general cost, including $L$-$p$. Our method not only computes sample based Wasserstein geodesics, but also provides Wasserstein distance and optimal map. We demonstrate the effectiveness of our scheme through a series of experiments with both low dimensional and high dimensional settings. It is worth noting that our model can be applied not only in machine learning such as image transfer, density estimation, but also in optimal control and robotics, where one needs to study the distribution of mobile agents. We should also be aware of the malicious usage for our method, for instance, it could be potentially used in some activities involving generating misleading data distributions.




\bibliography{a_paper}
\bibliographystyle{nips2021}





\appendix
\newpage
\section{Algorithm}
\subsection{Preconditioning technique for 2-Wasserstein case}
It's worth mentioning that we can apply preconditioning technique under the 2-Wasserstein cases, i.e., $L(\cdot) = \frac{|\cdot|^2}{2}$. When the support of distributions $\rho_a$ and $\rho_b$ are far away from each other, the computational process might get much more sensitive with respect to vector field ${F}$. In order to deal with this situation, we consider preconditioning to our initial distribution $\rho_a$. In our implementation, we treat $P:\mathbb{R}^d\rightarrow\mathbb{R}^d$ as our preconditioning map. We fix its structure as $P(x) = \sigma x + \mu$ with $\sigma\in\mathbb{R}_+, ~\mu\in\mathbb{R}^d$. Such preconditioning process can be treated as an operation aiming at relocating and rescaling the initial distribution $\rho_a$ so that the support of $P_{\sharp}\rho_a$ matches with the support of $\rho_b$ in a better way, which in turn facilitates the training process of our OT problem. A similar technique is also carried out by \citet{kuang2017preconditioning}.

Let us denote the optimal vector field of OT problem between $P_{\sharp}\rho_a$ and $\rho_b$ as $\nabla\widehat{\Phi}_0$, then for the vector field ${F}^*(x) = \nabla\widehat{\Phi}_0 \circ P(x) + P(x) - x $, the following theorem guarantees the optimality of ${F}^*$.
\begin{manualtheorem}{3}
\label{theorem on preconditioning}
 Suppose $L(\cdot)=\frac{|\cdot|^2}{2}$. We define the map $P(x)=\sigma x+\mu$ with $\sigma\in\mathbb{R}_+, \mu\in\mathbb{R}^d$. Recall \eqref{vector field = inv grad L grad Phi}, we denote ${v}(x,t)=\nabla\widehat{\Phi}(x, t)$ as the optimal solution to dynamical OT problem \eqref{dynamical OT} from $P_{\sharp}\rho_a$ to $\rho_b$. We set $\widehat{\Phi}_0 = \widehat{\Phi}(\cdot, 0)$. Furthermore, we denote ${v}(x,t)=\nabla\Phi^*(x,t)$ as the optimal solution to dynamical OT problem \eqref{dynamical OT} from $\rho_a$ to $\rho_b$,
 and set $\Phi^*_0=\Phi^*(\cdot, 0)$. Then we have
  \begin{equation*}
     \nabla\Phi^*_0(x) = \nabla\widehat{\Phi}_0\circ P(x)+P(x)-x,
~~\Phi^*_0(x) = \frac{1}{\sigma}\widehat{\Phi}_0(\sigma x + \mu)+\frac{\sigma - 1}{2}|x|^2 + \mu^Tx+\textrm{Const}.
\end{equation*}
\end{manualtheorem}
This theorem indicates that our constructed ${F}^*$ is exactly the optimal transport field $\nabla\Phi^*_0$ for the original OT problem from $\rho_a$ to $\rho_b$. The proof is in Appendix C.

Our computation procedure is summarized in Algorithm \ref{alg:1}. We set ${F}_{\theta_1}, {G}_{\theta_2}$ and $\Phi^F_{\omega_1},\Phi^G_{\omega_2}$ as fully connected neural networks and optimize over their parameters.

\begin{algorithm}[hb!]
\caption{Computing Wasserstein geodesic from $\rho_a$ to $\rho_b$ via bidirectional scheme \eqref{bidirection} and preconditioning}
\begin{algorithmic}[1]
\State Choose our preconditioning map $P(x)=\sigma x + \mu$. Denote $\hat{\rho_a} = P_{\sharp}\rho_a$ (This step is only applicable for 2-Wasserstein case. If we do not need preconditioning, we treat $P=\textrm{Id}$.)
\State Set up the threshold $\epsilon>0$ as the stopping criteria
\State Initialize ${F}_{\theta_1}, {G}_{\theta_2}$, $\Phi^F_{\omega_1}, \Phi^G_{\omega_2}$
\For{${F}_{\theta_1}$, ${G}_{\theta_2}$ steps}
\State Sample $\{(z^{a}_k,t^{a}_k)\}_{k=1}^N$ from $\hat{\rho}_a\otimes U(0,1)$ and $\{(z^b_k, t^b_k)\}_{k=1}^N$ from $\rho_b\otimes U(0,1)$; 
\State Set $x^a_k = z^a_k+t^a_k {F}_{\theta_1}(z^a_k)$,  $x^b_k = z^b_k + t^b_k {G}_{\theta_2}(z^b_k)$;
\State Sample $\{w^a_k\}_{k=1}^M$ from $\hat{\rho}_a$ and $\{w^b_k\}$ from $\rho_b$;
\For{$\Phi^F_{\omega_1}, \Phi^G_{\omega_2}$ steps}
\State Update (via gradient ascent) $\Phi^F_{\omega_1}, \Phi^G_{\omega_2}$ by:
\begin{equation*}
    \nabla_{\omega_1, \omega_2} (\mathcal{L}^{ab}(\Phi^F_{\omega_1}) +\mathcal{L}^{ba}(\Phi^G_{\omega_2}))
\end{equation*}
\EndFor
\State Sample $\{\xi^a_k\}_{k=1}^K$ from $\hat{\rho}_a$ and $\{\xi^b_k\}_{k=1}^K$ from $\rho_b$
\State Update (grad descent) ${F}_{\theta_1}$, ${G}_{\theta_2}$ by:
\begin{equation*}
 \nabla_{\theta_1, \theta_2} ( \mathcal{L}^{ab}(\Phi_{\omega_1}^F) + \mathcal{L}^{ba}(\Phi_{\omega_2}^G) +\mathcal{K}({F}_{\theta_1}, {G}_{\theta_2})  )
\end{equation*}
\State Whenever $|\widehat{W}^{ab}-\widehat{W}^{ba}|<\epsilon$, skip out of the loop.
\EndFor
\State Set ${F}^* = {F}_{\theta_1}\circ P + P - \textrm{Id}$ and ${G}^* = {G}_{\theta_2}\circ P + P - \textrm{Id}$.
\State Wasserstein geodesic from $\rho_a$ to $\rho_b$ is given by $\{(\textrm{Id}+t{F}_{\theta_1})_{\sharp}\rho_a\}$; Wasserstein geodesic from $\rho_b$ to $\rho_a$ is given by $\{(\textrm{Id}+t{G}_{\theta_2})_{\sharp}\rho_b\}$. 
\end{algorithmic}
\label{alg:1}
\end{algorithm}

\begin{remark}
 In Algorithm \ref{alg:1}, we need to sample points $\{z^a_k\}$ from the distribution $\hat{\rho}_a=P_\sharp \rho_a$. To achieve this, we first sample $\{u_k\}$ from $\rho_a$. Then $\{P(u_k)\}$ are our desired samples from $\hat{\rho}_a$.
\end{remark}
In Algorithm \ref{alg:1} we define
\begin{align*}
  &\mathcal{L}^{ab}(\Phi^F_{\omega_1}) = -\frac{1}{N}\sum_{k=1}^N \left[\frac{\partial}{\partial t}\Phi^F_{\omega_1}(x^{a}_k,t^{a}_k) + H(\nabla \Phi^F_{\omega_1} (x^a_k,t^a_k)) \right] +\frac{1}{M}\sum_{k=1}^M(\Phi^F_{\omega_1}(w^b_k,1)-\Phi^F_{\omega_1}(w^a_k , 0)),\\
  & \mathcal{L}^{ba}(\Phi^G_{\omega_2}) =   -\frac{1}{N}\sum_{k=1}^N \left[\frac{\partial}{\partial t}\Phi^G_{\omega_2}(x^{b}_k,t^{b}_k) + H(\nabla \Phi^G_{\omega_2} (x^b_k,t^b_k))\right]+\frac{1}{M}\sum_{k=1}^M(\Phi^G_{\omega_2}(w^b_k,1)-\Phi^G_{\omega_2}(w^a_k , 0)),  \\
  & \mathcal{K}({F}_{\theta_1}, {G}_{\theta_2}) =   \frac{\lambda}{ K}\sum_{k=1}^K|{G}_{\theta_2}(\xi^a_k+{F}_{\theta_1}(\xi^a_k))+{F}_{\theta_1}(\xi^a_k)|^2 + \frac{\lambda}{K} \sum_{k=1}^K|{F}_{\theta_1}(\xi^b_k+{G}_{\theta_2}(\xi^b_k))+{G}_{\theta_2}(\xi^b_k)|^2,\\
  & \widehat{W}^{ab} = \frac{1}{M}\sum_{k=1}^{M}L({F}_{\theta_1}(w^a_k)), \quad \widehat{W}^{ba} = \frac{1}{M}\sum_{k=1}^{M}L({G}_{\theta_2}(w^b_k)). \quad
\end{align*}

\section{Proof of Theorem 2}\label{study}
Let us denote $ {F}:\mathbb{R}^d\rightarrow \mathbb{R}^d$ as the transporting vector field. Recall that we are computing for the Wasserstein geodesic interpolating $\rho_a$ and $\rho_b$. We denote $\hat{\rho}_t = (I+t {F})_{\sharp}\rho_a$. Then we introduce the following functional of $F$ and $\Phi$:
\begin{align}
  \mathcal{L}( {F}, \Phi) = & \int_0^1 \int \left(-\frac{\partial \Phi(x,t)}{\partial t}-H(\nabla\Phi(x,t))\right)\hat{\rho}(x,t)~dxdt + \int \Phi(x,1)\rho_b(x) - \Phi(x, 0)\rho_a(x)~dx  \label{def of L} \nonumber \\
  = & \int_0^1 \int \left(-\frac{\partial \Phi(x+tF(x),t)}{\partial t}-H(\nabla\Phi(x+tF(x),t))\right)\rho_a(x,t)~dxdt + \int \Phi(x,1)\rho_b(x) \nonumber\\
  &- \Phi(x, 0)\rho_a(x)~dx 
\end{align}
As mentioned in \eqref{chosen scheme}, our numerical method is to solve the following saddle point problem
\begin{equation}
   \min_{ {F}}~\max_{\Phi} ~ \mathcal{L}( {F}, \Phi)  \label{saddle scheme}
\end{equation}
As stated in Theorem \ref{main}, the optimal solution obtained from dynamical OT problem (Brenier-Benamou formulation) \eqref{dynamical OT} is a critical point to the functional $\mathcal{L}( {F}, \Phi)$. Before we prove this result, we need the following lemmas:

\begin{lemma}\label{con eq}
  Given a distribution with density $\rho$ defined on $\mathbb{R}^d$, consider vector field $ {F}:\mathbb{R}^d\rightarrow\mathbb{R}^d$. Define time-varying density $\{\rho(\cdot,t)\}_{t\in [0,1]}$ as $\rho(\cdot, t) = (Id+t {F})_{\sharp}\rho_0$. Suppose for a given $f\in C^1(\mathbb{R}^d)$, $f(x)\rho(x,t)$ is integrable on $\mathbb{R}^d$. Then
  \begin{equation*}
     \int f(x)\frac{\partial}{\partial t}\rho(x,t) = \int \nabla f(x+t {F}(x))\cdot {F}(x)~\rho_a(x)~dx
  \end{equation*}
\end{lemma}
\begin{proof}
  We have
  \begin{align*}
     \int f(x)\frac{\partial}{\partial t}\rho(x,t) =\frac{d}{dt}\left(\int f(x)\rho(x,t)~dx\right)= &\frac{d}{dt}\left(\int f(x+t {F}(x))\rho_a(x)~dx\right)\\=& \int \nabla f(x+ {F}(x))\cdot {F}(x)~\rho_a(x)~dx
  \end{align*}
\end{proof}

\begin{lemma}\label{vf of HJ}
  Suppose $\Phi^*(x,t)$ is solved from \eqref{geodesic eq} in the paper with initial condition $\Phi^*(\cdot, 0) = \Phi^*_0(\cdot)$, we further assume $\Phi^*(\cdot, t)\in C^2(\mathbb{R}^d)$. Then we have
  \begin{equation}
    \nabla\Phi^*(x+t~\nabla L^{-1}(\nabla\Phi^*_0(x)), ~t) = \nabla\Phi^*_0(x).  \label{gradient of solution to hj}
  \end{equation}
\end{lemma}

\begin{proof}
   Now consider Hamilton-Jacobi equation of \eqref{geodesic eq} in the paper:
   \begin{equation*}
      \frac{\partial \Phi^*(y,t)}{\partial t} + H(\nabla\Phi^*(y,t)) = 0 \quad \Phi^*(\cdot, 0) = \Phi^*_0.
   \end{equation*}
   We take gradient with respect to $x$ on both sides, we have
   \begin{equation}
      \frac{\partial}{\partial t}(\nabla\Phi^*(y,t)) + \nabla^2\Phi^*(y,t)\nabla H( \nabla\Phi^*(y,t) ) = 0.  \label{grad of hj equation}
   \end{equation}
   Let us denote $T_t(x) = x+t\nabla H(\nabla\Phi^*_0(x))$ for simplicity. We now compute
   \begin{equation*}
      \frac{d}{dt}\nabla\Phi^*(T_t (x), t) = \frac{\partial}{\partial t}\nabla\Phi^*(T_t(x), t) + \nabla^2 \Phi^*(T_t(x), t) \nabla H(\nabla\Phi^*_0(x))
   \end{equation*}
   By plugging $y=T_t(x)$ into \eqref{grad of hj equation}, we are able to verify $\frac{d}{dt}\nabla\Phi^*(T_t (x), t) = 0$. Thus 
   \begin {equation}
      \nabla\Phi^*(T_t (x), t) = \nabla \Phi^*(T_0(x), 0) = \nabla\Phi^*_0(x) \quad \textrm{for} ~~ t\in [0,1]  \label{lemma2 eq}
   \end{equation}
   Recall $H$ defined in the paper, we can verify that $\nabla H = \nabla L^{-1}$. Thus \eqref{lemma2 eq} leads to
   \begin{equation*}
     \nabla\Phi^*(x+t~\nabla L^{-1}(\nabla\Phi^*_0(x)), ~t) = \nabla\Phi^*_0(x).
   \end{equation*}

\end{proof}


\begin{lemma}\label{continuity eq lemma}
   Suppose $\Phi^*(x,t)$ is solved from \eqref{geodesic eq} in the paper with initial condition $\Phi^*(\cdot, 0) = \Phi^*_0(\cdot)$, we further assume $\Phi^*(\cdot, t)\in C^2(\mathbb{R}^d)$. Now denote $\hat{\rho}(\cdot, t) = (\textrm{Id}+t\nabla L^{-1}(\nabla \Phi^*_0))_{\sharp}\rho_a$. Then $\hat{\rho}(\cdot, t)$ solves 
   \begin{equation*}
      \frac{\partial \hat{\rho}(x,t)}{\partial t} + \nabla\cdot(\hat{\rho}(x,t)\nabla L^{-1}(\nabla\Phi^*(x,t)))=0.
   \end{equation*}
\end{lemma}
\begin{proof}
  For arbitrary $f\in C_0^\infty(\mathbb{R}^d)$, we consider:
  \begin{align*}
    & \int f(x)\left(\frac{\partial \hat{\rho}(x,t)}{\partial t} + \nabla\cdot(\hat{\rho}(x,t)\nabla L^{-1}(\nabla\Phi^*(x,t)))\right)~dx \\
    =& \int f(x)\frac{\partial\hat{\rho}(x,t)}{\partial t} ~dx -\int \nabla f(x)\cdot \nabla L^{-1}(\nabla\Phi^*(x,t))\hat{\rho}(x,t)~dx
  \end{align*}
  By Lemma \ref{con eq}, the first term equals
  \begin{equation}
   \int \nabla f(x+t\nabla L^{-1}(\nabla\Phi^*_0(x)))\cdot \nabla L^{-1}(\nabla\Phi^*_0(x))~\rho_a(x)~dx  \label{first}
  \end{equation}
  The second term equals
  \begin{equation}
    \int \nabla f(x+t\nabla L^{-1}(\nabla\Phi^*_0(x)))\cdot \nabla L^{-1}(\nabla\Phi^*(x+t\nabla L^{-1}(\nabla\Phi^*_0(x)), t))~\rho_a(x)~ dx\label{second}
  \end{equation}
  Using Lemma \ref{vf of HJ}, we know the integrals \eqref{first} and \eqref{second} are the same. Thus we have
  \begin{equation*}
     \int f(x)\left(\frac{\partial \hat{\rho}(x,t)}{\partial t} + \nabla\cdot(\hat{\rho}(x,t)\nabla L^{-1}(\nabla\Phi^*(x,t)))\right)~dx = 0 \quad \forall ~ ~f\in C_0^\infty(\mathbb{R}^d).
  \end{equation*}
  This leads to our result.
\end{proof}

\begin{lemma}\label{simplify wasserstein}
 Suppose $\Phi^*(x,t)$ is solved from \eqref{geodesic eq} in the paper with initial condition $\Phi^*(\cdot, 0) = \Phi_0^*(\cdot)$, then 
 \begin{equation}
    W_{\textrm{Dym}}(\rho_a, \rho_b) = \int L(\nabla L^{-1}(\nabla\Phi_0^*(x)))\rho_a(x)~dx.  \label{compute Wass dist simplified}
 \end{equation}
\end{lemma}
\begin{proof}
   Consider particle dynamical OT with its optimal solution $ {v}^*(x,t) = \nabla L^{-1}(\nabla\Phi^*(x, t))$ as stated in \eqref{vector field = inv grad L grad Phi} in the paper. Recall Theorem 1 stating that the optimal plan is transporting each particle $\boldsymbol{X}_t$ along straight lines with constant velocity $ {v}^*(\boldsymbol{X}_0, 0) = \nabla L^{-1}(\nabla\Phi_0^*(\boldsymbol{X}_0))$, i.e. $ {v}^*(\boldsymbol{X}_t, t) =  {v}^*(\boldsymbol{X}_0, 0)$ for any $t\in [0,1]$. Combining these, we have
   \begin{equation*}
      W_{\textrm{Dym OT}}(\rho_a, \rho_b) = \int_0^1\mathbb{E}~ L( {v}^*(\boldsymbol{X}_t,t))~dt = \mathbb{E}\left(\int_0^1 L( {v}^*(\boldsymbol{X}_t,t))~dt\right)=\mathbb{E}~L(\nabla L^{-1}(\nabla\Phi_0^*(\boldsymbol{X}_0))).
   \end{equation*}
   Notice that we require $\boldsymbol{X}_0\sim\rho_a$. This will lead to \eqref{compute Wass dist simplified}.
\end{proof}


\begin{customthm}{2}
Recall the solution to the equation system \eqref{geodesic eq} in the paper is $\rho(x,t)$ and $\Phi^*(x,t)$, suppose $\Phi^*(\cdot, t)\in C^2(\mathbb{R}^d)$. Denote $\Phi_0^*(\cdot) = \Phi^*(\cdot,0)$, then $(\nabla L^{-1}(\nabla \Phi_0^*), \Phi^*)$ is a critical point to the functional $\mathcal{L}$, i.e.,
  \[ \frac{\partial \mathcal{L}}{\partial {F}}(\nabla L^{-1}(\nabla \Phi_0^*),\Phi^*) = 0, \quad \frac{\partial \mathcal{L}}{\partial \psi}(\nabla L^{-1}(\nabla \Phi_0^*),\Phi^*) = 0.\]
Furthermore,  $\mathcal{L}(\nabla L^{-1}(\nabla\Phi_0^*), \Phi^*)=W_{\textrm{Dym OT}}(\rho_a,\rho_b)$.
\end{customthm}

\begin{proof}
   Since we have assumed $\Phi^*(\cdot, t)\in C^2(\mathbb{R}^d)$, we restrict our $\Phi\in C^2(\mathbb{R}^d)$ as well.\\
   We first rewrite $\mathcal{L}( {F}, \Phi)$ by using integration by parts as:
   \begin{equation}
     \int_0^1 \int \Phi(x,t)\frac{\partial \hat{\rho}(x,t)}{\partial t} - H(\nabla\Phi(x,t)) \hat{\rho}(x,t)~dxdt + \int \Phi(x,1)(\rho_a(x) - \hat{\rho}(x,1))~dx . \label{L1}
   \end{equation}
  By Lemma \ref{con eq}, \eqref{L1} can be written as
  \begin{align}
     \mathcal{L}( {F}, \Phi) 
     = & \int_0^1\int_{\mathbb{R}^d} [\nabla\Phi(x+t {F}(x),t)\cdot  { F }(x) - H(\nabla \Phi(x+t {F}(x),t))]\rho_a(x)~dxdt \label{variation B of L} \\
     & + \int \Phi(x,1)\rho_b(x)~dx - \int \Phi(x+ {F}(x), 1)\rho_a(x)~dx.\nonumber
  \end{align}
  Now based on \eqref{variation B of L} here, we are able to compute $\frac{\partial \mathcal{L}( {F}, \Phi)}{\partial  {F}}(x)$ as
  \begin{align}
     \frac{\partial \mathcal{L}( {F}, \Phi)}{\partial {F}} 
     =& \int_0^1  t\nabla^2\Phi(x+t {F}(x), t)\cdot \underbrace{ [ {F}(x) - \nabla H( \nabla\Phi(x+t {F}(x), t)) ] }_{(A)} \rho_a(x)~dt  \label{dL/dF} \\
     & + \underbrace{ \left( \int_0^1\nabla\Phi(x+t {F}(x), t)~dt -  \nabla\Phi(x+ {F}(x), 1) \right) }_{(B)} \rho_a(x). \nonumber
  \end{align}
  Now we plug $ {F} = \nabla L^{-1} (\nabla\Phi^*_0)$, $\Phi = \Phi^*$ into \eqref{dL/dF}, by Lemma \ref{vf of HJ}, we have 
  \begin{equation}
     \nabla\Phi^*(x+t~\nabla L^{-1}(\nabla\Phi^*_0(x)), ~t) = \nabla\Phi^*_0(x). \label{important relation}
  \end{equation}
  Then using \eqref{important relation} and recall that $\nabla H = \nabla L^{-1}$, one can verify that $(A)=0$, similarly, for $(B)$, we have $\nabla\Phi(x+t {F}(x) ,t) = \nabla\Phi^*_0$ for all $t\in [0,1]$. Thus $(B)=0$ and we are able to verify $\frac{\partial \mathcal{L}}{\partial{F}}(\nabla L^{-1}(\nabla \Phi^*_0), \Phi^*) = 0$. 
  
  On the other hand, we can compute $\frac{\partial\mathcal{L}( {F}, \Phi)}{\partial \Phi}(x,t)$ as 
  \begin{align*}
     \frac{\delta\mathcal{L}( {F}, \Phi)}{\delta \Phi}(x,t) & = \underbrace{\left[\frac{\partial \hat{\rho}(x,t)}{\partial t} +\nabla\cdot(\hat{\rho}(x,t)\nabla H(\nabla\Phi(x,t)))\right]}_{(C)} + \underbrace{[\rho_b(x)-\hat{\rho}(x,1)]}_{(D)} \delta_1(t) 
  \end{align*}
  Now by Lemma \ref{continuity eq lemma}, we know $(C)=0$.
   Furthermore, since $\Phi^*$ solves Dynamical OT problem associated to the optimal transport problem between $\rho_a$ and $\rho_b$, by \eqref{Phi0, Phi1 as monge maps}, we have $\hat{\rho}(x,1) = (Id+\nabla\Phi^*_0)_{\sharp}\rho_a = \rho_b$, this verifies $(D) = 0$. Thus, we are able to verify $\frac{\partial \mathcal{L}( {F}, \Phi)}{\partial \Phi}( \nabla L^{-1}(\nabla \Phi^*_0), \Phi^* ) = 0$.
   
   At last, we plug $ {F}=\nabla L^{-1}(\nabla\Phi^*_0),\Phi = \Phi^*$ in \eqref{variation B of L} to obtain:
   \begin{align*}
      \mathcal{L}(\nabla L^{-1}(\nabla\Phi^*_0), \Phi^*) & = \int_0^1\int \nabla\Phi^*_0(x)\cdot\nabla L^{-1}(\nabla \Phi^*_0(x)) - H(\nabla\Phi^*_0(x)) ~\rho_a(x) ~dxdt \\
     & = \int L(\nabla L^{-1}(\nabla\Phi^*_0(x)))\rho_a(x)~dx.
   \end{align*}
   Now by Lemma \ref{simplify wasserstein}, we have verified $\mathcal{L}(\nabla L^{-1}(\nabla\Phi^*_0), \Phi^*)=W_{\textrm{Dym OT}}(\rho_a,\rho_b)$.
\end{proof}

\section{Proof of Theorem \ref{theorem on preconditioning}} \label{theorem3}

\begin{customthm}{3}
Suppose $L(\cdot)=\frac{|\cdot|^2}{2}$. We define the map $P(x)=\sigma x+\mu$ with $\sigma\in\mathbb{R}_+, \mu\in\mathbb{R}^d$. Recall \eqref{vector field = inv grad L grad Phi} in the paper, we denote ${v}(x,t)=\nabla\widehat{\Phi}(x, t)$ as the optimal solution to dynamical OT problem \eqref{dynamical OT} in the paper from $P_{\sharp}\rho_a$ to $\rho_b$. We set $\widehat{\Phi}_0 = \widehat{\Phi}(\cdot, 0)$. Furthermore, we denote ${v}(x,t)=\nabla\Phi^*(x,t)$ as the optimal solution to dynamical OT problem \eqref{dynamical OT} from $\rho_a$ to $\rho_b$. We set $\Phi^*_0=\Phi^*(\cdot, 0)$. Then we have
\begin{equation*}
  \nabla\Phi^*_0(x) = \nabla\widehat{\Phi}_0\circ P(x)+P(x)-x,
\end{equation*}
and $\Phi^*_0(x) = \frac{1}{\sigma}\widehat{\Phi}_0(\sigma x + \mu)+\frac{\sigma - 1}{2}|x|^2 + \mu^Tx+\textrm{Const}$.
\end{customthm}

\begin{proof}
 According to \eqref{Phi0, Phi1 as monge maps} in the paper, we have
 \begin{equation*}
   (\textrm{Id}+\nabla\widehat{\Phi}_0)_{\sharp}(P_{\sharp}\rho_a) = \rho_b
 \end{equation*}
 This yields
 \begin{equation*}
    (P+\nabla\Phi_0\circ P)_{\sharp}\rho_a = \rho_b
 \end{equation*}
 We rewrite this as
 \begin{equation}
    (\textrm{Id} + \nabla\widehat{\Phi}_0\circ P + P - \textrm{Id})_{\sharp}\rho_a = \rho_b  \label{thm precondition eq1}
 \end{equation}
 We denote $u(x) =  \frac{1}{\sigma}\widehat{\Phi}_0(\sigma x + \mu) + \frac{\sigma-1}{2}|x|^2 + \mu^Tx$.
 
 Then we can directly verify that
 \begin{equation*}
    \nabla u(x) = \nabla\widehat{\Phi}_0(\sigma x + \mu) + (\sigma x + \mu) - x = \nabla\widehat{\Phi}_0\circ P(x) + P(x) - x
 \end{equation*}
 Plug this into \eqref{thm precondition eq1} above we get:
 \begin{equation*}
    (\textrm{Id}+\nabla u)_{\sharp}\rho_a = \rho_b
 \end{equation*}
 Using the uniqueness of the solution to Monge-Ampere equation, we have $\Phi^*_0 = u +\textrm{Const}$, or equivalently, $\nabla\Phi^*_0(x) = \nabla u(x) = \nabla\widehat{\Phi}_0\circ P(x) + P(x)-x$
\end{proof}

\newpage
\section{Complete Geodesics in Experiments}
\subsection{Synthetic Data}
\textbf{Syn-1:}
\begin{figure*}[h!]
\centering
 \subfloat[][$\rho_a $ to $\rho_b$ at $t_1$]{\includegraphics[width=0.18\textwidth,height=0.18\textheight,keepaspectratio]{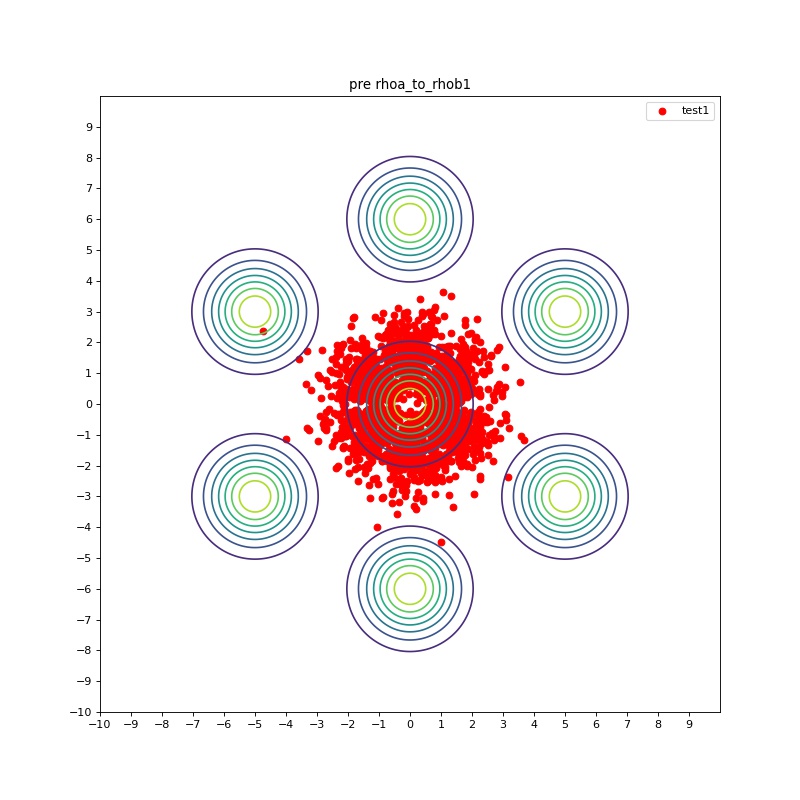}}
 \subfloat[][$\rho_a $ to $\rho_b$ at $t_2$]{\includegraphics[width=.18\linewidth]{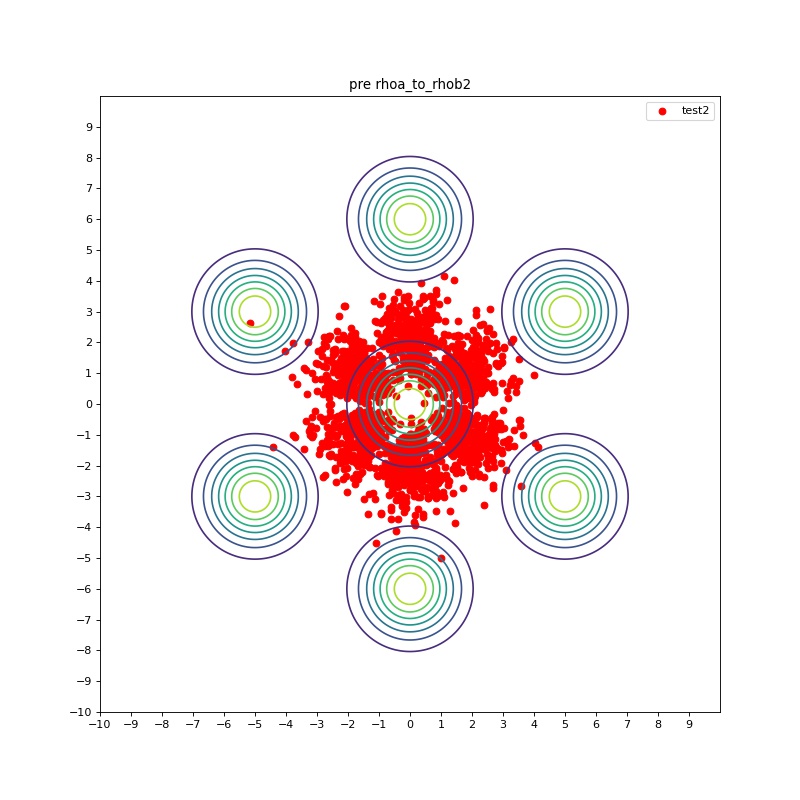}}
 \subfloat[][$\rho_a $ to $\rho_b$ at $t_3$]{\includegraphics[width=.18\linewidth]{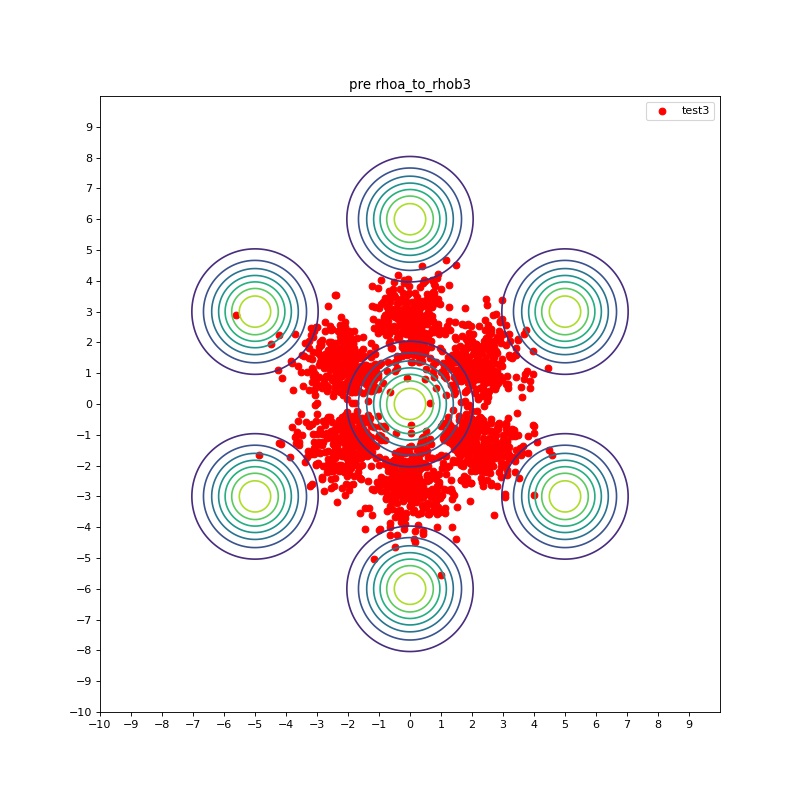}}
 \subfloat[][$\rho_a $ to $\rho_b$ at $t_4$]{\includegraphics[width=.18\linewidth]{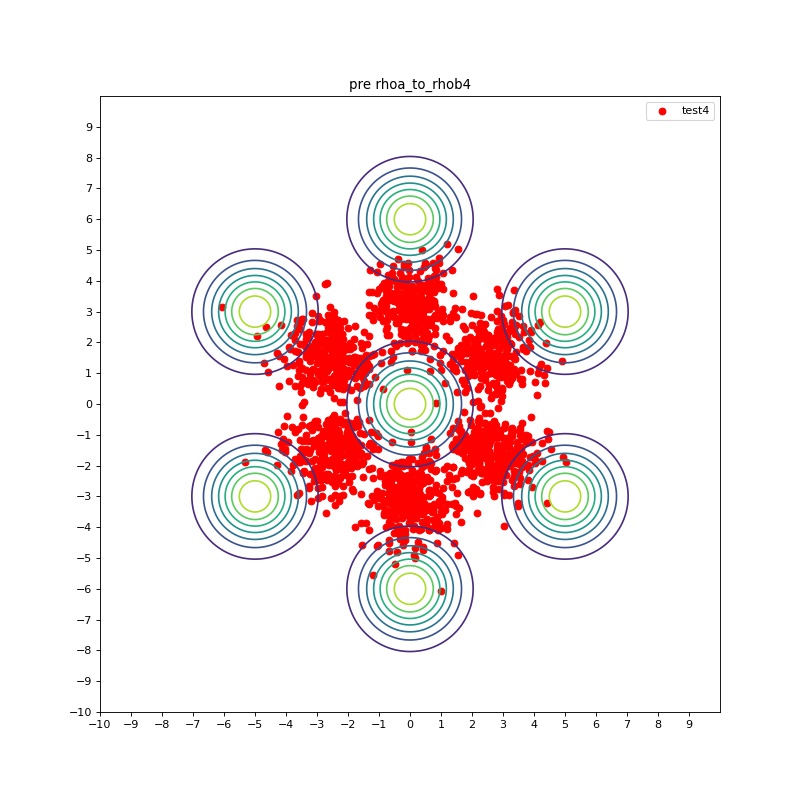}}
 \subfloat[][$\rho_a $ to $\rho_b$ at $t_5$]{\includegraphics[width=.18\linewidth]{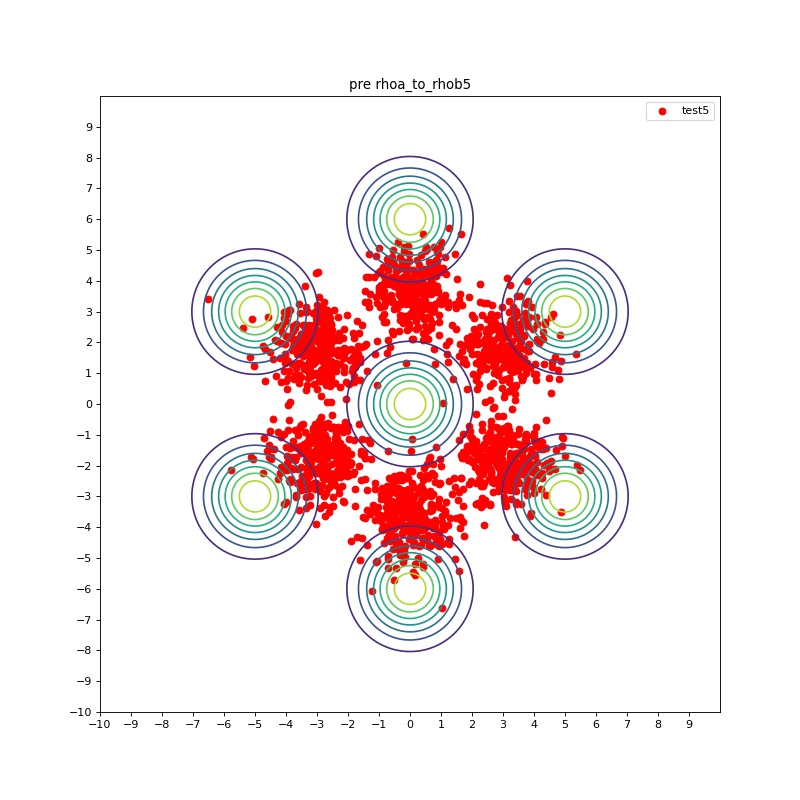}}\\
 \subfloat[][$\rho_a $ to $\rho_b$ at $t_6$]{\includegraphics[width=.18\linewidth]{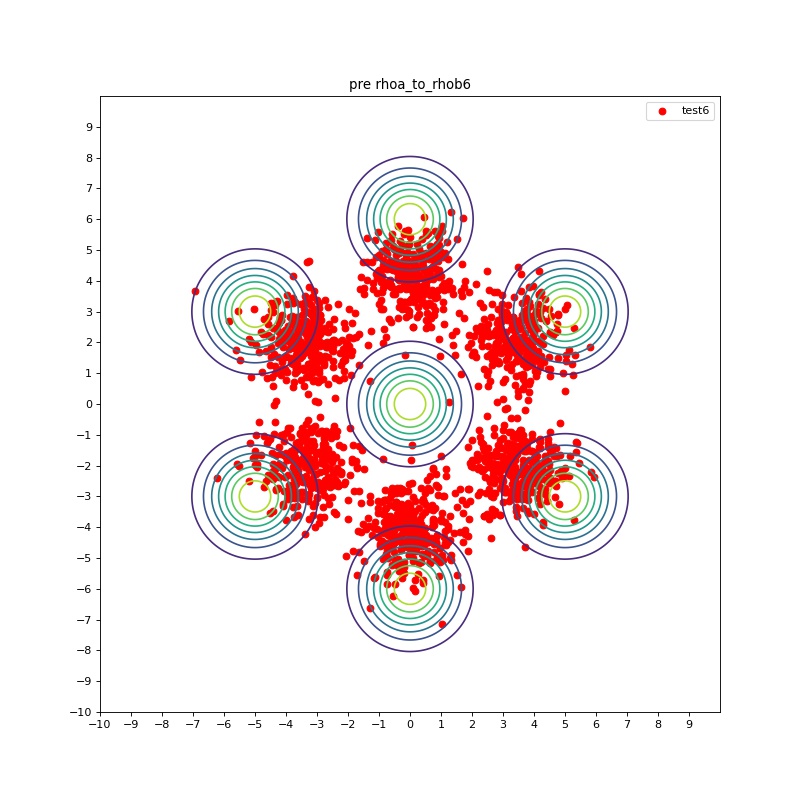}}
 \subfloat[][$\rho_a $ to $\rho_b$ at $t_7$]{\includegraphics[width=.18\linewidth]{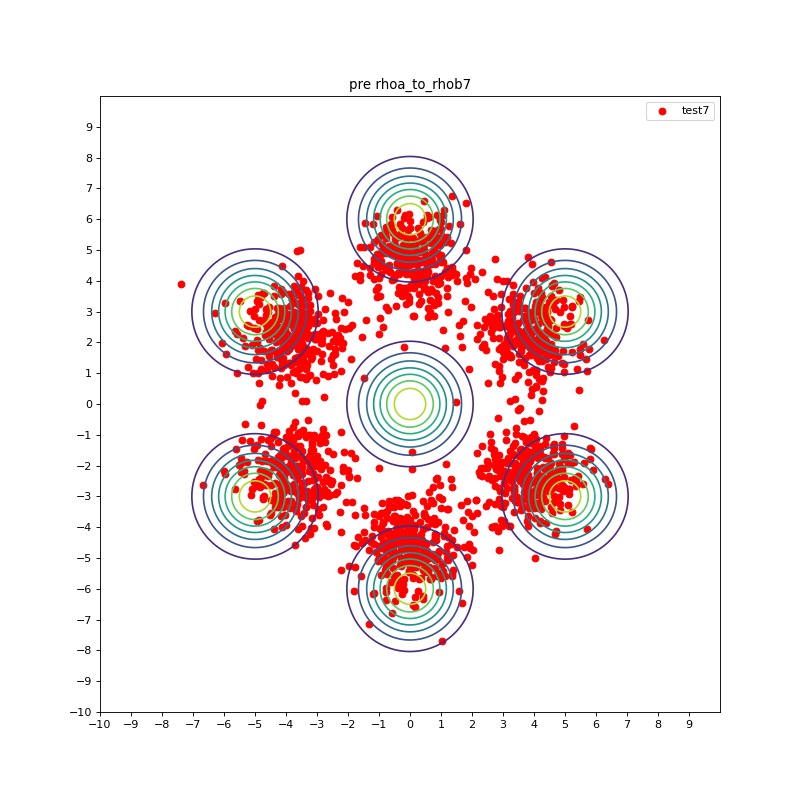}}
 \subfloat[][$\rho_a $ to $\rho_b$ at $t_8$]{\includegraphics[width=.18\linewidth]{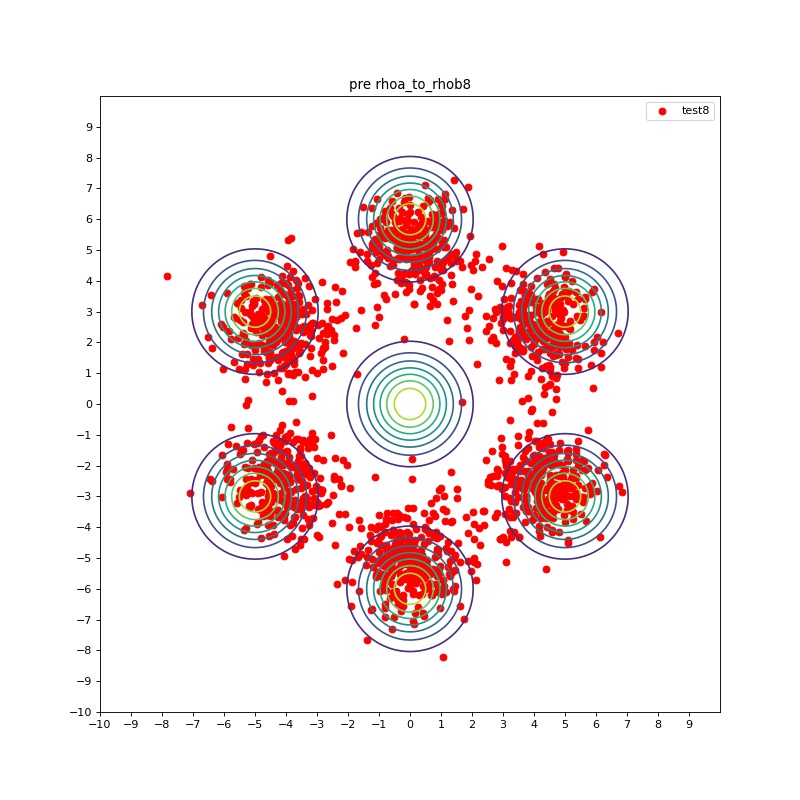}}
 \subfloat[][$\rho_a $ to $\rho_b$ at $t_9$]{\includegraphics[width=.18\linewidth]{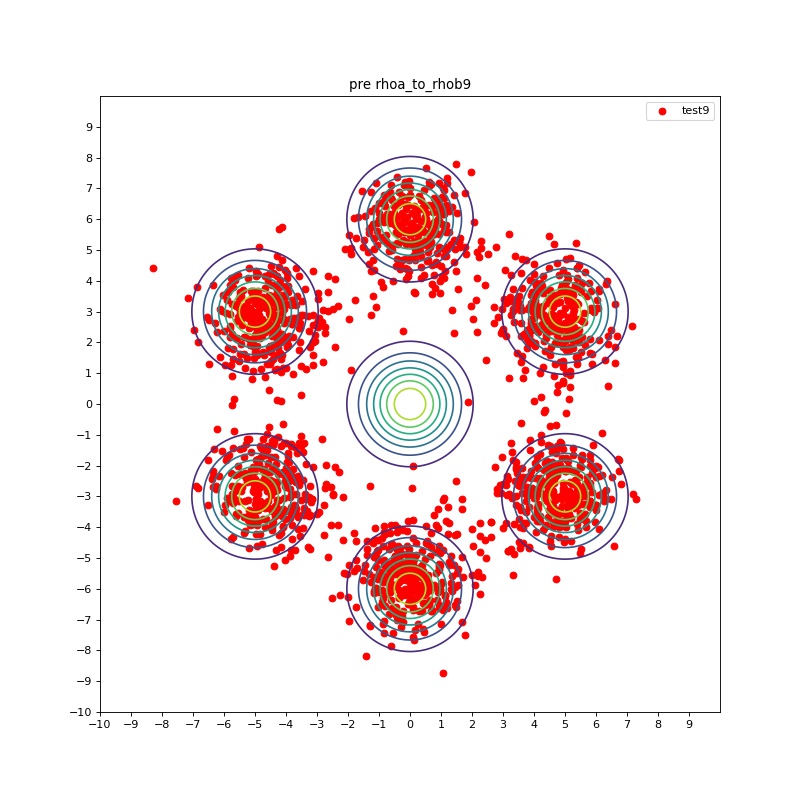}}
 \subfloat[][$\rho_a $ to $\rho_b$ at $t_{10}$]{\includegraphics[width=.18\linewidth]{pre_rhoa_to_rhob1700010.jpg}}\\
 \subfloat[][$\rho_b $ to $\rho_a$ at $t_1$]{\includegraphics[width=.18\linewidth]{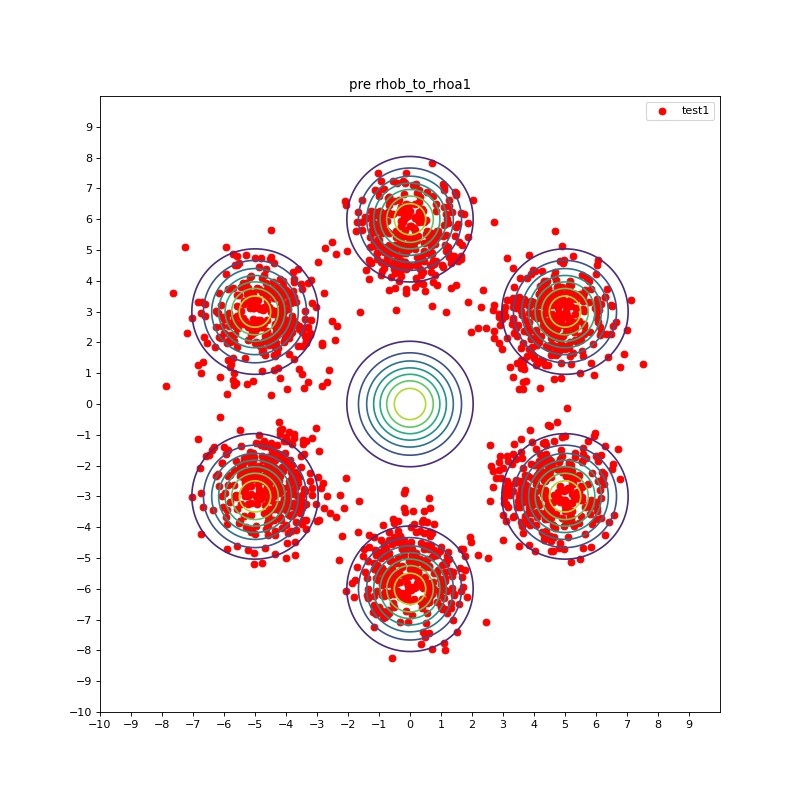}}
 \subfloat[][$\rho_b $ to $\rho_a$ at $t_2$]{\includegraphics[width=.18\linewidth]{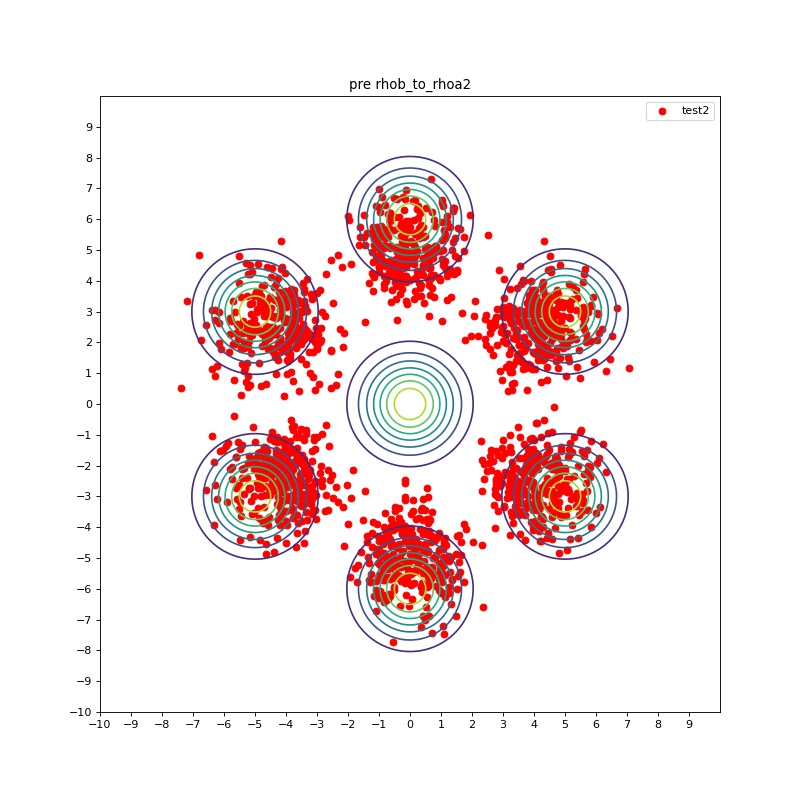}}
 \subfloat[][$\rho_b $ to $\rho_a$ at $t_3$]{\includegraphics[width=.18\linewidth]{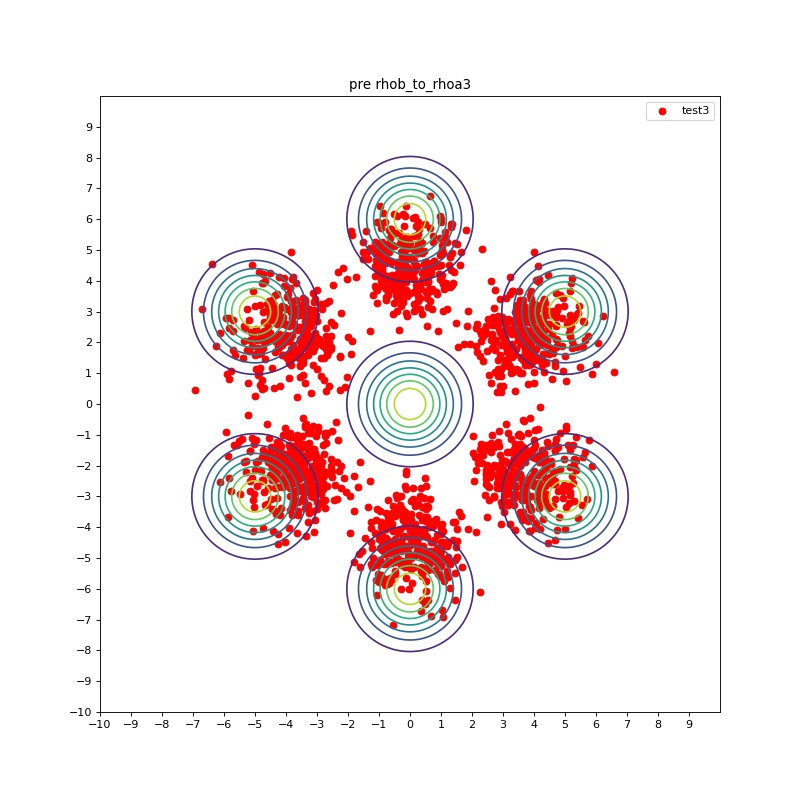}}
 \subfloat[][$\rho_b $ to $\rho_a$ at $t_4$]{\includegraphics[width=.18\linewidth]{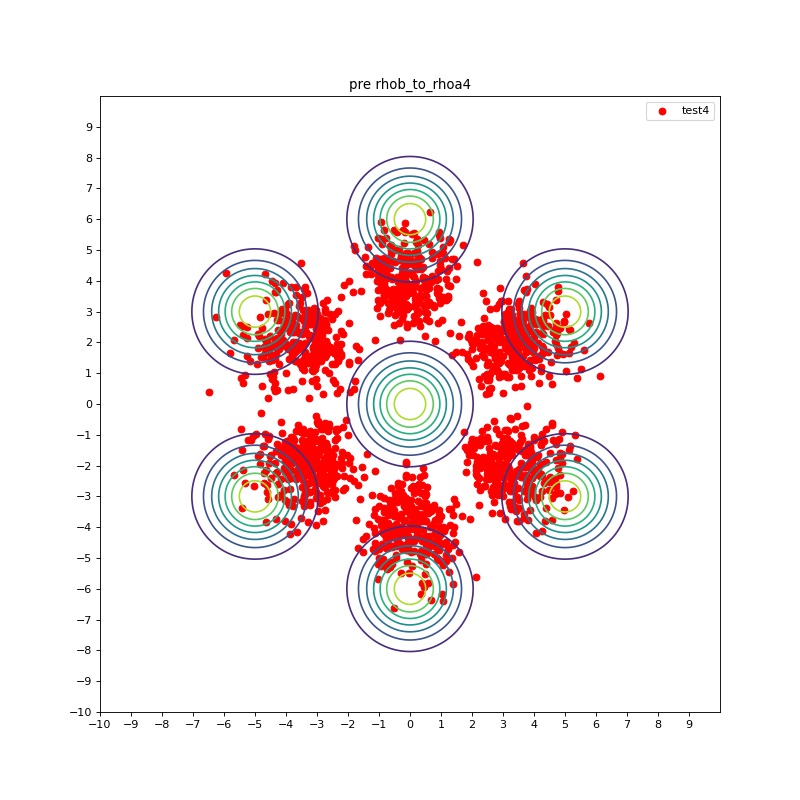}}
 \subfloat[][$\rho_b $ to $\rho_a$ at $t_5$]{\includegraphics[width=.18\linewidth]{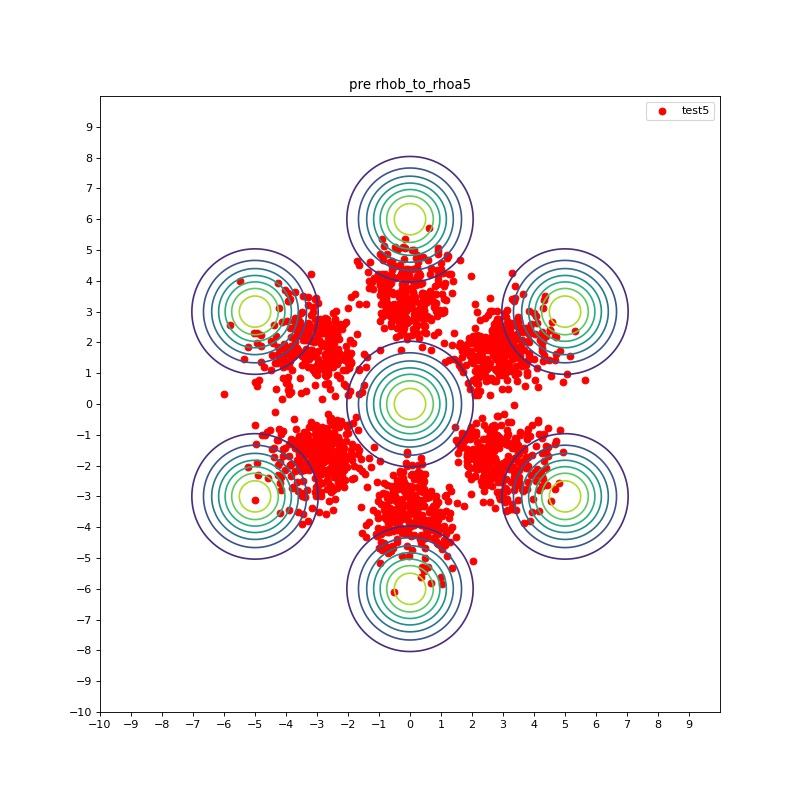}}\\
 \subfloat[][$\rho_b $ to $\rho_a$ at $t_6$]{\includegraphics[width=.18\linewidth]{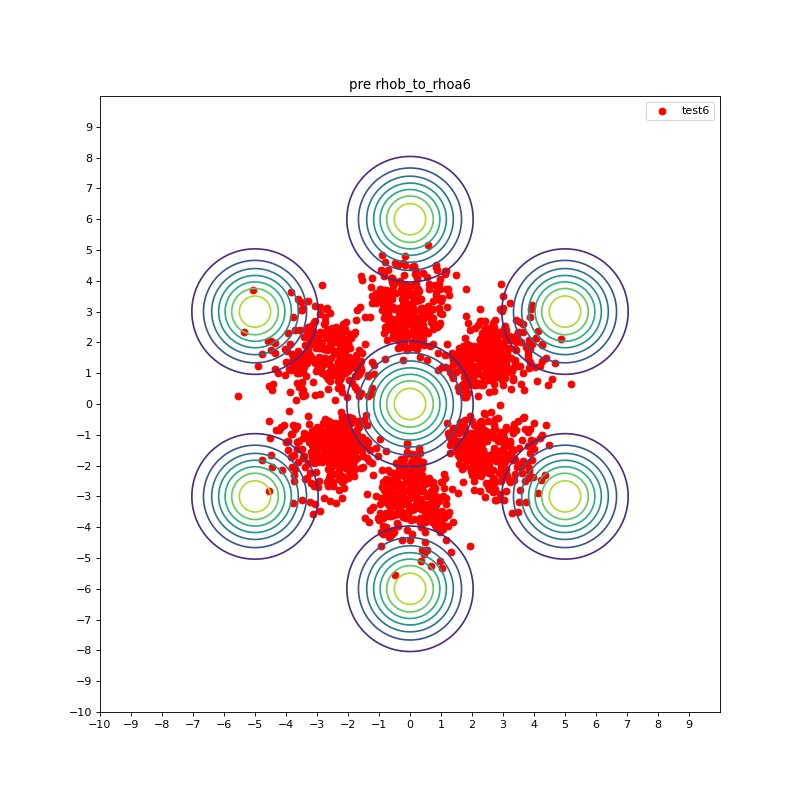}}
 \subfloat[][$\rho_b $ to $\rho_a$ at $t_7$]{\includegraphics[width=.18\linewidth]{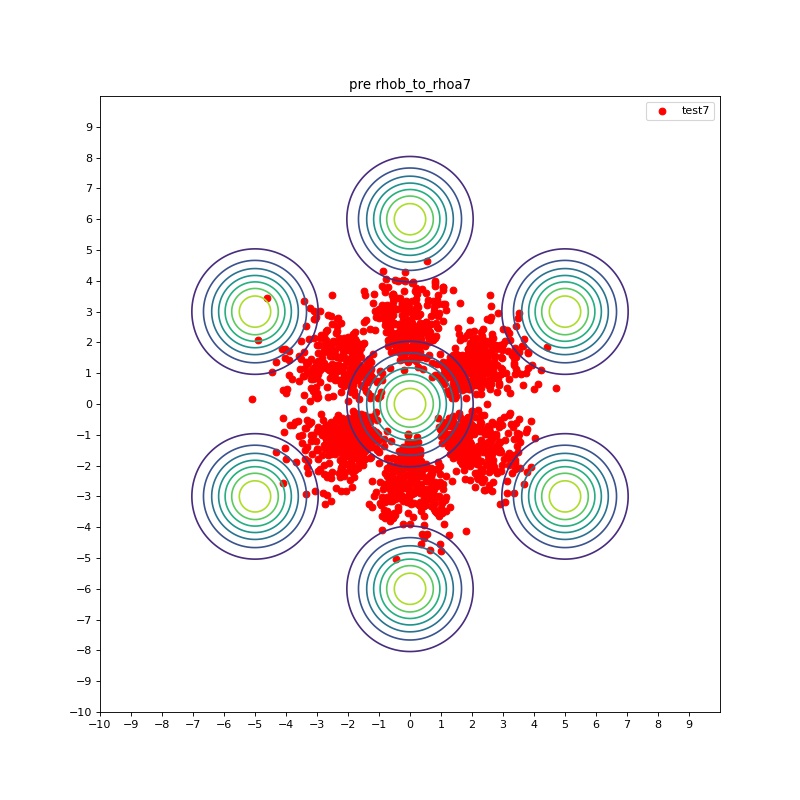}}
 \subfloat[][$\rho_b $ to $\rho_a$ at $t_8$]{\includegraphics[width=.18\linewidth]{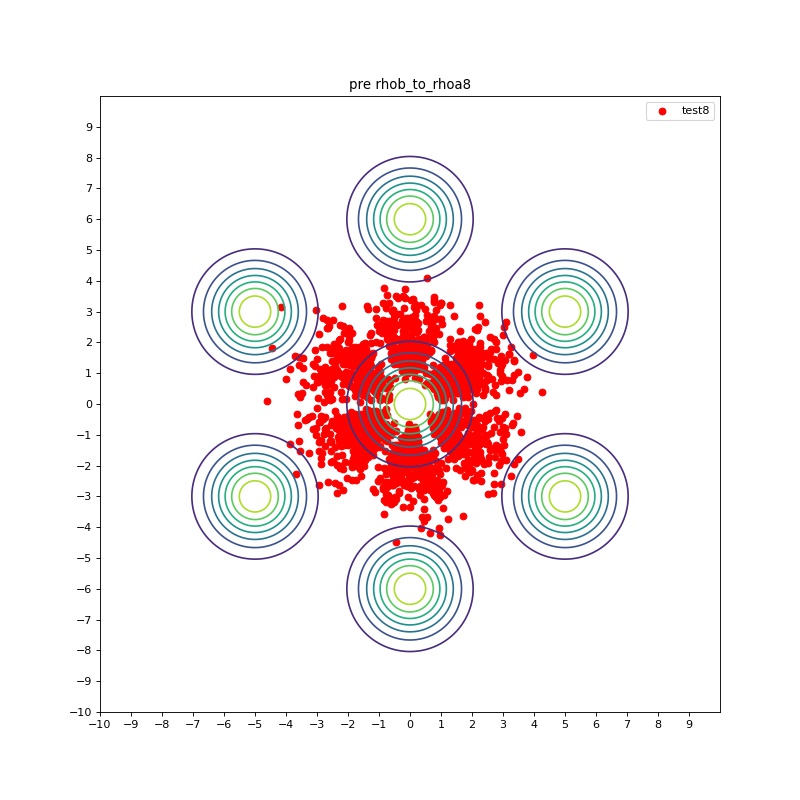}}
 \subfloat[][$\rho_b $ to $\rho_a$ at $t_9$]{\includegraphics[width=.18\linewidth]{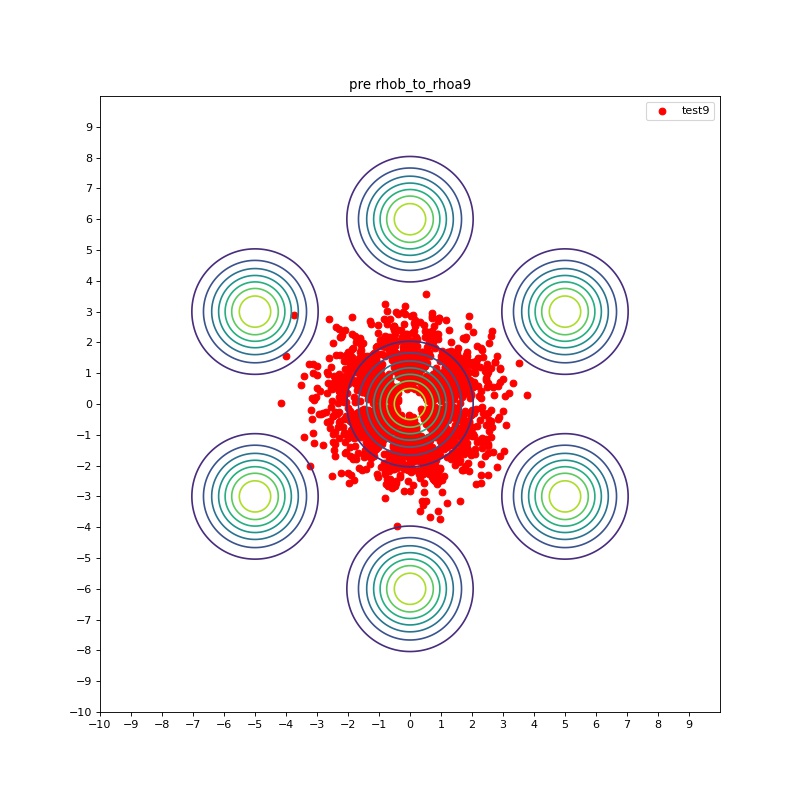}}
 \subfloat[][$\rho_b $ to $\rho_a$ at $t_{10}$]{\includegraphics[width=.18\linewidth]{pre_rhob_to_rhoa1700010.jpg}}\\
 \subfloat[][points track:$\rho_a$ to $\rho_b$]{\includegraphics[width=.18\linewidth]{colored_tracks_rhoa_to_rhob17000.jpg}}
 \subfloat[][vector field:$\rho_a$ to $\rho_b$]{\includegraphics[width=.18\linewidth]{vec_fd_rhoa_to_rhob1700010.jpg}}
 \subfloat[][points track:$\rho_b$ to $\rho_a$]{\includegraphics[width=.18\linewidth]{colored_tracks_rhob_to_rhoa17000.jpg}}
 \subfloat[][vector field:$\rho_b$ to $\rho_a$]{\includegraphics[width=.18\linewidth]{vec_fd_rhob_to_rhoa1700010.jpg}}
\caption{2-dimensional Gaussian to Gaussian}
\label{fig:syn-111}
\end{figure*}

\newpage
\textbf{Syn-2: Wasserstein-1.5}
\begin{figure*}[h!]
\centering
\subfloat[][$\rho_a $ to $\rho_b$ at $t_1$]{\includegraphics[width=0.18\textwidth,height=0.18\textheight,keepaspectratio]{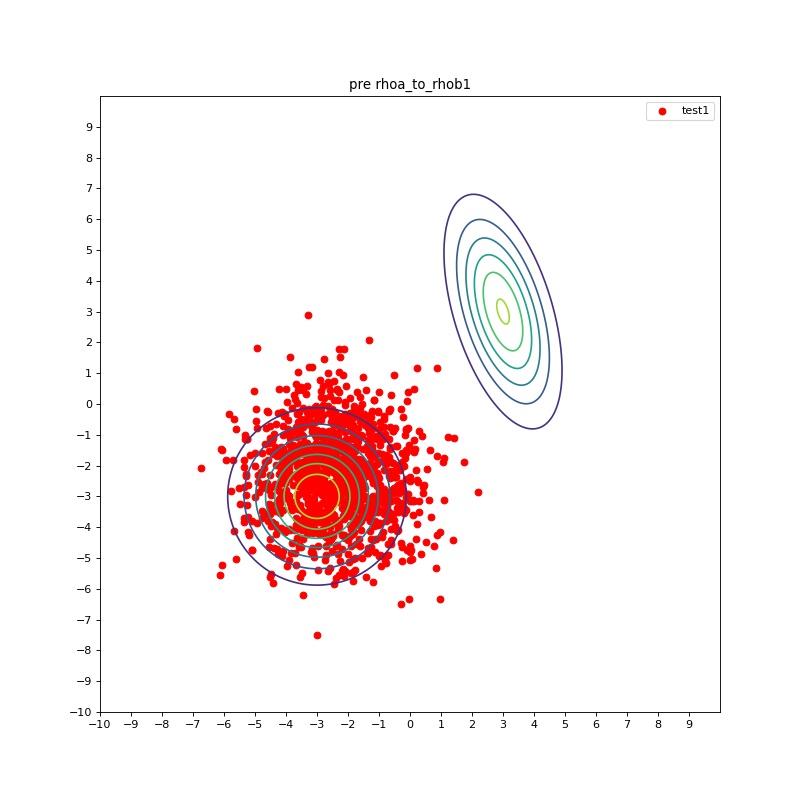}}
\subfloat[][$\rho_a $ to $\rho_b$ at $t_2$]{\includegraphics[width=.18\linewidth]{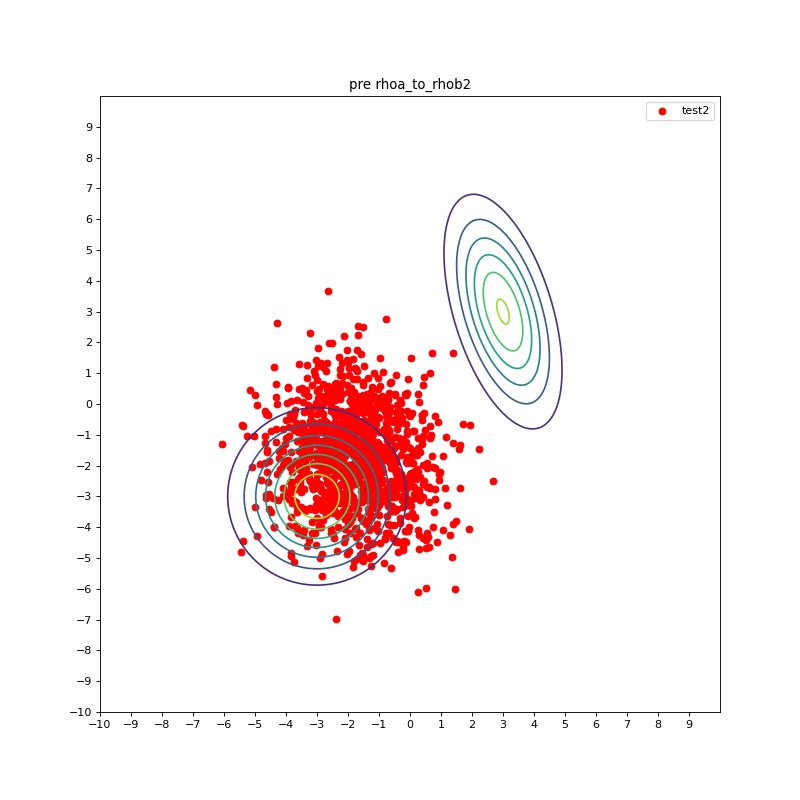}}
\subfloat[][$\rho_a $ to $\rho_b$ at $t_3$]{\includegraphics[width=.18\linewidth]{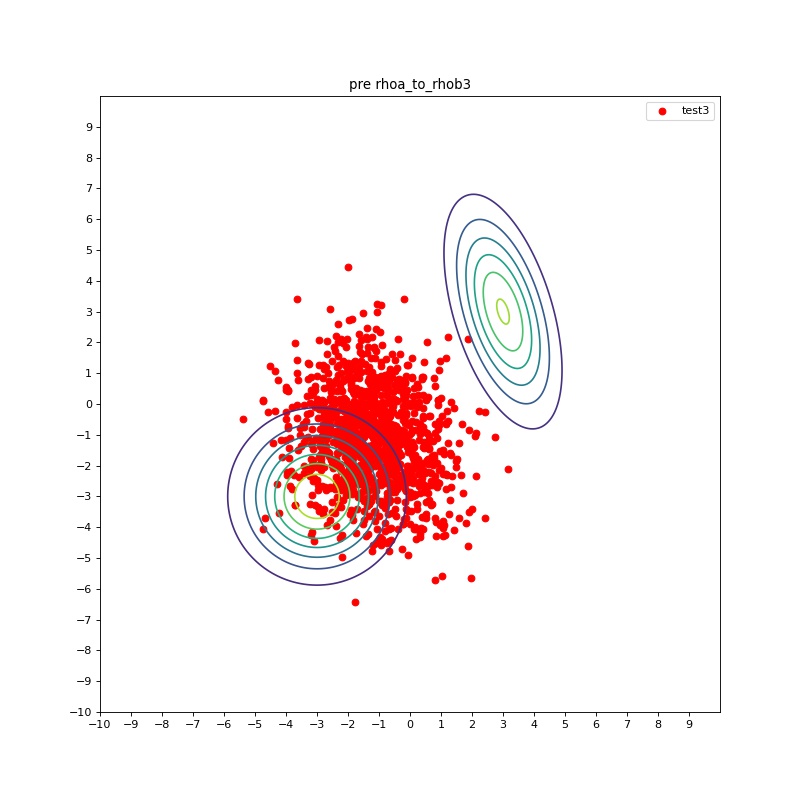}}
\subfloat[][$\rho_a $ to $\rho_b$ at $t_4$]{\includegraphics[width=.18\linewidth]{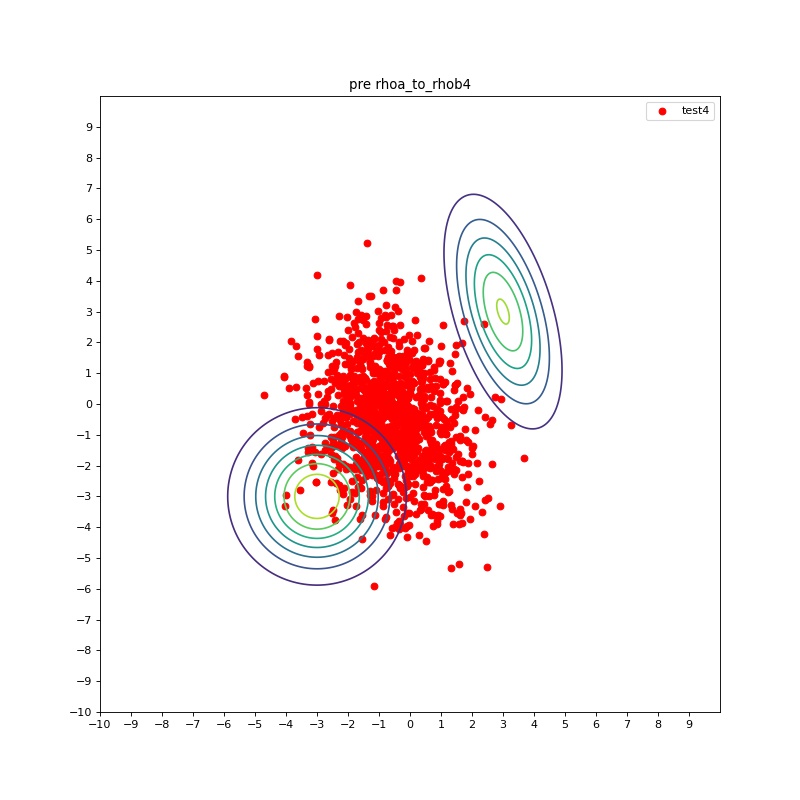}}
\subfloat[][$\rho_a $ to $\rho_b$ at $t_5$]{\includegraphics[width=.18\linewidth]{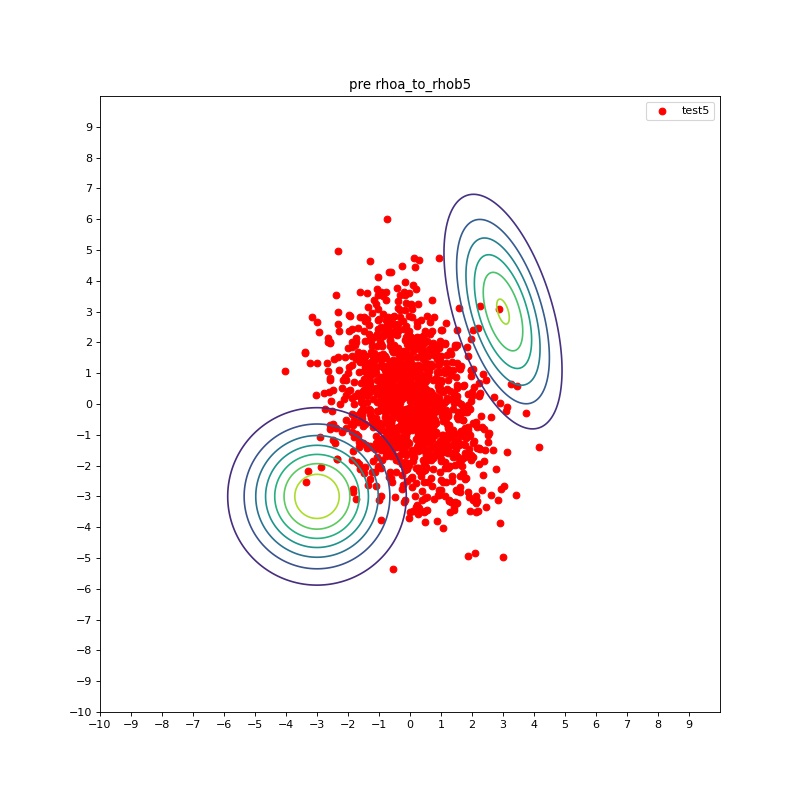}}\\
\subfloat[][$\rho_a $ to $\rho_b$ at $t_6$]{\includegraphics[width=.18\linewidth]{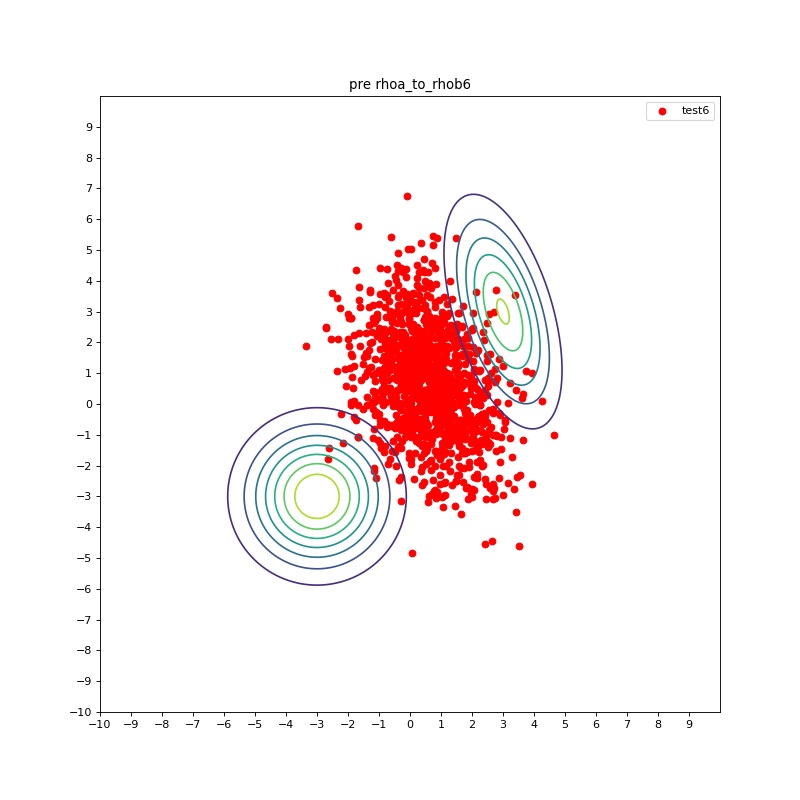}}
\subfloat[][$\rho_a $ to $\rho_b$ at $t_7$]{\includegraphics[width=.18\linewidth]{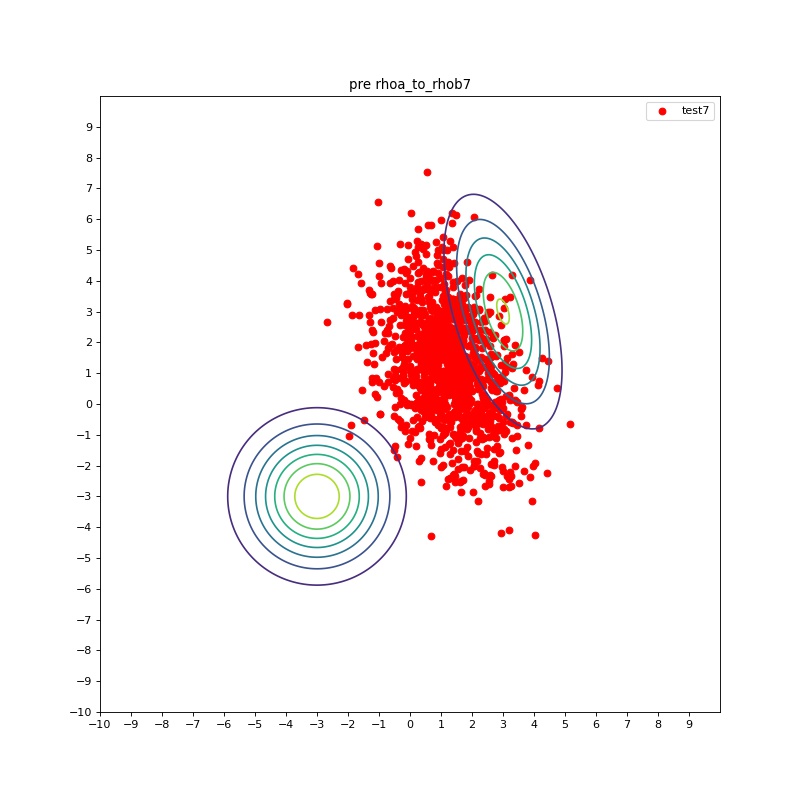}}
\subfloat[][$\rho_a $ to $\rho_b$ at $t_8$]{\includegraphics[width=.18\linewidth]{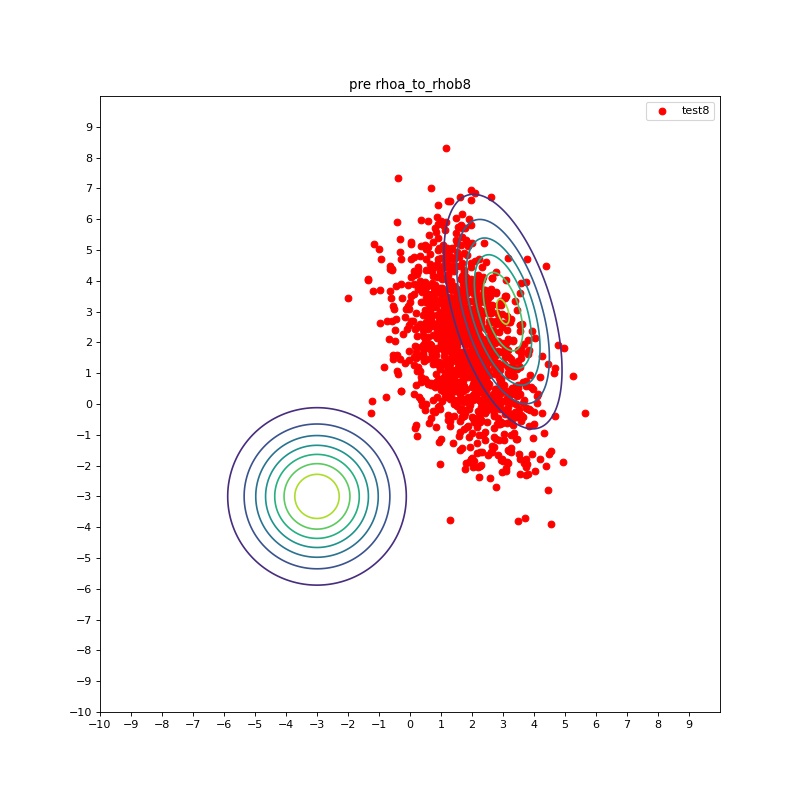}}
\subfloat[][$\rho_a $ to $\rho_b$ at $t_9$]{\includegraphics[width=.18\linewidth]{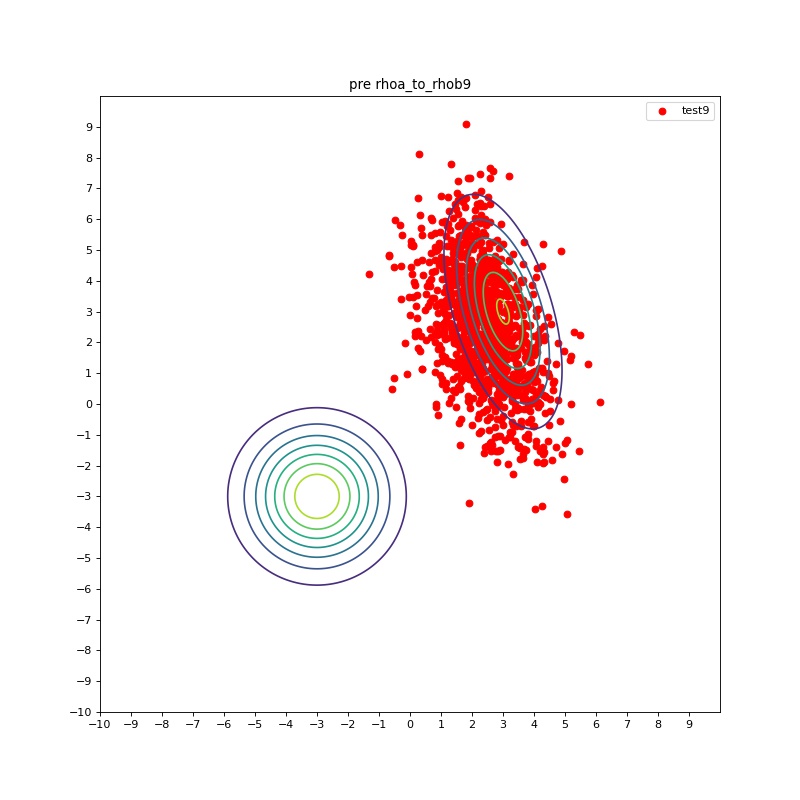}}
\subfloat[][$\rho_a $ to $\rho_b$ at $t_{10}$]{\includegraphics[width=.18\linewidth]{pre_rhoa_to_rhob_11000_10.jpg}}\\
\subfloat[][$\rho_b $ to $\rho_a$ at $t_1$]{\includegraphics[width=.18\linewidth]{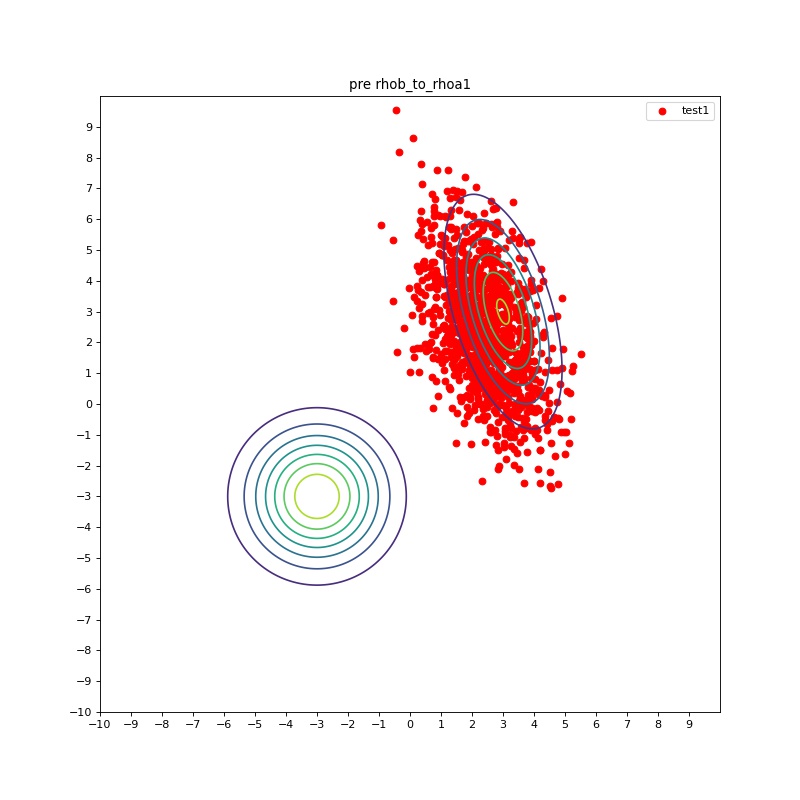}}
\subfloat[][$\rho_b $ to $\rho_a$ at $t_2$]{\includegraphics[width=.18\linewidth]{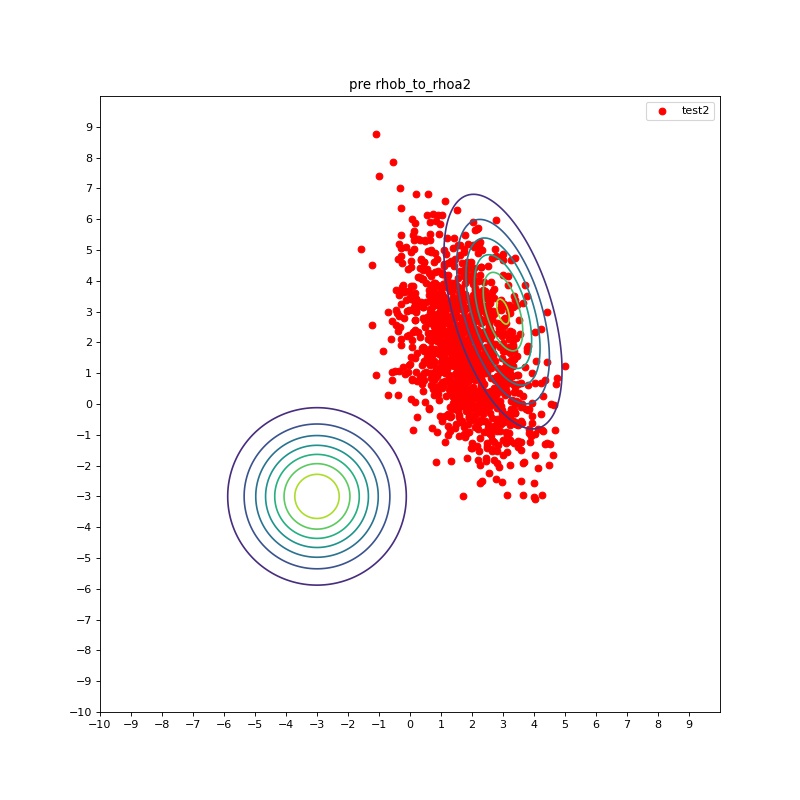}}
\subfloat[][$\rho_b $ to $\rho_a$ at $t_3$]{\includegraphics[width=.18\linewidth]{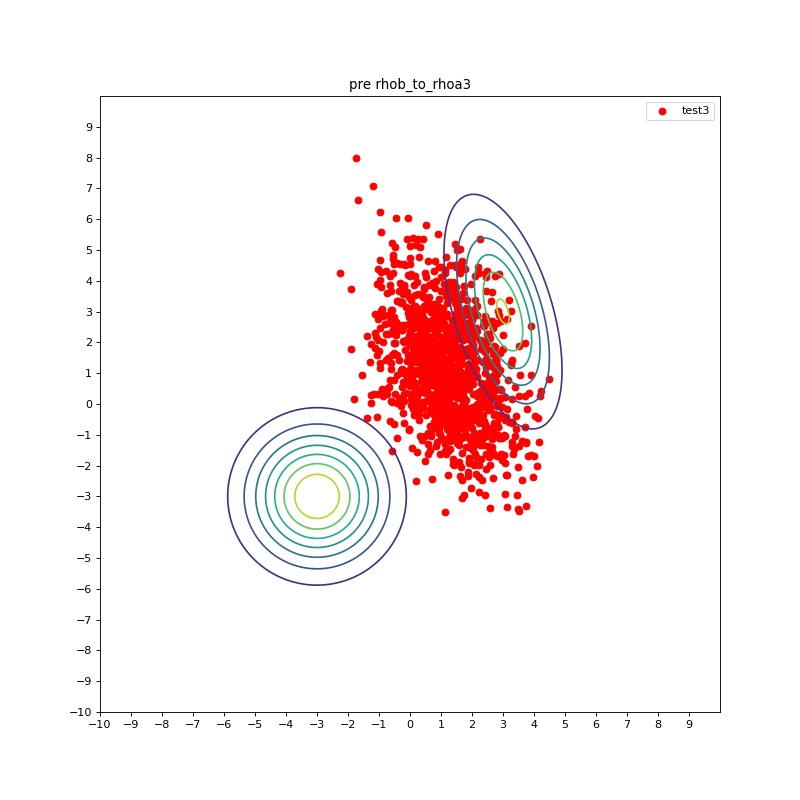}}
\subfloat[][$\rho_b $ to $\rho_a$ at $t_4$]{\includegraphics[width=.18\linewidth]{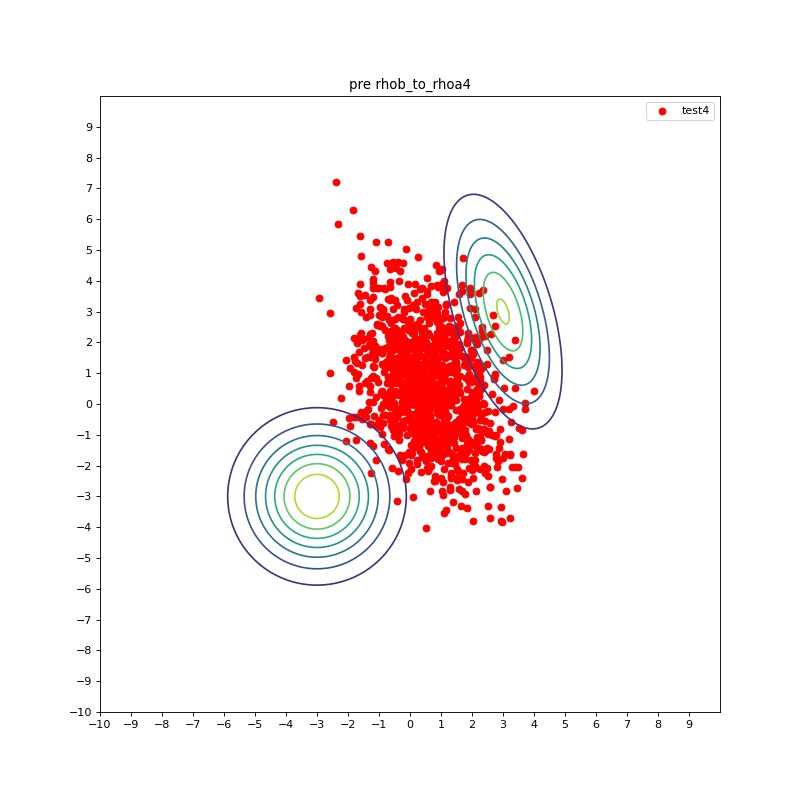}}
\subfloat[][$\rho_b $ to $\rho_a$ at $t_5$]{\includegraphics[width=.18\linewidth]{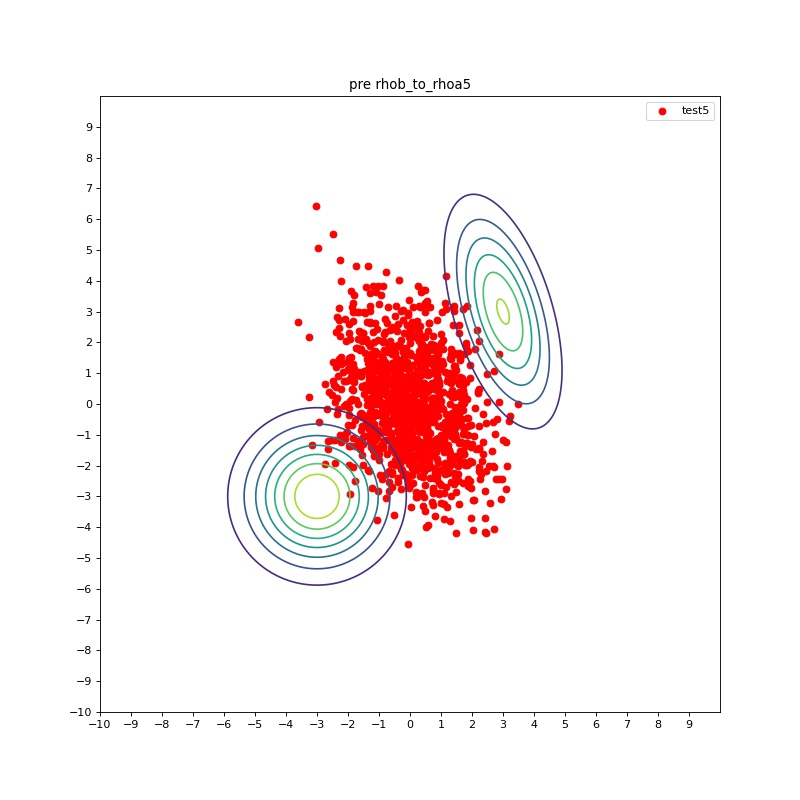}}\\
\subfloat[][$\rho_b $ to $\rho_a$ at $t_6$]{\includegraphics[width=.18\linewidth]{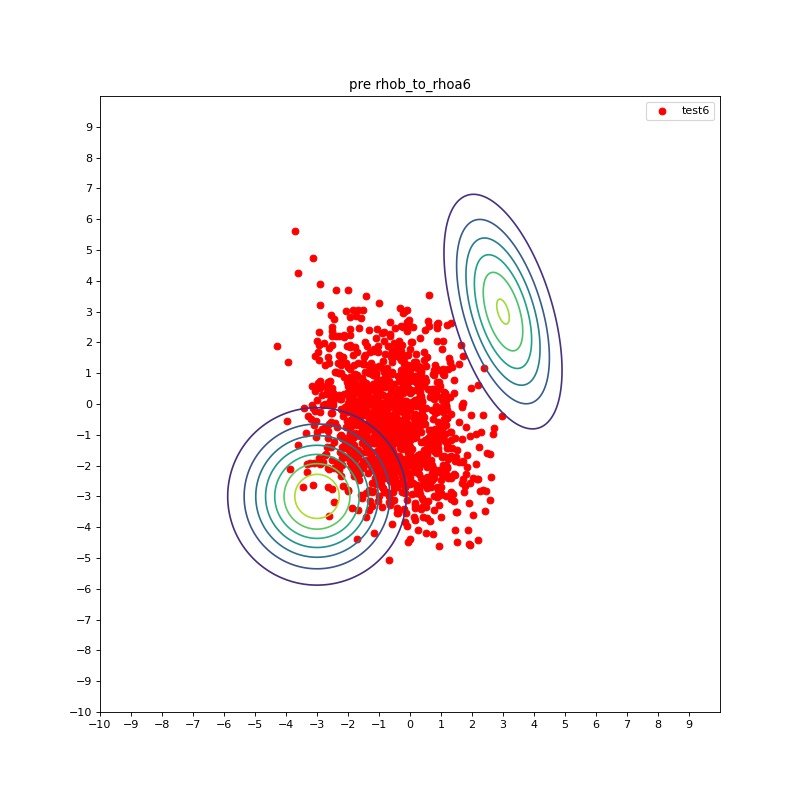}}
\subfloat[][$\rho_b $ to $\rho_a$ at $t_7$]{\includegraphics[width=.18\linewidth]{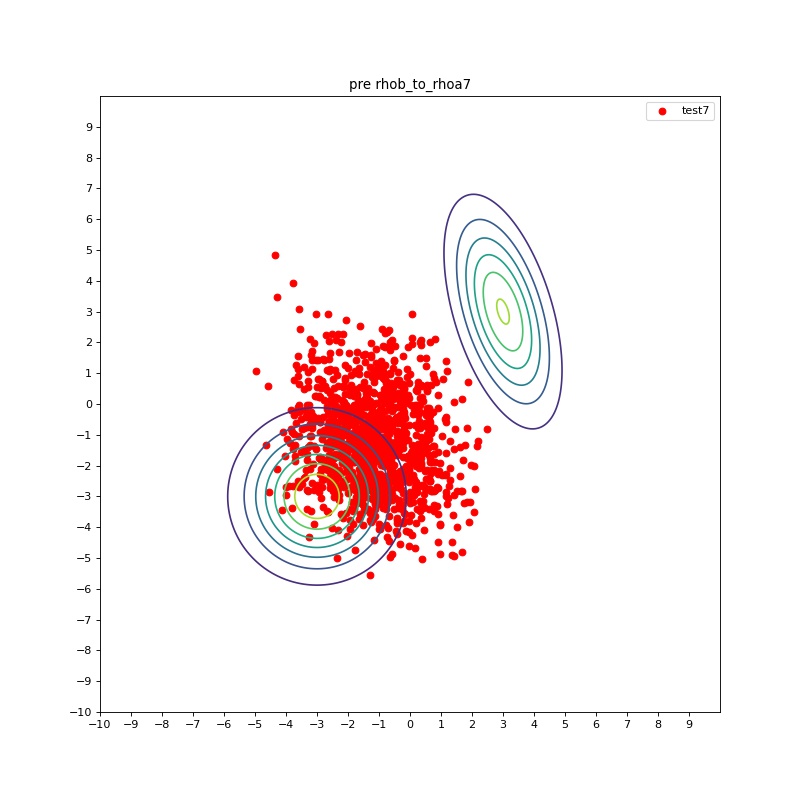}}
\subfloat[][$\rho_b $ to $\rho_a$ at $t_8$]{\includegraphics[width=.18\linewidth]{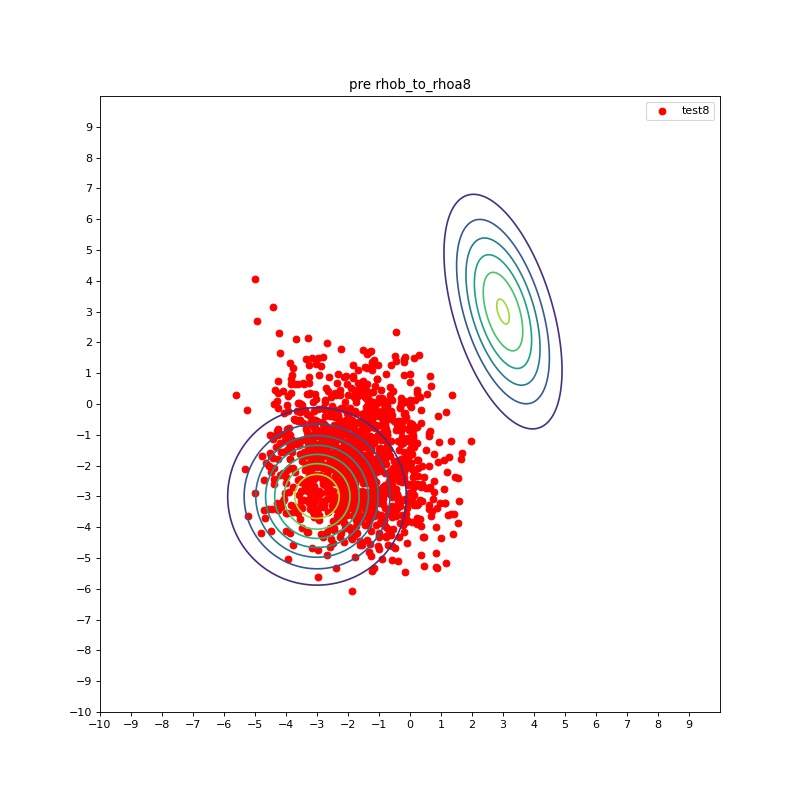}}
\subfloat[][$\rho_b $ to $\rho_a$ at $t_9$]{\includegraphics[width=.18\linewidth]{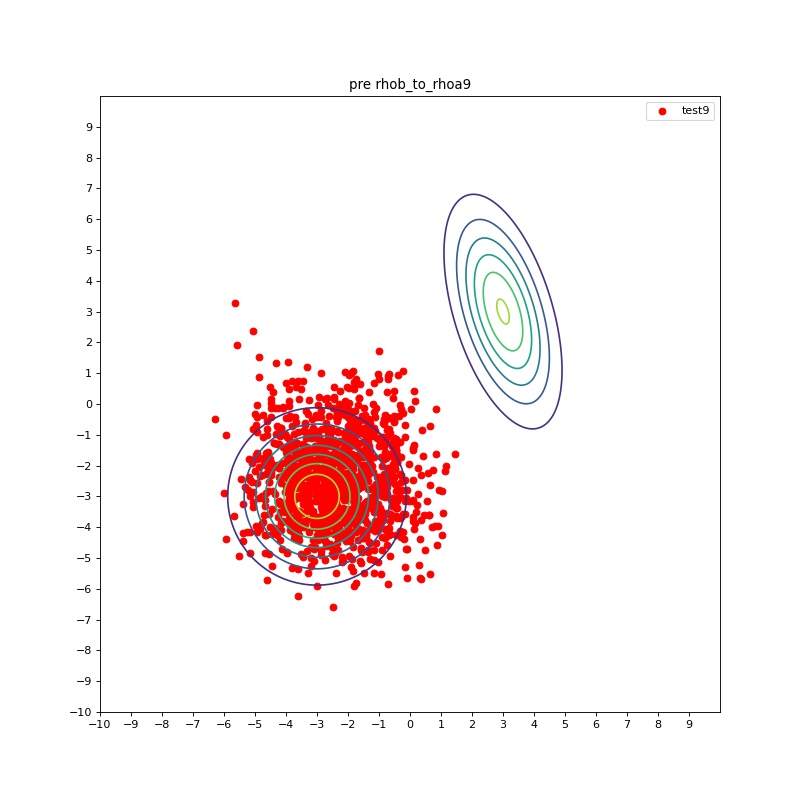}}
\subfloat[][$\rho_b $ to $\rho_a$ at $t_{10}$]{\includegraphics[width=.18\linewidth]{pre_rhob_to_rhoa_11000_10.jpg}}\\
\subfloat[][points track:$\rho_a$ to $\rho_b$]{\includegraphics[width=.18\linewidth]{colored_tracks_rhoa_to_rhob_11000.jpg}}
\subfloat[][vector field:$\rho_a$ to $\rho_b$]{\includegraphics[width=.18\linewidth]{vec_fd_rhoa_to_rhob_11000.jpg}}
\subfloat[][points track:$\rho_b$ to $\rho_a$]{\includegraphics[width=.18\linewidth]{colored_tracks_rhob_to_rhoa_11000.jpg}}
\subfloat[][vector field:$\rho_b$ to $\rho_a$]{\includegraphics[width=.18\linewidth]{vec_fd_rhob_to_rhoa_11000.jpg}}
\caption{5-dimensional Gaussian to Gaussian}
\label{fig:syn-21}
\end{figure*}

\newpage
\textbf{Syn-2: Wasserstein-2}
\begin{figure*}[h!]
\centering
\subfloat[][$\rho_a $ to $\rho_b$ at $t_1$]{\includegraphics[width=0.18\textwidth,height=0.18\textheight,keepaspectratio]{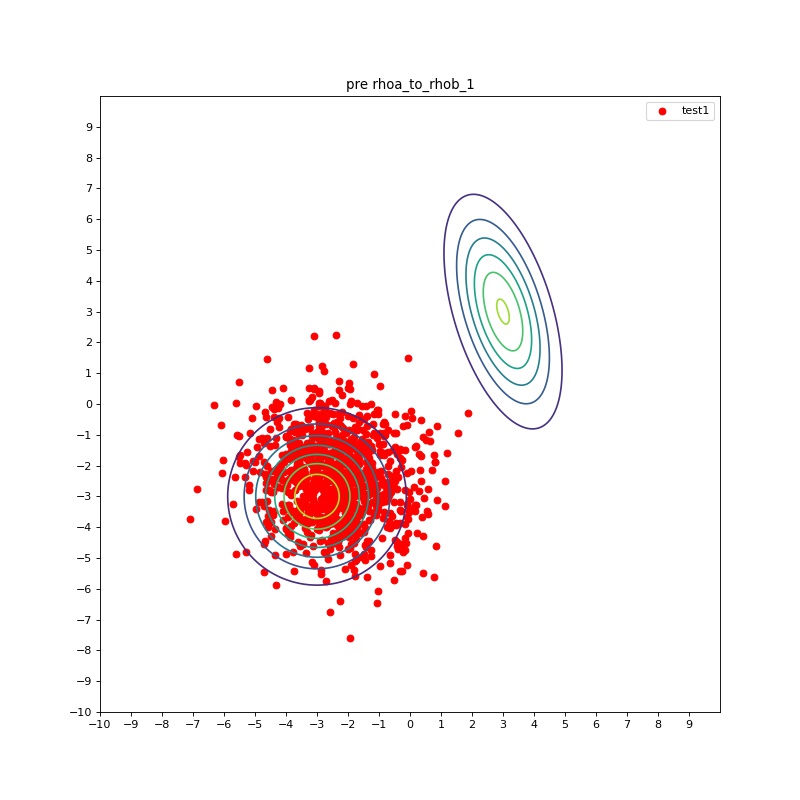}}
\subfloat[][$\rho_a $ to $\rho_b$ at $t_2$]{\includegraphics[width=.18\linewidth]{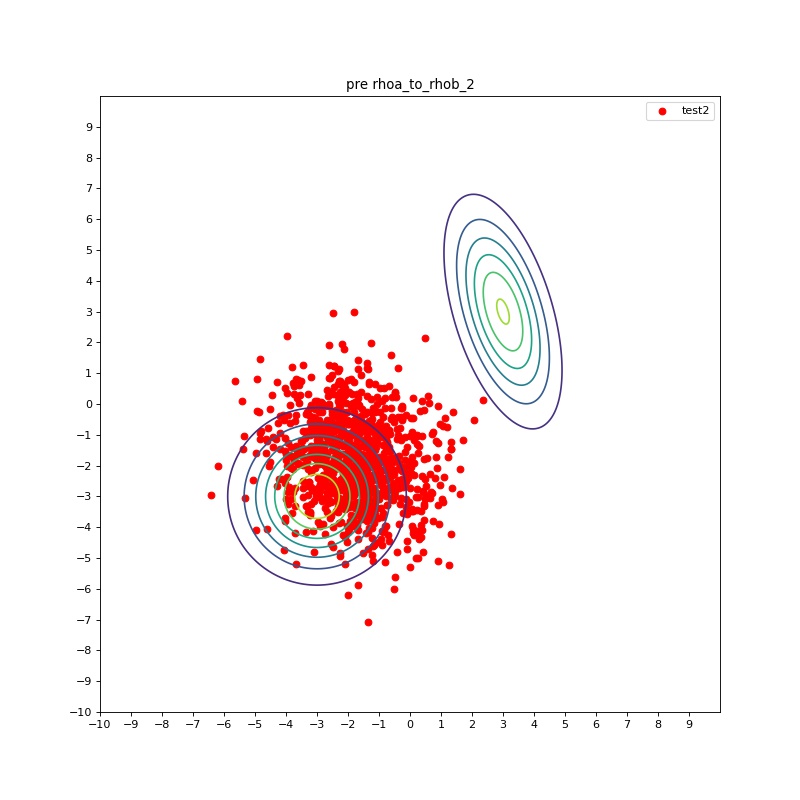}}
\subfloat[][$\rho_a $ to $\rho_b$ at $t_3$]{\includegraphics[width=.18\linewidth]{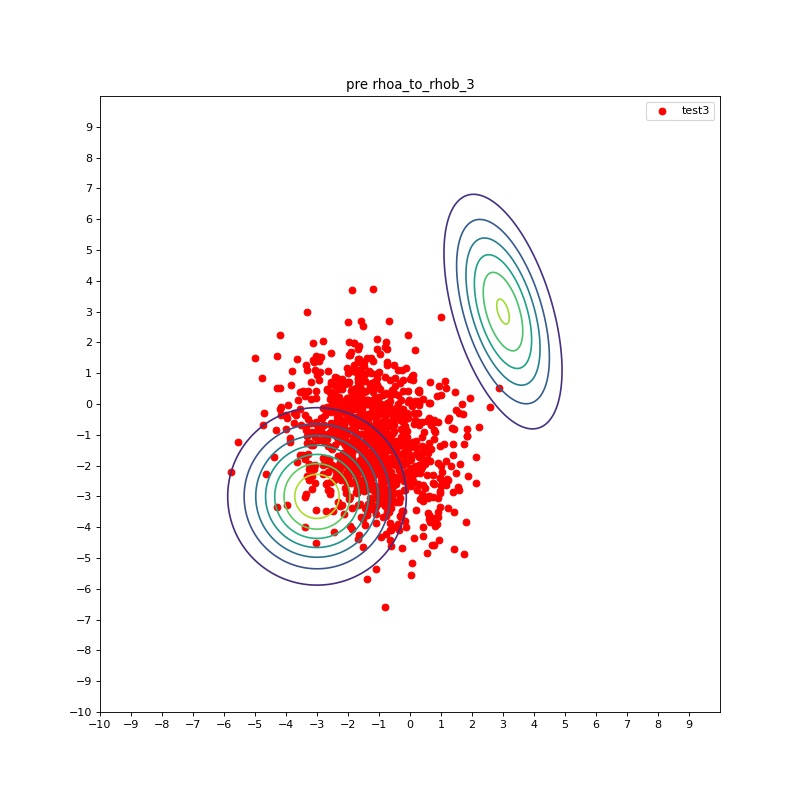}}
\subfloat[][$\rho_a $ to $\rho_b$ at $t_4$]{\includegraphics[width=.18\linewidth]{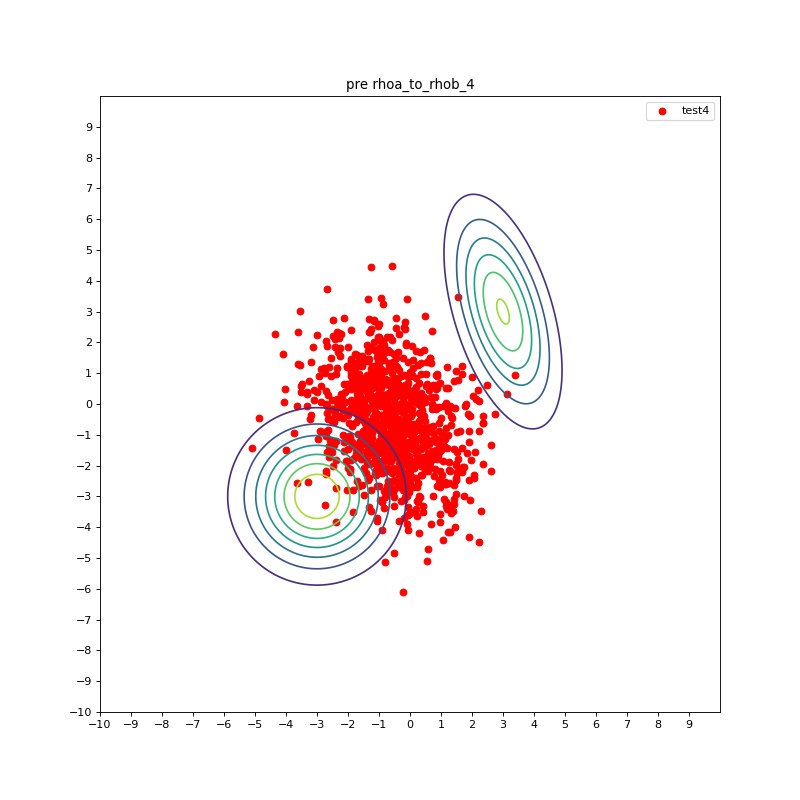}}
\subfloat[][$\rho_a $ to $\rho_b$ at $t_5$]{\includegraphics[width=.18\linewidth]{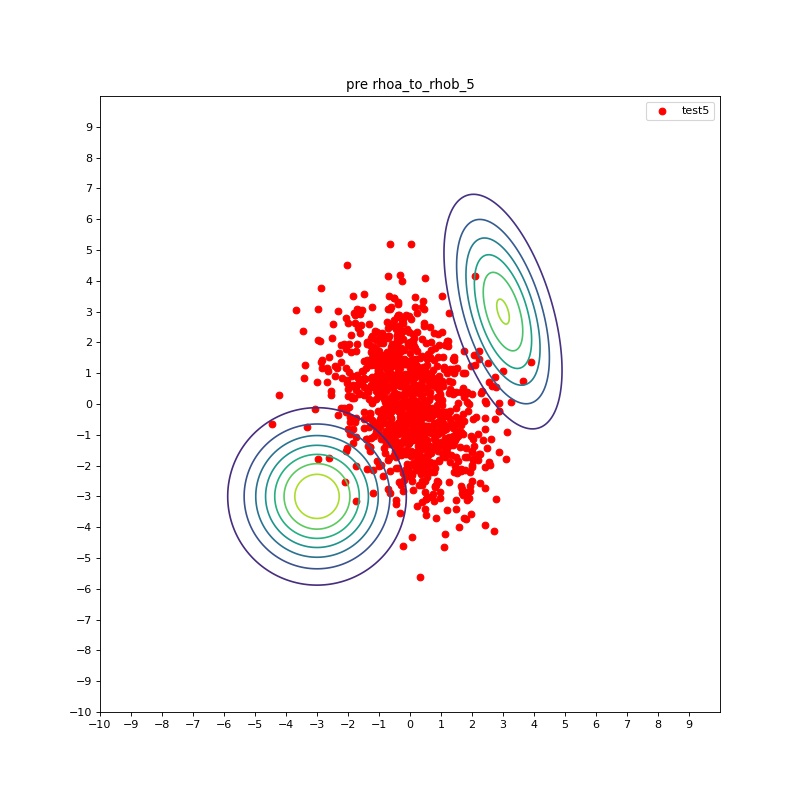}}\\
\subfloat[][$\rho_a $ to $\rho_b$ at $t_6$]{\includegraphics[width=.18\linewidth]{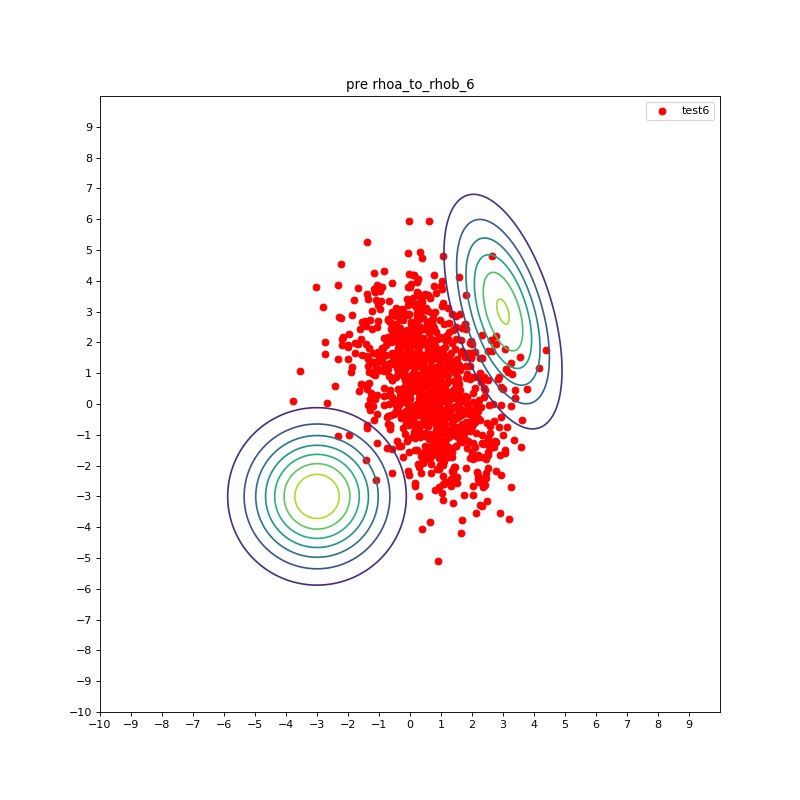}}
\subfloat[][$\rho_a $ to $\rho_b$ at $t_7$]{\includegraphics[width=.18\linewidth]{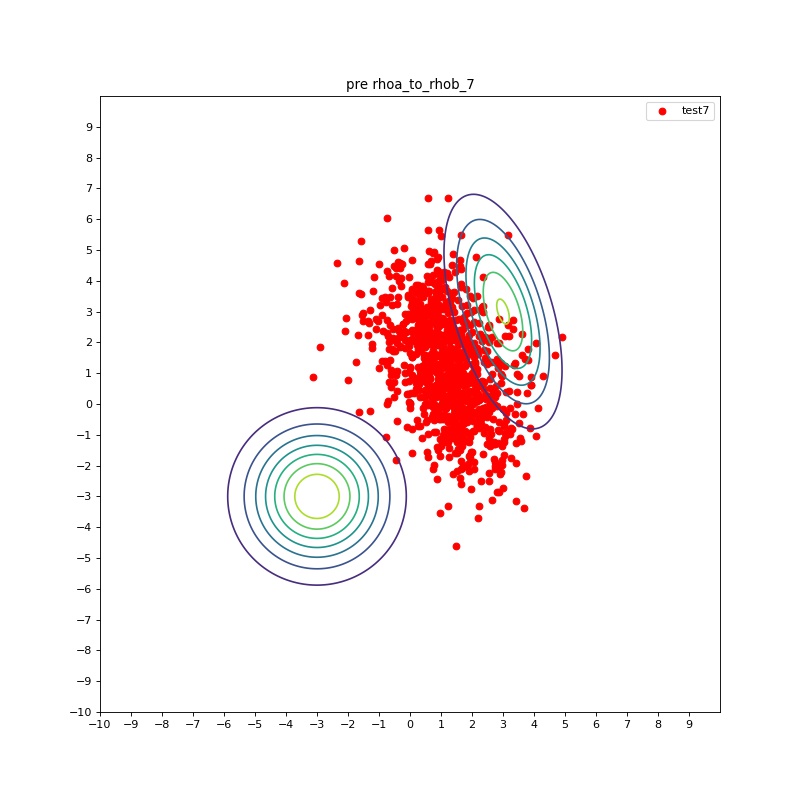}}
\subfloat[][$\rho_a $ to $\rho_b$ at $t_8$]{\includegraphics[width=.18\linewidth]{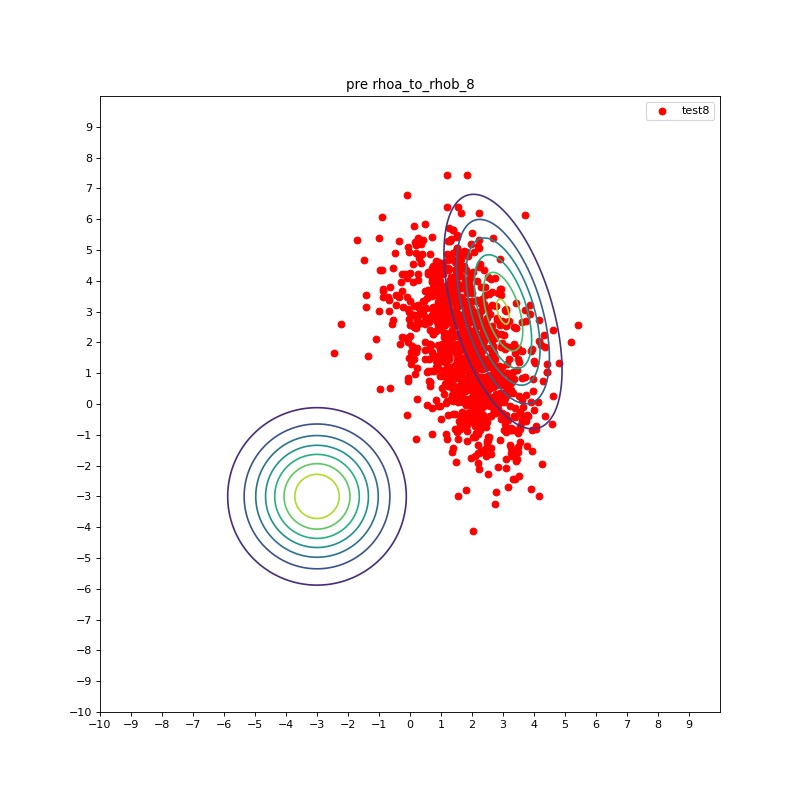}}
\subfloat[][$\rho_a $ to $\rho_b$ at $t_9$]{\includegraphics[width=.18\linewidth]{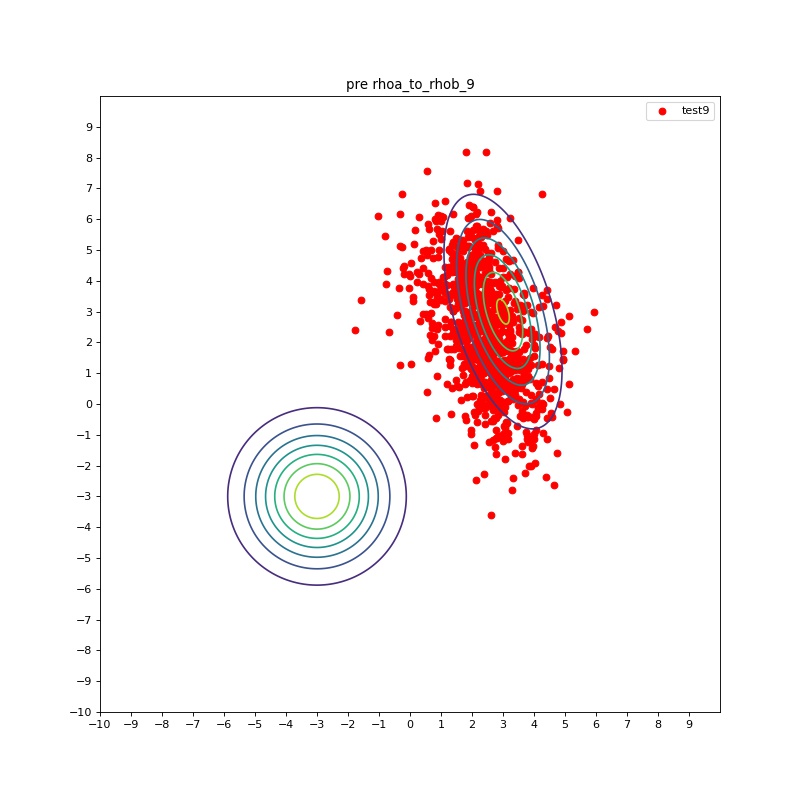}}
\subfloat[][$\rho_a $ to $\rho_b$ at $t_{10}$]{\includegraphics[width=.18\linewidth]{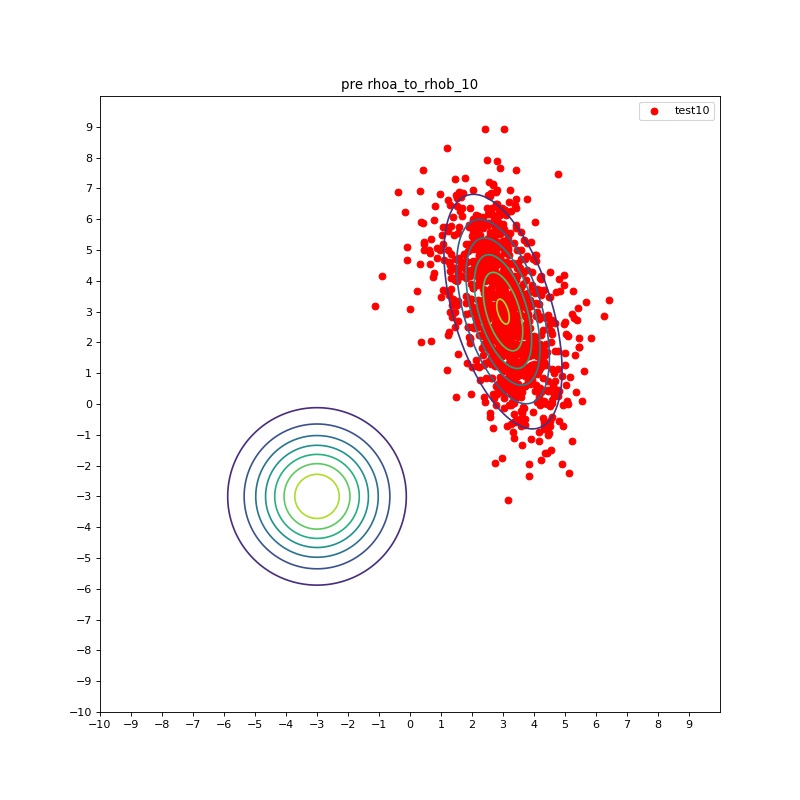}}\\
\subfloat[][$\rho_b $ to $\rho_a$ at $t_1$]{\includegraphics[width=.18\linewidth]{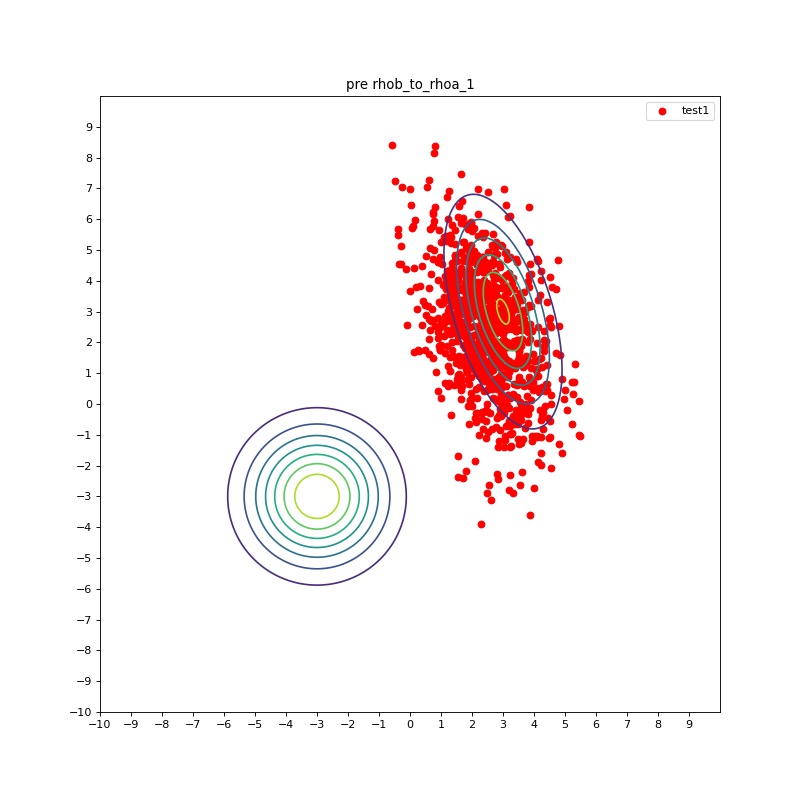}}
\subfloat[][$\rho_b $ to $\rho_a$ at $t_2$]{\includegraphics[width=.18\linewidth]{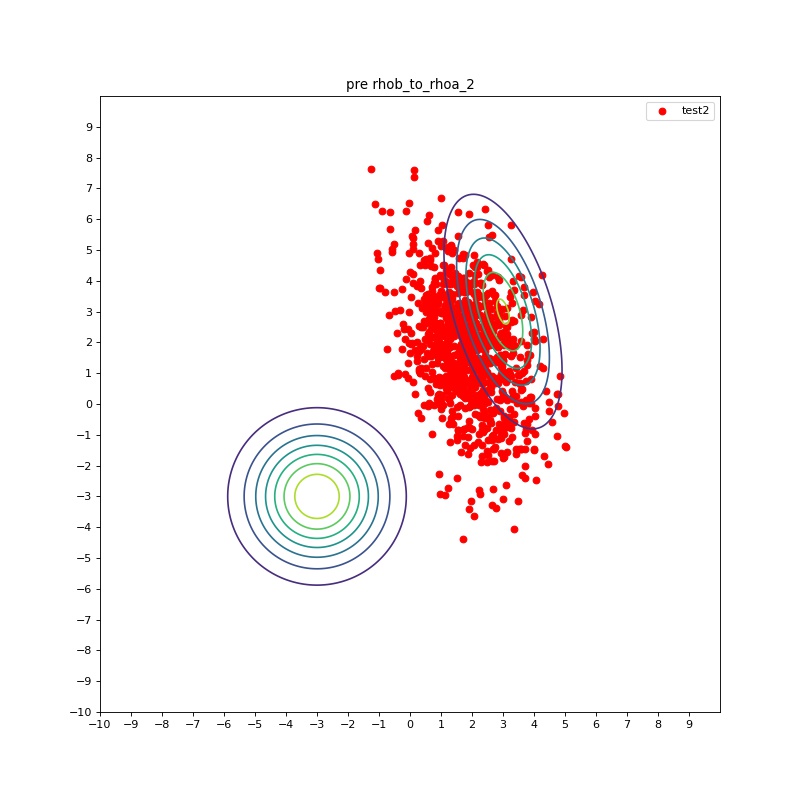}}
\subfloat[][$\rho_b $ to $\rho_a$ at $t_3$]{\includegraphics[width=.18\linewidth]{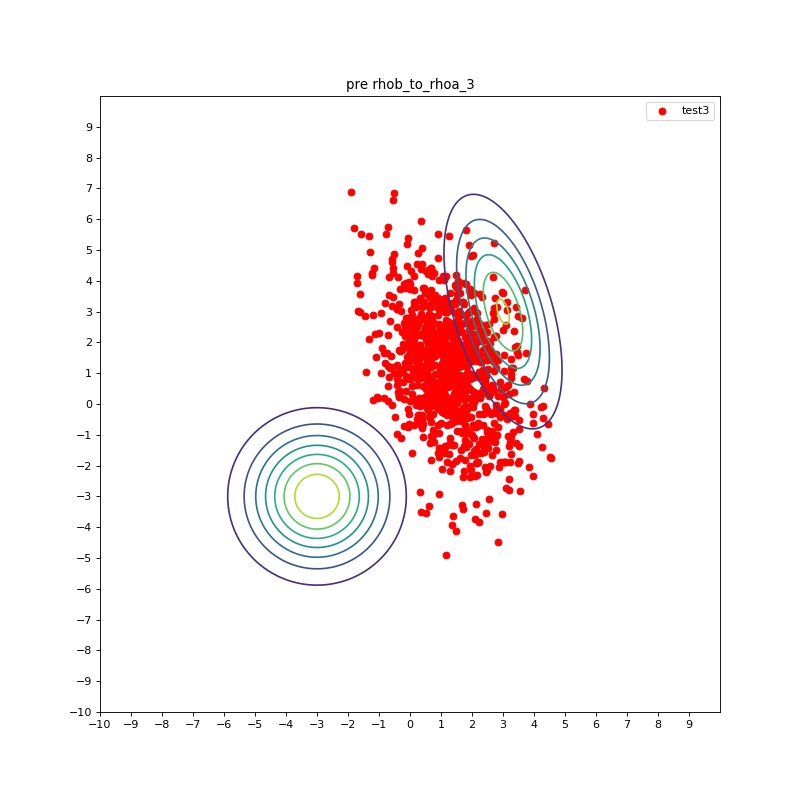}}
\subfloat[][$\rho_b $ to $\rho_a$ at $t_4$]{\includegraphics[width=.18\linewidth]{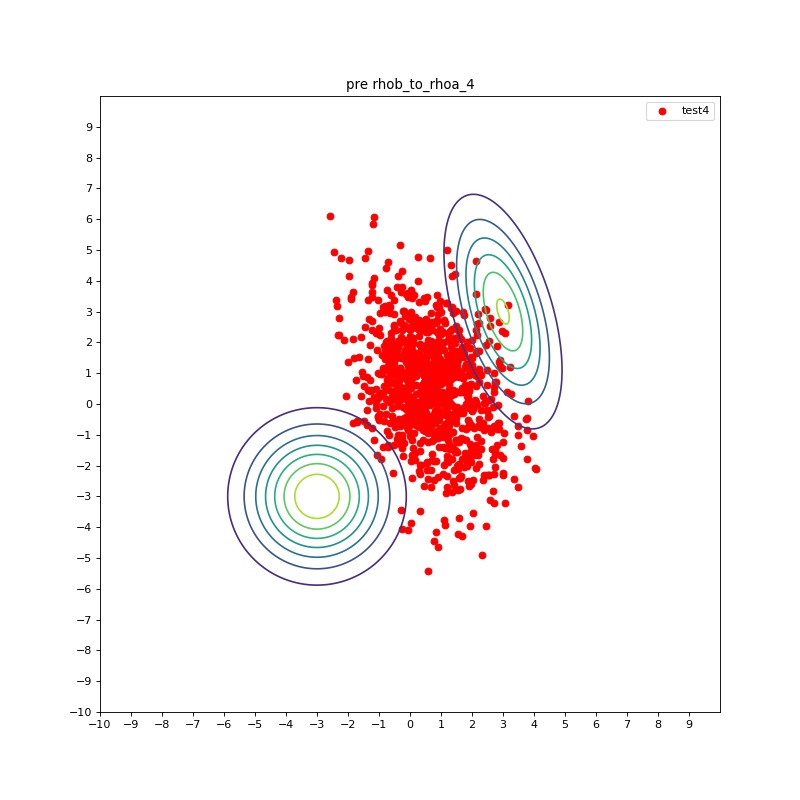}}
\subfloat[][$\rho_b $ to $\rho_a$ at $t_5$]{\includegraphics[width=.18\linewidth]{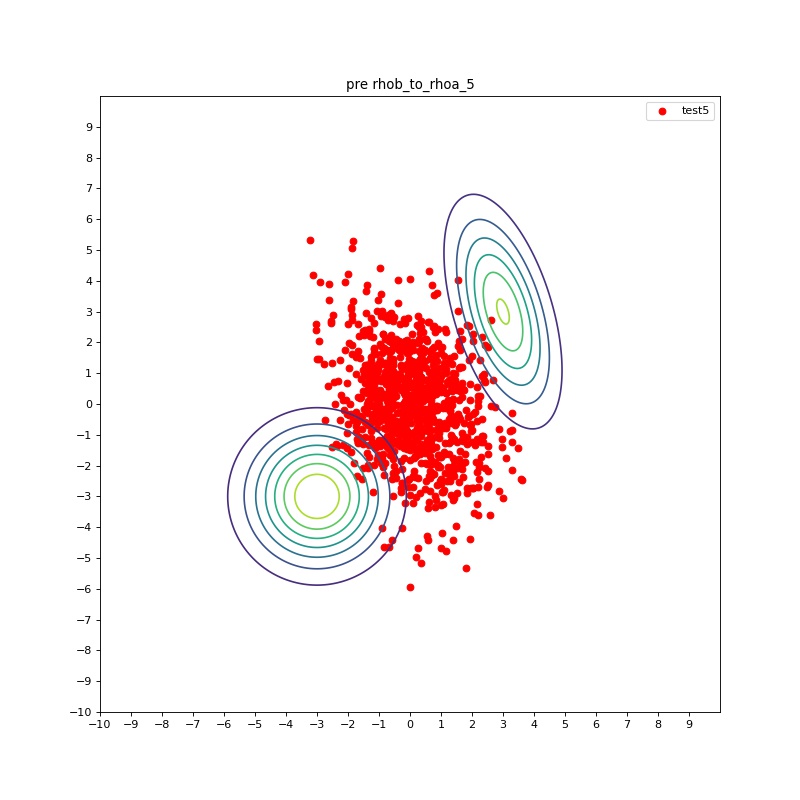}}\\
\subfloat[][$\rho_b $ to $\rho_a$ at $t_6$]{\includegraphics[width=.18\linewidth]{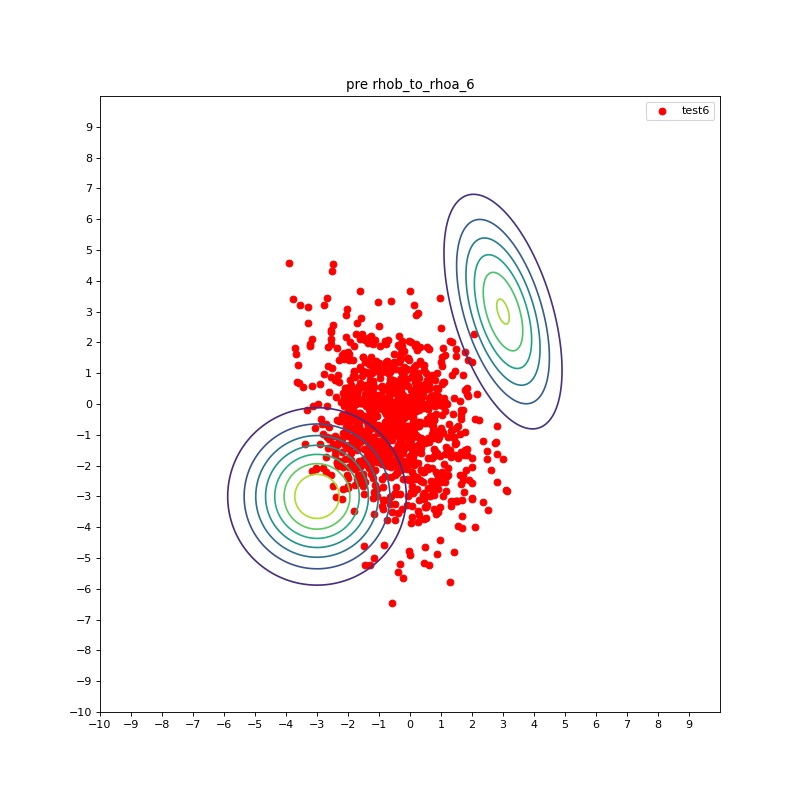}}
\subfloat[][$\rho_b $ to $\rho_a$ at $t_7$]{\includegraphics[width=.18\linewidth]{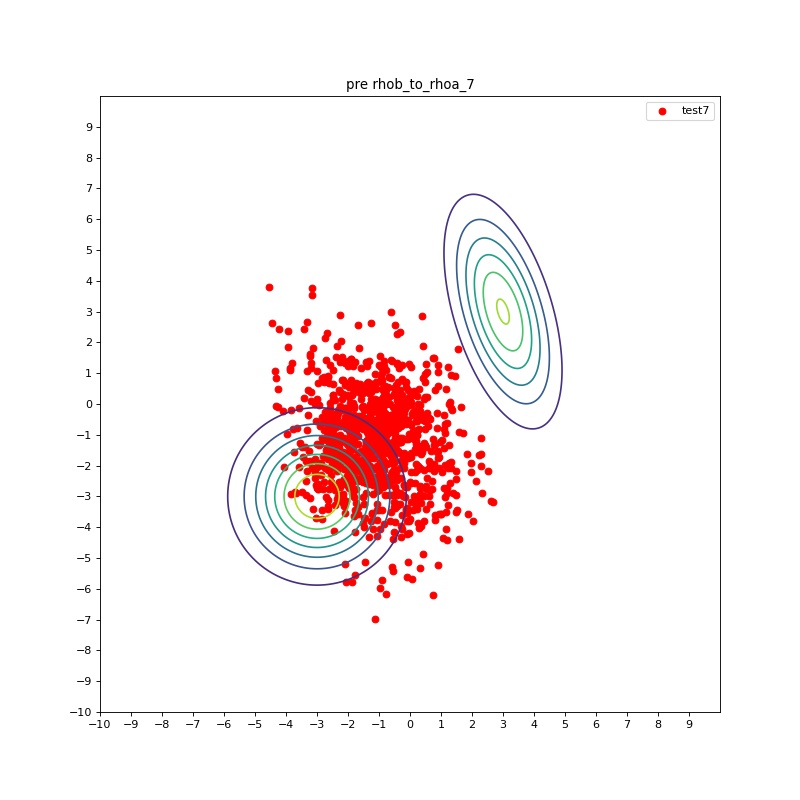}}
\subfloat[][$\rho_b $ to $\rho_a$ at $t_8$]{\includegraphics[width=.18\linewidth]{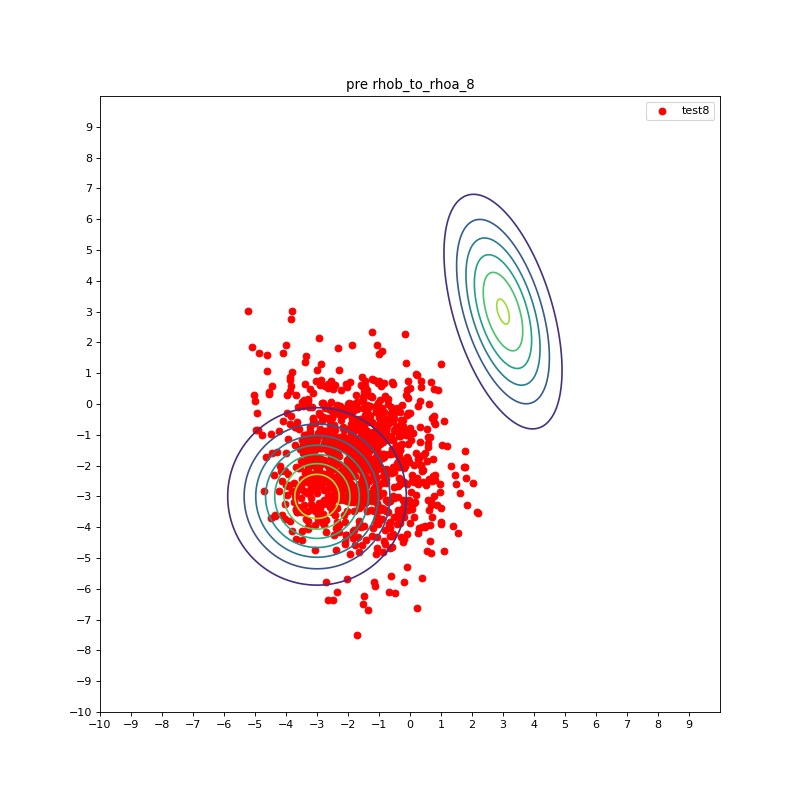}}
\subfloat[][$\rho_b $ to $\rho_a$ at $t_9$]{\includegraphics[width=.18\linewidth]{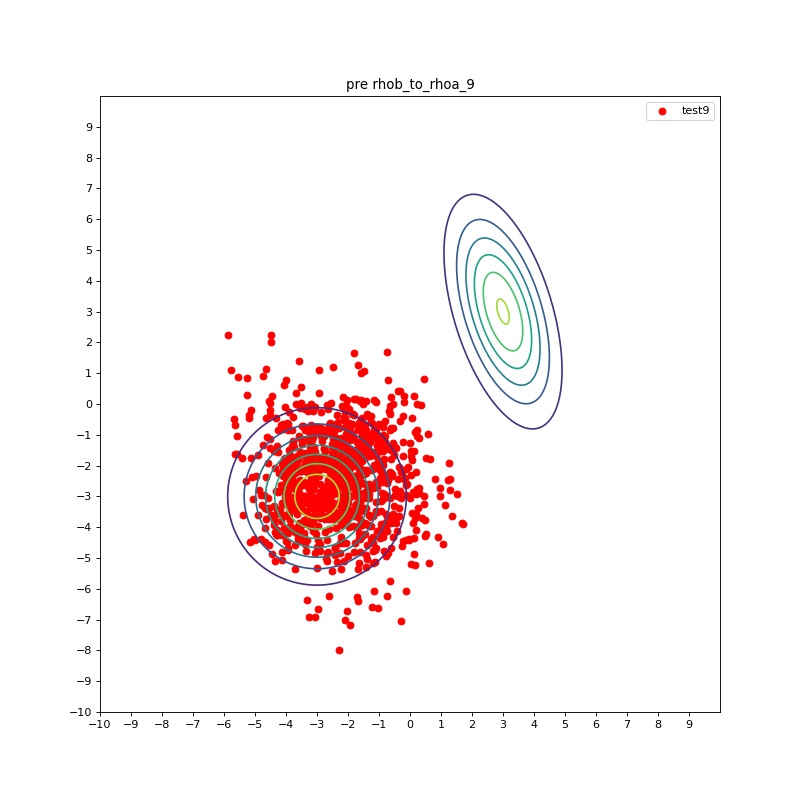}}
\subfloat[][$\rho_b $ to $\rho_a$ at $t_{10}$]{\includegraphics[width=.18\linewidth]{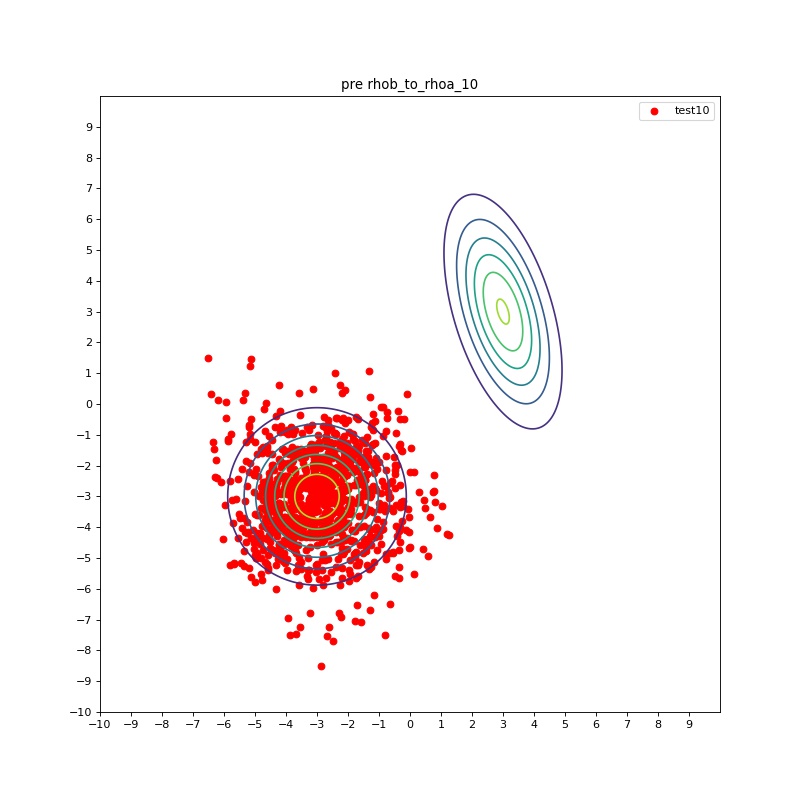}}\\
\subfloat[][points track:$\rho_a$ to $\rho_b$]{\includegraphics[width=.18\linewidth]{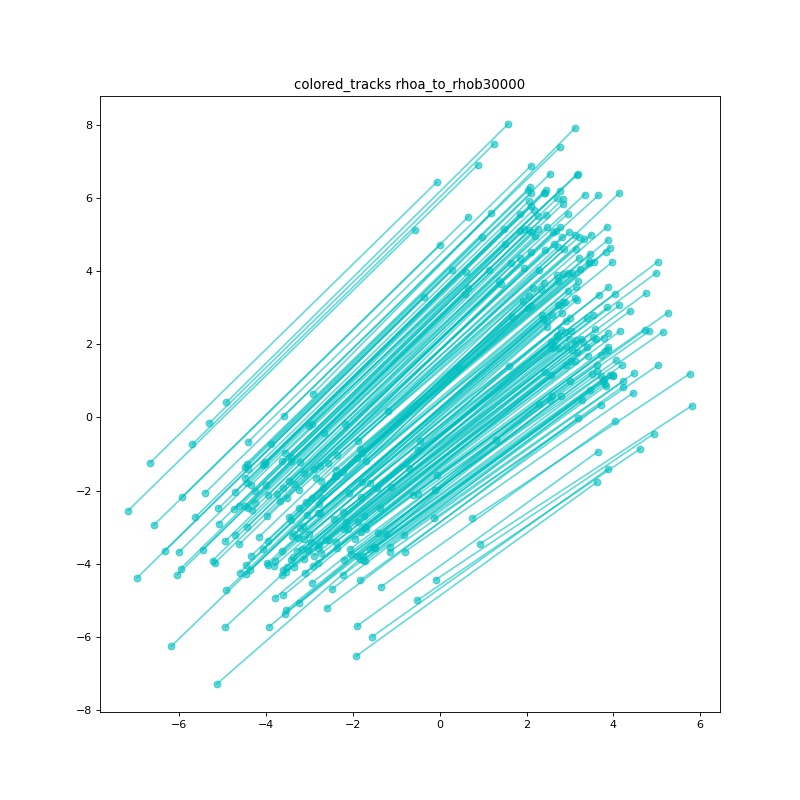}}
\subfloat[][vector field:$\rho_a$ to $\rho_b$]{\includegraphics[width=.18\linewidth]{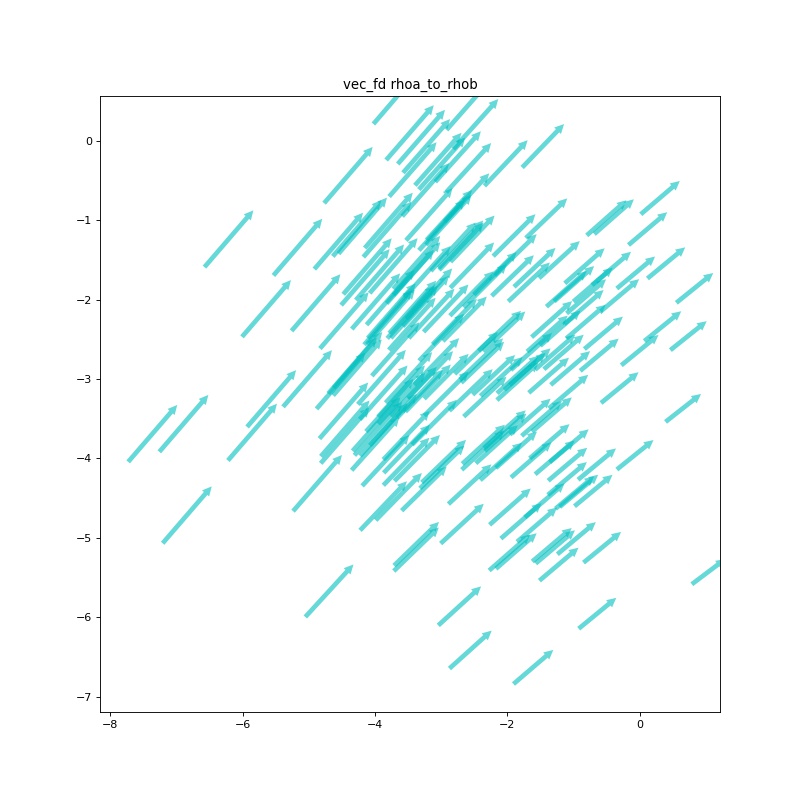}}
\subfloat[][points track:$\rho_b$ to $\rho_a$]{\includegraphics[width=.18\linewidth]{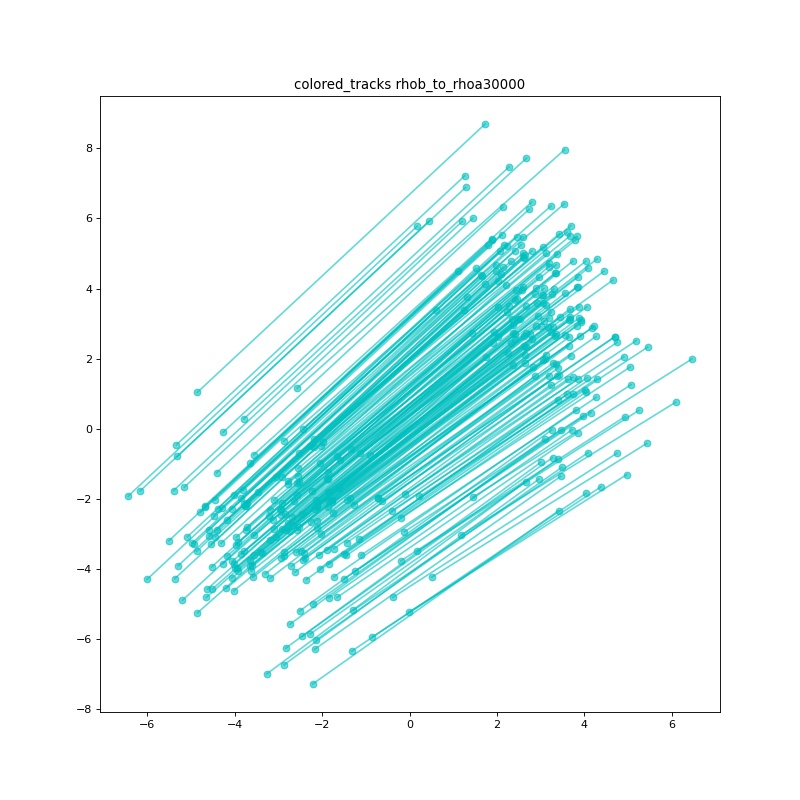}}
\subfloat[][vector field:$\rho_b$ to $\rho_a$]{\includegraphics[width=.18\linewidth]{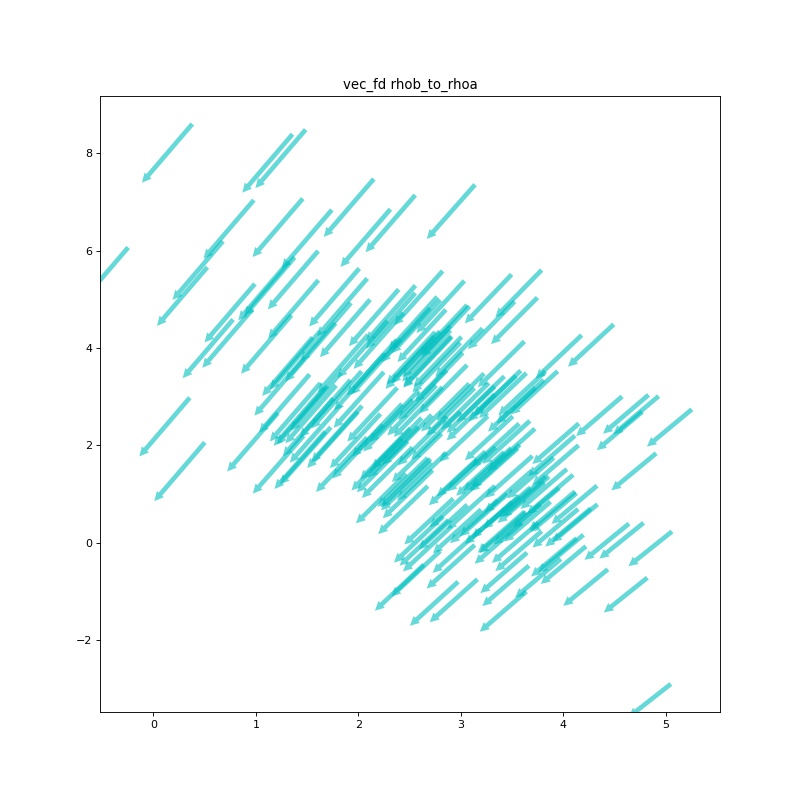}}
\caption{5-dimensional Gaussian to Gaussian}
\label{fig:syn-22}
\end{figure*}

\newpage
\textbf{Syn-3:}
\begin{figure}[ht!]
    \centering
     \subfloat[][$\rho_a $ to $\rho_b$ at $t_1$]{\includegraphics[width=.18\linewidth]{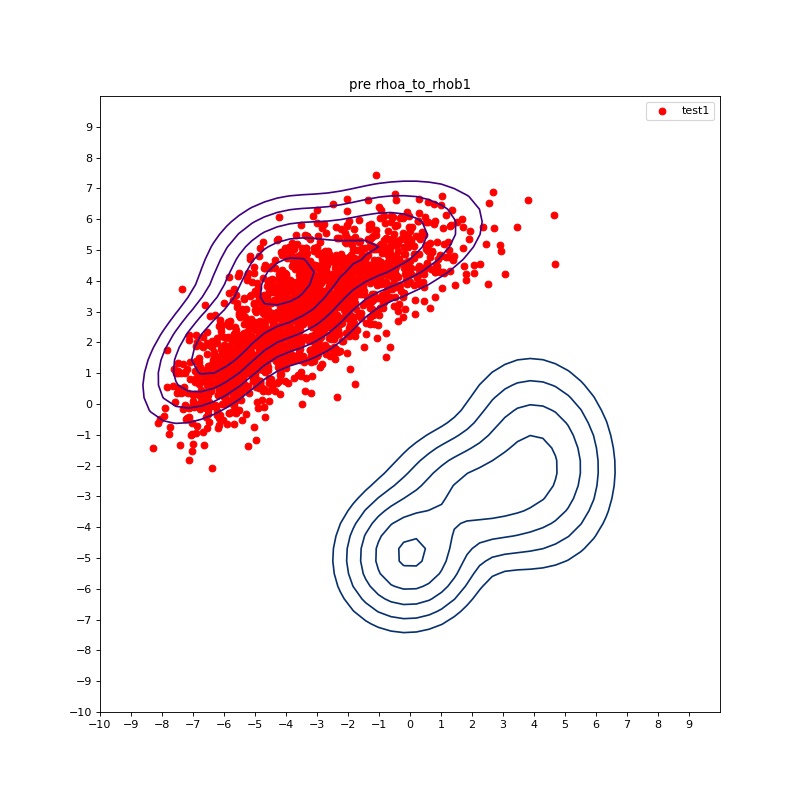}}
     \subfloat[][$\rho_a $ to $\rho_b$ at $t_2$]{\includegraphics[width=.18\linewidth]{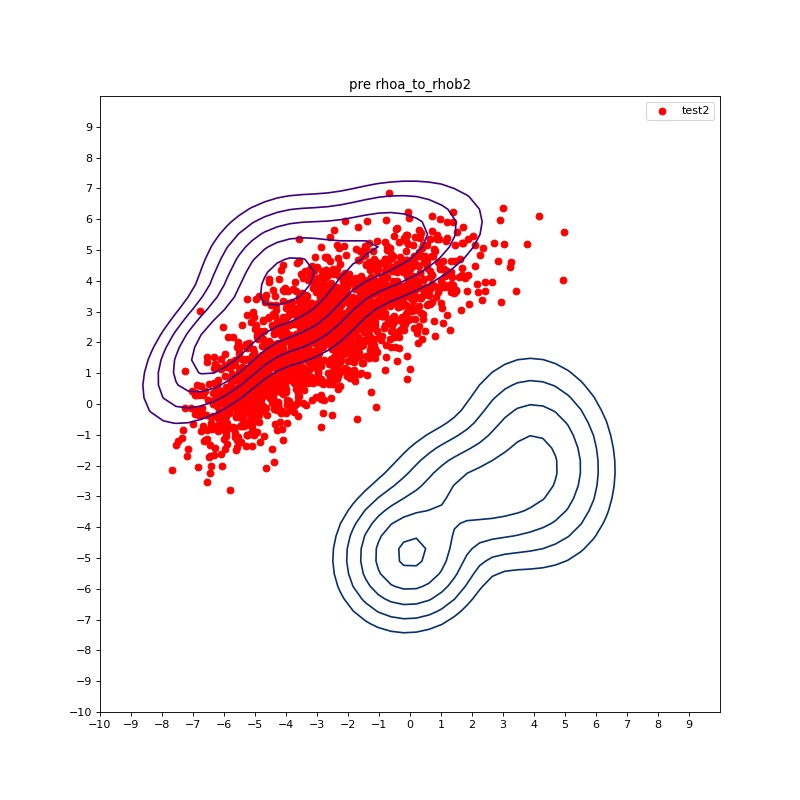}}
     \subfloat[][$\rho_a $ to $\rho_b$ at $t_3$]{\includegraphics[width=.18\linewidth]{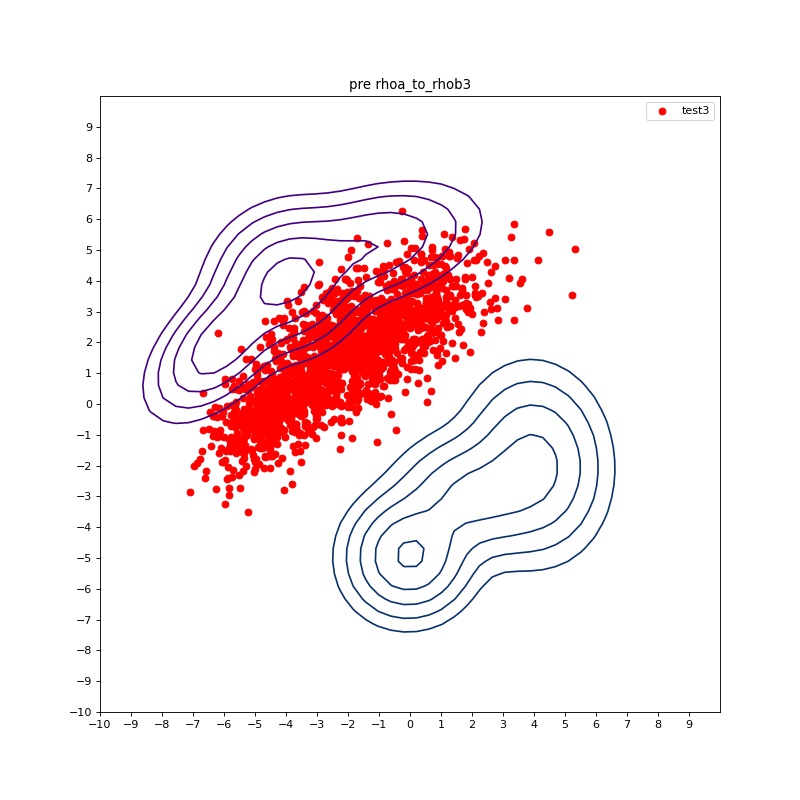}}
     \subfloat[][$\rho_a $ to $\rho_b$ at $t_4$]{\includegraphics[width=.18\linewidth]{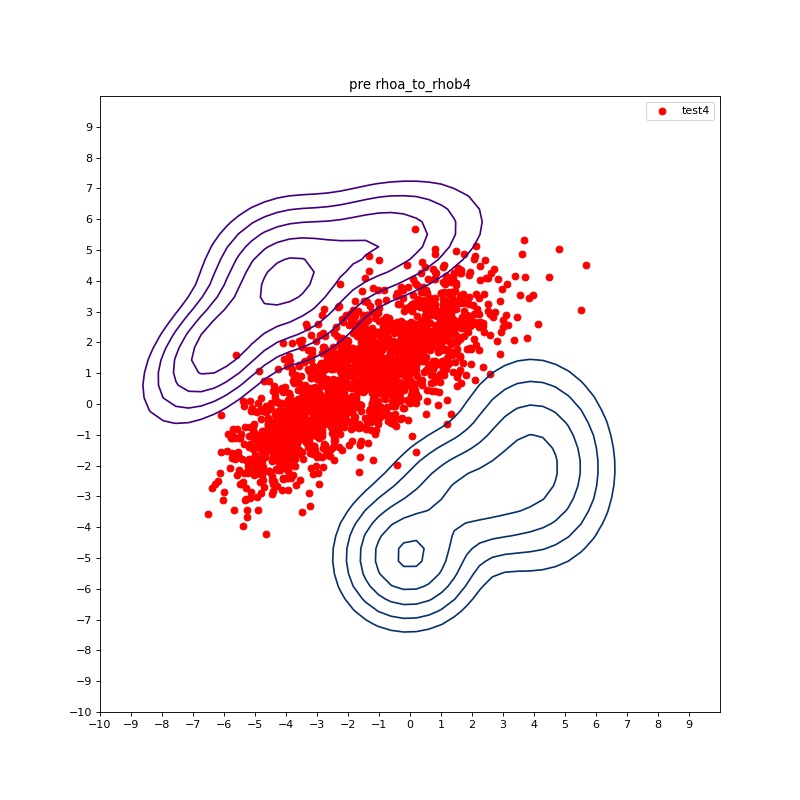}}
     \subfloat[][$\rho_a $ to $\rho_b$ at $t_5$]{\includegraphics[width=.18\linewidth]{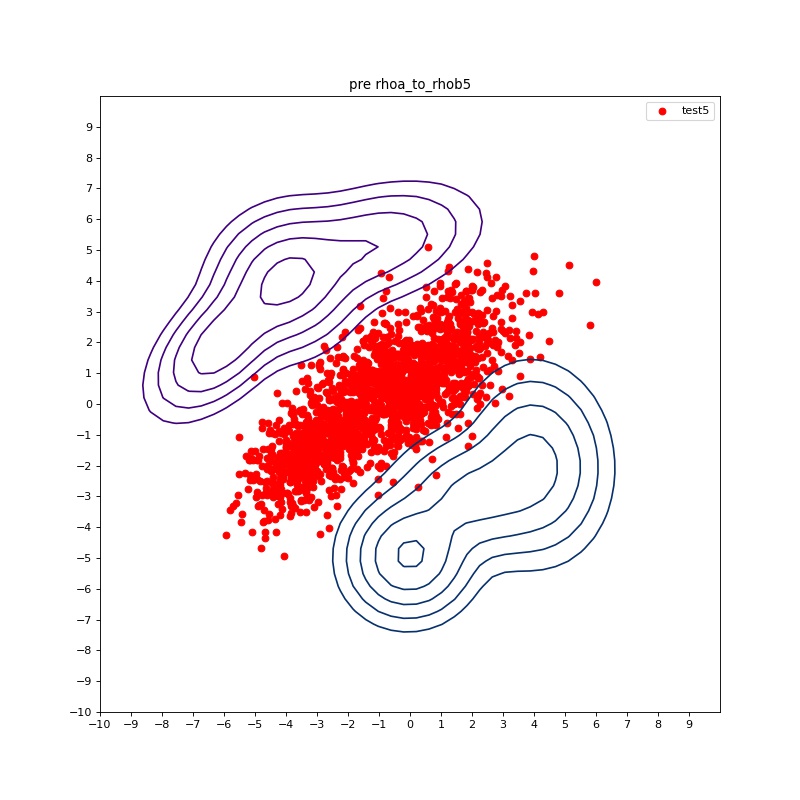}}\\
     \subfloat[][$\rho_a $ to $\rho_b$ at $t_6$]{\includegraphics[width=.18\linewidth]{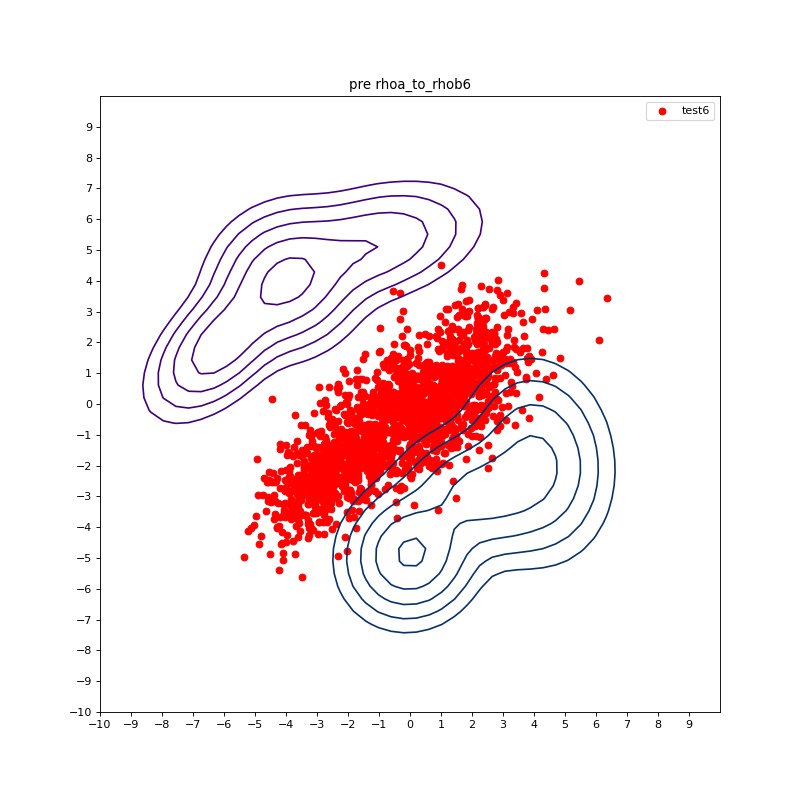}}
     \subfloat[][$\rho_a $ to $\rho_b$ at $t_7$]{\includegraphics[width=.18\linewidth]{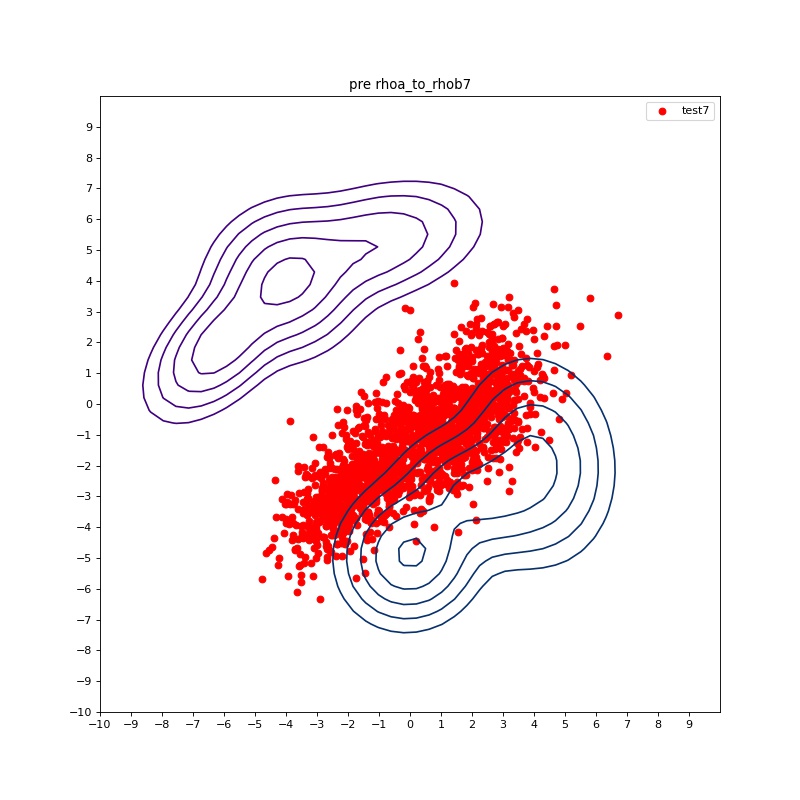}}
     \subfloat[][$\rho_a $ to $\rho_b$ at $t_8$]{\includegraphics[width=.18\linewidth]{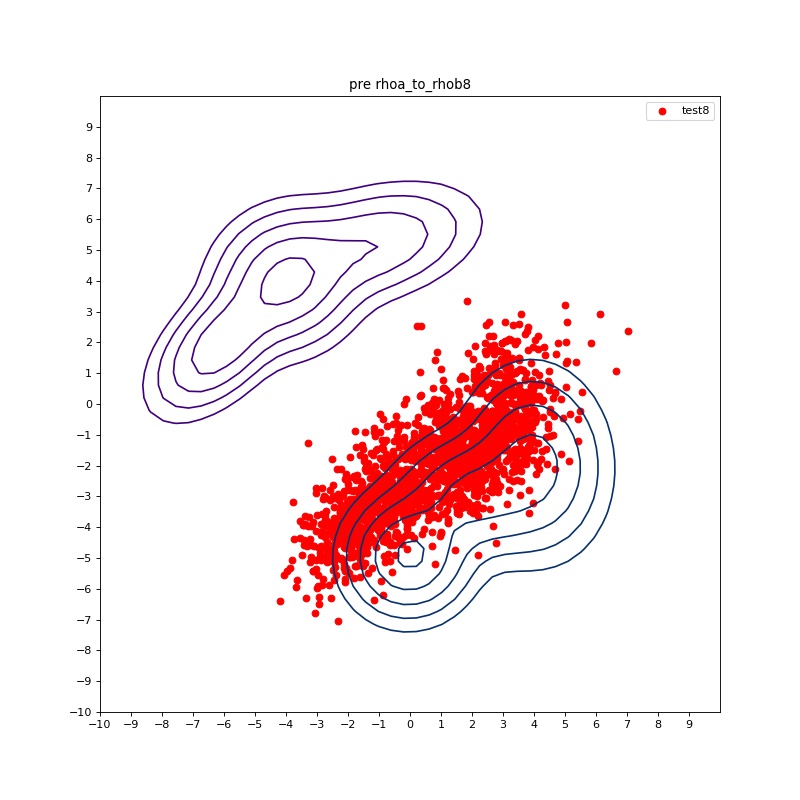}}
     \subfloat[][$\rho_a $ to $\rho_b$ at $t_9$]{\includegraphics[width=.18\linewidth]{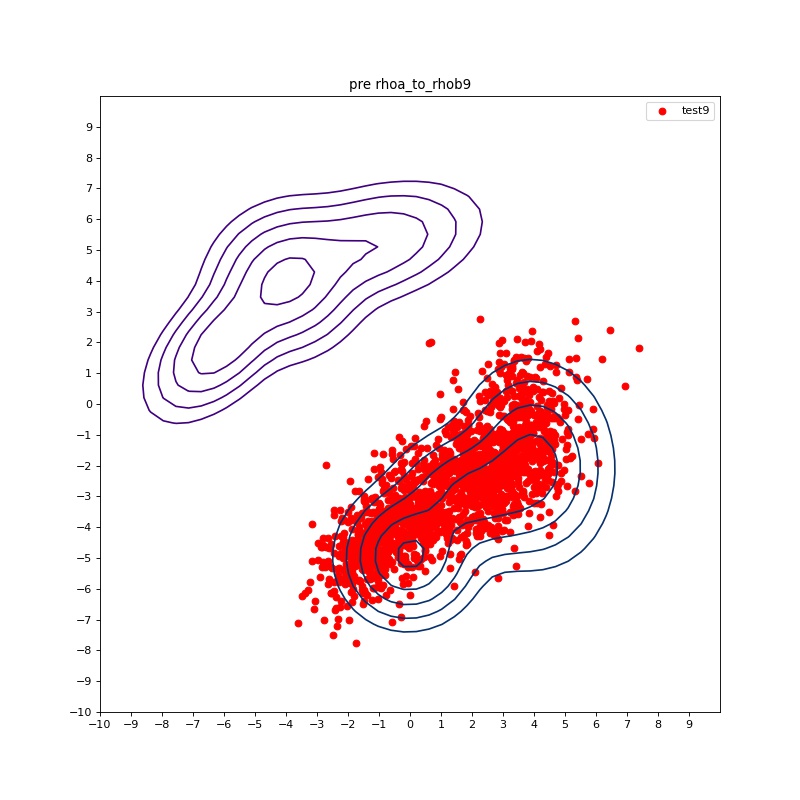}}
     \subfloat[][$\rho_a $ to $\rho_b$ at $t_{10}$]{\includegraphics[width=.18\linewidth]{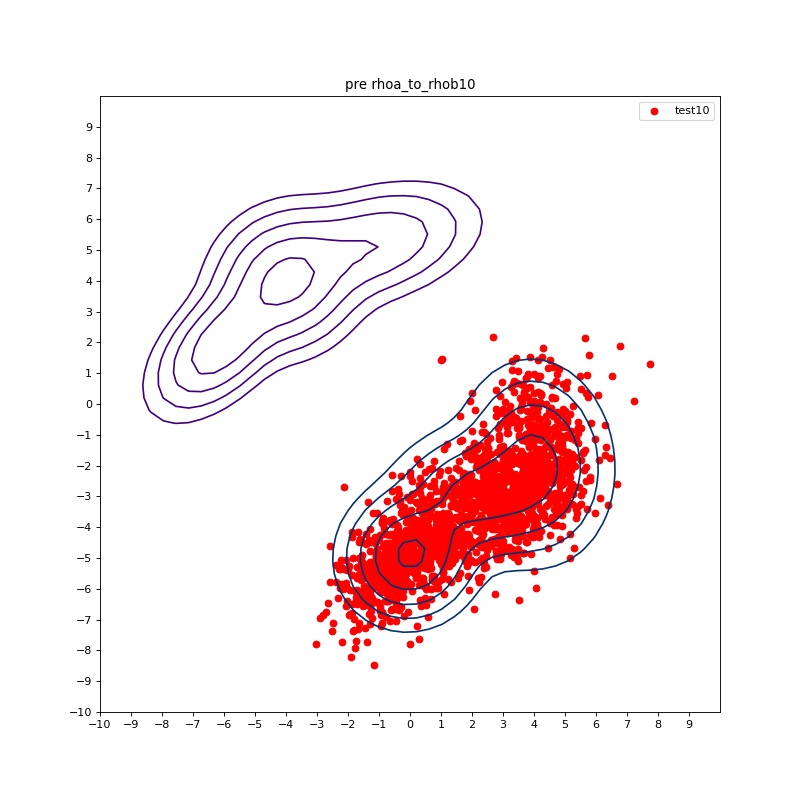}}\\
     \subfloat[][$\rho_b $ to $\rho_a$ at $t_1$]{\includegraphics[width=.18\linewidth]{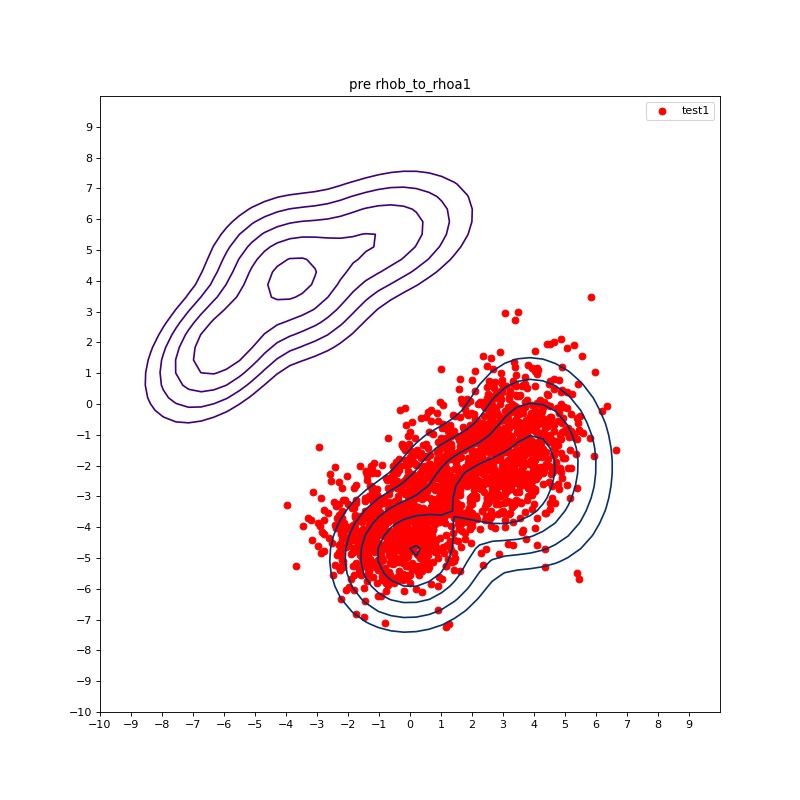}}
     \subfloat[][$\rho_b $ to $\rho_a$ at $t_2$]{\includegraphics[width=.18\linewidth]{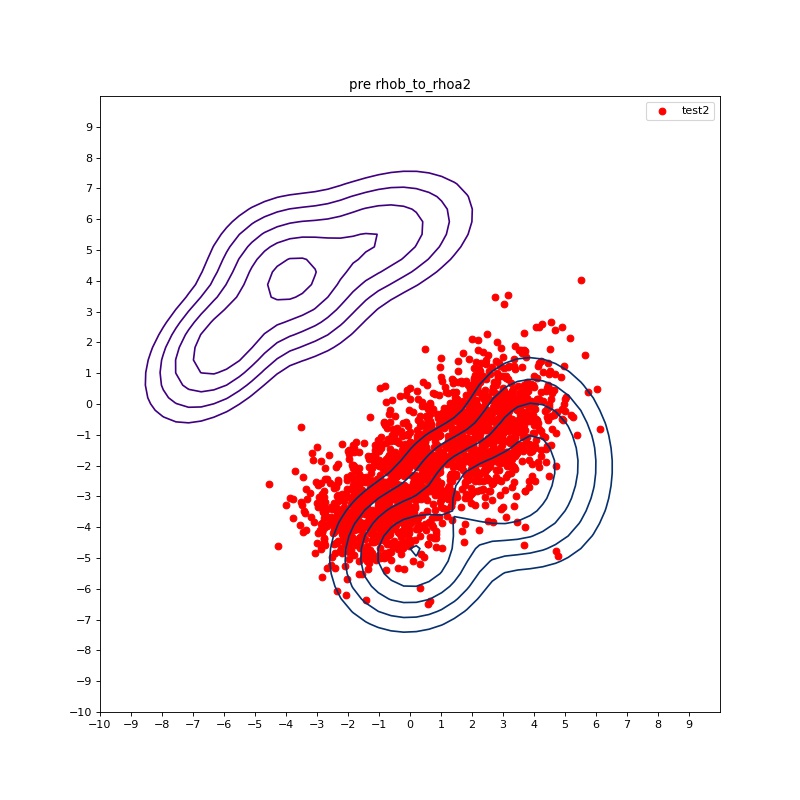}}
     \subfloat[][$\rho_b $ to $\rho_a$ at $t_3$]{\includegraphics[width=.18\linewidth]{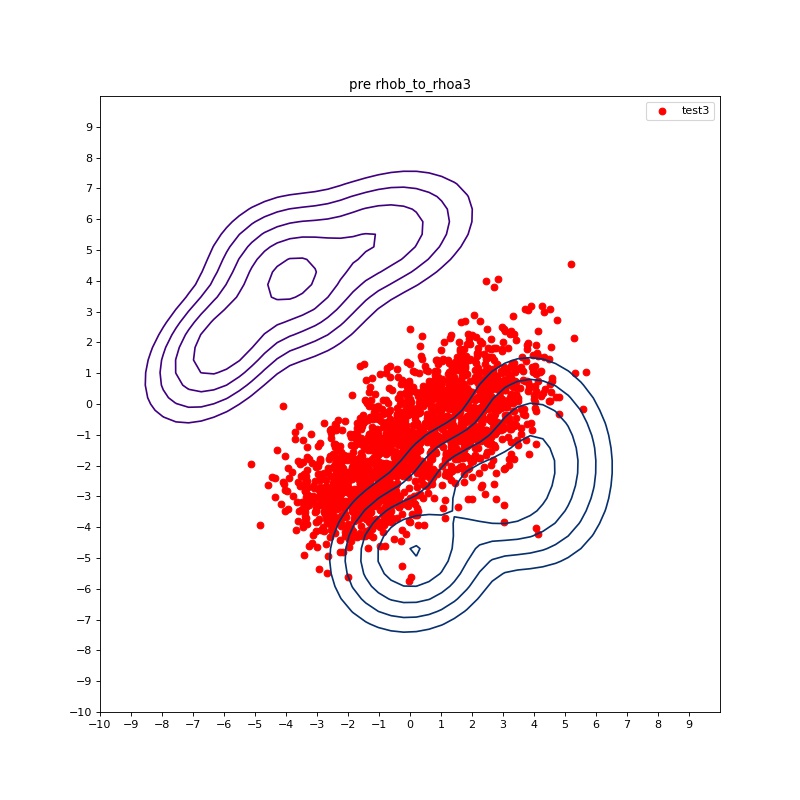}}
     \subfloat[][$\rho_b $ to $\rho_a$ at $t_4$]{\includegraphics[width=.18\linewidth]{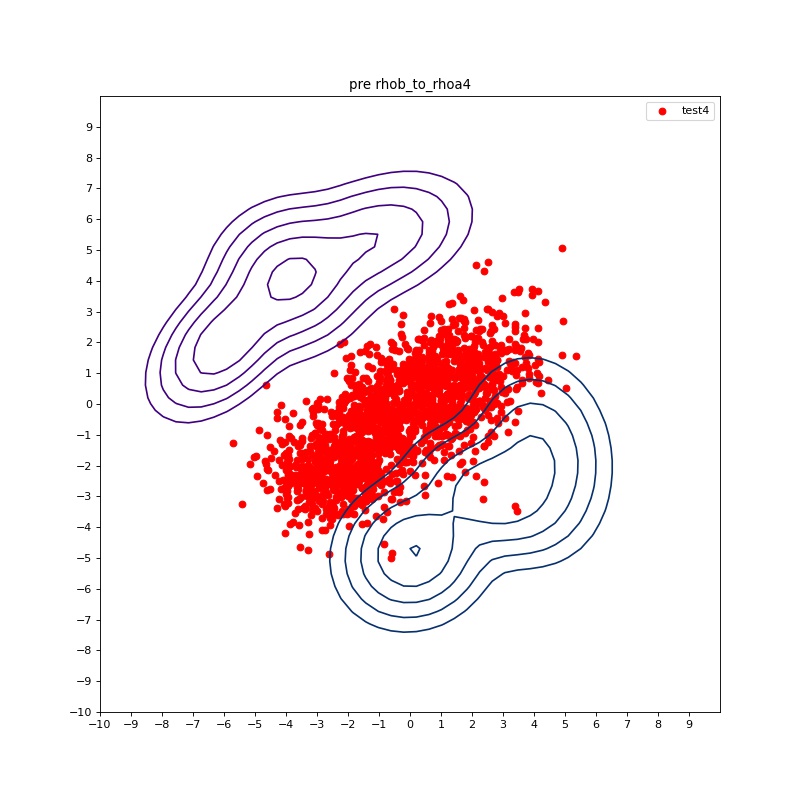}}
     \subfloat[][$\rho_b $ to $\rho_a$ at $t_5$]{\includegraphics[width=.18\linewidth]{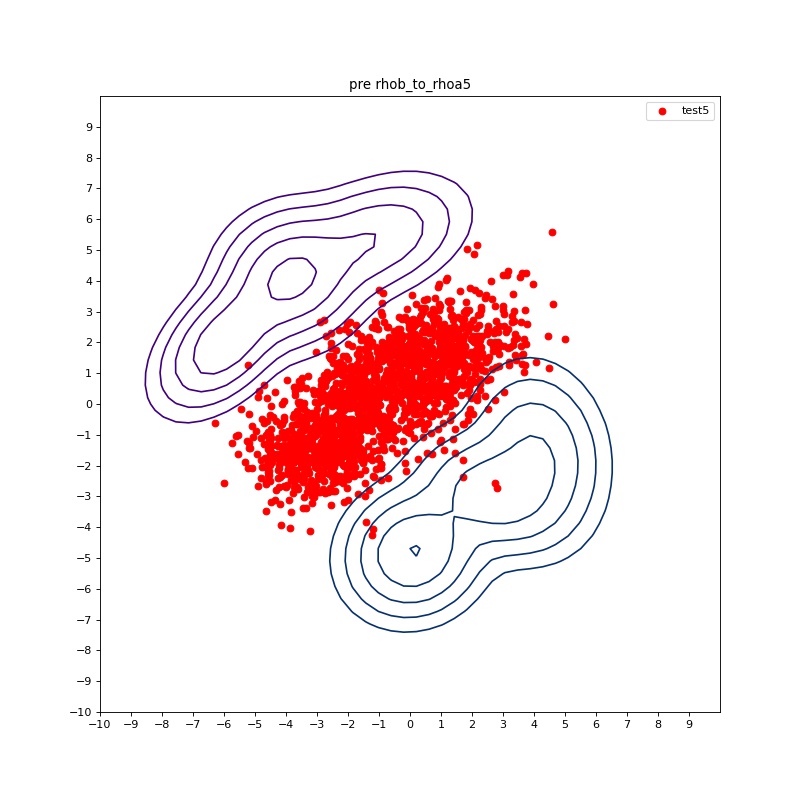}}\\
     \subfloat[][$\rho_b $ to $\rho_a$ at $t_6$]{\includegraphics[width=.18\linewidth]{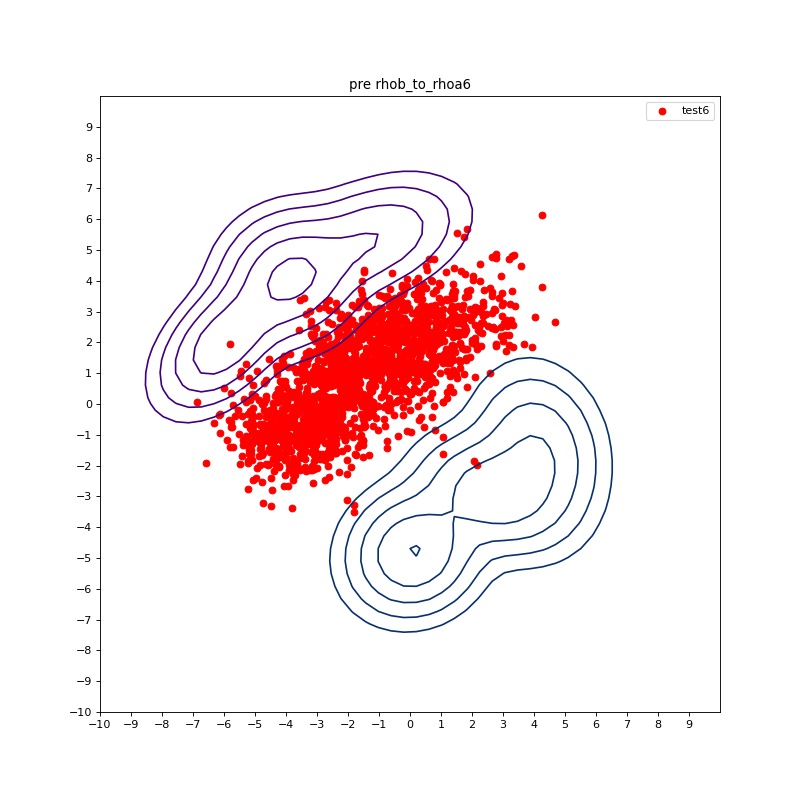}}
     \subfloat[][$\rho_b $ to $\rho_a$ at $t_7$]{\includegraphics[width=.18\linewidth]{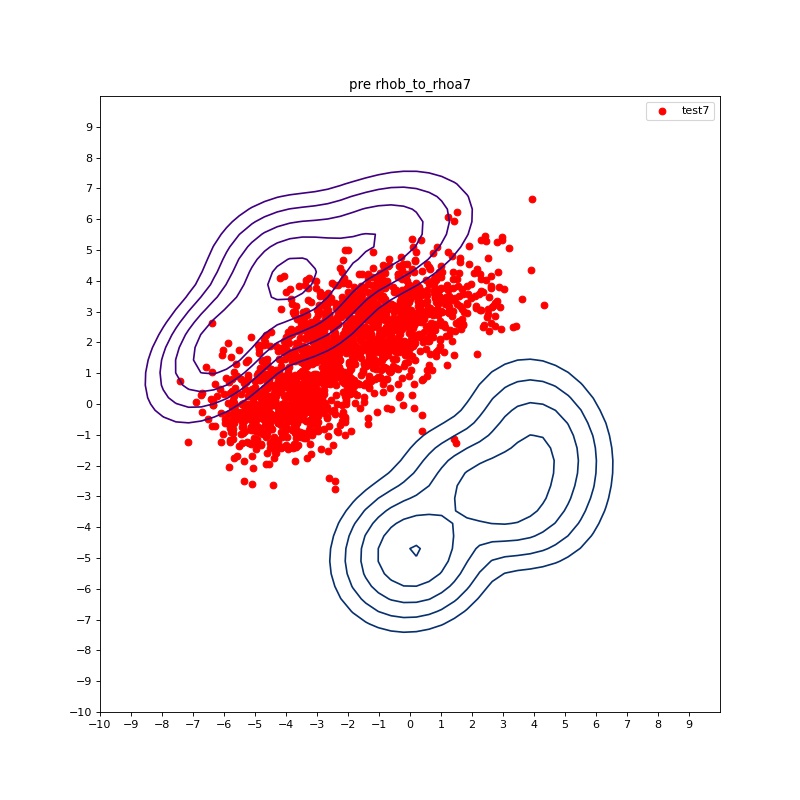}}
     \subfloat[][$\rho_b $ to $\rho_a$ at $t_8$]{\includegraphics[width=.18\linewidth]{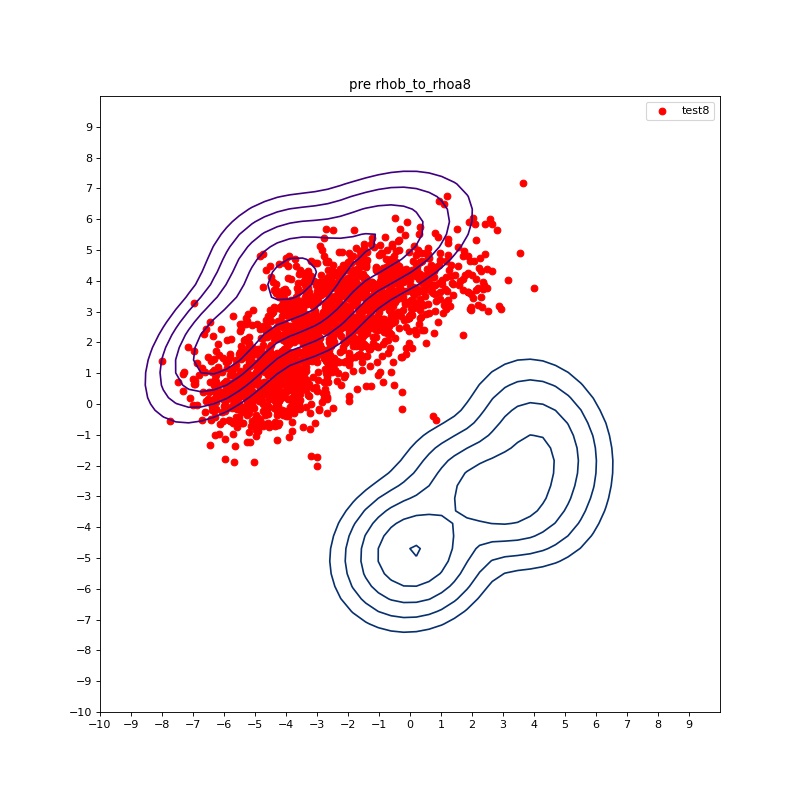}}
     \subfloat[][$\rho_b $ to $\rho_a$ at $t_9$]{\includegraphics[width=.18\linewidth]{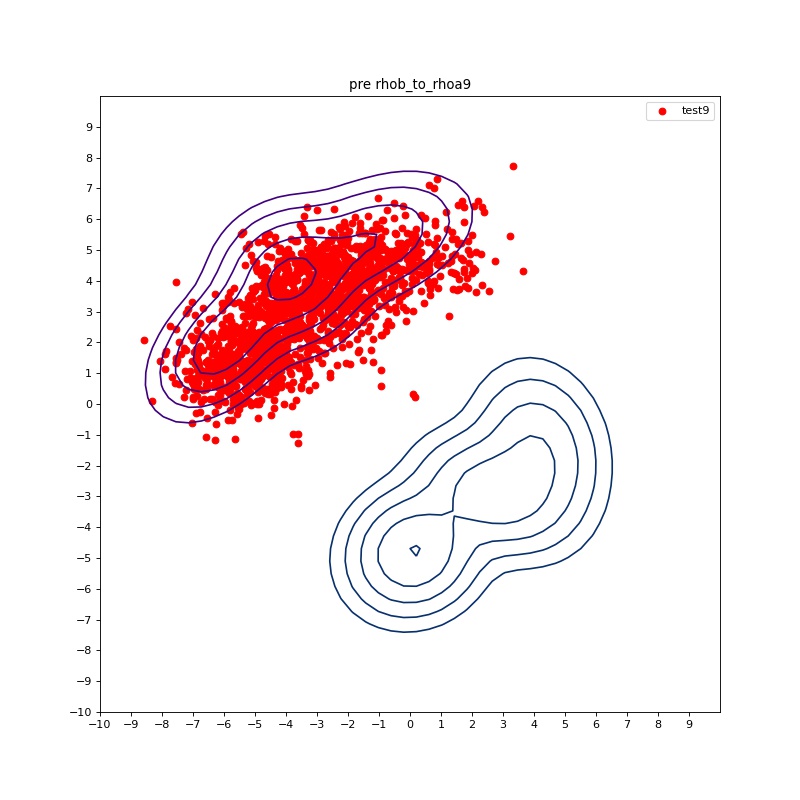}}
     \subfloat[][$\rho_b $ to $\rho_a$ at $t_{10}$]{\includegraphics[width=.18\linewidth]{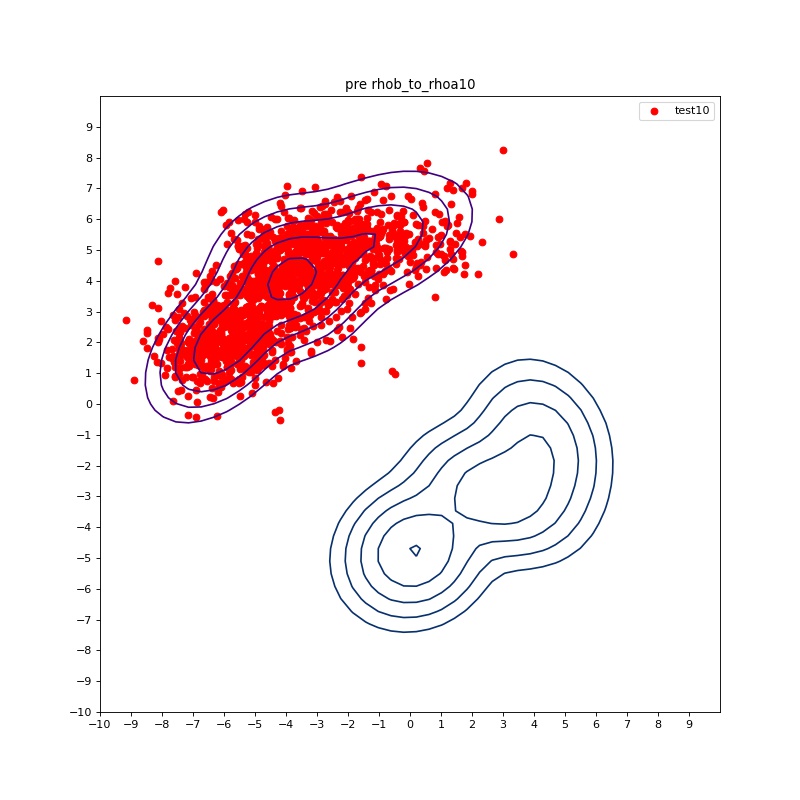}}
\caption{5-dimensional Gaussian Mixture}
\label{fig:syn-33}
\end{figure}


\newpage
\textbf{Syn-4:}
\begin{figure*}[ht!]
\centering
 \subfloat[][$\rho_a $ to $\rho_b$ at $t_1$]{\includegraphics[width=0.18\textwidth,height=0.18\textheight,keepaspectratio]{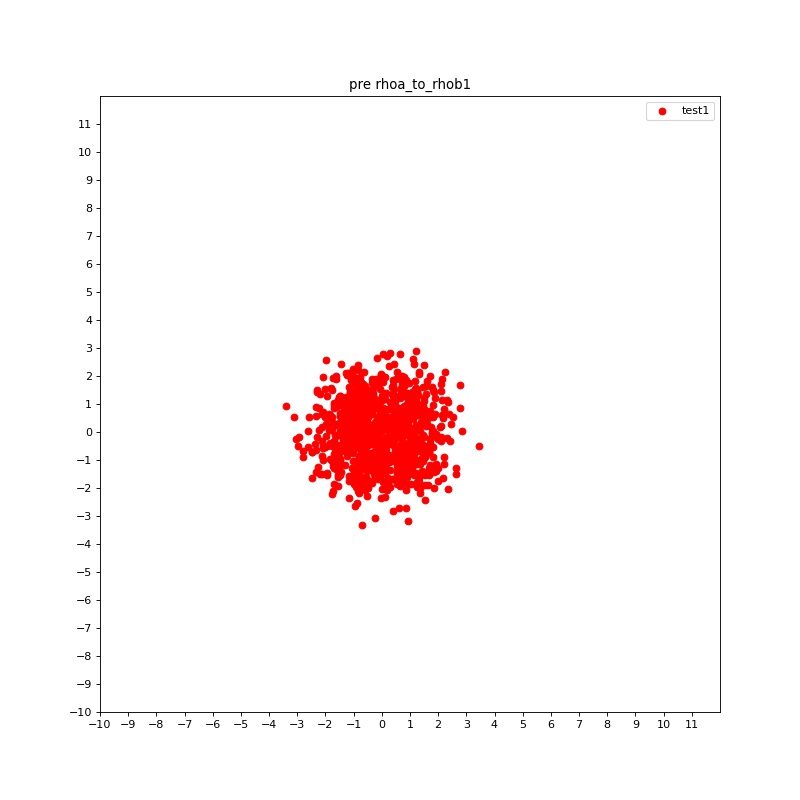}}
 \subfloat[][$\rho_a $ to $\rho_b$ at $t_2$]{\includegraphics[width=.18\linewidth]{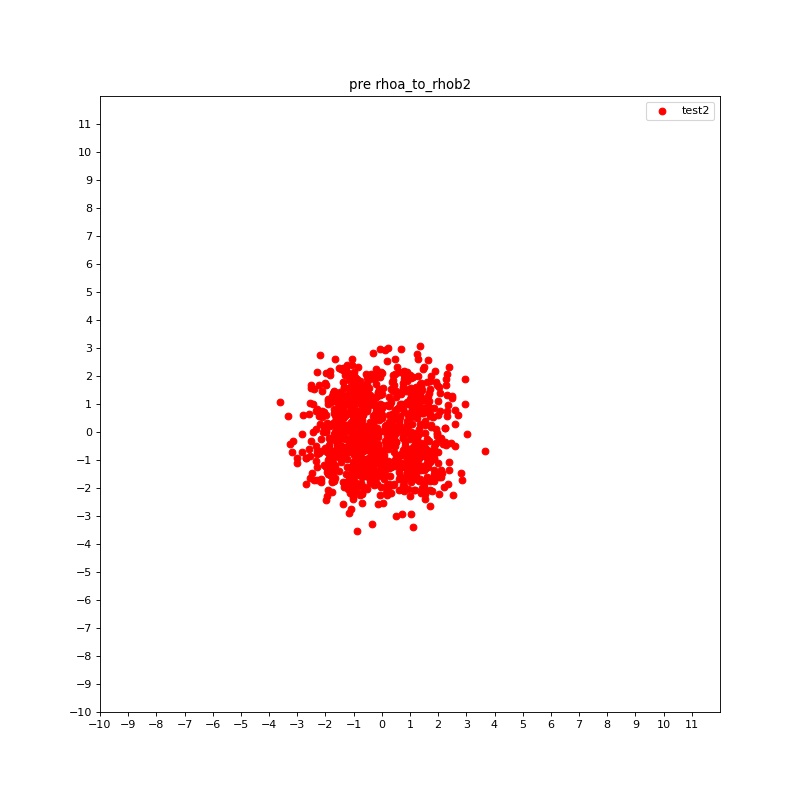}}
 \subfloat[][$\rho_a $ to $\rho_b$ at $t_3$]{\includegraphics[width=.18\linewidth]{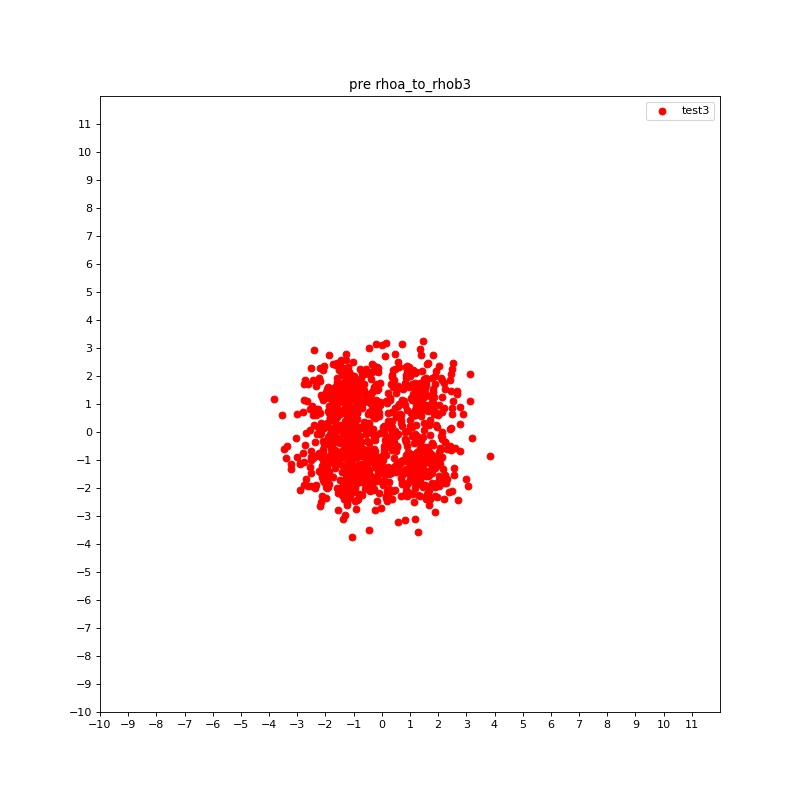}}
 \subfloat[][$\rho_a $ to $\rho_b$ at $t_4$]{\includegraphics[width=.18\linewidth]{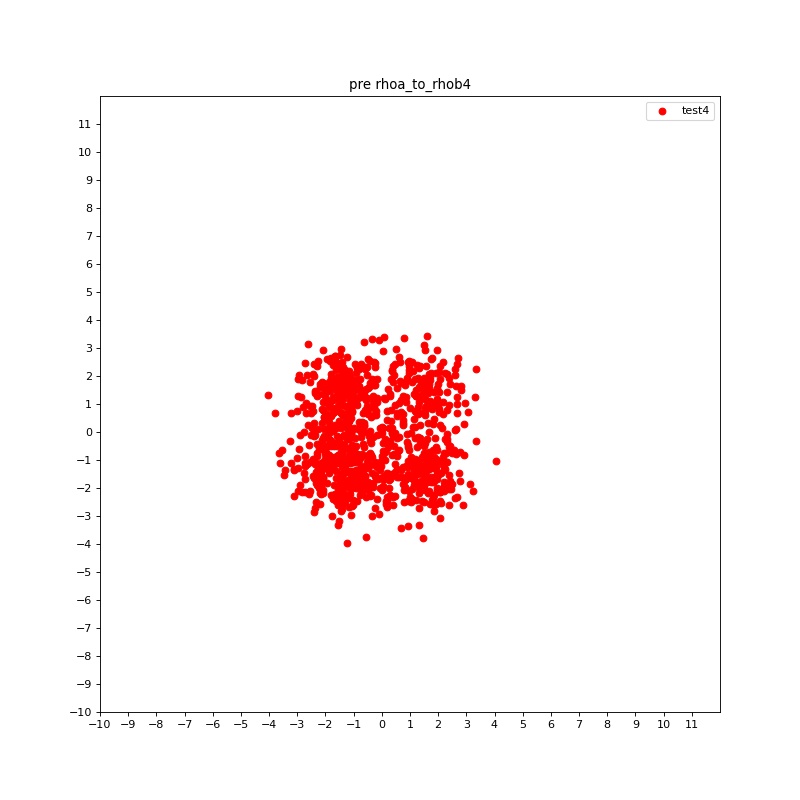}}
 \subfloat[][$\rho_a $ to $\rho_b$ at $t_5$]{\includegraphics[width=.18\linewidth]{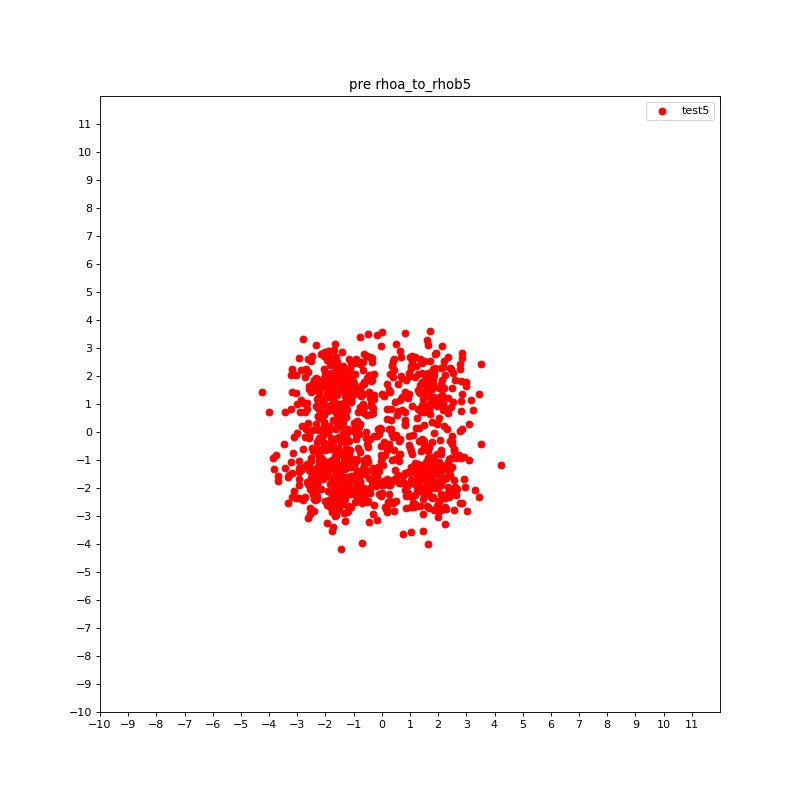}}\\
 \subfloat[][$\rho_a $ to $\rho_b$ at $t_6$]{\includegraphics[width=.18\linewidth]{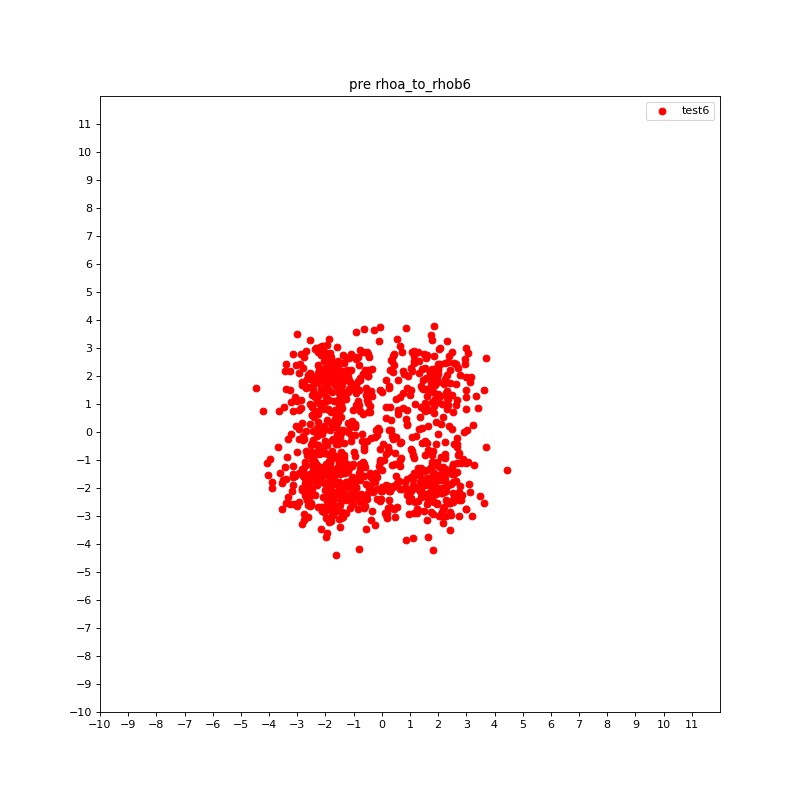}}
 \subfloat[][$\rho_a $ to $\rho_b$ at $t_7$]{\includegraphics[width=.18\linewidth]{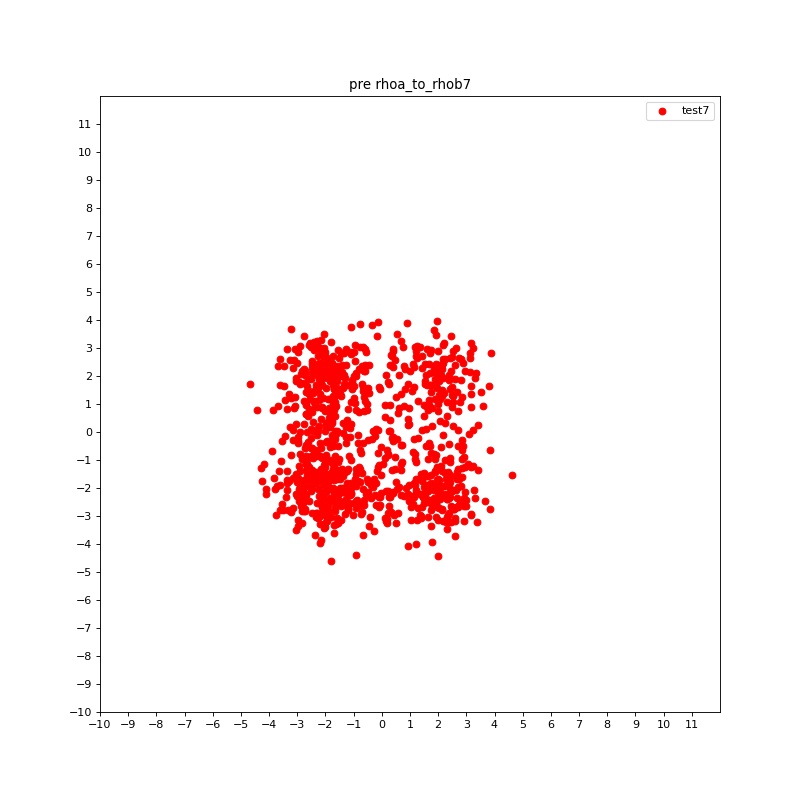}}
 \subfloat[][$\rho_a $ to $\rho_b$ at $t_8$]{\includegraphics[width=.18\linewidth]{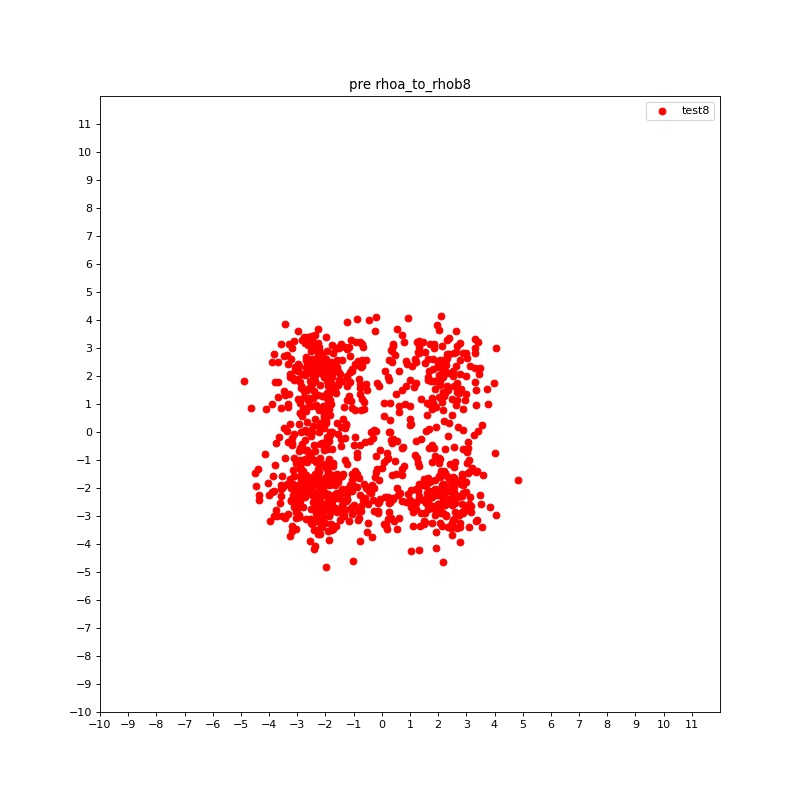}}
 \subfloat[][$\rho_a $ to $\rho_b$ at $t_9$]{\includegraphics[width=.18\linewidth]{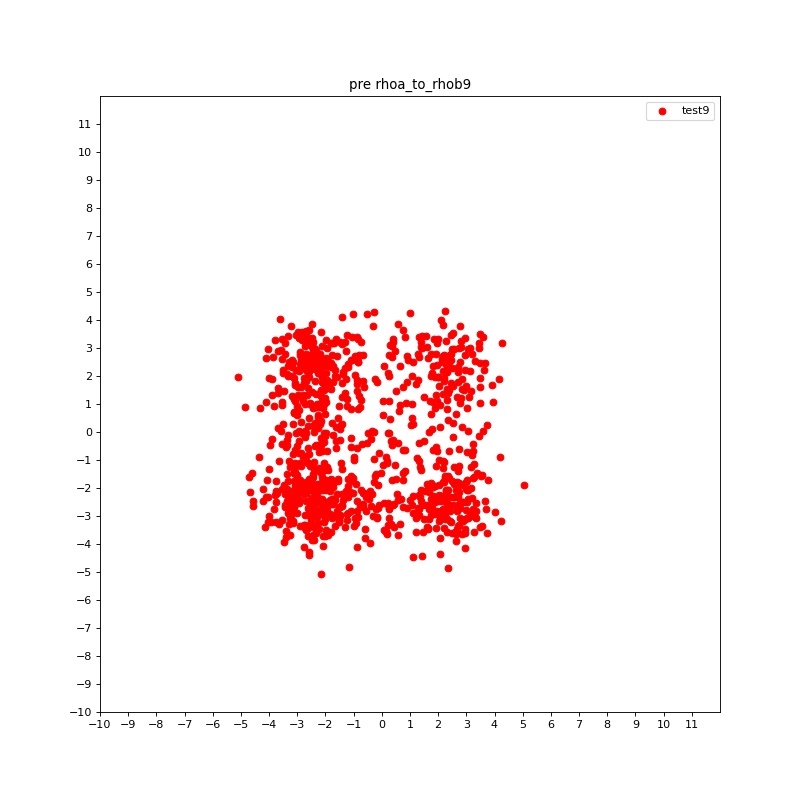}}
 \subfloat[][$\rho_a $ to $\rho_b$ at $t_{10}$]{\includegraphics[width=.18\linewidth]{pre_rhoa_to_rhob2550010.jpg}}\\
 \subfloat[][$\rho_b $ to $\rho_a$ at $t_1$]{\includegraphics[width=.18\linewidth]{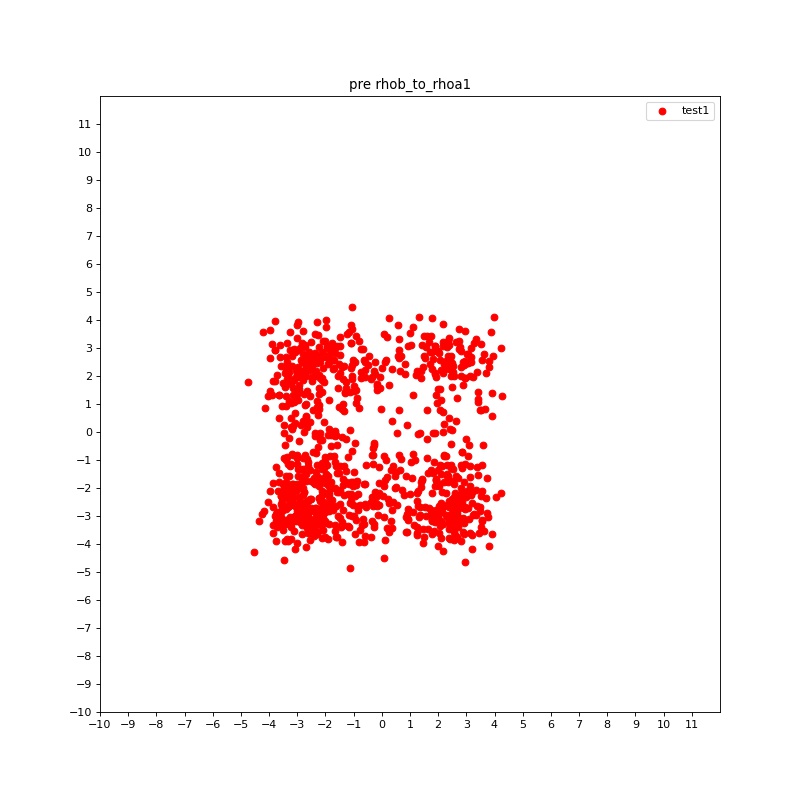}}
 \subfloat[][$\rho_b $ to $\rho_a$ at $t_2$]{\includegraphics[width=.18\linewidth]{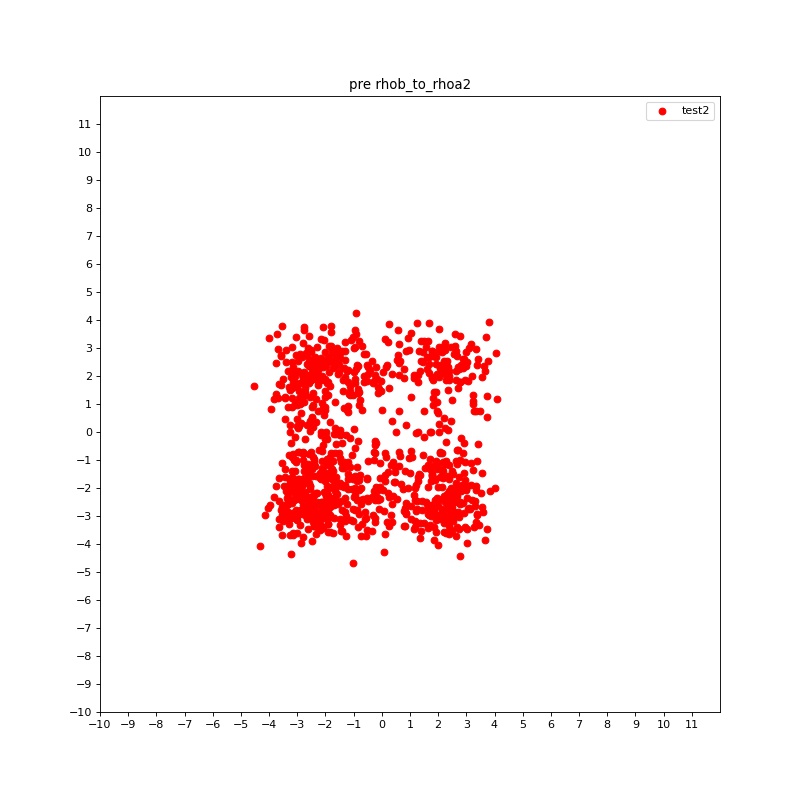}}
 \subfloat[][$\rho_b $ to $\rho_a$ at $t_3$]{\includegraphics[width=.18\linewidth]{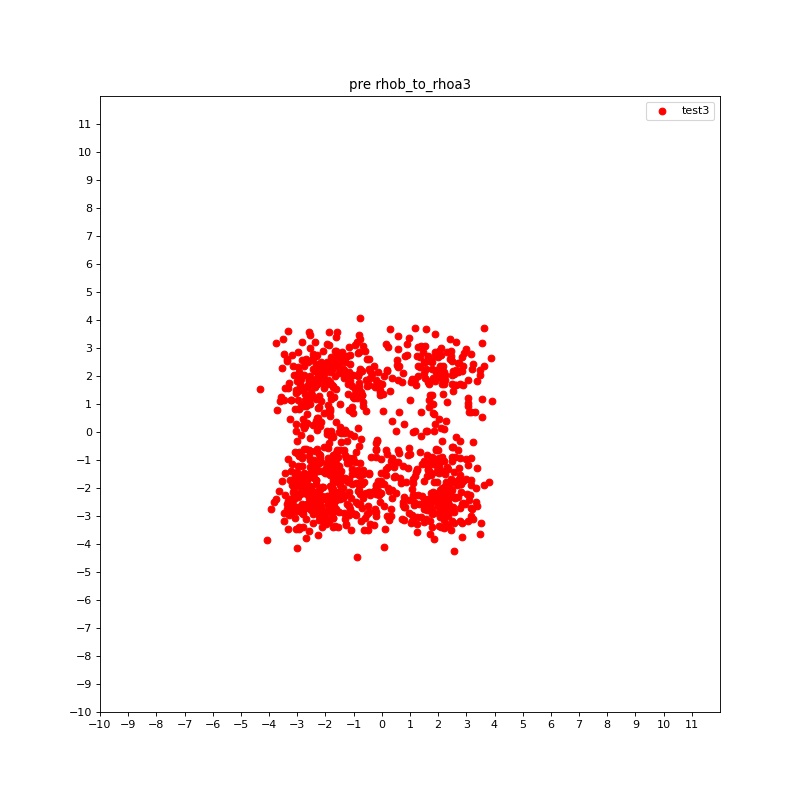}}
 \subfloat[][$\rho_b $ to $\rho_a$ at $t_4$]{\includegraphics[width=.18\linewidth]{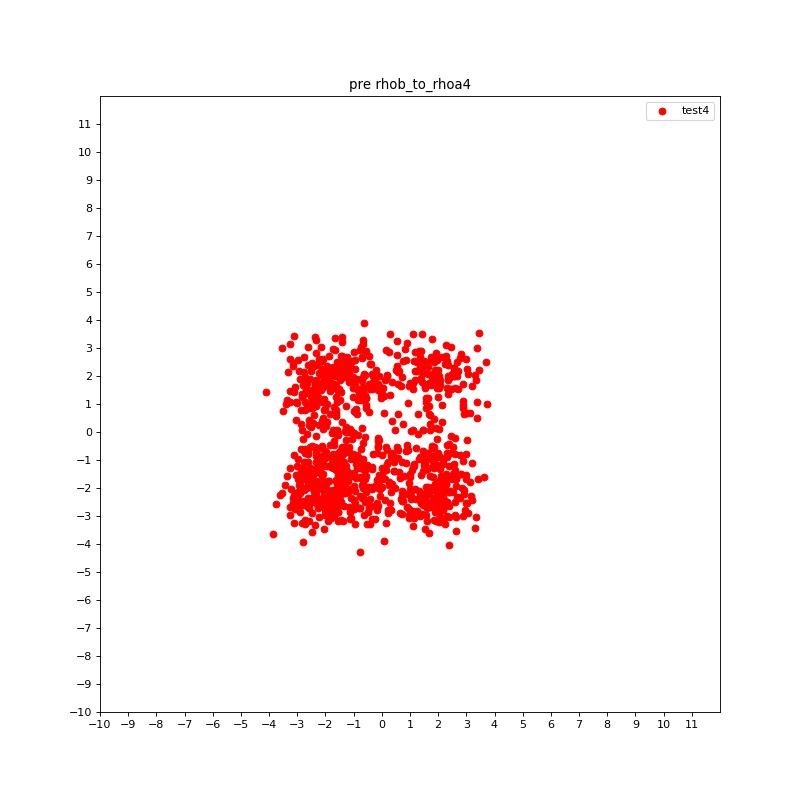}}
 \subfloat[][$\rho_b $ to $\rho_a$ at $t_5$]{\includegraphics[width=.18\linewidth]{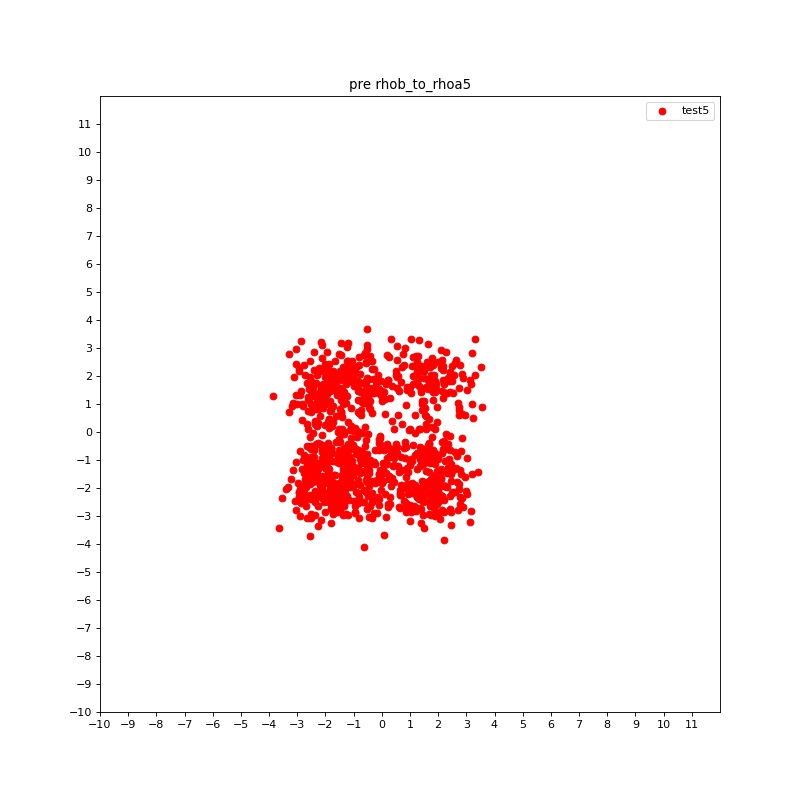}}\\
 \subfloat[][$\rho_b $ to $\rho_a$ at $t_6$]{\includegraphics[width=.18\linewidth]{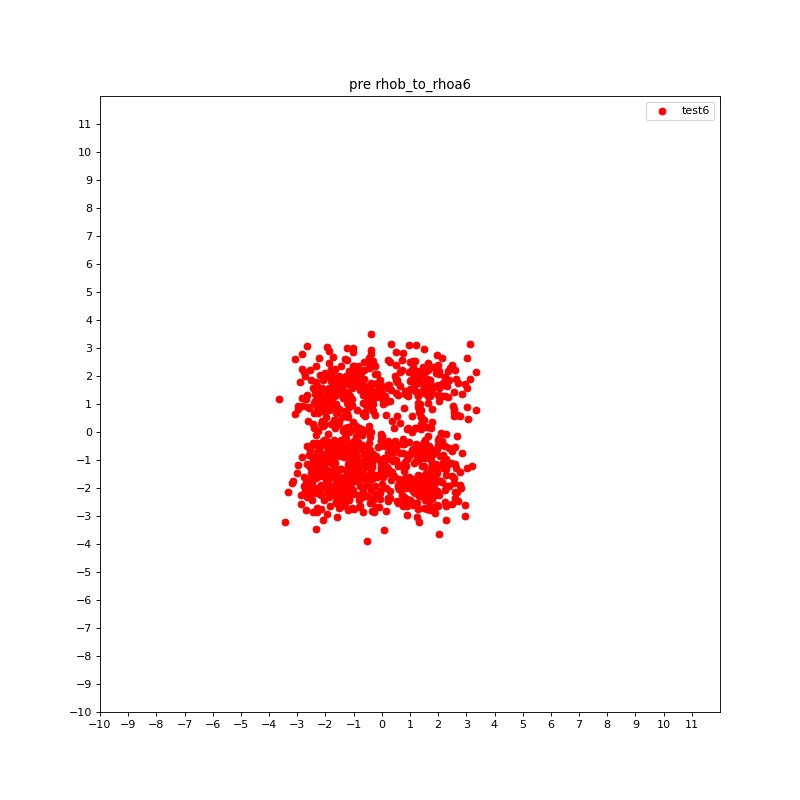}}
 \subfloat[][$\rho_b $ to $\rho_a$ at $t_7$]{\includegraphics[width=.18\linewidth]{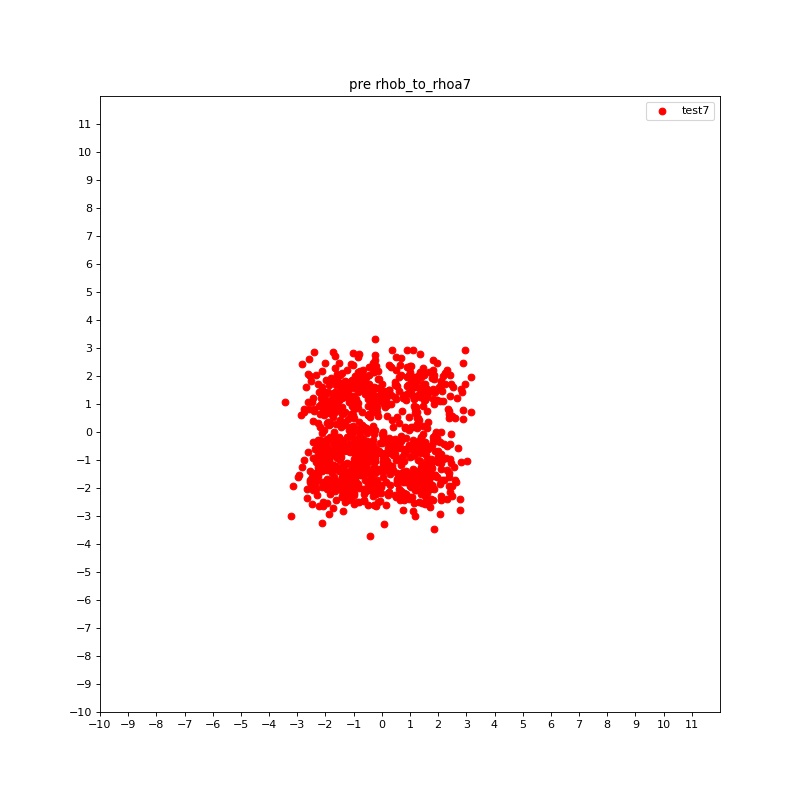}}
 \subfloat[][$\rho_b $ to $\rho_a$ at $t_8$]{\includegraphics[width=.18\linewidth]{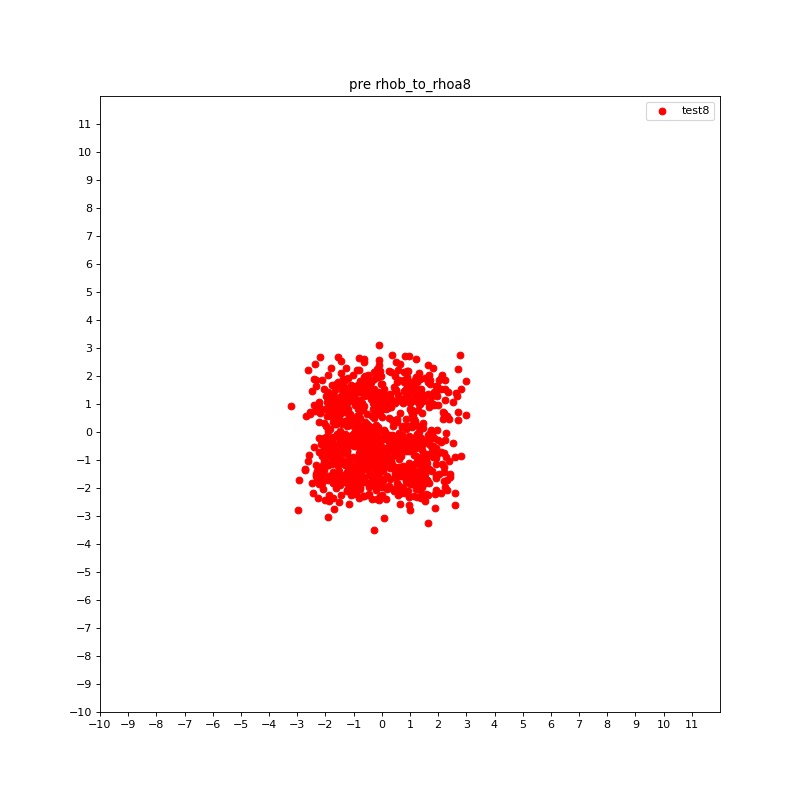}}
 \subfloat[][$\rho_b $ to $\rho_a$ at $t_9$]{\includegraphics[width=.18\linewidth]{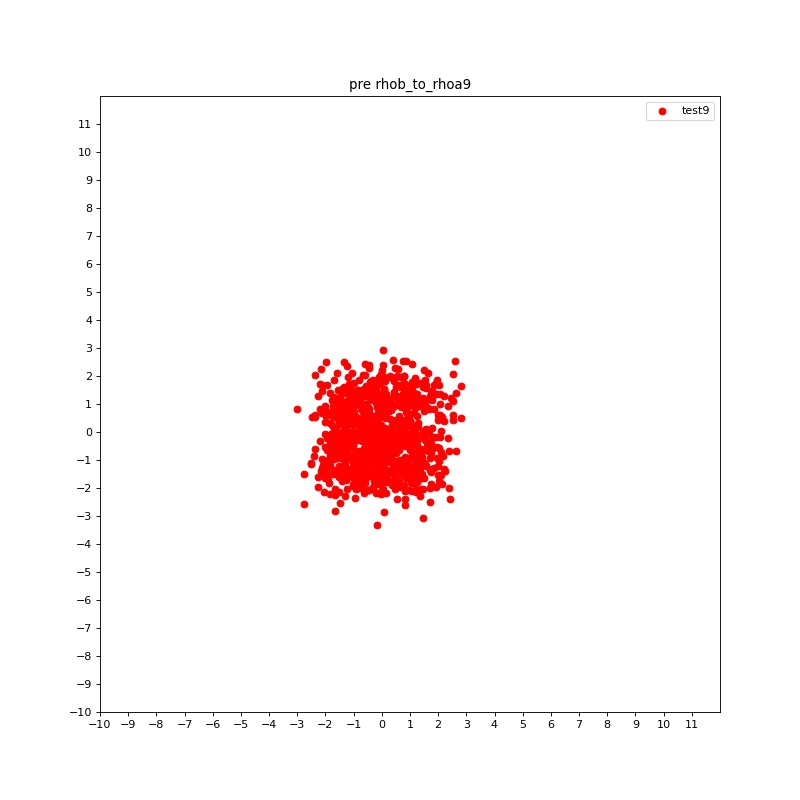}}
 \subfloat[][$\rho_b $ to $\rho_a$ at $t_{10}$]{\includegraphics[width=.18\linewidth]{pre_rhob_to_rhoa2550010.jpg}}\\
 \subfloat[][points track:$\rho_a$ to $\rho_b$]{\includegraphics[width=.18\linewidth]{colored_tracks_rhoa_to_rhob25500.jpg}}
 \subfloat[][vector field:$\rho_a$ to $\rho_b$]{\includegraphics[width=.18\linewidth]{vec_fd_rhoa_to_rhob2550010.jpg}}
 \subfloat[][points track:$\rho_b$ to $\rho_a$]{\includegraphics[width=.18\linewidth]{colored_tracks_rhob_to_rhoa25500.jpg}}
 \subfloat[][vector field:$\rho_b$ to $\rho_a$]{\includegraphics[width=.18\linewidth]{vec_fd_rhob_to_rhoa2550010.jpg}}
\caption{10-dimensional unbalanced Gaussian}
\label{fig:syn-4}
\end{figure*}

\newpage
\subsection{Realistic Data}
\textbf{Real-1:}
\begin{figure}[ht!]
    \centering
     \subfloat[][$\rho_a $ to $\rho_b$ at $t_1$]{\includegraphics[width=.18\linewidth]{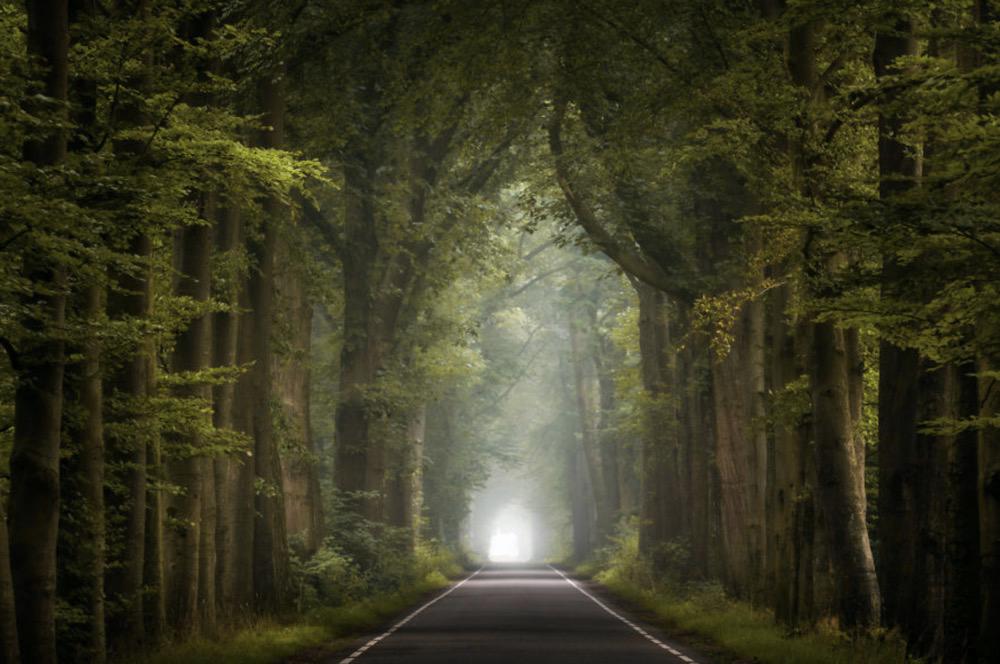}}
     \subfloat[][$\rho_a $ to $\rho_b$ at $t_2$]{\includegraphics[width=.18\linewidth]{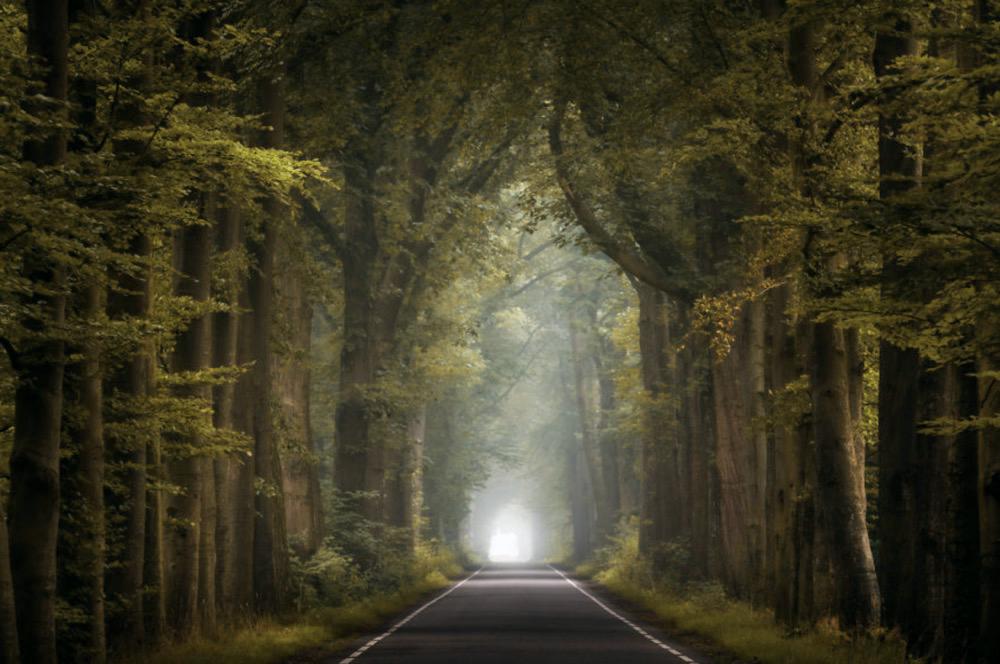}}
     \subfloat[][$\rho_a $ to $\rho_b$ at $t_3$]{\includegraphics[width=.18\linewidth]{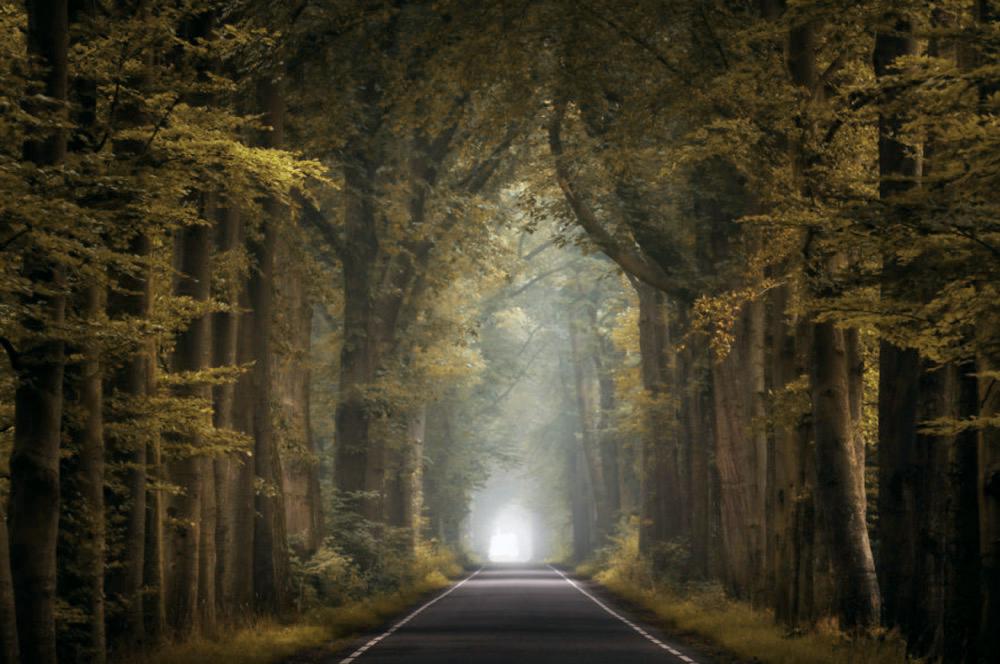}}
     \subfloat[][$\rho_a $ to $\rho_b$ at $t_4$]{\includegraphics[width=.18\linewidth]{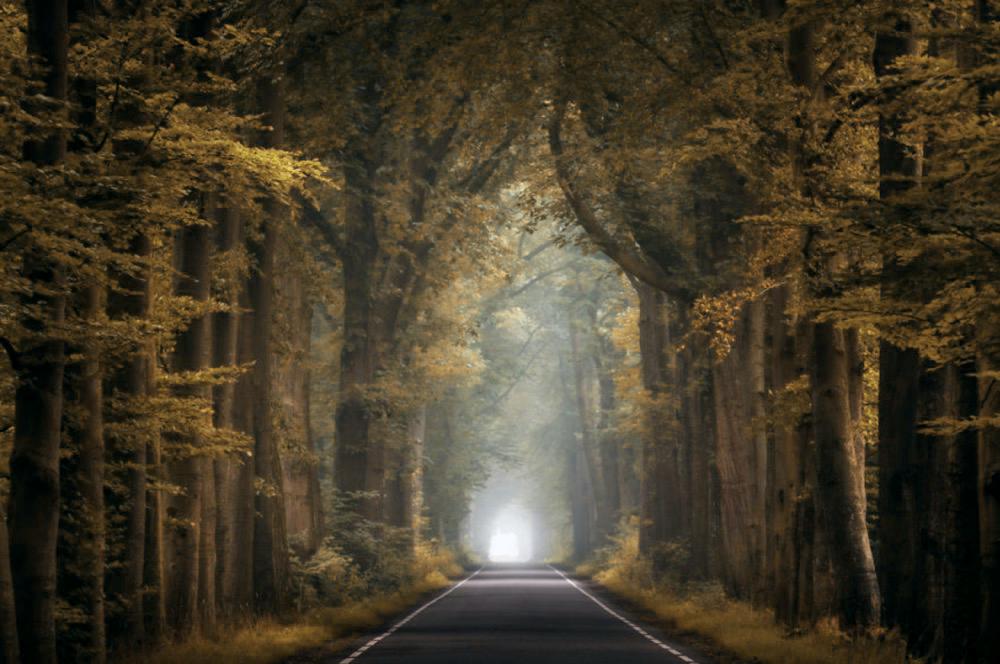}}
     \subfloat[][$\rho_a $ to $\rho_b$ at $t_5$]{\includegraphics[width=.18\linewidth]{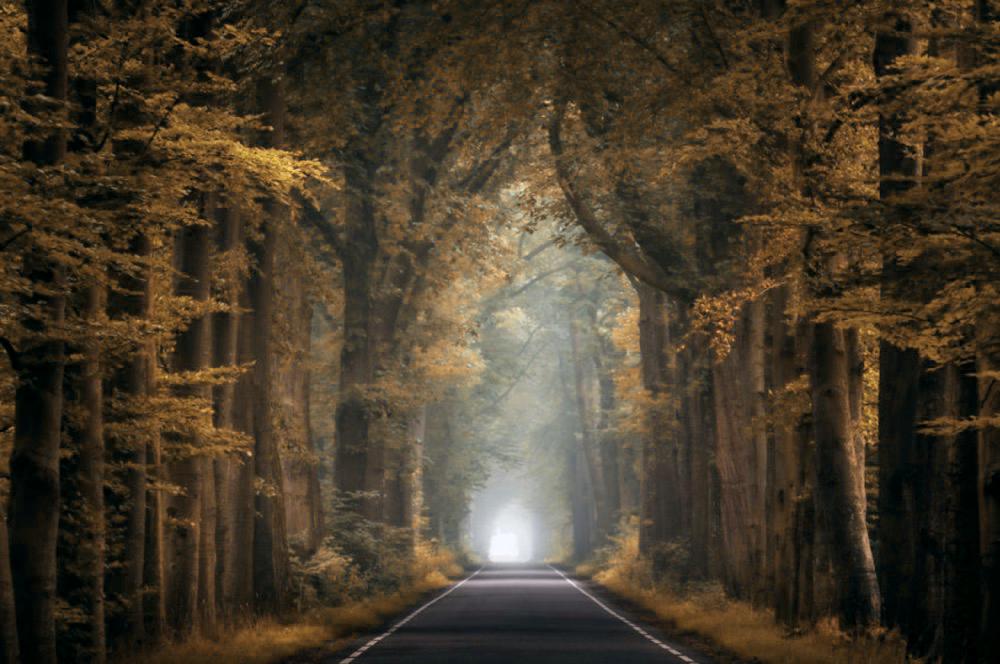}}\\
     \subfloat[][$\rho_a $ to $\rho_b$ at $t_6$]{\includegraphics[width=.18\linewidth]{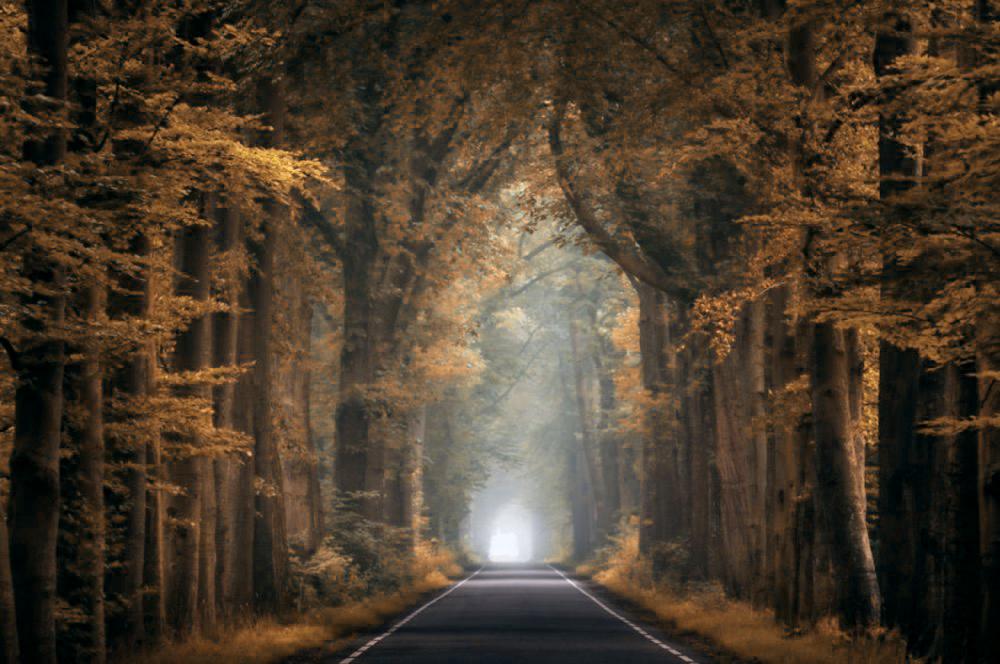}}
     \subfloat[][$\rho_a $ to $\rho_b$ at $t_7$]{\includegraphics[width=.18\linewidth]{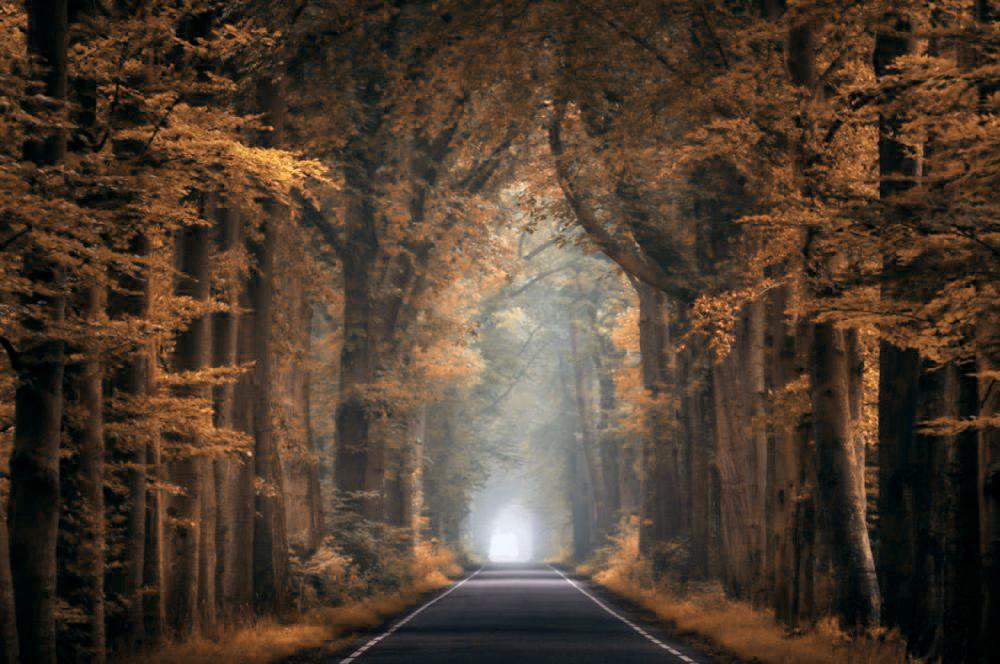}}
     \subfloat[][$\rho_a $ to $\rho_b$ at $t_8$]{\includegraphics[width=.18\linewidth]{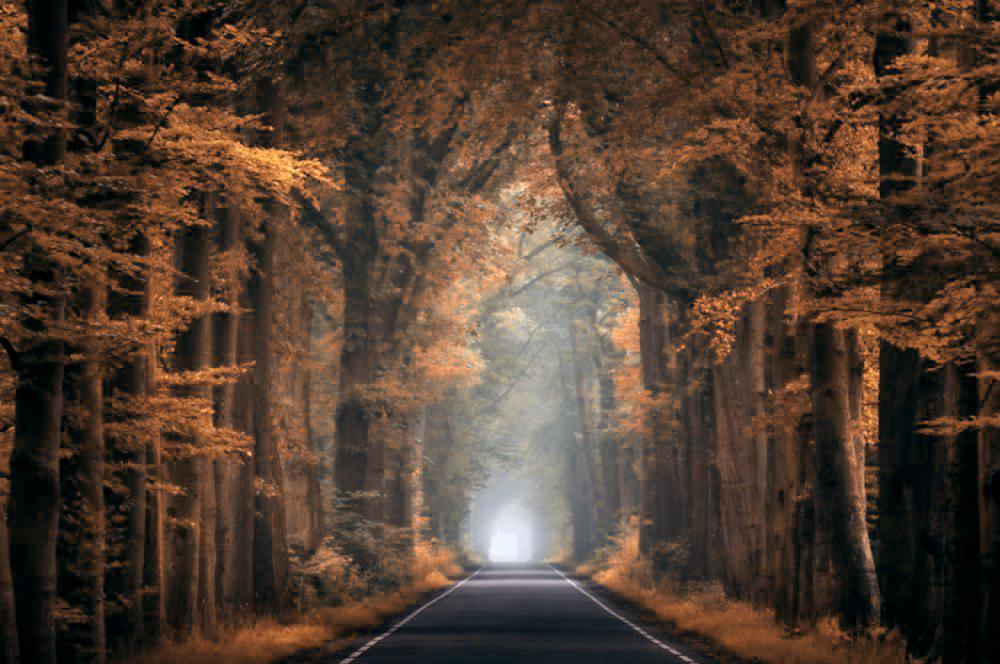}}
     \subfloat[][$\rho_a $ to $\rho_b$ at $t_9$]{\includegraphics[width=.18\linewidth]{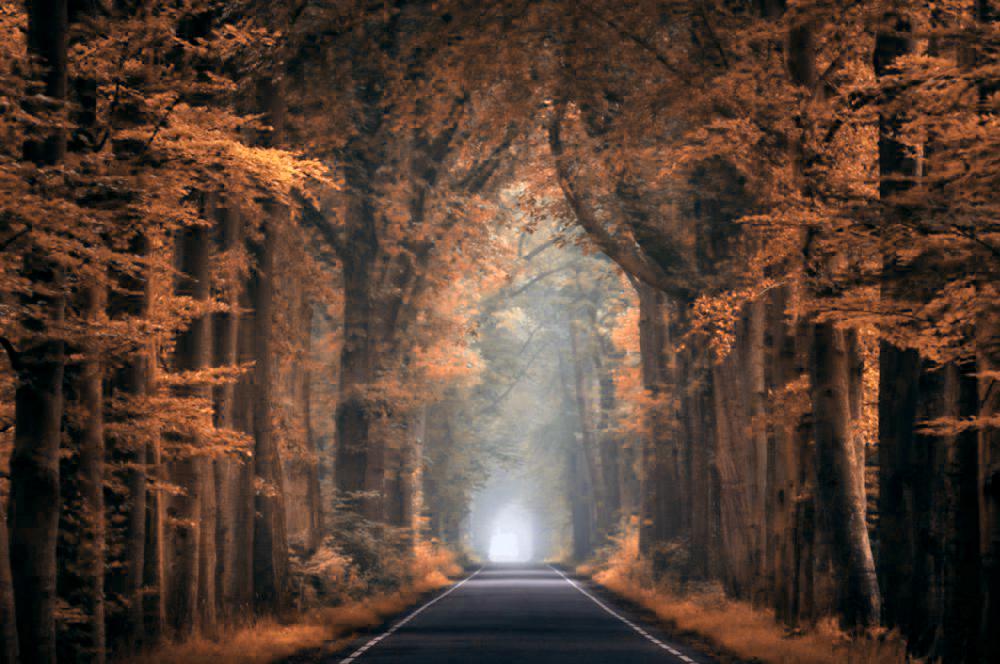}}
     \subfloat[][$\rho_a $ to $\rho_b$ at $t_{10}$]{\includegraphics[width=.18\linewidth]{colortransform1to2attime10outiteration1000.jpg}}\\
     \subfloat[][$\rho_b $ to $\rho_a$ at $t_1$]{\includegraphics[width=.18\linewidth]{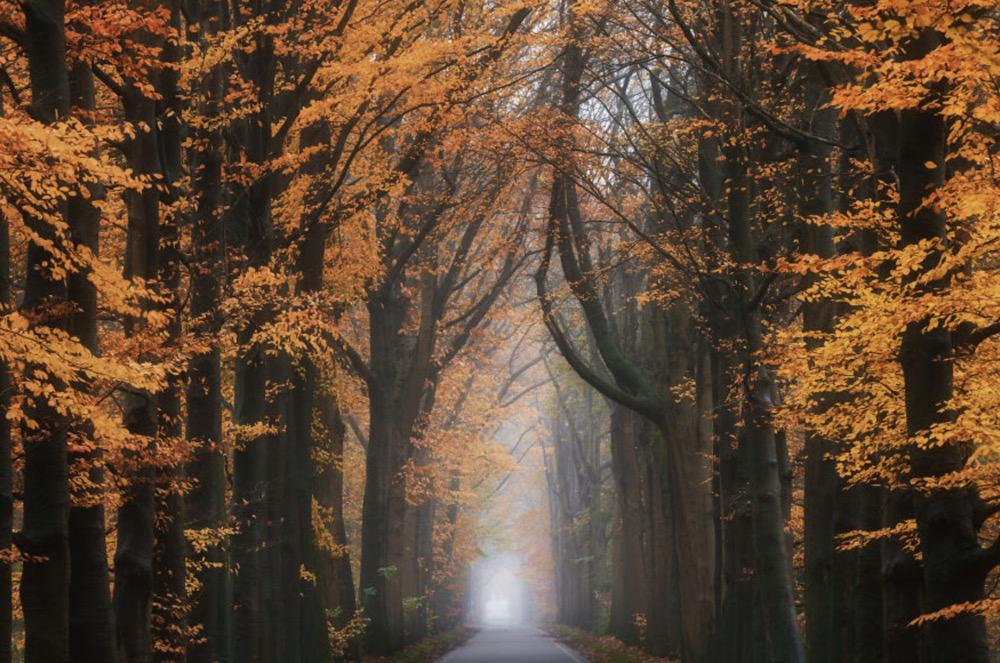}}
     \subfloat[][$\rho_b $ to $\rho_a$ at $t_2$]{\includegraphics[width=.18\linewidth]{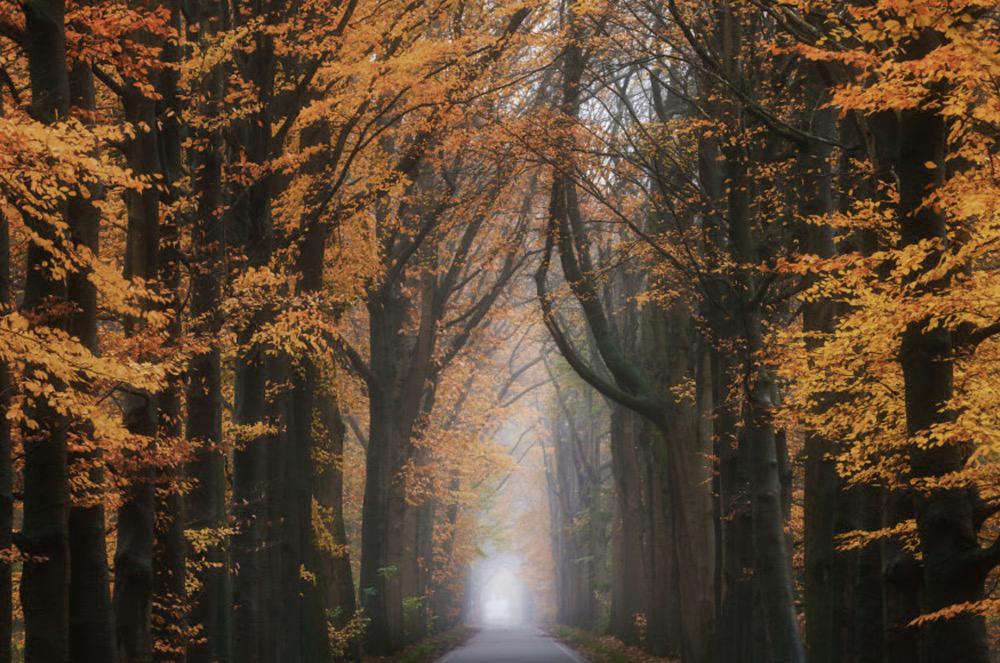}}
     \subfloat[][$\rho_b $ to $\rho_a$ at $t_3$]{\includegraphics[width=.18\linewidth]{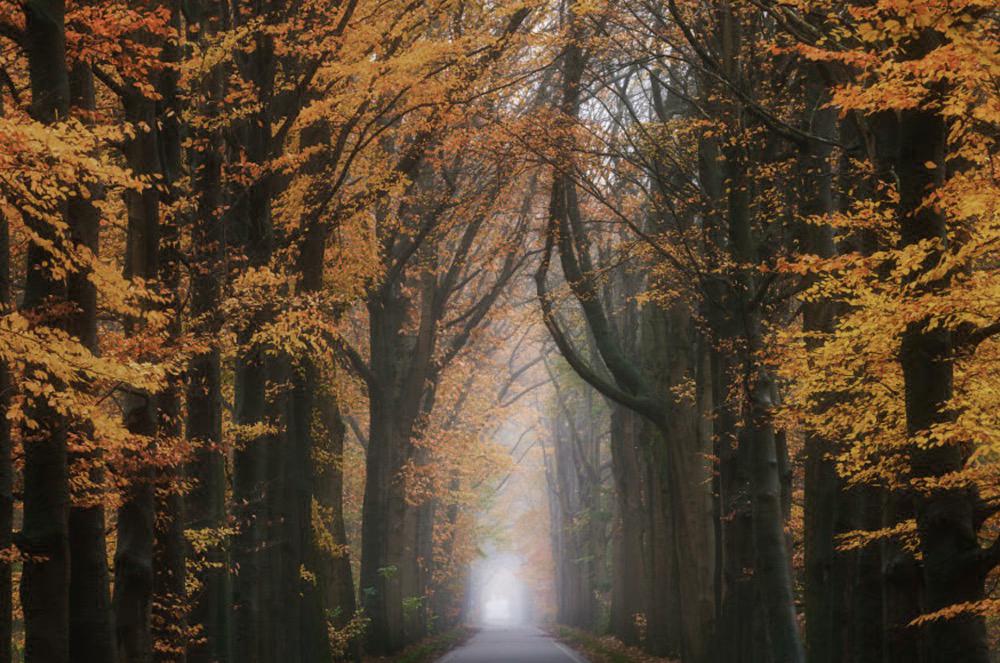}}
     \subfloat[][$\rho_b $ to $\rho_a$ at $t_4$]{\includegraphics[width=.18\linewidth]{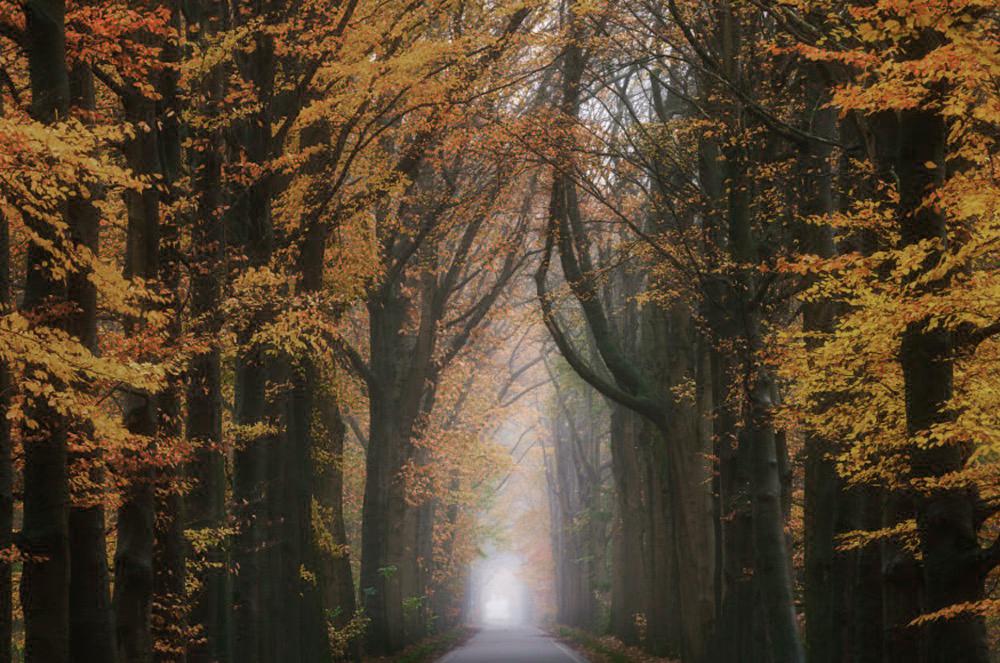}}
     \subfloat[][$\rho_b $ to $\rho_a$ at $t_5$]{\includegraphics[width=.18\linewidth]{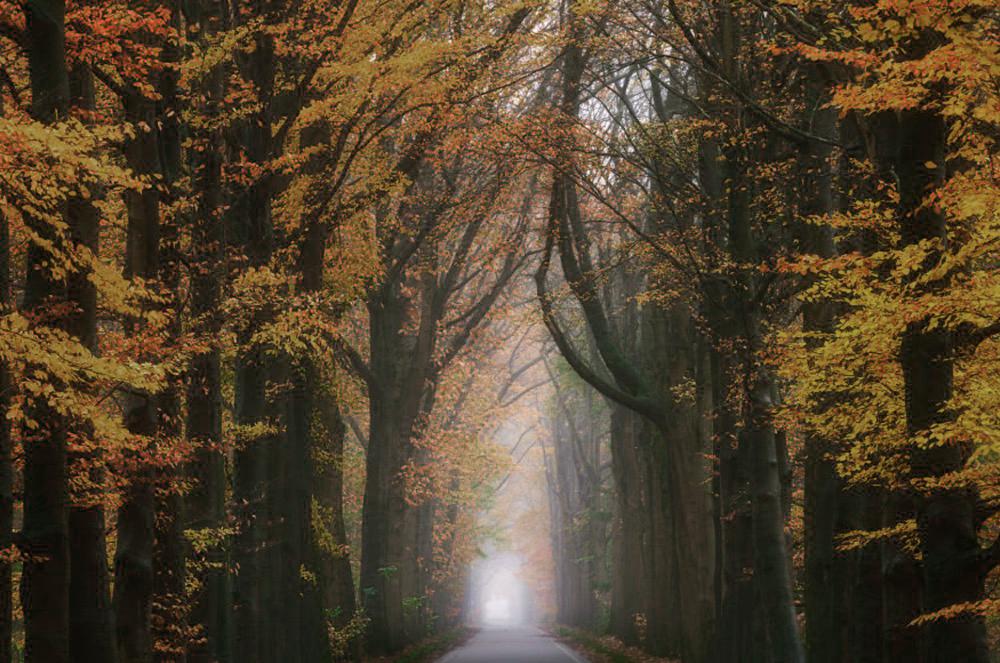}}\\
     \subfloat[][$\rho_b $ to $\rho_a$ at $t_6$]{\includegraphics[width=.18\linewidth]{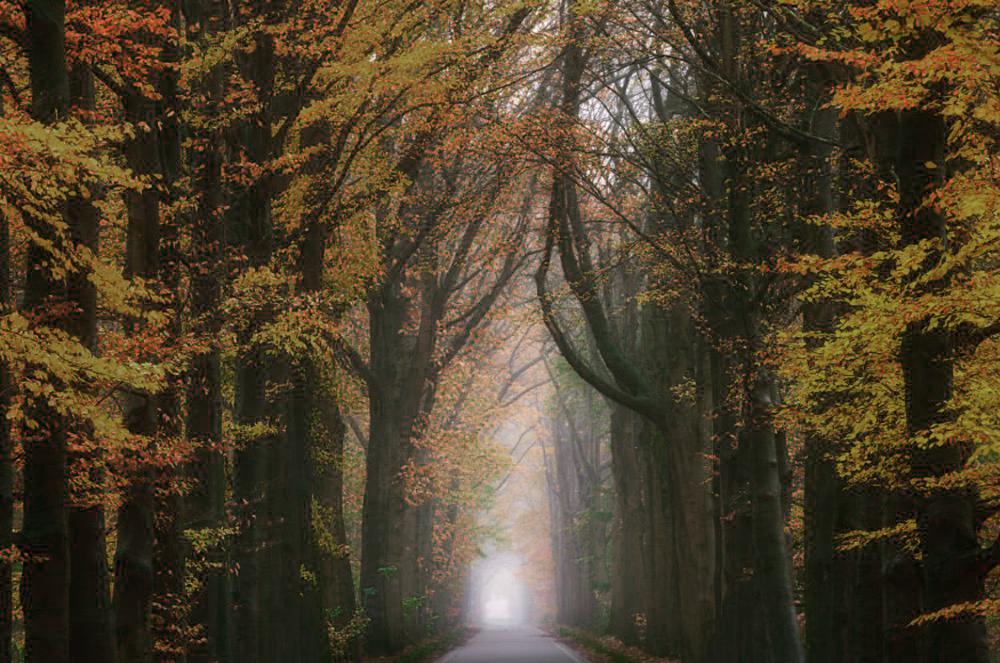}}
     \subfloat[][$\rho_b $ to $\rho_a$ at $t_7$]{\includegraphics[width=.18\linewidth]{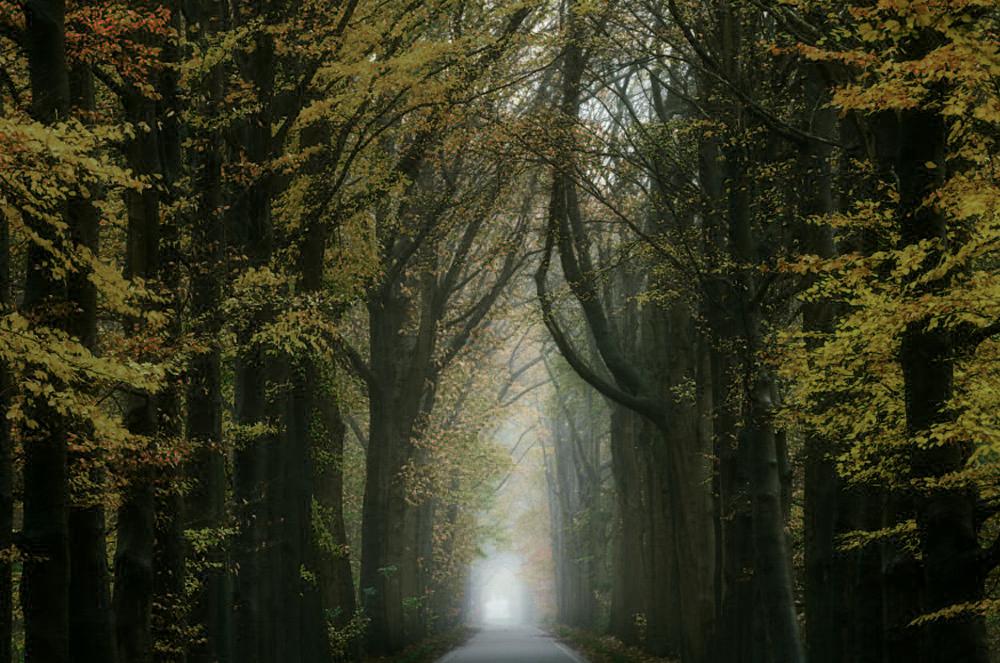}}
     \subfloat[][$\rho_b $ to $\rho_a$ at $t_8$]{\includegraphics[width=.18\linewidth]{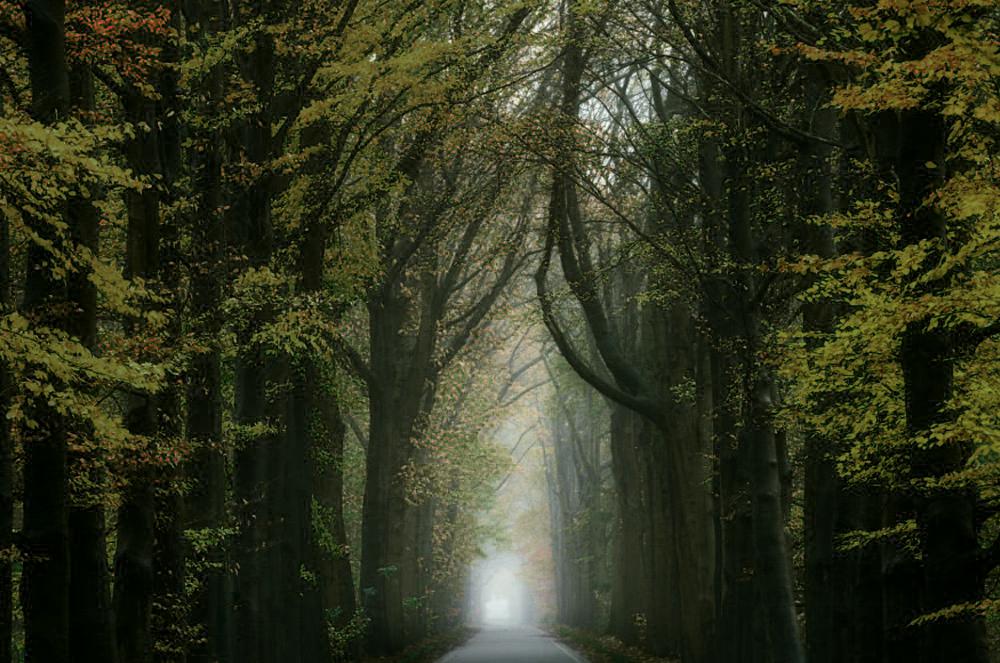}}
     \subfloat[][$\rho_b $ to $\rho_a$ at $t_9$]{\includegraphics[width=.18\linewidth]{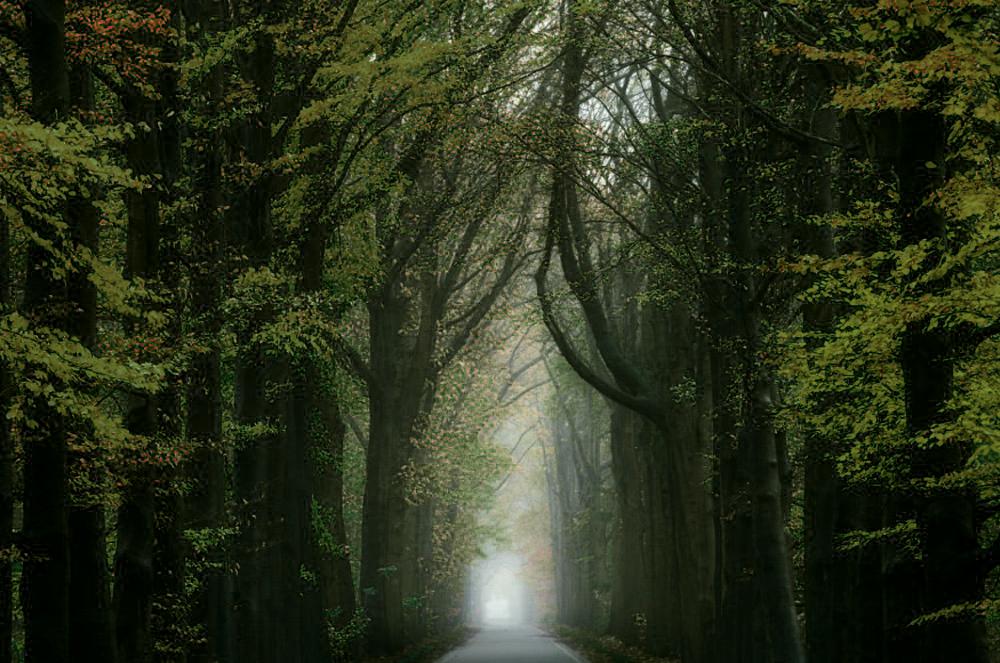}}
     \subfloat[][$\rho_b $ to $\rho_a$ at $t_{10}$]{\includegraphics[width=.18\linewidth]{colortransform2to1attime10outiteration1100.jpg}}\\
      \subfloat[][Color dist of Summer View]{\includegraphics[width=.2\linewidth]{scatterplot_of_RGB_values_of_picture_1_forest.png}}\quad
     \subfloat[][Color dist of Autumn View]{\includegraphics[width=.2\linewidth]{scatterplot_of_RGB_values_of_picture_2_forest.png}}\quad
     \subfloat[][Generated color dist of Summer View]{\includegraphics[width=.2\linewidth]{2to1iteration800samplesatt121.jpg}}\quad
     \subfloat[][Generated color dist of Autumn View]{\includegraphics[width=.2\linewidth]{1to2iteration800samplesatt121.jpg}}
\caption{Forest summer view and autumn view}
\label{fig:real-11}
\end{figure}

\newpage
\textbf{Real-2:}
\begin{figure}[ht!]
    \centering
     \subfloat[][Summer West Lake]{\includegraphics[width=.25\linewidth]{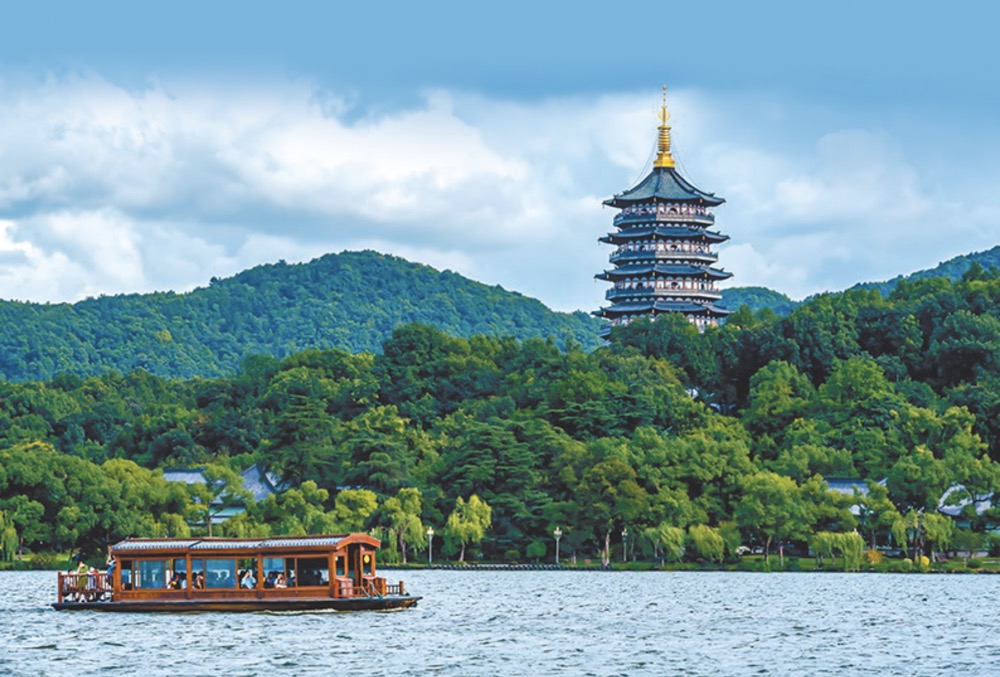}}
     \quad
     \subfloat[][Autumn White Tower]{\includegraphics[width=.25\linewidth]{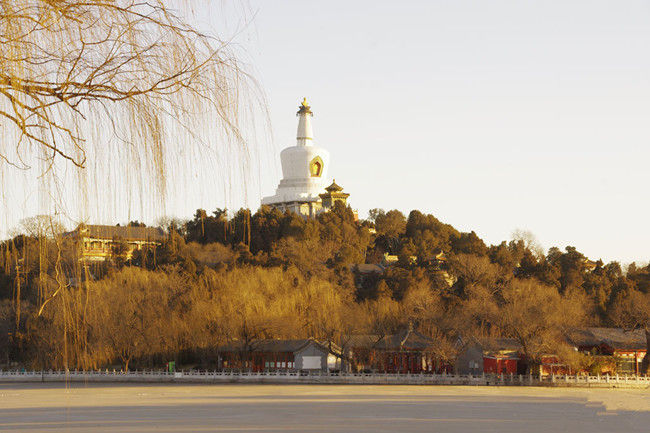}}
     \subfloat[][Color dist of West Lake]{\includegraphics[width=.2\linewidth]{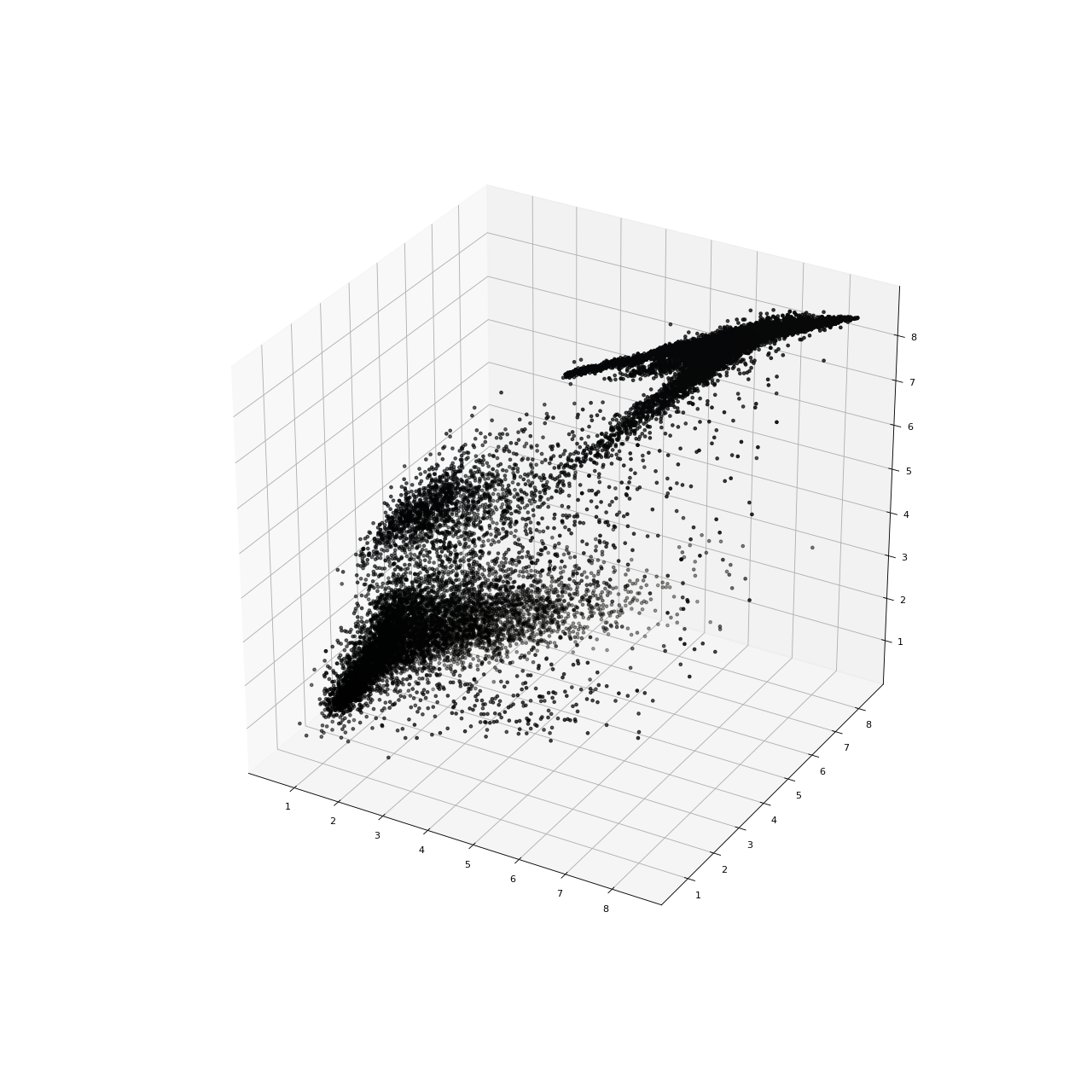}}\quad
     \subfloat[][Color dist of White Tower]{\includegraphics[width=.2\linewidth]{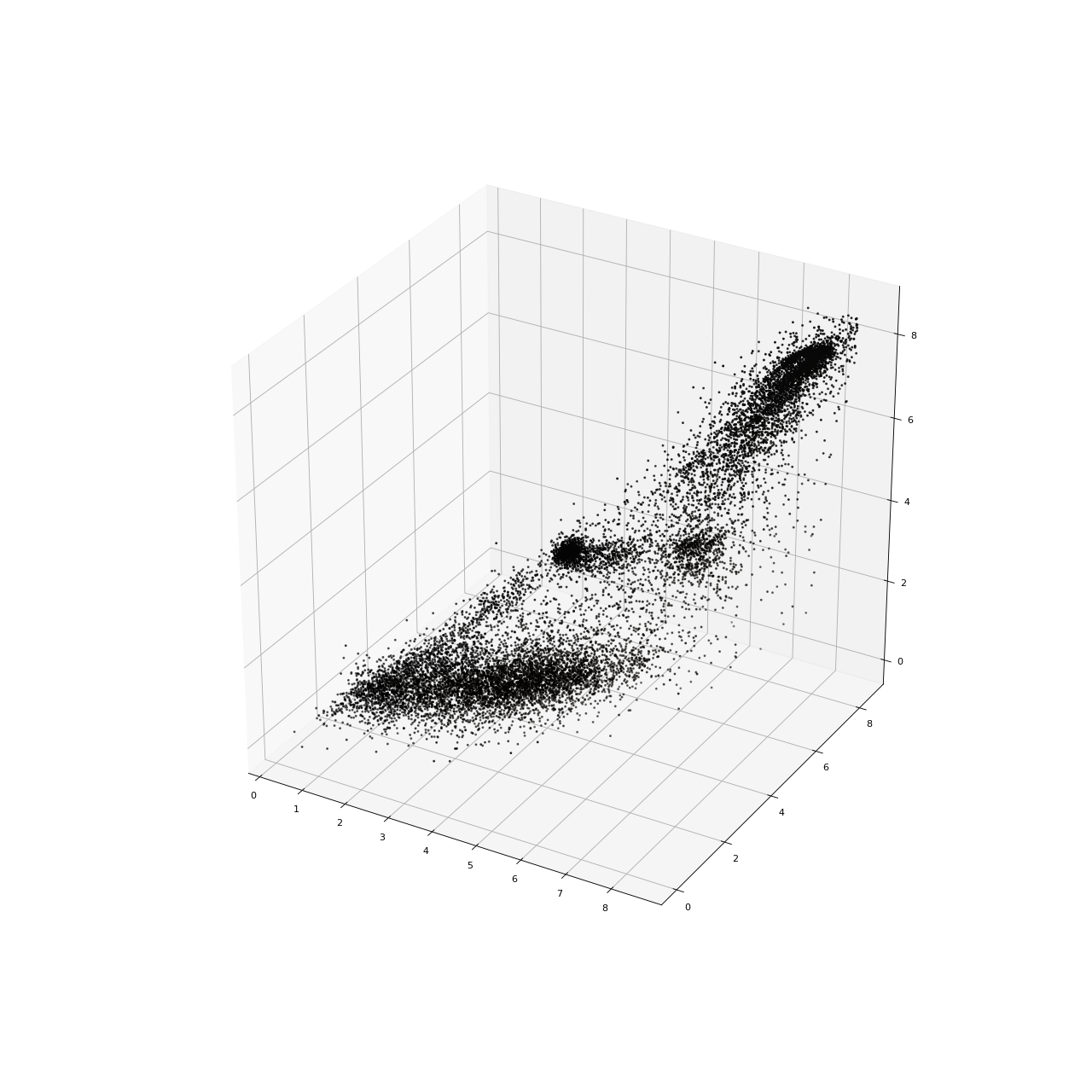}}\\
     \subfloat[][Autumn West Lake]{\includegraphics[width=.25\linewidth]{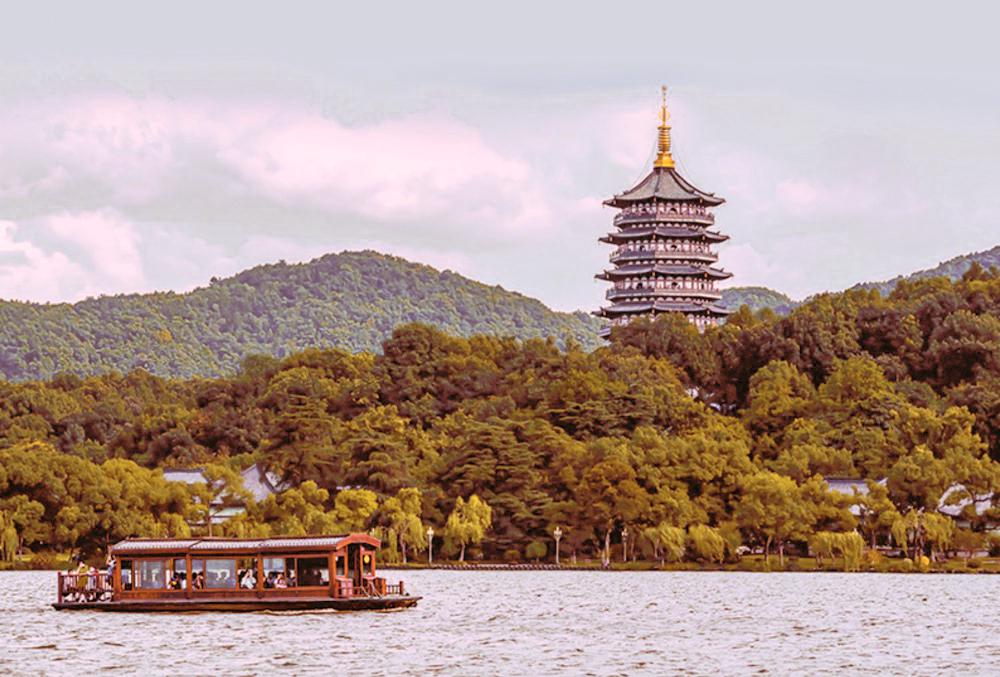}}\qquad
     \subfloat[][Summer White Tower]{\includegraphics[width=.25\linewidth]{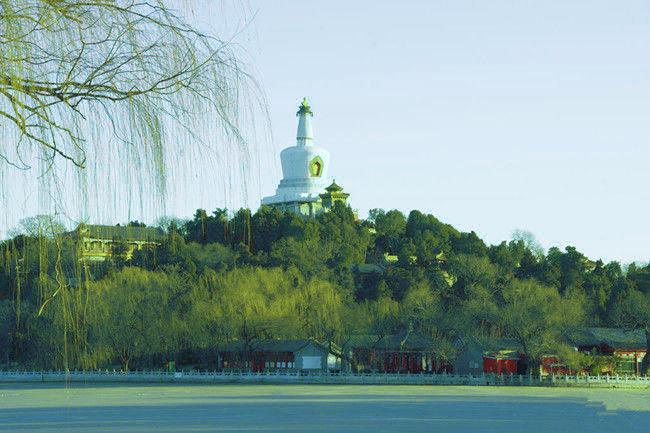}}
    \subfloat[][Generated color dist of West Lake]{\includegraphics[width=.2\linewidth]{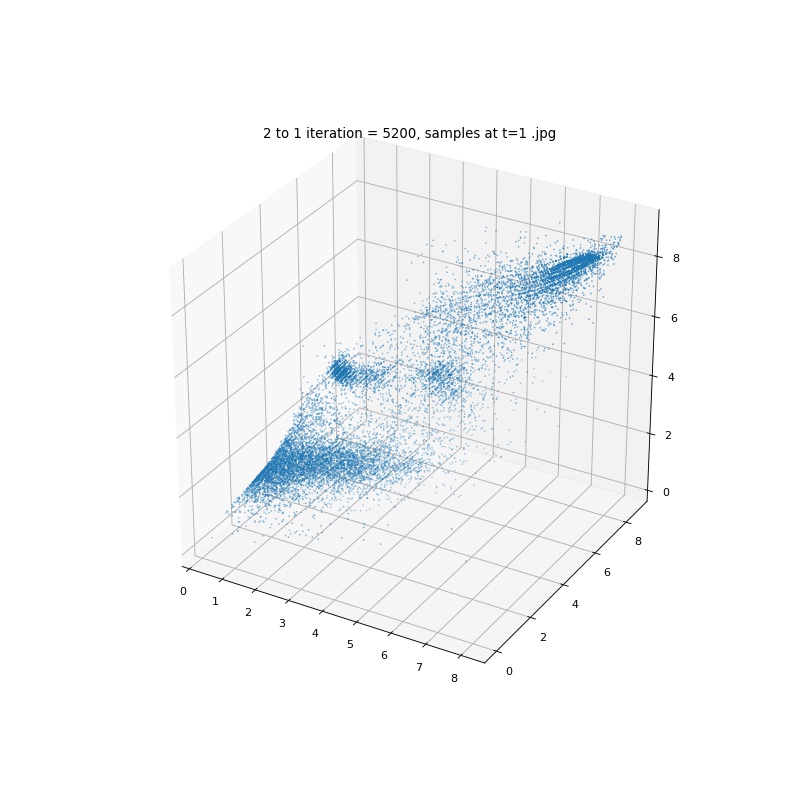}}\quad
     \subfloat[][Generated color of White Tower]{\includegraphics[width=.2\linewidth]{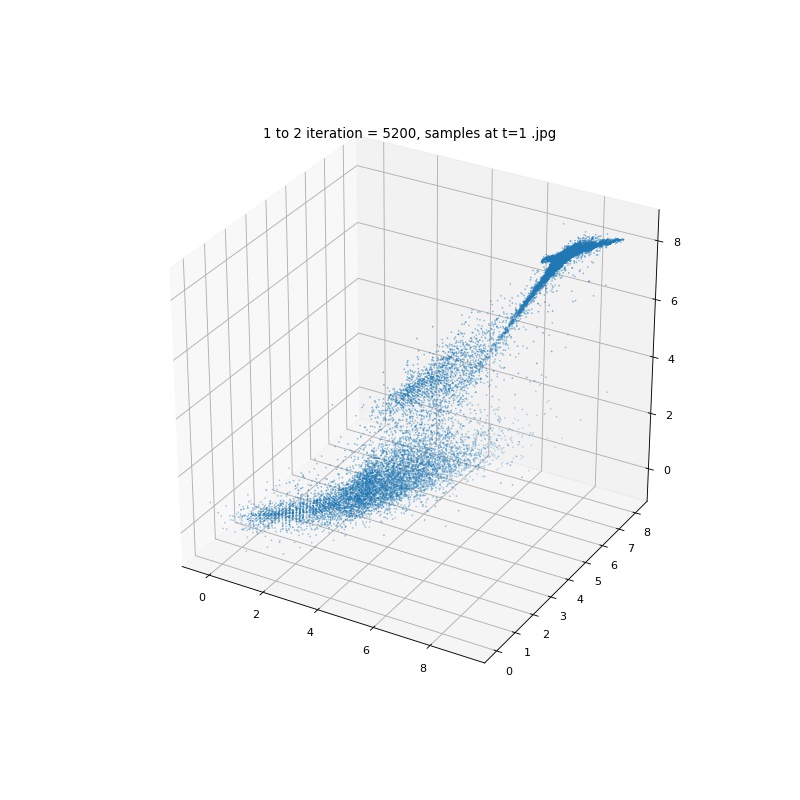}}
\caption{Color transfer}
\label{fig:real-22}
\end{figure}

\newpage
\textbf{Real-3:}
\begin{figure}[ht!]
    \centering
    \subfloat[][$\rho_a $ to $\rho_b$ at $t_1$]{\includegraphics[width=.18\linewidth]{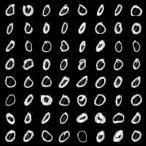}}
     \subfloat[][$\rho_a $ to $\rho_b$ at $t_2$]{\includegraphics[width=.18\linewidth]{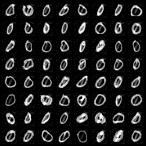}}
     \subfloat[][$\rho_a $ to $\rho_b$ at $t_3$]{\includegraphics[width=.18\linewidth]{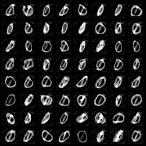}}
     \subfloat[][$\rho_a $ to $\rho_b$ at $t_4$]{\includegraphics[width=.18\linewidth]{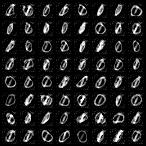}}
     \subfloat[][$\rho_a $ to $\rho_b$ at $t_5$]{\includegraphics[width=.18\linewidth]{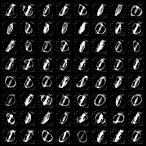}}\\
     \subfloat[][$\rho_a $ to $\rho_b$ at $t_6$]{\includegraphics[width=.18\linewidth]{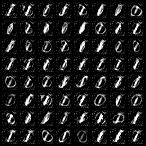}}
     \subfloat[][$\rho_a $ to $\rho_b$ at $t_7$]{\includegraphics[width=.18\linewidth]{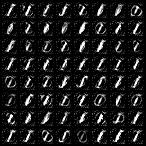}}
     \subfloat[][$\rho_a $ to $\rho_b$ at $t_8$]{\includegraphics[width=.18\linewidth]{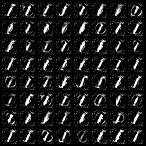}}
     \subfloat[][$\rho_a $ to $\rho_b$ at $t_9$]{\includegraphics[width=.18\linewidth]{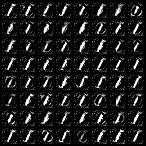}}
     \subfloat[][$\rho_a $ to $\rho_b$ at $t_{10}$]{\includegraphics[width=.18\linewidth]{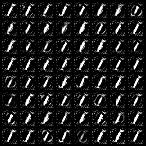}}\\
     \subfloat[][$\rho_b $ to $\rho_a$ at $t_1$]{\includegraphics[width=.18\linewidth]{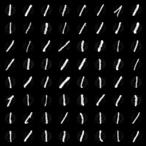}}
     \subfloat[][$\rho_b $ to $\rho_a$ at $t_2$]{\includegraphics[width=.18\linewidth]{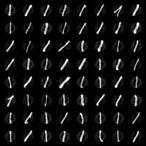}}
     \subfloat[][$\rho_b $ to $\rho_a$ at $t_3$]{\includegraphics[width=.18\linewidth]{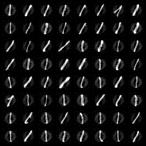}}
     \subfloat[][$\rho_b $ to $\rho_a$ at $t_4$]{\includegraphics[width=.18\linewidth]{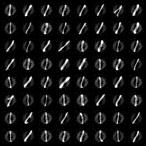}}
     \subfloat[][$\rho_b $ to $\rho_a$ at $t_5$]{\includegraphics[width=.18\linewidth]{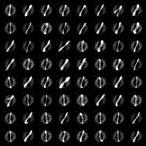}}\\
     \subfloat[][$\rho_b $ to $\rho_a$ at $t_6$]{\includegraphics[width=.18\linewidth]{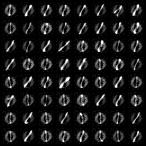}}
     \subfloat[][$\rho_b $ to $\rho_a$ at $t_7$]{\includegraphics[width=.18\linewidth]{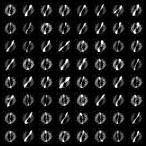}}
     \subfloat[][$\rho_b $ to $\rho_a$ at $t_8$]{\includegraphics[width=.18\linewidth]{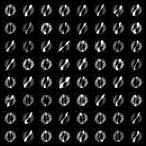}}
     \subfloat[][$\rho_b $ to $\rho_a$ at $t_9$]{\includegraphics[width=.18\linewidth]{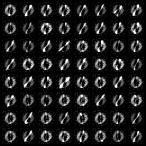}}
     \subfloat[][$\rho_b $ to $\rho_a$ at $t_{10}$]{\includegraphics[width=.18\linewidth]{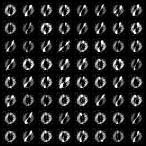}}
\caption{Geodesics between "0" and "1"}
\label{fig:real-33}
\end{figure}

\newpage
\begin{figure}[ht!]
    \centering
     \subfloat[][$\rho_a $ to $\rho_b$ at $t_1$]{\includegraphics[width=.18\linewidth]{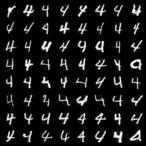}}
     \subfloat[][$\rho_a $ to $\rho_b$ at $t_2$]{\includegraphics[width=.18\linewidth]{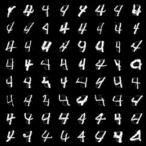}}
     \subfloat[][$\rho_a $ to $\rho_b$ at $t_3$]{\includegraphics[width=.18\linewidth]{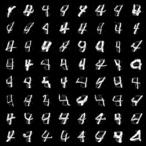}}
     \subfloat[][$\rho_a $ to $\rho_b$ at $t_4$]{\includegraphics[width=.18\linewidth]{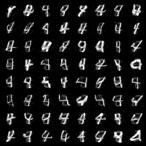}}
     \subfloat[][$\rho_a $ to $\rho_b$ at $t_5$]{\includegraphics[width=.18\linewidth]{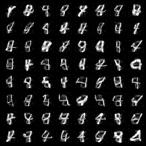}}\\
     \subfloat[][$\rho_a $ to $\rho_b$ at $t_6$]{\includegraphics[width=.18\linewidth]{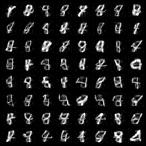}}
     \subfloat[][$\rho_a $ to $\rho_b$ at $t_7$]{\includegraphics[width=.18\linewidth]{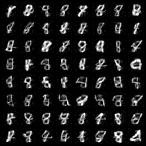}}
     \subfloat[][$\rho_a $ to $\rho_b$ at $t_8$]{\includegraphics[width=.18\linewidth]{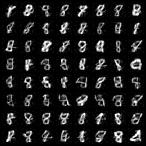}}
     \subfloat[][$\rho_a $ to $\rho_b$ at $t_9$]{\includegraphics[width=.18\linewidth]{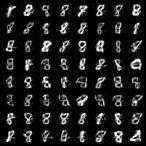}}
     \subfloat[][$\rho_a $ to $\rho_b$ at $t_{10}$]{\includegraphics[width=.18\linewidth]{approx_m816400at10.jpg}}\\
     \subfloat[][$\rho_b $ to $\rho_a$ at $t_1$]{\includegraphics[width=.18\linewidth]{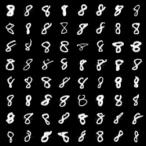}}
     \subfloat[][$\rho_b $ to $\rho_a$ at $t_2$]{\includegraphics[width=.18\linewidth]{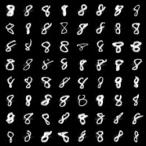}}
     \subfloat[][$\rho_b $ to $\rho_a$ at $t_3$]{\includegraphics[width=.18\linewidth]{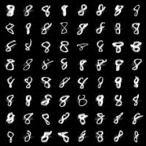}}
     \subfloat[][$\rho_b $ to $\rho_a$ at $t_4$]{\includegraphics[width=.18\linewidth]{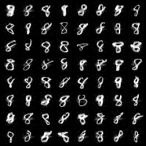}}
     \subfloat[][$\rho_b $ to $\rho_a$ at $t_5$]{\includegraphics[width=.18\linewidth]{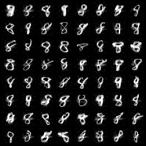}}\\
     \subfloat[][$\rho_b $ to $\rho_a$ at $t_6$]{\includegraphics[width=.18\linewidth]{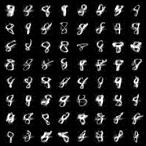}}
     \subfloat[][$\rho_b $ to $\rho_a$ at $t_7$]{\includegraphics[width=.18\linewidth]{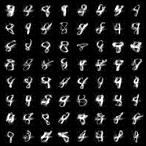}}
     \subfloat[][$\rho_b $ to $\rho_a$ at $t_8$]{\includegraphics[width=.18\linewidth]{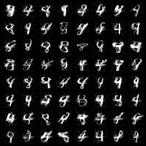}}
     \subfloat[][$\rho_b $ to $\rho_a$ at $t_9$]{\includegraphics[width=.18\linewidth]{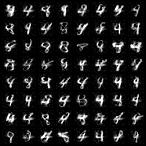}}
     \subfloat[][$\rho_b $ to $\rho_a$ at $t_{10}$]{\includegraphics[width=.18\linewidth]{approx_m416400at10.jpg}}
\caption{Geodesics between "4" and "8"}
\label{fig:real-4}
\end{figure}

\newpage
\begin{figure}[ht!]
    \centering
     \subfloat[][$\rho_a $ to $\rho_b$ at $t_1$]{\includegraphics[width=.18\linewidth]{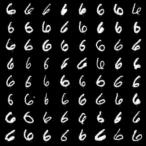}}
     \subfloat[][$\rho_a $ to $\rho_b$ at $t_2$]{\includegraphics[width=.18\linewidth]{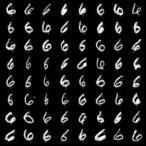}}
     \subfloat[][$\rho_a $ to $\rho_b$ at $t_3$]{\includegraphics[width=.18\linewidth]{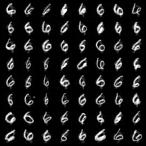}}
     \subfloat[][$\rho_a $ to $\rho_b$ at $t_4$]{\includegraphics[width=.18\linewidth]{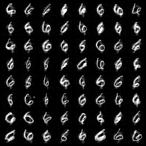}}
     \subfloat[][$\rho_a $ to $\rho_b$ at $t_5$]{\includegraphics[width=.18\linewidth]{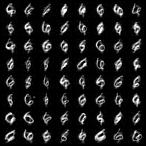}}\\
     \subfloat[][$\rho_a $ to $\rho_b$ at $t_6$]{\includegraphics[width=.18\linewidth]{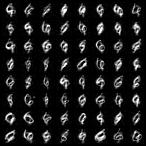}}
     \subfloat[][$\rho_a $ to $\rho_b$ at $t_7$]{\includegraphics[width=.18\linewidth]{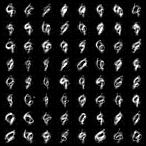}}
     \subfloat[][$\rho_a $ to $\rho_b$ at $t_8$]{\includegraphics[width=.18\linewidth]{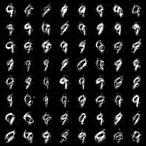}}
     \subfloat[][$\rho_a $ to $\rho_b$ at $t_9$]{\includegraphics[width=.18\linewidth]{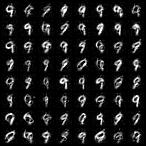}}
     \subfloat[][$\rho_a $ to $\rho_b$ at $t_{10}$]{\includegraphics[width=.18\linewidth]{approx_m923600at10.jpg}}\\
     \subfloat[][$\rho_b $ to $\rho_a$ at $t_1$]{\includegraphics[width=.18\linewidth]{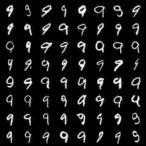}}
     \subfloat[][$\rho_b $ to $\rho_a$ at $t_2$]{\includegraphics[width=.18\linewidth]{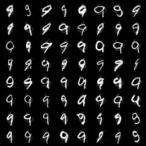}}
     \subfloat[][$\rho_b $ to $\rho_a$ at $t_3$]{\includegraphics[width=.18\linewidth]{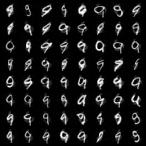}}
     \subfloat[][$\rho_b $ to $\rho_a$ at $t_4$]{\includegraphics[width=.18\linewidth]{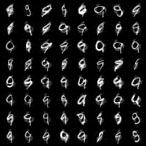}}
     \subfloat[][$\rho_b $ to $\rho_a$ at $t_5$]{\includegraphics[width=.18\linewidth]{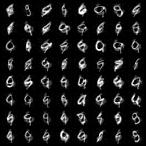}}\\
     \subfloat[][$\rho_b $ to $\rho_a$ at $t_6$]{\includegraphics[width=.18\linewidth]{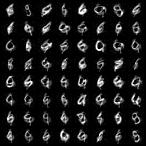}}
     \subfloat[][$\rho_b $ to $\rho_a$ at $t_7$]{\includegraphics[width=.18\linewidth]{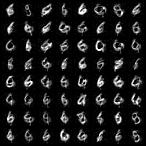}}
     \subfloat[][$\rho_b $ to $\rho_a$ at $t_8$]{\includegraphics[width=.18\linewidth]{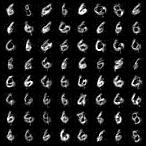}}
     \subfloat[][$\rho_b $ to $\rho_a$ at $t_9$]{\includegraphics[width=.18\linewidth]{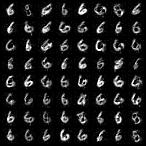}}
     \subfloat[][$\rho_b $ to $\rho_a$ at $t_{10}$]{\includegraphics[width=.18\linewidth]{approx_m623600at10.jpg}}
\caption{Geodesics between "6" and "9"}
\label{fig:real-5}
\end{figure}

\end{document}